\documentclass{article}

\usepackage{microtype}
\usepackage{float}
\usepackage{graphicx}
\usepackage{booktabs} 
\usepackage{dsfont}
\usepackage{enumitem}
\usepackage{verbatim}
\usepackage{multirow}
\usepackage{longtable}
\usepackage{makecell}

\usepackage{caption}
\usepackage{hyperref}

\usepackage[accepted]{icml2026}

\usepackage{mathtools}
\usepackage{amsmath,amsthm,amsfonts,amssymb}

\usepackage[capitalize,noabbrev]{cleveref}

\theoremstyle{plain}
\newtheorem{theorem}{Theorem}[section]
\newtheorem{proposition}[theorem]{Proposition}
\newtheorem{lemma}[theorem]{Lemma}
\newtheorem{corollary}[theorem]{Corollary}
\theoremstyle{definition}
\newtheorem{definition}[theorem]{Definition}

\newtheorem{example}{Example}[section]
\newtheorem{remark}[theorem]{Remark}
\newtheorem{erratum}{Erratum}
\usepackage[textsize=tiny]{todonotes}

\newcommand{\set}[1]{\left\{ #1 \right\}}
\newcommand{\squarebrac}[1]{\left[ #1 \right]}
\newcommand{\roundbrac}[1]{\left( #1 \right)}

\newcommand{\absbrac}[1]{\left| #1 \right|}

\renewcommand{\ge}{\geqslant}
\renewcommand{\le}{\leqslant}

\newcommand{\R}{\mathbb{R}}
\newcommand{\E}{\mathbb{E}}
\newcommand{\F}{\mathcal{F}}
\newcommand{\X}{\mathcal{X}}

\newcommand{\Rad}{{\bf \Sigma}}
\newcommand{\RC}{\mathfrak{R}}
\newcommand{\ERC}{\widehat{\mathfrak{R}}}

\renewcommand{\P}{\mathbb{P}}
\newcommand{\1}[1]{{\mathds{1}}_{\{#1\}}}
\newcommand{\Lnorm}{\mathrm{L}}

\newcommand{\Lun}{\mathrm{L}_\phi}
\renewcommand{\E}{\mathbb{E}}

\newcommand{\barkappa}{\overline{\kappa}}
\newcommand{\z}{{\bf z}}
\newcommand{\x}{{\bf x}}
\renewcommand{\u}{{\bf u}}

\icmltitlerunning{Refined Generalization Analysis for SCRL}
\begin{document}

\twocolumn[
  \icmltitle{A Refined Generalization Analysis for Extreme Multi-class Supervised Contrastive Representation Learning}

  \icmlsetsymbol{equal}{*}

  \begin{icmlauthorlist}
    \icmlauthor{Nong Minh Hieu}{smu}
    \icmlauthor{Antoine Ledent}{smu}
  \end{icmlauthorlist}

  \icmlaffiliation{smu}{School of Computing and Information Systems, Singapore Management University}
  \icmlcorrespondingauthor{Nong Minh Hieu}{mh.nong.2024@phdcs.smu.edu.sg}
  \vskip 0.3in
]



\printAffiliationsAndNotice{}

\begin{abstract}
    Contrastive Representation Learning (CRL) has achieved strong empirical success in multiple machine learning disciplines, yet its theoretical sample complexity remains poorly understood. Existing analyses usually assume that input tuples are identically and independently distributed, an assumption violated in most practical settings where contrastive tuples are constructed from a finite pool of labeled data, inducing dependencies among tuples. While one recent work analyzed this learning setting using U-Statistics to estimate the population risk, the techniques used therein require the risk of each class to concentrate uniformly, making excess risk bounds scale in the order of $\rho_{\min}^{-{1}/{2}}$ where $\rho_{\min}$ denotes the probability of the rarest class. Such a dependency can be overly pessimistic in the extreme multiclass settings where there are many tail classes which contribute minimally to the overall population risk. Our contributions are two-fold. Firstly, we improve upon the previous work and prove a bound with a sample complexity of the same order as the number of classes $R$, regardless of the distribution over classes. Furthermore, we formulate a different estimator that captures the concentration of the risk \textit{across classes}, enabling sharper bounds in extreme multi-class learning scenarios, especially where class distributions are long-tailed. Under mild assumptions on the class distributions, the resulting sample complexity is $\mathcal{{O}}(k)$ where $k$ is the number of samples per tuple. 
\end{abstract}

\section{Introduction}
The performance of machine learning models depends critically on data representation. In multi-class classification, effective representations typically promote intra-class compactness and inter-class separability. Contrastive Representation Learning (CRL) has therefore become a popular approach for learning structured representations from raw data prior to downstream tasks. Despite its empirical success, theoretical understanding of CRL is still in the early stage. Existing analyses of excess risk bounds, such as \citet{article:arora19} and \citet{article:lei2023}, commonly assume that unlabeled input $(k+2)$-tuples are drawn i.i.d. from an unknown distribution. While analytically convenient, this assumption may not reflect practical CRL pipelines. In many real-world implementations, contrastive tuples are constructed from a finite pool of \textbf{labeled examples} \citep{article:sohn2016, article:khosla2020, article:dejiao2021}, leading to repeated use of data points across tuples, thereby violating the independence assumption.

To explicitly model tuple dependencies in supervised CRL, \citet{article:hieu2025} introduced a U-statistic based estimator that decomposes the population risk into a sum of class-wise risks then estimates each component independently (cf. Eqn.~\eqref{eq:hieu_ledent_estimator}). While this formulation is natural and effective when the number of classes is moderate, it implicitly ties the overall sample complexity to the worst class-wise estimation error. In particular, their analysis requires all class-wise risk estimators to concentrate simultaneously, leading to a sample complexity that scales with $\rho_{\min}^{-1}$ where $\rho_{\min}$ denotes the rarest class probability. As a consequence, the bound becomes overly pessimistic in regimes with many classes, where $\rho_{\min}$ is necessarily small, and reduces to a dependence on the number of classes $R$ only in the restrictive case of perfectly balanced class distributions.

Our first contribution relaxes this worst-case requirement by allowing heterogeneous accuracy across class-wise estimators. Intuitively, classes with negligible contribution to the population risk need not be estimated with the same precision as dominant ones. This eliminates the parasitic dependence on $\rho_{\min}$ and yields bounds that depend on $R$ irrespective of the class distribution. Nevertheless, even this scale remains pessimistic in extreme multi-class settings where $R$ is very large, and it stands in contrast with classical results on excess risk bounds for $k$-tuple learning where sample complexity depends only on the tuple size $k$ \citep{article:clemencon2008, journal:clemencon2016}. This mismatch suggests that controlling class-wise estimators, even in aggregate, is fundamentally sub-optimal. We therefore introduce an \textbf{entirely different} U-Statistic formulation (cf. Eqn.~\eqref{eq:natural_estimator}), targeting a debiased surrogate of the collision-free risk (cf. Eqn.~\eqref{eq:population_contrastive_risk}) that enforces joint concentration across classes. The resulting sample complexity scales as $k(1-\tau)^{-2}$, where $\tau$ is the overall class collision probability (see, e.g., \citet{article:ash2022, article:awasthi2022} or \citet[Section 4.2]{article:arora19}). Specifically, when negatives are sampled without excluding positive classes, $\tau$ represents the overall likelihood that there exists at least a positive-class sample among the selected negatives (cf. Eqn.~\eqref{eq:collision_probability}). In extreme multi-class regime, $\tau\to0$, recovering a dependency on $k$, aligning with established results on tuple-learning.
\section{Related Work}
\textbf{Empirical Minimization of U-Statistics:} U-Statistics were introduced by \citet{article:hoeffding1948} for parameters estimation using symmetric kernels. Later, the concentration properties of U-Statistics were analyzed rigorously by \citet{article:ginezinn1984} and \citet{article:arconesgine1993}, laying the foundation for works in statistical learning theory where losses rely on two or more samples. Concentration inequalities for U-Statistics have been utilized in multiple machine learning disciplines, including ranking \citep{article:clemencon2008}, metric learning \citep{article:Cao2016, article:jin2009}, pairwise learning \citep{article:yleiregularizedpairwise2018, article:yleiledent2020}, clustering \citep{article:liliu2021} and self-supervised learning \citep{article:haochen2021}.

\textbf{Generalization under the i.i.d. Tuples Regime:} The i.i.d. tuple framework was introduced by \citet{article:arora19}, who formalized unsupervised risk for contrastive learning. Ignoring the dependency on the function class, their excess risk bounds scale as $\mathcal{O}(\sqrt{{k}/{\mathrm{N}}})$ where $k$ denotes the number of negative samples and ${\mathrm{N}}$ is the number of i.i.d. \textit{tuples} available to the learner. Using tools such as worst-case Rademacher complexity and fat-shattering dimension,  \citet{article:lei2023} subsequently improved this dependence to $\widetilde{\mathcal{O}}(\sqrt{{1}/{\mathrm{N}}})$. We note that in these works, `sample complexity' is expressed in terms of the number $\mathrm{N}$ of required $(k+2)$-\emph{tuples} (with each tuple being one `sample'), while in our work, it refers to the number $N$ of required \emph{labeled data points} which are subsequently used to form collision-free tuples. Hence a sample complexity of $\widetilde{\mathcal{O}}(k)$ in the supervised CRL formulation is comparable to a sample complexity of $\widetilde{\mathcal{O}}(1)$ in i.i.d. tuples regime. Indeed, this is what can be expected from a Hoeffding-type argument in the pure tuple-wise learning (without constraints on collisions), since $\mathcal{O}({N}/{k})$ tuples can be formed with $N$ samples. Subsequently, the work of \citet{article:arora19} was extended to other settings such as adversarial CRL \cite{article:xinandwei2023, article:ghanooni2024},  PAC-Bayes analysis and deep CRL \cite{article:nozawa2020, article:nozawa2021,article:alves2024,article:hieu2024}. Apart from generalization, prior works inspired by the i.i.d. regime also investigates other angles of contrastive learning theory such as the role of negative samples \citep{article:ash2022, article:awasthi2022, bao2022surrogate}, mitigating spurious features in CRL \citep{article:ghanooni2025}, and the effect of the unsupervised risk on downstream tasks \cite{article:arora19, article:chuang2020,article:liliu2021, bao2022surrogate, article:huang2023,cui2025inclusive,article:cui2025}. 

\textbf{Theoretical Analysis for SS-CRL:} A strong body of work inspired by \citet{article:arora19} explores generalization and other theoretical aspects of self-supervised CRL. For example, \citet{article:wang2022} proposed augmentation overlap theory, suggesting that aggressive data augmentation can implicitly improve intra-class compactness. Along similar lines, \citet{article:cui2025} study how augmentation affects downstream performance by relating the supervised risk to the empirical self-supervised risk (over a set of i.i.d. tuples) and the expected same-class or same-instance augmentation distance. \citet{cui2025inclusive} provide both theoretical and practical insights on a variant of supervised CRL with non-collided i.i.d. tuples with label noise. \citet{article:huang2023} studied the effect of the self-supervised risk on classification by bounding the error rate of nearest-neighbor classifiers built on top of representation functions using misalignment probability. \citet{article:nozawa2021} provided a theoretical explanation for why using many negative samples improves downstream performance. In \citet{article:haochen2021}, the authors established guarantees for the excess contrastive risk for loss functions in \textit{two arguments} constructed from fixed pool of \textit{unlabeled} samples via independent augmentation of both samples. Whilst this involves a second order U-statistic, this work makes no constraint on the valid pairs/tuples in terms of class-collision, making the concentration of the unsupervised risk straightforward: the core contribution instead lies in the relationship between the contrastive risk and the classification risk of a linear probe under structural assumptions on the informativeness of the data-augmentation generation. Their framework was extended to other generalization analyses for variants of SS-CRL including domain adaptation \citep{article:shen2022}, weakly-supervised learning \citep{article:cui2023}, multi-modal learning \citep{article:zhang2023multimodal}.

\textbf{Generalization in Supervised CRL:} \citet{article:hieu2025} explored the CRL setting where input tuples are selected from a finite pool of labeled samples. In this regime, the authors proposed a U-Statistic formulation that estimates each class-wise contrastive risk separately. In their analysis, all class-wise estimators are required to concentrate with equal accuracy, leading to a dependency on $\rho_{\min}$, the smallest class probability. For their estimator, we achieve a dependency of $R$ instead. We also achieve a faster rate of $k$ by constructing a radically different estimator and analysis.


\section{Preliminaries}
\subsection{Definitions and Notations}
Let $\mathcal{X}$ be the input space and $[R]=\{1,\dots, R\}$ be a labels set. We are given a (labeled) dataset $S=\set{X_j}_{j=1}^N\sim \mathcal{\bar D}^{\otimes N}$ where $\mathcal{\bar D}: =\sum_{r=1}^R \rho_r\mathcal{D}_r$ is a mixture of distributions with $\rho_1,\dots, \rho_R$ representing class probabilities and $\mathcal{D}_r$ being the class-conditional distributions over $\mathcal{X}$. In supervised CRL, we aim to estimate a risk associated to some contrastive loss function using the available labeled data points. Specifically, for a certain representation function $f\in\mathcal{F}$ where $\mathcal{F}$ is a prefixed class of functions (e.g. neural networks, linear maps, etc.), we define $\ell_{\phi, f}:\mathcal{X}^{k+2}\to\R_+$ as a tuple-wise loss evaluated on the representations produced by $f$ as follows:
\begin{align}
    &\ell_{\phi, f}\left(X, X^+, \set{X_i^-}_{i=1}^k\right) = \\
    &\qquad\qquad\phi\Big(\Big\{
        f(X)^\top\left[f(X^+) - f(X_i^-)\right]
    \Big\}_{i=1}^k\Big),\nonumber
\end{align}
where $k\ge1$ is the number of negative samples and $\phi:\R^k\to\R_+$ is a vector function often referred to as the contrastive loss. For example, if we use the logistic loss \citep{article:sohn2016} for training, then for all $k$-dimensional vectors ${\bf v}=(v_1,\dots,v_k)^\top$, we have $\phi({\bf v}) := \ln\Big(1+\sum_{i=1}^k e^{-v_i}\Big)$. In this paper, we are interested in bounding the \textbf{uniform excess risk} (over a class of representation functions) associated to the contrastive risk defined as follows.

\begin{definition}[Population Contrastive Risk]
    Let $f\in\mathcal{F}$ be a representation function and $\phi:\R^k\to\R_+$ be a contrastive loss (where $k\ge1$ is the number of negative samples). Then, the contrastive risk of $f$, denoted $\Lnorm_\phi(f)$, is defined as:
    \begin{align}
        \label{eq:population_contrastive_risk}
        &\Lun(f) = \\
        &\underset{r\sim\rho}{\E}\underset{\substack{X,X^+\sim\mathcal{D}_r^{\otimes 2}\\\set{X_i^-}_{i=1}^k\sim\mathcal{\bar D}_r^{\otimes k}}}{\E}\left[\ell_{\phi, f}\left(X, X^+, \set{X_i^-}_{i=1}^k\right)\right],\nonumber
    \end{align}
    where $\rho$ is the discrete distribution over the classes that assigns probability $\rho_r$ to class $r$ for all $r\in[R]$ and $\mathcal{\bar D}_r$ is a distribution over $\mathcal{X}$ defined by the following mixture:
    \begin{align}
        \label{eq:collision-free-negative-dist}
        \mathcal{\bar D}_r := \frac{1}{1-\rho_r}\sum_{q\ne r}\rho_q\mathcal{D}_q,
    \end{align}
    which is the re-normalized mixture of class-conditional distributions excluding class $r$, i.e., the distribution of negative samples when anchor-positive pairs come from class $r$.
\end{definition}
\begin{remark}
    The definition in Eqn.~\eqref{eq:population_contrastive_risk} differs from the definition of negative distribution used in \citet{article:arora19}. Our formulation forbids negative samples $\set{X_i^-}_{i=1}^k$ from sharing classes with the anchor–positive pair $(X, X^+)$, whereas \citet{article:arora19} allow such class collisions by drawing negatives i.i.d. from the full mixture $\bar{\mathcal D}$ rather than $\bar{\mathcal D}_r$. In the self-supervised case~\cite{article:chen2020,He_2020_CVPR,Caron20,article:cui2025}, class-collision is often unavoidable since the labels are not available to prevent it.  On the other hand, in supervised CRL, it is natural to avoid class-collision, as it has been shown to improve performance \citep{article:khosla2020, article:chuang2020, article:robinson2021, chen2022incrementalfalsenegativedetection,article:huynh2022}. Therefore,  it is the primary setting we investigate. However, we also provide excess risk guarantee (Thm. \ref{thm:concentration_of_bar_Uomega}) for the collision-allowed risk in \citet{article:arora19} although this result mostly serves as an intermediary step to provide excess risk guarantees for the desired risk $\Lnorm_\phi$.
\end{remark}
\textbf{Notations}: Throughout this manuscript, we use ${\E}_{X\sim\mu}[\cdot]$ to denote the population expectation over some distribution $\mu$ or simply $\E[\cdot]$ when the distribution is obvious in given context. Additionally, we use $\widehat\E_V[\cdot]$ or $\widehat\E_{v\sim V}[\cdot]$ to denote the empirical average over elements $v$ of some set $V$. For example, $\widehat\E_V[g]=\widehat\E_{v\sim V}[g(v)]:=\frac{1}{|V|}\sum_{v\in V}g(v)$. For each $r\in[R]$, we also denote $S_r\subseteq S$ as the set of labeled data points belonging to class $r$, $N_r:= |S_r|$ and $\widehat\rho_r:=\frac{N_r}{N}$. Finally, we assume that the contrastive loss is \textbf{uniformly bounded}, i.e., $|\ell_{\phi, f}|\le\mathcal{B}$ for all $f\in\mathcal{F}$ for some $\mathcal{B}\in\R_+$. For the reader's convenience, we provide an exhaustive list of notations in the appendix (Table \ref{tab:notations}).

\subsection{Contributions}
In this work, we are interested in formulating estimators for the class-collision free population risk defined in Eqn.~\eqref{eq:population_contrastive_risk} and bounding the corresponding uniform excess risk. Our contributions can be summarized as follows:
\begin{enumerate}[label=(\roman*)]
    \item We provide a refined excess risk guarantee for the U-Statistic estimator proposed in \citet{article:hieu2025} by completely removing the dependency on $\rho_{\min}^{-1}$, the rarest class probability, from the sample complexity. By imposing non-uniform accuracy for the concentration of each class-wise U-Statistic, we prove that the sample complexity of their estimator scales linearly with $R$, the number of classes, regardless of the class distribution's structure.
    \item We propose a different U-Statistic formula inspired by de-biasing the population unsupervised risk used by previous works in the i.i.d. tuples regime \citep{article:arora19,article:lei2023} in which class-collision is arbitrarily allowed. Specifically, we define a debiased surrogate of the desired collision-free risk as a combination of the risks with arbitrary collisions and with one enforced collision (Eqn.~\eqref{eq:decomposed_population_contrastive_risk}). Then, we approximate each component separately to formulate the final estimator.
    \item We prove that the proposed U-Statistic yields a sample complexity that scales with $k(1-\tau)^{-2}$ where $\tau$ is the population class-collision probability (Eqn.~\eqref{eq:collision_probability}). In extreme multi-class scenarios where class probabilities are small, the collision probability $\tau$ is approximately zero and the sample complexity recovers the $\mathcal{{O}}(k)$ rate. Therefore, this type of bound is particularly useful for large multi-class representation learning problems where $R$ is significantly larger than $k$. Furthermore, for certain long-tailed imbalanced class distributions, we have $1-\tau\in\Omega(1)$ with high probability. Therefore, the dominant dependency on $k$ remains unchanged.
\end{enumerate}
An outline of our main results (summarized in sample complexities) with comparison to prior works is provided in Table \ref{tab:results_comparison}. Additionally, we also provide a more thorough summary of the literature in Appendix \ref{sec:summary_of_literature}. Furthermore, in Appendix \ref{sec:proof_discussion}, we provide a brief discussion on the differences between our analysis and previous works, as well as the additional technical difficulties involved. 

\begin{table*}[ht]
	\begin{center}
		\begin{small}
			\caption{Comparison of our work with the previous result in terms of sample complexity. For simplicity, we assume that the loss upper bound $\mathcal{B}=1$. The notation ignores polylogarithmic factors of all relevant quantities ($\mathfrak{C}^2_N$, $R$, $k$, ${1}/{(1-\tau)}$ and the failure probability $\Delta$). For $\mathcal{L}$-Lipschitz parameterized function classes $\mathcal{H}$, the complexity term $\mathfrak{C}_N(\mathcal{H})$ grows at most as $\mathcal{\widetilde O}[(W)^\frac{1}{2}]$ where $W$ denotes the number of parameters and the $\widetilde{\mathcal{O}}$ notation hides polylogarithmic factors of $\mathcal{L}$ and $W$. The probability $\gamma_k$ is defined in Eqn.~\eqref{eq:gamma_k}.}
            \label{tab:results_comparison}
			\centering
            \begin{tabular}{c|c|c|c|c}
            \toprule
            \textbf{Reference} & \textbf{Est.} & \textbf{Default} & $\gamma_k \in \Omega(1)$ & \textbf{Relevant quantity} \\
            \midrule
            \citet{article:hieu2025} & \multirow[c]{2}{*}{\makecell{\vspace{0.3em}\\$U_N^\mathrm{hl}$\\(For $\mathrm{L}_\phi$)}} & \multicolumn{2}{c|}{$[\mathfrak{C}^2_N(\mathcal{H})+1]\max\left[\rho_{\min}^{-1}, (1-\rho_{\max})^{-1}\right]$} & \makecell{$\rho_{\min}:=\min_{r\in[R]}\rho_r$\\$\rho_{\max}:=\max_{r\in[R]}\rho_r$}\\
            \cmidrule(lr){1-1}\cmidrule(lr){3-5}
            This Work & & \makecell{$\mathfrak{C}^2_N(\mathcal{H})[\widehat\theta_{k+2}R+(1-\widehat\theta_{k+2})^2k] \ +$\\$[R+k^2]$\\(Theorem \ref{thm:refined_hieu_ledent_general}/\ref{thm:refined_hieu_ledent_general_appendix})}  & \makecell{$\mathfrak{C}^2_N(\mathcal{H})R \ +$\\$[R+k^2]$\\(Theorem \ref{thm:refined_hieu_ledent_small_probabilities})} & $\widehat\theta_{k+2}:={\P_{r\sim\widehat\rho}}\big[\widehat\rho_r\le \frac{2}{k+2}\big]$ \\
            \midrule
            This Work & \makecell{${U_N/\bar U_N}$\\(For $\mathrm{\widetilde L}_\phi$)} & \makecell{$\mathfrak{C}^2_N(\mathcal{H})(1-\tau)^{-2}k \ +$\\$[R(1-\tau)^{-4} + Rk(1-\tau)^{-1}+k^2]$\\(Theorem \ref{thm:main_result_more_general}/\ref{thm:main_result_more_general_appendix})} & \makecell{$\mathfrak{C}^2_N(\mathcal{H})k \ +$\\$[R + Rk]$\\(Theorem \ref{thm:second_main_result_appendix})} & $\tau:=1-\sum_{r=1}^R\rho_r(1-\rho_r)^k$\\
            \bottomrule
            \end{tabular}
        \end{small}
    \end{center}
    \vskip -0.1in
\end{table*}

\section{Main Results}
\subsection{Refined Analysis of Prior Work}
For a given function $f\in\mathcal{F}$, \citet{article:hieu2025} estimates the collision-free population contrastive risk in Eqn.~\eqref{eq:population_contrastive_risk} using a U-Statistic formulation $U_N^\mathrm{hl}$ defined as follows:
\begin{align}
    \label{eq:hieu_ledent_estimator}
    U_N^\mathrm{hl}(f) :=\sum_{r=1}^R\widehat\rho_r U_{\Theta_r}(f).
\end{align}
\noindent  Here, for all $r\in[R]$, $U_{\Theta_r}(f) := \widehat\E_{\Theta_r}\left[\ell_{\phi,f}\right]$ and $\Theta_r$ is a collection of $(k+2)$-tuples, defined as:
\begin{align}
    \Theta_r := \Big\{\Big(&X, X^+, \set{X_i^-}_{i=1}^k\Big): \label{eq:no_collision_tuples} \\ 
    &X, X^+\in S_r, \set{X_i^-}_{i=1}^k\subseteq S\setminus S_r\Big\}. \nonumber
\end{align}
\noindent Essentially, their proposed estimator is a weighted sum of class-wise U-Statistics where each class' estimator averages over the set of class-collision free tuples with anchor-positive pairs belonging to that class. In their analysis, the uniform excess risk bound is proved to scale roughly as:
\begin{align}
    \label{eq:hieu_ledent_old_result}
    \mathcal{O}\Bigg[\sum_{r=1}^R\frac{\rho_r\mathfrak{C}_N(\mathcal{H})}{\sqrt{\min\left(\frac{N_r}{2}, \frac{N-N_r}{k}\right)}} + \mathcal{B}\sqrt{\frac{\ln (R/\Delta)}{N\rho_{\min}}}\Bigg],
\end{align}
\noindent with probability of at least $1-\Delta$ for any $\Delta\in(0,1)$ where $\mathcal{H}$ is the class of loss functions:
\begin{align}
    &\mathcal{H} := \big\{\big(X, X^+,\set{X_i^-}_{i=1}^k\big)\mapsto  \label{eq:loss_class}\\
    &\qquad\qquad\ell_{\phi, f}\big(X, X^+,\set{X_i^-}_{i=1}^k\big):f\in\mathcal{F}\big\},\nonumber
\end{align}
\noindent and $\mathfrak{C}_N(\mathcal{H})$ measures the worst-case Dudley entropy integral of $\mathcal{H}$ over the collections of $(k+2)$-tuples with size $N$ (cf. Eqn.~\eqref{eq:worst_case_dudley_integral}). For common function spaces (linear maps, neural networks, etc), $\mathfrak{C}_N(\mathcal{H})$ grows at most in polylogarithmic order of $N$. For instance, a parameter-counting argument \citep{article:longsedghi2020,article:graf2023} yields $\mathfrak{C}_N(\mathcal{H})\in \mathcal{\widetilde O}(\sqrt{WL})$ for an $L$-layered neural network with $W$ parameters, where the $\mathcal{\widetilde{O}}$ notation hides logarithmic factors in weight norms. In \citet{article:hieu2025}, the authors used a high-probability argument to show that the following lower bound holds simultaneously for all $r\in[R]$:
\begin{align*}
    \min\left(\frac{N_r}{2}, \frac{N-N_r}{k}\right)\gtrsim\min\left(\frac{N\rho_{\min}}{2}, \frac{N[1-\rho_{\max}]}{k}\right),
\end{align*}
\noindent resulting in a sample complexity that scales in the order of $\mathcal{O}\left(\max\left(\frac{1}{\rho_{\min}}, \frac{k}{1-\rho_{\max}}\right)\right)$. While this estimation of the overall sample complexity is valid, it is vastly pessimistic. Specifically, it requires every class in the distribution to concentrate uniformly well around their respective class-wise risk. However, this is not necessary since tail classes often contribute minimally to the overall population risk. Therefore, the estimation of these small classes need not to be as accurate as major classes. Using this intuition, we improved upon previous proof techniques to derive a sample complexity that scales in the order of at most $\mathcal{O}\left(\max\left(R, k\right)\right)$.
\begin{theorem}[cf. Theorem \ref{thm:refined_hieu_ledent_general_appendix}, Table \ref{tab:results_comparison}]
    \label{thm:refined_hieu_ledent_general}
    Let $\mathcal{F}$ be a class of representation functions and let $U_N^\mathrm{hl}(f)$ be defined for each $f\in\mathcal{F}$ as in Eqn.~\eqref{eq:hieu_ledent_estimator}. Suppose $\rho_r\le\frac{1}{2}$ for all $r\in[R]$. Then, for $\Delta\in(0,1)$, as long as $N\ge \mathcal{\widetilde{O}}\left(k^2\right)$, we have:
    \begin{align}
        &\sup_{f\in\mathcal{F}}\left|U_N^\mathrm{hl}(f) - \Lun(f)\right| \le \\
        &\mathcal{O}\Bigg( \mathfrak{C}_N(\mathcal{H})\sqrt{\frac{\max(k, R)}{N}} + \mathcal{B}\sqrt{\frac{\max(k,R)\ln(R/\Delta)}{N}}\Bigg),\nonumber
    \end{align}
    \noindent with probability of at least $1-\Delta$. The complexity $\mathfrak{C}_N(\mathcal{H})$ of $\mathcal{H}$ is defined in Eqn.~\eqref{eq:worst_case_dudley_integral} or Definition \ref{def:worst_case_dudley_entropy} and the $\mathcal{\widetilde O}$ hides logarithmic orders of $R, \Delta$.
\end{theorem}
\noindent In the above, we presented a simplified result where the bound scales in worst-case order of $\mathcal{O}(\sqrt{\max(k,R)/N})$. As previously claimed, the result is independent of the class distribution. In the precise version of the theorem (Theorem \ref{thm:refined_hieu_ledent_general_appendix}), the order of the excess risk interpolates between square-root dependencies on both $k$ and $R$. Specifically, when the total mass of small probabilities (cf. Eqn.~\eqref{eq:gamma_k}) isn't negligible, the reliance on the number of classes dominates the dependency on tuple size. For instance, we demonstrate in Theorem \ref{thm:refined_hieu_ledent_small_probabilities} that the order of the bound collapses to depending on $\sqrt{R}$ entirely when $\rho_r\lesssim 1/k,\forall r\in[R]$. In the next section, we formulate a fundamentally different estimator that encourages $\mathcal{O}(\sqrt{k})$ dependency when the class distribution consists of mostly small probabilities, making it suitable for extreme multi-class scenarios.

\subsection{Refined U-Statistics Formulation}
\subsection*{U-Statistics Formulation}
As seen in the previous section, the desired collision-free risk can be estimated by combining the class-wise U-Statistics, each weighted by the empirical class probability. However, we found that faster concentration can be achieved with an indirect approach involving the estimation of risks that \textbf{allow class-collision} to various extents. Specifically, we introduce the following population risks:
\begin{definition}[Contrastive Risks with Class-collision]
    Let $f\in\mathcal{F}$ be a representation function and $\phi:\R^k\to\R_+$ be a contrastive loss. We define $\Lnorm_\Omega(f)$ and $\Lnorm_\Lambda(f)$ as follows:
    \begin{align}
        &\Lnorm_\Omega(f):= \label{eq:lomega}\\
        &\underset{r\sim\rho}{\E}\underset{\substack{X,X^+\sim\mathcal{D}_r^{\otimes 2}\\\set{X_i^-}_{i=1}^k\sim\mathcal{\bar D}^{\otimes k}}}{\E}\left[\ell_{\phi, f}\left(X, X^+, \set{X_i^-}_{i=1}^k\right)\right], \nonumber \\
        &\Lnorm_\Lambda(f):=  \label{eq:llambda} \\
        &\underset{r\sim\rho}{\E}\underset{\substack{X,X^+, \bar X\sim\mathcal{D}_r^{\otimes 3}\\\set{X_i^-}_{i=1}^{k-1}\sim\mathcal{\bar D}^{\otimes k-1}}}{\E}\left[\ell_{\phi, f}\left(X, X^+, \bar X, \set{X_i^-}_{i=1}^{k-1}\right)\right].\nonumber
    \end{align}
    \noindent Here, $\Lnorm_\Omega(f)$ denotes the risk of $f$ when negative samples either collide with the anchor-positive pairs or not. On the other hand, $\Lnorm_\Lambda(f)$ denotes the risk where negative samples must collide \textbf{at least once} with the anchor-positive pairs. 
    
    \noindent We then define a debiased surrogate $\mathrm{\widetilde L}_\phi$ of the collision-free contrastive risk $\Lun$ (Eqn.~\eqref{eq:population_contrastive_risk}) in terms of $\Lnorm_\Omega$ and $\Lnorm_\Lambda$ using a debiased formula as follows:
    \begin{align}
    	\mathrm{\widetilde L}_\phi(f) &:= \frac{1}{1-\tau}\Lnorm_\Omega(f) - \frac{\tau}{1-\tau}\Lnorm_\Lambda(f). \label{eq:decomposed_population_contrastive_risk} \\
    	\tau    &:= 1- \sum_{r=1}^R \rho_r\squarebrac{1-\rho_r}^k.     \label{eq:collision_probability}
    \end{align}
    \noindent Estimators of $\Lun$ itself along with uniform deviation bounds are provided in Appendix \ref{sec:exact-debias} (cf. Theorems \ref{thm:exact_class_weighted}, \ref{lem:general_tuple_functions}). Using the above formula, instead of directly estimating $\Lun$ with collision-free tuples, we can approach this indirectly by estimating $\Lnorm_\Omega$, $\Lnorm_\Lambda$ and $\tau$ separately. Specifically, we can formulate a U-Statistic to estimate the surrogate $\mathrm{\widetilde L}_\phi$ as:
    \begin{align}
    	U_N(f) = \frac{1}{1-\widehat\tau}U_\Omega(f) - \frac{\widehat\tau}{1-\widehat\tau}U_\Lambda(f). \label{eq:natural_estimator}
    \end{align}
\end{definition}
\noindent Where $\widehat\tau$ is an estimator of $\tau$, $U_\Omega(f):=\sum_{r=1}^R\widehat\rho_r\widehat\E_{\Omega_r}[\ell_{\phi, f}]$ and $U_\Lambda(f):=\sum_{r=1}^R\widehat\rho_r\widehat\E_{\Lambda_r}[\ell_{\phi, f}]$ where for all $r\in[R]$, the collections $\Omega_r$ and $\Lambda_r$ of $(k+2)$-tuples are defined as:
\begin{align}
    \Omega_r &:=\Big\{\Big(X, X^+, \set{X_i^-}_{i=1}^k\Big): \label{eq:arbitrary_collision_tuples} \\
        &X, X^+\in S_r, \set{X_i^-}_{i=1}^k\subseteq S\setminus\set{X, X^+}\Big\}, \nonumber \\
    \Lambda_r&:=\Big\{\Big(X, X^+, \bar X, \set{X_i^-}_{i=1}^{k-1}\Big): \label{eq:atleast_one_collision_tuples} \\
        &X, X^+, \bar X\in S_r, \set{X_i^-}_{i=1}^{k-1}\subseteq S\setminus\set{X, X^+, \bar X}\Big\}. \nonumber
\end{align}
\noindent Intuitively, when $U_\Omega(f), U_\Lambda(f)$ estimates the intermediary risks $\Lnorm_\Omega(f),\Lnorm_\Lambda(f)$ accurately, then $U_N(f)$ becomes a good estimator for the surrogate risk $\mathrm{\widetilde L}_\phi(f)$ (depending also on the accuracy of $\widehat\tau$). Unfortunately, both $U_\Omega(f)$ and $U_\Lambda(f)$ are \textbf{biased estimators}, making it difficult to conduct any concentration analysis. To be more specific, for every tuple in each collection $\Omega_r$, the selection of negative samples $\set{X_i^-}_{i=1}^k$ is dependent on the previously selected anchor-positive pair $(X,X^+)$, making these negative samples no longer distributed according to the full mixture of class-conditional distributions $\mathcal{\bar D}$. Similarly, the same dependency structure exists for every tuple in each collection $\Lambda_r$.

\noindent To address the biased-ness caused by dependent tuples, we construct an \textbf{auxiliary estimator} $\bar U_N$ defined as follows:
\begin{align}
    \label{eq:combined_aux_estimators}
    \bar U_N(f) := \frac{1}{1-\widehat\tau}\bar U_\Omega(f) - \frac{\widehat \tau}{1-\widehat\tau}\bar U_\Lambda(f),
\end{align}
\noindent which is structurally similar to the definition of the proposed estimator $U_N$ in Eqn.~\eqref{eq:natural_estimator} except $U_\Omega$ and $U_\Lambda$ are replaced by $\bar U_\Omega$ and $\bar U_\Lambda$, which are (asymptotically) unbiased estimators of the intermediary population risks $\Lnorm_\Omega$ and $\Lnorm_\Lambda$. Particularly, in the construction of $\bar U_\Omega$ (and $\bar U_\Lambda$), we disassociate the selection of negative samples from anchor-positive pairs by first splitting the labeled dataset $S$ into non-overlapping partitions. Then, anchor-positives are only selected from one partition while negative samples are selected from the other, making all elements of each tuple independent. To preserve the overall flow, we defer the specific construction of both $\bar U_\Omega$ and $\bar U_\Lambda$ to Appendix \ref{sec:aux_estimators_construction}. Additionally, we highlight the difference between the construction of our estimator and that of the previous work in Figure \ref{fig:tuples_selection}. Furthermore, we also prove that when the sample size $N$ is large, the differences between $\bar U_\Omega$ and $U_\Omega$ as well as between $\bar U_\Lambda$ and $U_\Lambda$ can be relatively well-controlled (cf. Proposition \ref{prop:diff_between_natural_and_nonatural}). Therefore, we can focus primarily on concentration properties of the proposed auxiliary estimator $\bar U_N$.

\begin{remark}
	The idea of debiasing the empirical risk by decomposing the population risk based on class-collision has been explored in \citet[Section 4.2]{article:arora19}. This decomposition has also been employed by later works to bound the surrogate gap between supervised and unsupervised risk. Recently, the same idea was utilized by \citet{article:chuang2020} to mitigate negative sampling bias, although for a much simpler variant of contrastive learning (PU-learning) in \textbf{semi-supervised settings}. Therefore, we do not claim novelty for the empirical performance of the proposed U-Statistic $U_N$ in Eqn.~\eqref{eq:natural_estimator}. The primary focus of this work is providing sharpened excess risk bounds.
\end{remark}
\subsection*{Concentration of Auxiliary Estimators}
In this section, we will focus on the concentration properties of the auxiliary estimator defined in Eqn.~\eqref{eq:combined_aux_estimators}. Firstly, we introduce the following important definition.

\begin{definition}[Worst-case Dudley Entropy]
    \label{def:worst_case_dudley_entropy}
    Let $\mathcal{G}$ be a class of functions $g:\mathcal{Z}\to\R$ such that $|g|\le \mathcal{B}$ for all $g\in\mathcal{G}$ and $n\ge1$ be an integer. Then, the worst-case Dudley entropy of $\mathcal{G}$ over samples of size $n$ is defined as:
    \begin{align}
        \mathfrak{C}_n(\mathcal{G}):=\sup_{S\in\mathcal{Z}^n}\int_{\frac{1}{n}}^\mathcal{B} \sqrt{\ln2\mathcal{N}\left(\mathcal{G}, \varepsilon, \Lnorm_2(S)\right)}d\varepsilon,
    \end{align}
    \noindent where $\mathcal{N}(\mathcal{G}, \varepsilon, \Lnorm_2(S))$ denotes the $\varepsilon$-covering number of $\mathcal{G}$ with respect to the $\Lnorm_2(S)$ norm restricted to the dataset $S=\set{\z_j}_{j=1}^n$, defined for any $g,\bar g\in\mathcal{G}$ as follows:
    \begin{align}
        \|g-\bar g\|_{\Lnorm_2(S)}^2:= \frac{1}{n}\sum_{j=1}^n|g(\z_j)-\bar g(\z_j)|^2.
    \end{align}
\end{definition}
\noindent In the following result, for a tuple-wise loss class $\mathcal{H}$, we express our generalization bounds in terms of the quantity $\mathfrak{C}_N(\mathcal{H})$, which is an upper bound on the worst-case Rademacher complexity \citep{LeiTIT, srebroSLFR} of the loss class $\mathcal{H}$ over any set of $N$ tuples $Z_1,\ldots,Z_N\in\mathcal{X}^{k+2}$. For a parametric model, $\mathfrak{C}_N(\mathcal{G})$ will be dominated by the number of parameters in the model, up to polylogarithmic factors~\cite{article:longsedghi2020,article:graf2023}, and the argument can be extended to $\mathfrak{C}_N(\mathcal{H})$ with arguments such as those in~\cite{article:hieu2024} at the cost of polylogarithmic factors in all quantities including the number of negative samples $k$. For a class of linear maps $A\in\R^{m\times d}$ with norm bounded by $s$, the complexity will scale like $\widetilde{\mathcal{O}}(\sqrt{md})$ where the $\widetilde{\mathcal{O}}$ notation includes polylogarithmic factors of both $s$ and $k$, for a neural network, the complexity will scale as $\widetilde{\mathcal{O}}(\sqrt{WL})$ where $L$ is the depth of the network and $W$ is the number of parameters~\cite{anthony1999neural,article:graf2023,NEURIPS2025_d9556431}.

\begin{theorem}[cf. Theorem \ref{thm:main_result_more_general_appendix}, Table \ref{tab:results_comparison}]
	\label{thm:main_result_more_general}
	Let $\mathcal{F}$ be a class of representation functions and $\bar U_N(f)$ be defined for each $f\in\mathcal{F}$ as in Eqn.~\eqref{eq:combined_aux_estimators}. Suppose that $R\ge k$. Then, for any $\Delta\in(0,1)$, with probability of at least $1-\Delta$, we have:
	\begin{align}
		&\sup_{f\in\mathcal{F}}\left|\bar U_N(f)-\mathrm{\widetilde L}_\phi(f)\right| \le  \\
		&\mathcal{O}\Bigg(\frac{\mathfrak{C}_N(\mathcal{H})}{1-\tau}\sqrt{\frac{k}{N}}+\frac{\mathcal{B}\sqrt{R}}{(1-\tau)^2}\sqrt{\frac{\ln \left(R/\Delta\right)}{N}}+\frac{\mathcal{B}Rk}{N(1-\tau)}\Bigg), \nonumber
	\end{align}
	\noindent as long as $N\ge \mathcal{\widetilde O}\left(k\cdot\max\left\{R, k+\frac{1}{k(1-\tau)}\right\}\right)$. Here, the $\mathcal{\widetilde O}$ notation hides logarithmic orders of $R, \Delta$.
\end{theorem}
\noindent First, we note that the above result and the corresponding sample complexity in Table \ref{tab:results_comparison} also applies to the natural estimator $U_N$ (cf. Remark \ref{rem:remark_on_natural_estimator}). For a fixed $\tau$ and sufficiently large $N$, the dominant term is $\mathfrak{C}_N(\mathcal{H})\sqrt{{k}/{N}}$, which indicates that the required number of samples grows like $k$. This is exactly what one would expect from a pure Hoeffding-type argument~\cite{article:clemencon2008} over tuples in the case where collisions are ignored.  We first note that even though a square-root dependency of $R$ shows up in the second term, it is completely detached from the \textit{function class complexity factor} $\mathfrak{C}_N(\mathcal{H})$, which is a dominant quantity. For instance, if the number of parameters is $W$ (for a Lipschitz parametrized function class), Theorem~\ref{thm:main_result_more_general} corresponds to a sample complexity of $\mathcal{\widetilde{O}}({Wk}{(1-\tau)^{-2}}+{R}{(1-\tau)^{-4}})$: the additive factor of $R$ is only relevant if $W\le {R}/{k}$. This is also largely unavoidable: if the number of samples $N$ is less than $\mathcal{O}(R)$, there are simply no valid tuples to be formed due to the inability to select \textit{any} (positive, anchor positive) pair with distinct elements. Next, the bound also depends on the inverse of the non-collision probability $1-\tau$. However, it should be noted that in extreme multi-class scenarios, this probability is typically well-controlled (i.e., constant) given that the ``total mass of small class probabilities'' (cf. Eqn.~\eqref{eq:gamma_k}) is sufficiently large, which certainly occurs in the approximately uniform extreme multi-class scenario: in this case we have $1-\tau \simeq 1$. To provide better insights, we present several illustrative examples.
\begin{figure*}[ht]
	\centering
	\centerline{\includegraphics[width=0.95\linewidth]{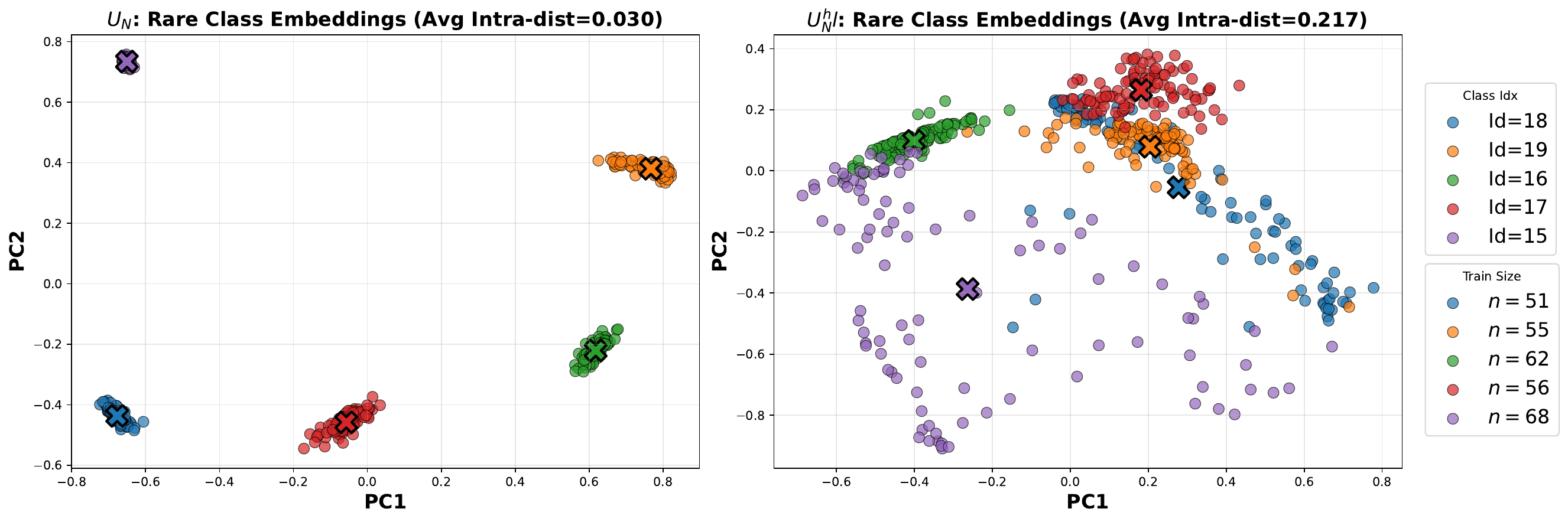}}
	\caption{Synthetic data experiment results for representation learning on an imbalanced dataset. \textbf{Left}: PCA reduced representations of the model trained with sub-sampled estimation of $U_N$ as empirical risk. \textbf{Right}: PCA reduced representations of model trained with sub-sampled estimation of $U_N^\mathrm{hl}$ as empirical risk. Both plots show representations of $100$ data points from top five rarest classes in the synthesized dataset. The legends show class indices as well as the number of data points from each class in the labeled training dataset.}
	\label{fig:rare_class_embeddings}
\end{figure*}
\begin{figure}[ht!]
	\centering
	\centerline{\includegraphics[width=\columnwidth]{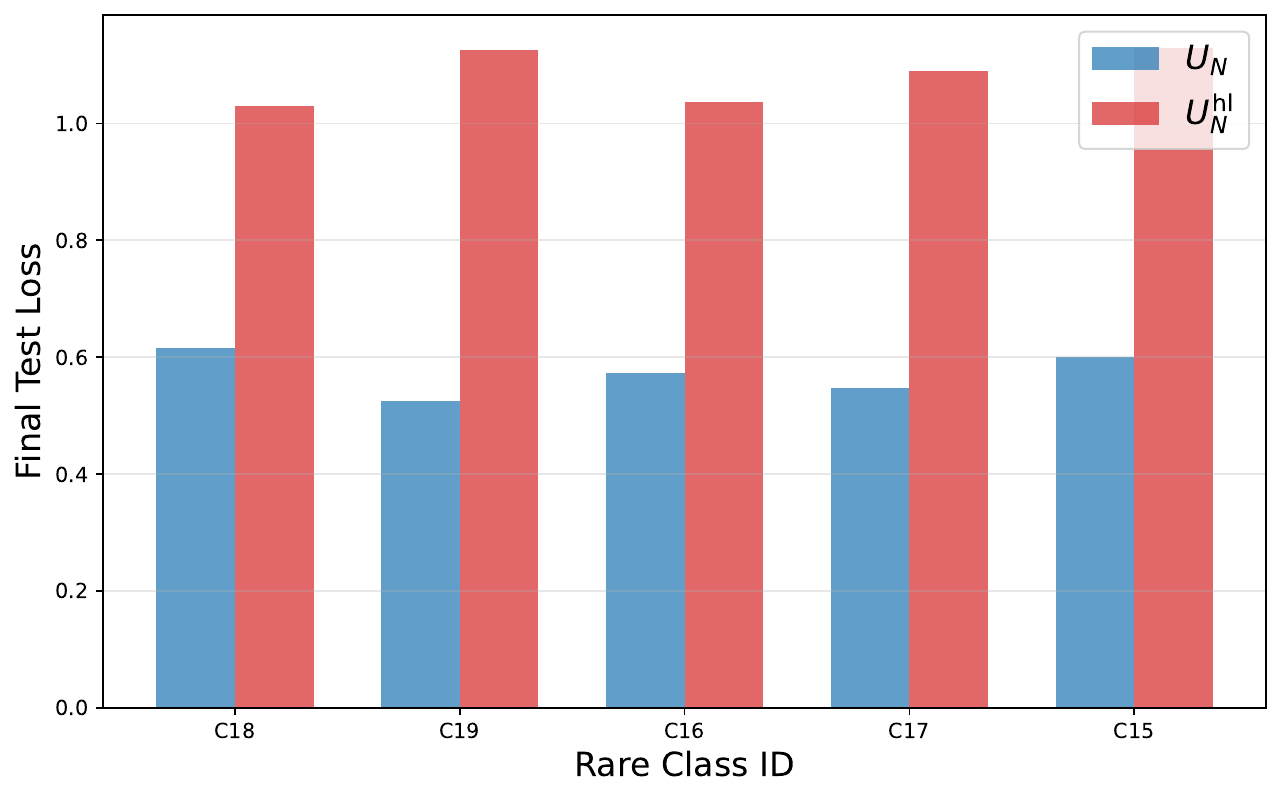}}
	\caption{
		Final contrastive test loss for the model trained using 
		sub-sampled estimations of two U-Statistics formulations, $U_N$ (blue) 
		and $U_N^\mathrm{hl}$, as empirical risks (red) on synthetic dataset.}
	\label{fig:rare_class_test_loss}
\end{figure}
\begin{example}[Typical Extreme Scenario]
    Suppose that $\rho_r\le \frac{1}{k+1}$ for all $r\in[R]$. Then, for triplet learning ($k=1$), we have $1-\tau = 1 - \sum_{r=1}^R\rho_r^2 \ge 1-\frac{1}{2}\sum_{r=1}^R\rho_r=\frac{1}{2}$. On the other hand, for $k\ge2$, we have:
    \begin{align*}
        1-\tau = \sum_{r=1}^R\rho_r(1-\rho_r)^k \ge \left(1-\frac{1}{k+1}\right)^k \ge \frac{1}{e}.
    \end{align*}
    \noindent Essentially, when $R$ is large and all classes are relatively well-spread we have $1-\tau\in\Omega(1)$.
\end{example} 

\begin{example}[General Extreme Scenario]
    Let $\gamma_k$ denote the total mass of small probabilities bounded by $1/k$. Formally, we define $\gamma_k$ as follows:
    \begin{align}
        \label{eq:gamma_k}
        \gamma_k:=\underset{r\sim\rho}{\P}(\rho_r\le1/k)=\sum_{r=1}^R\rho_r\1{\rho_r\le1/k}.
    \end{align}
    \noindent Suppose that $\gamma_k\in\Omega(1)$, i.e., the total mass of small probabilities is not vanishingly small. Then, for $k\ge2$, we have the following lower bound on $1-\tau$:
    \begin{align*}
        1-\tau &\ge \sum_{r:\rho_r\le1/k}\rho_r(1-\rho_r)^k \ge \gamma_k\left(1-\frac{1}{k}\right)^k \ge \frac{\gamma_k}{4}.
    \end{align*}
    \noindent Therefore, we also have $1-\tau\in\Omega(1)$ in this case. In other words, as long as $\gamma_k$ is not infinitesimally small, the main result (Theorem \ref{thm:main_result_more_general}) yields a sample complexity of $k\cdot\mathfrak{C}^2_N(\mathcal{H})$
    \noindent with probability of at least $1-\Delta$ as long as $N\ge\mathcal{\widetilde O}(\max(Rk, k^2))=\mathcal{\widetilde{O}}(Rk)$. 
\end{example}
\begin{example}[Imbalanced Distribution]
\noindent Another interesting case where our result will be tighter is in imbalanced distributions favoring a few major classes (e.g., anomaly detection, medical imaging problems). As a simple example, let $\{\rho_1,\dots,\rho_R\}$ be a class distribution with $\rho_{\max}=\frac{1}{2}$ and $\rho_r\le\frac{1}{k+1},\forall\rho_r\ne\rho_{\max}$. Since $[(1-1/(1+x))^x]'<0$:
\begin{align*}
    1-\tau &\ge \sum_{r:\rho_r<\rho_{\max}} \rho_r(1-\rho_r)^k \ge \frac{1}{2}\left(1-\frac{1}{k+1}\right)^k\\
    &\ge\frac{1}{2}\lim_{k\to\infty}\left(1-\frac{1}{k+1}\right)^k = \frac{1}{2e}.
\end{align*}
\noindent For clarity, this is a notable special case of the previous example with sufficiently large total mass of small probabilities (where $\gamma_k\ge1/2$ in this example). In the Appendix, we also provide some dedicated results (Theorem \ref{thm:main_result_appendix}, \ref{thm:second_main_result_appendix}) for the scenario where $\gamma_k\in\Omega(1)$. In the following section, we will present empirical results to verify that the proposed estimator performs well under class imbalance.
\end{example}
\begin{table*}[ht]
	\begin{center}
		\begin{small}
			\caption{Average precisions, recalls and F1-scores of a simple deep neural network trained using sub-sampled estimations of $U_N$ and $U_N^\mathrm{hl}$ as empirical risks. The one with higher score is highlighted in bold font.}
			\label{tab:experiment_results}
			\centering
			\begin{tabular}{cl|cc|cc|cc}
				\toprule
				&
				& \multicolumn{2}{c}{\textbf{MNIST}}
				& \multicolumn{2}{c}{\textbf{FashionMNIST}}
				& \multicolumn{2}{c}{\textbf{CIFAR10}} \\
				\cmidrule(lr){3-4}
				\cmidrule(lr){5-6}
				\cmidrule(lr){7-8}
				\textbf{No. Negatives} & \textbf{Metric}
				& $U_N$ & $U_N^{\mathrm{hl}}$
				& $U_N$ & $U_N^{\mathrm{hl}}$
				& $U_N$ & $U_N^{\mathrm{hl}}$ \\
				\midrule
				\multirow{3}{4em}{$k=3$} &Avg.\ Precision &  
				{\bf 0.9473} & 0.8243 & {\bf 0.7729} & 0.6369 & 0.2516 & {\bf 0.2916} \\
				& Avg.\ Recall    &  
				{\bf 0.8965} & 0.7650 & {\bf 0.7280} & 0.5545 & {\bf 0.1745} & 0.1240 \\
				& Avg.\ F1-Score  &  
				{\bf 0.9208} & 0.7830 & {\bf 0.7473} & 0.5595 & {\bf 0.1994} & 0.1311 \\
				\midrule
				\multirow{3}{4em}{$k=5$} &Avg.\ Precision &  
				{\bf 0.9460} & 0.6119 & {\bf 0.4891} & 0.3361 & {\bf 0.2817} & 0.2668 \\
				& Avg.\ Recall    &  
				{\bf 0.8930} & 0.4860 & {\bf 0.4501} & 0.2139 & {\bf 0.1465} & 0.1455 \\
				& Avg.\ F1-Score  &  
				{\bf 0.9186} & 0.4867 & {\bf 0.4123} & 0.1783 & {\bf 0.1795} & 0.1746 \\
				\midrule
				\multirow{3}{4em}{$k=7$} &Avg.\ Precision &  
				{\bf 0.9446} & 0.3928 & {\bf 0.4195} & 0.2822 & 0.2616 & {\bf 0.2968} \\
				& Avg.\ Recall    &  
				{\bf 0.9180} & 0.1970 & {\bf 0.3710} & 0.1870 & {\bf 0.1821} & 0.1455 \\
				& Avg.\ F1-Score  &  
				{\bf 0.9311} & 0.2507 & {\bf 0.3637} & 0.1471 & {\bf 0.2111} & 0.1708 \\
				\bottomrule
			\end{tabular}
		\end{small}
	\end{center}
    \vskip -0.1in
\end{table*}
\section{Numerical Experiments}
\subsection{Experiment on Synthetic Imbalanced Dataset}
\label{sec:synth_imbl_dataset}
We hypothesize that the estimator $U_N(f)$ in Eqn.~\eqref{eq:natural_estimator} generalizes better than the formula $U_N^\mathrm{hl}(f)$ used by prior work in certain long-tailed scenarios, especially for minority classes. For example, when a single class is overly dominant while others are extremely small and we have $k\ll R$. To verify this hypothesis, we synthesized an imbalanced dataset using a pre-defined mixture of Gaussian distributions $\mathcal{\bar D}=\sum_{r=1}^R\rho_r\mathcal{N}(\cdot|\mu_r, \sigma_r^2)$ where $R=20$, $\sigma_r^2=1$ for all $r\in[R]$ and $\rho_{\max}=0.5$ (one class dominating half of the dataset). Then, we trained the same model architecture (shallow neural network with ReLU activations) on the synthesized dataset with the \textbf{sub-sampled estimators} (cf. Section \ref{sec:subsample_estimators}) corresponding to $U_N$ and $U_N^\mathrm{hl}$ as empirical risks until convergence. After training, we evaluate the performance of each estimator on the test set. In Figure \ref{fig:rare_class_test_loss}, we report the final contrastive loss computed on test tuples whose anchor–positive pairs belong to the top five rarest classes. From the result, it can be seen that the model trained with the sub-sampled estimation of $U_N$ outperforms that trained using $U_N^\mathrm{hl}$ on every rare class. Additionally, we visualize the PCA transformed representations of each model on test data points of the rarest classes in Figure \ref{fig:rare_class_embeddings} and compare the intra-class compactness of each model's output. From the experiment results, we observe that the model trained with sub-sampled $U_N$ as the empirical risk achieved much better generalization performances on every rare class. Furthermore, the representations produced by this model are also much more compact compared to the representations produced by the model trained with the old estimator $U_N^\mathrm{hl}$. In Appendix \ref{sec:experiment_details}, we provide a more thorough description of experiment settings as well as detailed algorithms for calculating sub-sampled estimations of the two U-Statistics formulas $U_N$ and $U_N^\mathrm{hl}$.

\begin{table*}[ht]
	\begin{center}
		\begin{small}
        \caption{Average (macro) precisions, recalls and F1-scores of a simple \emph{convolutional} neural network trained using sub-sampled estimations of $U_N$, $U_N^\mathrm{hl}$ and SupCon as empirical risks. The one with higher score between the two estimators $U_N$ and $U_N^\mathrm{hl}$ is highlighted in bold font. The metrics of SupCon is underlined if it outperforms both $U_N$, $U_N^\mathrm{hl}$.}
			\label{tab:experiment_results_cnn}
			\centering
			\resizebox{\textwidth}{!}{\begin{tabular}{ll|ccc|ccc|ccc|ccc}
				\toprule
				&
				& \multicolumn{3}{c}{\textbf{MNIST} (imbl)}
				& \multicolumn{3}{c}{\textbf{FashionMNIST} (imbl)}
				& \multicolumn{3}{c}{\textbf{CIFAR10} (imbl)} 
                & \multicolumn{3}{c}{\textbf{CIFAR100}} \\
				\cmidrule(lr){3-5}
				\cmidrule(lr){6-8}
				\cmidrule(lr){9-11}
                \cmidrule(lr){12-14}
				\textbf{Metr.} & \textbf{$\sharp$Neg.}
				& $U_N$ & $U_N^{\mathrm{hl}}$ & SupCon 
				& $U_N$ & $U_N^{\mathrm{hl}}$ & SupCon
				& $U_N$ & $U_N^{\mathrm{hl}}$ & SupCon 
                & $U_N$ & $U_N^{\mathrm{hl}}$ & SupCon \\
				\midrule
				\multirow{3}{*}{Prec.} & $k=3$  &
				{\bf 0.9835} & 0.9238 & \multirow{3}{*}{0.9874} & {\bf 0.4146} & 0.3487 & \multirow{3}{*}{\underline{0.9382}} & {\bf 0.8128} & 0.7794 & \multirow{3}{*}{0.8057} & {\bf 0.3490} & 0.2836 & \multirow{3}{*}{0.2772} \\
				& $k=5$ &
				{\bf 0.9898} & 0.5031 & & {\bf 0.9212} & 0.3560 & & {\bf 0.8078} & 0.8045 & & {\bf 0.3239} & 0.3196 \\
				& $k=7$ &
				{\bf 0.9903} & 0.2845 & & {\bf 0.9253} & 0.3765 & & {0.7977} & {\bf 0.8033} & & {\bf 0.3173} & 0.2778 \\
				\midrule
				\multirow{3}{*}{Rec.} & $k=3$ &
				{\bf 0.9809} & 0.7599 & \multirow{3}{*}{0.9801} & {\bf 0.4998} & 0.4076 & \multirow{3}{*}{0.8435} & {\bf 0.5649} & 0.5404 & \multirow{3}{*}{\underline{0.6045}} & {\bf 0.3298} & 0.2675 & \multirow{3}{*}{0.3095} \\
				& $k=5$ &
				{\bf 0.9869} & 0.4315 & & {\bf 0.8629} & 0.4071 & & {\bf 0.5507} & {0.5324} & & {\bf 0.3422} & 0.2710 \\
				& $k=7$ &
				{\bf 0.9791} & 0.3711 & & {\bf 0.8451} & 0.3966 & & {\bf 0.5623} & 0.5515 & & {\bf 0.2895} & 0.2544 \\
				\midrule
				\multirow{3}{*}{F1} & $k=3$ &
				{\bf 0.9821} & 0.7771 & \multirow{3}{*}{0.9837} & {\bf 0.4195} & 0.3222 & \multirow{3}{*}{0.8745} & {\bf 0.6592} & 0.6328 & \multirow{3}{*}{\underline{0.6813}} & {\bf 0.3335} & 0.2723 & \multirow{3}{*}{0.2919} \\
				& $k=5$ &
				{\bf 0.9883} & 0.4209 & & {\bf 0.8862} & 0.3212 & & {\bf 0.6406} & {0.6366} & & {\bf 0.3306} & 0.2929 \\
				& $k=7$ &
				{\bf 0.9846} & 0.3141 & & {\bf 0.8744} & 0.3111 & & {0.6481} & {\bf 0.6505} & & {\bf 0.2997} & 0.2640 \\
				\bottomrule
			\end{tabular}}
		\end{small}
	\end{center}
    \vskip -0.1in
\end{table*}
\subsection{Experiment on Real Imbalanced Dataset}
\label{sec:real_imbl_dataset}
We follow the same representation learning procedure on synthetic data for real datasets. For the real-data experiments, we extract subsets from MNIST, FashionMNIST, and CIFAR-10 to simulate a long-tailed class distribution. Specifically, we choose samples from the full datasets such that one class is vastly dominant (occupying roughly half of the extracted subset) while the sizes of the remaining classes decay at an exponential rate. Using the representations learned by each model, one trained with sub-sampled $U_N$ and the other with sub-sampled  $U_N^\mathrm{hl}$, we then train a simple linear classifier on top. Finally, we report the average precisions, recalls, and F1-scores of both models evaluated on the five rarest classes in each test set of the corresponding dataset. These metrics are reported for different values of negative samples $k$ in  Table \ref{tab:experiment_results}. Overall, all performance metric deteriorates as the number of negatives increases. This observation is consistent with the assessment of several empirical works in contrastive learning for long-tailed classification problems \citep{article:khosla2020, article:wang2021, article:li2022, article:zhu2022, article:hou2023}. Intuitively, standard supervised contrastive learning typically performs poorly on long-tailed data because the overwhelming number of negative samples from major classes distorts the feature space and harms tail-class generalization. Interestingly, we can see that, despite the long-tailed condition, the model trained on the sub-sampled version of $U_N$ outperforms the one trained with sub-sampled $U_N^\mathrm{hl}$ on almost every metric across considered benchmarks. This empirical result further confirms our initial hypothesis about the tail-class generalization behavior of both U-Statistics.

\subsection{Typical Extreme Multi-class Scenarios}
\label{sec:typical_dataset}
In order to further verify that the proposed estimator $U_N$ outperforms the class-wise estimator $U_N^\mathrm{hl}$ in typical extreme multi-class scenarios, we conducted further experiments on MNIST, FashionMNIST, CIFAR10 and CIFAR100 with convolutional neural networks (CNNs). For MNIST, FashionMNIST and CIFAR10, we simulate the class-imbalanced distribution illustrated in Figure \ref{fig:class_distributions} as originally done for deep neural networks in Section \ref{sec:real_imbl_dataset}. For CIFAR100, we use the entire balanced dataset without modifying the class distribution. The CNN architecture used in this experiment comprises of three $(\text{Conv}\to\text{BN}\to\text{ReLU}\to\text{MaxPool})$ blocks with $32\to64\to128$ channels, respectively. Aside from the comparison to $U^\mathrm{hl}_N$, we also ran this experiment on SupCon \citep{article:khosla2020} to provide a more comprehensive comparative study. Overall, the debiased estimator $U_N$ proposed in this work outperforms the class-wise estimator proposed in \citet{article:hieu2025}, as expected, with only a few exceptions on CIFAR10. Notably, in the typical extreme multi-class scenario reflected in CIFAR100, the performance of $U_N$ is better than both $U_N^\mathrm{hl}$ and SupCon.

\subsection{Estimation of U-Statistics with Sub-sampling}
\label{sec:subsample_estimators}
Direct computation of either $U_N$ or $U_N^\mathrm{hl}$ requires contrastive loss evaluation over massive collections of valid tuples. Therefore, in the experiments described in Sections \ref{sec:synth_imbl_dataset}, \ref{sec:real_imbl_dataset} and \ref{sec:typical_dataset}, we used the \textbf{sub-sampled versions} of the U-Statistics estimators instead of the full averages. In this section, we provide a brief overview of sub-sampled estimators and a more detailed description along with experiment details are provided in Appendix \ref{sec:experiment_details}. Let $\{\z_j\}_{j=1}^K\subset\mathcal{Z}$ be a finite population of elements in some input space $\mathcal{Z}$ and $K$ is an extremely large number\footnote{In our case, $K$ refers to the size of valid tuples collection.}. Let $h:\mathcal{Z}\to\R$ and $A^w_K(h) := \sum_{j=1}^Kw(\z_j)h(\z_j)$ where $\forall j\in[K]$, $w(\z_j)\in\R$ and $\sum_{j=1}^Kw(\z_j)=1$. Let $q$ be a distribution over the full population $\set{\z_j}_{j=1}^K$ and let $\set{\z^*_\ell}_{\ell=1}^M\subseteq\set{\z_j}_{j=1}^K$ be drawn i.i.d. from $q$. We define the sub-sampled average:
\begin{align}
    \widehat A^q_M(h) = \frac{1}{M}\sum_{\ell=1}^M \frac{w(\z^*_\ell)}{q(\z^*_\ell)}h(\z^*_\ell).
\end{align}
\noindent Trivially, we can show that $\E_q\big[\widehat A^q_M(h)\big]=A_K^w(h)$. Hence, by the law of large numbers, we have the convergence $\widehat A_M^q(h)\to A_K^w(h)$ with probability one as $M\to\infty$. In fact, one can easily show $|A_K^w(h)-\widehat A_M^q(h)|\in\mathcal{O}(1/\sqrt{M})$ with high-probability (with respect to draws from $q$) using simple concentration inequalities (e.g., see \citet[Theorem 5.2]{article:hieu2025}). Therefore, if the weight function $w$ is given, we can always construct a sub-sampled estimator that is close to the weighted sum over the full population using a proposed sub-sampling distribution $q$. Indeed, both U-Statistics in Eqn.~\eqref{eq:hieu_ledent_estimator} and Eqn.~\eqref{eq:natural_estimator} can be written in the form of $A_K^w(h)$, which means that we can construct sub-sampled estimators corresponding to both formulations to conduct computationally feasible experiments.

\section{Conclusion}
In this work, we conduct a refined generalization analysis for the supervised contrastive representation learning framework. On the one hand, we provide a much tighter excess risk bound for the combined class-wise U-Statistic proposed by the previous work of \citet{article:hieu2025}. Specifically, we proved that the U-Statistic in Eqn.~\eqref{eq:hieu_ledent_estimator} yields a sample complexity that scales in the worst case with $R$, the number of classes, a rate only achievable by the previous work under the perfectly balanced classes assumption. On the other hand, we propose an entirely different estimator based on separately estimating two auxiliary contrastive risks, one with arbitrary number of class collisions and the other with at least one collision. Then, we combine both estimators via the overall class-collision probability. Under a mild assumption on the class distributions, we show that this estimator, which targets the surrogate risk of Eqn.~\eqref{eq:decomposed_population_contrastive_risk}, achieves a sample complexity that scales with $k$, the tuple size, in extreme multi-class scenarios and atypical conditions such as class imbalance. Using both theoretical examples and empirical experiments, we show that the U-Statistic estimator in Eqn.~\eqref{eq:natural_estimator} reflects better generalization performance in a typical extreme multi-class scenario (cf. Section \ref{sec:typical_dataset}, CIFAR100). Furthermore, in certain long-tailed class distributions, our experiments (Section \ref{sec:real_imbl_dataset}) also demonstrate that the proposed U-Statistic achieves better generalization performance compared to the class-wise formulation in prior work, especially for tail classes.

\section*{Impact Statement}

This paper is primarily theoretical in nature and we cannot foresee any negative societal impact.

\section*{Acknowledgement}

This research is supported by the National Research Foundation, Singapore under its AI Singapore Programme (AISG
Award No: AISG3-PhD-2025-08-066T).

\bibliographystyle{icml2026}
\bibliography{main}

\newpage\appendix\onecolumn

\newpage 
\section{Table of Notations}
\begin{longtable}{ll}
	\caption{Summary of key notations.} 
	\label{tab:notations}
	\\
	\toprule 
	\textbf{Ntn.} & \textbf{Description} \\
	
	\midrule
	\multicolumn{2}{c}{\textbf{Frequently-used Notations}} \\
	\midrule
	$\X$ & Input data space \\
	$R$ & The number of classes \\	
	$k$ & The number of negative samples \\	
	$\rho_r$ & The occurrence probability of class $r\in[R]$ \\
	$\mathcal{D}_r$ & The distribution of data points coming from class $r\in[R]$ \\
	$\mathcal{\bar D}$ & $\mathcal{\bar D}:=\sum_{r=1}^R\rho_r\mathcal{D}_r$, i.e., the full mixture of all class-conditional distributions \\
	$\mathcal{\bar D}_r$ & $\mathcal{\bar D}_r:=\frac{1}{1-\rho_r}\sum_{q\ne r}\rho_q\mathcal{D}_q$, i.e., re-normalized mixture that excludes class $r\in[R]$ \\
	$S$ & $S=\set{X_j}_{j=1}^N\sim\mathcal{\bar D}^{\otimes N}$, i.e., the given dataset of $N$ labeled data points \\
	$S_r$ & The set of data points in $S$ that belongs to class $r\in[R]$ \\
	$\bar S_r$ & The set of data points in $S$ that \textbf{do not} belong to class $r\in[R]$ \\
	$N_r$ & The cardinality of $S_r$ \\
	$\widehat\rho_r$ & $\widehat\rho_r:=\frac{N_r}{N}$, i.e., the empirical occurrence probability of class $r\in[R]$ \\
	$\rho_{\min}$ & $\rho_{\min} := \min_{r\in[R]}\rho_r$, i.e., the probability of the rarest class \\
	$\rho_{\max}$ & $\rho_{\max} := \max_{r\in[R]}\rho_r$, i.e., the probability of the most dominant class \\
	$\tau$ & $\tau:=1-\sum_{r=1}^R\rho_r(1-\rho_r)^k$, i.e., the population collision probability \\
	$\widehat\tau$ & $\widehat\tau:=1-\sum_{r=1}^R\widehat\rho_r(1-\widehat\rho_r)^k$, i.e., the plug-in estimator of $\tau$ \\
	$\phi$ & The $k$-dimensional vector-valued contrastive loss \hfill (e.g., see \citet{article:sohn2016}) \\
	$\ell_{\phi, f}$ & $\ell_{\phi, f}: \left(X, X^+, \set{X_i^-}_{i=1}^k\right)\mapsto \phi\left(\set{f(X)^\top\left[f(X^+)-f(X_i^-)\right]}_{i=1}^k\right)$ \\
    $\mathcal{B}$ & The upper-bound of $\ell_{\phi, f}$, i.e., $|\ell_{\phi, f}|\le \mathcal{B}$ for all $f\in\mathcal{F}$ \\
    $\Pi$ & The set of all bijections (i.e., permutations) $\pi:[N]\to[N]$ \\
	
	\midrule
	\multicolumn{2}{c}{\textbf{Operators}} \\
	\midrule
	${\E}[\cdot]$ & The population expectation over the distribution $\mu$ , written as $\E_\mu[\cdot]$ or $\underset{X\sim\mu}{\E}[\cdot]$ \\
	& to specify the underlying distribution ($\mu$) in given contexts \\
	$\mathbb{\widehat E}_V[\cdot]$ & $\mathbb{\widehat E}_V[g]:=\mathbb{\widehat E}_{v\sim V}[g(v)]=\frac{1}{|V|}\sum_{v\in V}g(v)$, i.e., empirical average over $V$ \\
	$\|\cdot\|_\mathcal{G}$ & Given a class of functions $\mathcal{G}$ and $F:\mathcal{G}\to\R$, $G:\mathcal{G}\to\R$ be two functionals, \\
	& we have $\|F-G\|_\mathcal{G} := \sup_{g\in\mathcal{G}}|F(g) - G(g)|$, i.e., the uniform difference over $\mathcal{G}$\\ 
	
	\midrule
	\multicolumn{2}{c}{\textbf{Key Population Risks}} \\
	\midrule
	$\Lun$ & Population contrastive risk with no class-collision \hfill (cf. Eqn.~\eqref{eq:population_contrastive_risk}) \\
    $\mathrm{\widetilde L}_\phi$ & Debiasing surrogate risk of $\Lun$ \hfill (cf. Eqn.~\eqref{eq:decomposed_population_contrastive_risk}) \\
	$\Lnorm_\Omega$ & Population contrastive risk with zero or more class-collisions \hfill (cf. Eqn.~\eqref{eq:lomega}) \\
	$\Lnorm_\Lambda$ & Population contrastive risk with one or more class-collisions \hfill (cf. Eqn.~\eqref{eq:llambda}) \\
	$\Lnorm_\phi^r$ & Population class-wise risk with no class-collision \hfill (cf. Eqn.~\eqref{eq:collision_free_classwise_risk}) \\
	$\Lnorm^r_\Omega$ & Population class-wise risk with zero or more class-collisions \hfill (cf. Eqn.~\eqref{eq:lomega_classwise}) \\
	$\Lnorm^r_\Lambda$ & Population class-wise risk with one or more class-collisions \hfill (cf. Eqn.~\eqref{eq:llambda_classwise}) \\
	
	\midrule
	\multicolumn{2}{c}{\textbf{Key Collections of Tuples}} \\
	\multicolumn{2}{c}{The following are collections of tuples constructed from $S$ whose} \\
	\multicolumn{2}{c}{anchor-positive pairs are of class $r\in[R]$} \\
	\midrule
	$\Theta_r$ & The negative samples in each tuple are not allowed to be of class $r$ \\
	$\Omega_r$ & The negative samples in each tuple include \textbf{zero or more} sample from class $r$ \\
	$\Lambda_r$ & The negative samples in each tuple include \textbf{one or more} sample from class $r$ \\
	
	\midrule
	\multicolumn{2}{c}{\textbf{Key Estimators used in Main Text}} \\
	\midrule
	$U_{\Theta_r}$ & $U_{\Theta_r}(f) := \widehat\E_{\Theta_r}[\ell_{\phi, f}]$, i.e., the (asymptotically) unbiased estimator for $\Lnorm_\phi^r(f)$ \\
	$U_{\Omega_r}$ & $U_{\Omega_r}(f):=\widehat\E_{\Omega_r}[\ell_{\phi, f}]$, i.e., the ``natural estimator" of $\Lnorm^r_\Omega(f)$ \\
	$U_{\Lambda_r}$ & $U_{\Lambda_r}(f):=\widehat\E_{\Lambda_r}[\ell_{\phi, f}]$, i.e., the ``natural estimator" of $\Lnorm^r_\Lambda(f)$ \\
	$U_\Omega$ & $U_\Omega(f):=\sum_{r=1}^R\widehat\rho_r U_{\Omega_r}(f)$, i.e., the ``natural estimator" of $\Lnorm_\Omega(f)$ (Zero or more collisions)\\
	$U_\Lambda$ & $U_\Lambda(f):=\sum_{r=1}^R\widehat\rho_r U_{\Lambda_r}(f)$, i.e., the ``natural estimator" of $\Lnorm_\Lambda(f)$ (One or more collisions)\\
	$U_N^\mathrm{hl}$ & $U_N^\mathrm{hl}(f):=\sum_{r=1}^R\widehat\rho_rU_{\Theta_r}(f)$, i.e., the U-Statistic in \citet{article:hieu2025} \\
	$U_N$ & $U_N(f):=\frac{1}{1-\widehat\tau}U_\Omega(f)-\frac{\widehat\tau}{1-\widehat\tau}U_\Lambda(f)$, i.e., the proposed U-Statistic estimator \\
	
	\midrule
	\multicolumn{2}{c}{\textbf{Construction of Auxiliary Estimators}} \\
	\multicolumn{2}{c}{For the following notations, we refer to $\Pi$ as the set of all bijections $\pi:[N]\to[N]$} \\
	\multicolumn{2}{c}{Furthermore, we let $n=2\lfloor N / (k+2)\rfloor$ and $m=3\lfloor N / (k+2)\rfloor$} \\
	\midrule
	$S_\pi$ & $S_\pi := \set{X_{\pi(j)}}_{j=1}^N$, i.e., the dataset $S$ permuted by $\pi$, meaning for each position \\
	& $j\in[N]$, the data point $X_j$ is swapped with $X_{\pi(j)}$\\
	$S^{(n)}$ & $S^{(n)} := \set{X_j}_{j=1}^n$, i.e., the set of the first $n$ data points in $S$ \\
	$\bar S^{(n)}$ & $\bar S^{(n)}:=\set{X_j}_{j=n+1}^{N}$, i.e., the set of the last $N-n$ data points in $S$ \\
	$S_\pi^{(n)}$ & $S_\pi^{(n)}:=\set{X_{\pi(j)}}_{j=1}^n$, i.e., the set of the first $n$ data points of $S_\pi$  \\
	$\bar S_\pi^{(n)}$ & $\bar S_\pi^{(n)}:=\set{X_{\pi(j)}}_{j=n+1}^{N}$, i.e., the set of the last $N-n$ data points of $S_\pi$ \\
	\midrule
	\multicolumn{2}{c}{The above notations extend to $S^{(m)}, \bar S^{(m)}, S^{(m)}_\pi$ and $\bar S^{(m)}_\pi$} \\
	\midrule
	$\mathrm{id}$ & $\mathrm{id}\in\Pi$ is the identity permutation, i.e., $\pi(j)=j,\forall j\in[N]$ \\
	$n_r^\pi$ & $n_r^\pi:=|S^{(n)}_\pi\cap S_r|$, i.e., number of data points in class $r$ from $S_\pi^{(n)}$ \\
	$m_r^\pi$ & $m_r^\pi:=|S^{(m)}_\pi\cap S_r|$, i.e., number of data points in class $r$ from $S_\pi^{(m)}$ \\
	$\omega_r^\pi$ & $\omega_r^\pi:=\lfloor n_r^\pi/2\rfloor\Big/\left[\sum_{q=1}^R\lfloor n_q^\pi/2\rfloor\right]$, i.e., the fraction of independent $2$-blocks \\
	& from class $r$ in $S_\pi^{(n)}$ \\
	$\lambda_r^\pi$ & $\lambda_r^\pi:=\lfloor m_r^\pi/3\rfloor\Big/\left[\sum_{q=1}^R\lfloor m_q^\pi/3\rfloor\right]$, i.e., the fraction of independent $3$-blocks \\
	& from class $r$ in $S_\pi^{(m)}$ \\
	\midrule
	\multicolumn{2}{c}{We also define $n_r=n_r^\mathrm{id}, m_r = m_r^\mathrm{id}, \omega_r = \omega_r^\mathrm{id}$ and $\lambda_r=\lambda_r^\mathrm{id}$,} \\
	\multicolumn{2}{c}{i.e., variants of $n_r^\pi, m_r^\pi, \omega_r^\pi,\lambda_r^\pi$ when no permutation takes place ($\pi=\mathrm{id}$)} \\
	\midrule
	$\Omega_r^\pi$ & The set of tuples with \textbf{zero or more} collisions whose anchor-positive pairs \\
	& come from $S^{(n)}_\pi$ and negative samples come from $\bar S^{(n)}_\pi$ \\
	$\Lambda_r^\pi$ & The set of tuples with \textbf{one or more} collisions whose anchor-positive-collided \\
	& triplets come from $S^{(m)}_\pi$ and negative samples come from $\bar S^{(m)}_\pi$ \\
	$U_{\Omega_r^\pi}$ & $U_{\Omega_r^\pi}(f):=\widehat\E_{\Omega_r^\pi}[\ell_{\phi, f}]$, i.e., an (asymptotically) unbiased estimator for $\Lnorm_\Omega^r(f)$ \\	
	$U_{\Lambda_r^\pi}$ & $U_{\Lambda_r^\pi}(f):=\widehat\E_{\Lambda_r^\pi}[\ell_{\phi, f}]$, i.e., an (asymptotically) unbiased estimator for $\Lnorm_\Lambda^r(f)$ \\
	$\bar U_\Omega$ & $\bar U_\Omega(f):=\widehat\E_\Pi\left[\sum_{r=1}^R\omega_r^\pi U_{\Omega_r^\pi}(f)\right]$, i.e., the auxiliary estimator for $U_\Omega(f)$ (Zero or more collisions) \\
	$\bar U_\Lambda$ & $\bar U_\Lambda(f):=\widehat{\E}_\Pi\left[\sum_{r=1}^R\lambda_r^\pi U_{\Lambda_r^\pi}(f)\right]$, i.e., the auxiliary estimator for $U_\Lambda(f)$ (One or more collisions) \\
	$\bar U_N$ & $\bar U_N(f):=\frac{1}{1-\widehat\tau}\bar U_\Omega(f)-\frac{\widehat\tau}{1-\widehat\tau}\bar U_\Lambda(f)$, i.e., the auxiliary estimator for $U_N(f)$ \\
	
	\midrule
	\multicolumn{2}{c}{\textbf{Auxiliary Risks}} \\
	\multicolumn{2}{c}{The following are auxiliary risks that we define for concentration analysis} \\
	\multicolumn{2}{c}{(Meaning they will be important for understanding the main proofs)} \\
	\midrule
	$\overline{\omega}_r$ & $\overline{\omega}_r:=\E[\omega_r]$, i.e., the \textbf{population expectation} of the fraction $\omega_r$ \\
	& Note that $\E[\omega_r]=\E[\omega_r^\pi]$ for any permutation $\pi\in\Pi$ \\
	$\omega_r^*$ & $\omega_r^* := \widehat\E_\Pi[\omega_r^\pi]$, i.e., the \textbf{empirical average} of fractions $\omega_r^\pi$ over all permutations \\	
	$\Lnorm_\Omega^*$ & $\Lnorm_\Omega^*(f):=\sum_{r=1}^R\omega_r^*\Lnorm_\Omega^r(f)$, i.e., the mixture of class-wise risks weighted by $\set{\omega_r^*}_{r=1}^R$ \\
	$\overline{\Lnorm}_\Omega$ & $\overline{\Lnorm}_\Omega(f):=\sum_{r=1}^R\overline{\omega}_r\Lnorm_\Omega^r(f)$, i.e., the mixture of class-wise risks weighted by $\set{\overline{\omega}_r}_{r=1}^R$ \\
	$\widehat\Lnorm_\Omega$ & $\widehat\Lnorm_\Omega(f|\pi):=\sum_{r=1}^R\omega_r^\pi\Lnorm_\Omega^r(f)$, i.e., the mixture of class-wise risks weighted by $\set{\omega_r^\pi}_{r=1}^R$ \\
	& Note that we also define $\widehat\Lnorm_\Omega(f):=\widehat\Lnorm_\Omega(f|\mathrm{id})=\sum_{r=1}^R\omega_r\Lnorm_\Omega^r(f)$ \\
	\midrule
	$\overline{\lambda}_r$ & $\overline{\lambda}_r:=\E[\lambda_r]$, i.e., the \textbf{population expectation} of the fraction $\lambda_r$ \\
	& Note that $\E[\lambda_r]=\E[\lambda_r^\pi]$ for any permutation $\pi\in\Pi$ \\
	$\lambda_r^*$ & $\lambda_r^* := \widehat\E_\Pi[\lambda_r^\pi]$, i.e., the \textbf{empirical average} of fractions $\lambda_r^\pi$ over all permutations \\
	$\Lnorm_\Lambda^*$ & $\Lnorm_\Lambda^*(f):=\sum_{r=1}^R\lambda_r^*\Lnorm_\Lambda^r(f)$, i.e., the mixture of class-wise risks weighted by $\set{\lambda_r^*}_{r=1}^R$ \\
	$\overline{\Lnorm}_\Lambda$ & $\overline{\Lnorm}_\Lambda(f):=\sum_{r=1}^R\overline{\lambda}_r\Lnorm_\Lambda^r(f)$, i.e., the mixture of class-wise risks weighted by $\set{\overline{\lambda}_r}_{r=1}^R$ \\
	$\widehat\Lnorm_\Lambda$ & $\widehat\Lnorm_\Lambda(f|\pi):=\sum_{r=1}^R\lambda_r^\pi\Lnorm_\Lambda^r(f)$, i.e., the mixture of class-wise risks weighted by $\set{\lambda_r^\pi}_{r=1}^R$ \\
	& Note that we also define $\widehat\Lnorm_\Lambda(f):=\widehat\Lnorm_\Lambda(f|\mathrm{id})=\sum_{r=1}^R\lambda_r\Lnorm_\Lambda^r(f)$ \\
    \midrule
	\multicolumn{2}{c}{\textbf{Other Special Quantities}} \\
	\midrule
    $\gamma_\alpha$ & $\gamma_\alpha:=\P_\rho(\rho_r\le\alpha^{-1})=\sum_{r:\rho_r\le\alpha^{-1}}\rho_r$, i.e., total mass of small class probabilities \\
    $\widehat\gamma_\alpha$ & $\widehat\gamma_\alpha:=\P_{\widehat\rho}(\widehat\rho_r\le \alpha^{-1})=\sum_{r:\widehat\rho_r\le\alpha^{-1}}\widehat\rho_r$, i.e., total mass of small empirical probabilities \\
    $\widehat\theta_{k+2}$ & $\widehat\theta_{k+2}:=\P_{\widehat\rho}(\widehat\rho_r\le 2/(k+2))$, used in the refined concentration analysis of $U_N^\mathrm{hl}$ \\
    $W_k$ & $(1-\rho_{\max})^{-k}$, used in the exact debiasing of $\Lun$ \\
  	\bottomrule
\end{longtable}

\begin{longtable}{ll}
	\caption{Function classes and complexity measures.} 
	\label{tab:complexity_measures}
	\\
	\toprule 
	\textbf{Ntn.} & \textbf{Description} \\
	
	\midrule
	\multicolumn{2}{c}{\textbf{Function Classes}} \\
	\midrule
	$\mathcal{F}$ & The class of representation functions \\
	$\mathcal{H}$ & $\mathcal{H}:=\Big\{\left(X, X^+, \set{X_i^-}_{i=1}^k\right)\mapsto\ell_{\phi, f}\left(X, X^+, \set{X_i^-}_{i=1}^k\right):f\in\mathcal{F}\Big\}$ \\
	
	\midrule
	\multicolumn{2}{c}{\textbf{Complexity Measures}} \\
	\multicolumn{2}{c}{Given an input space $\mathcal{Z}$, a class $\mathcal{G}$ of real-valued functions $g:\mathcal{Z}\to\R$, a distribution $\mathcal{\mu}$} \\
	\multicolumn{2}{c}{over $\mathcal{Z}^m$ and let ${\bf\Sigma}_m=\set{\sigma_j}_{j=1}^m$ be a sequence of independent Rademacher variables} \\
	\multicolumn{2}{c}{Finally, let $S=\set{\z_j}_{j=1}^m\sim\mu$ be a dataset drawn from $\mu$ and suppose $|g|\le \mathcal{B}$ for all $g\in\mathcal{G}$} \\
	\midrule
	$\mathcal{N}(\mathcal{G}, \epsilon, \Lnorm_p(S))$ & The size of the smallest cover $\mathcal{C}\subseteq\mathcal{G}$ such that for all $g\in\mathcal{G}$, \\
	& there exists $\bar g\in\mathcal{C}$ such that $\frac{1}{m}\sum_{j=1}^m|g(\z_j) - \bar g(\z_j)|^p\le \epsilon^p$ \\
	$\mathfrak{C}_m(\mathcal{G})$ & $\mathfrak{C}_m(\mathcal{G}) := \sup_{S^*\in\mathcal{Z}^m}\int_{1/m}^\mathcal{B}\sqrt{\ln 2\mathcal{N}(\mathcal{G},\epsilon, \Lnorm_2(S^*))}d\epsilon$, i.e., the \\
	& worst-case Dudley integral over datasets of size $m$\\
	$\mathfrak{\widehat R}_S(\mathcal{G})$ & $\mathfrak{\widehat R}_S(\mathcal{G}) : = \E_{\Rad_m}\left[\sup_{g\in\mathcal{G}}\left|\frac{1}{m}\sum_{j=1}^m\sigma_jg(\z_j)\right|\right]$, i.e., the empirical Rademacher \\
	& complexity of $\mathcal{G}$ given the dataset $S$ \\
	$\mathfrak{R}_\mu(\mathcal{G})$ & $\mathfrak{R}_\mu(\mathcal{G})=\E_{S\sim\mu}\left[\mathfrak{R}_S(\mathcal{G})\right]$, i.e., the expected Rademacher complexity where \\
    & expectation is taken over the distribution $\mu$ over $\mathcal{Z}^m$ \\
    $\mathfrak{R}_m^\mathrm{wc}(\mathcal{G})$ & $\mathfrak{R}_m^\mathrm{wc}(\mathcal{G}):=\sup_{S^*\in\mathcal{Z}^m}\mathfrak{\widehat R}_{S^*}(\mathcal{G})$, i.e., the worst-case Rademacher complexity, \\
    & over datasets of size $m$ \\
	\bottomrule
\end{longtable}

\newpage
\newpage 
\section{Summary of the Literature on the Concentration of the Unsupervised Risk}
\label{sec:summary_of_literature}
\begin{table*}[ht]
	\begin{center}
		\begin{small}
			\caption{Review of current works on \textit{generalization bounds} (i.e., finite-sample concentration of the contrastive risk) for various regimes of CRL.  Note that all results are expressed as sample complexity in terms of total number of labeled data points $N$: for the references studying the i.i.d. tuples regime~\citep{article:arora19,article:lei2023}, $N$ scales as $\mathcal{{O}}(N_{\mathrm{tup}}k)$, which explains the rate of $k^2$ in~\citet{article:arora19}. Indeed~\citet{article:arora19} requires $N_{\mathrm{tup}}\in \mathcal{\widetilde{O}}(k)$ tuples, which corresponds to $N\in \mathcal{\widetilde{O}}(k^2)$ individual samples. Similarly, ~\citet{article:lei2023} requires $N_{\mathrm{tup}}\in\mathcal{\widetilde{O}}(1)$ tuples, which corresponds to $\mathcal{\widetilde{O}}(k)$ individual samples.  }
            \label{tab:literature review}
			\centering
            \begin{tabular}{c|c|c|c|c}
            \toprule
            \textbf{Reference} & \textbf{Estimator} & \textbf{Default} & $\gamma_k \in \mathcal{O}(1)$ & $\rho_{\max}\in\mathcal{O}(1/k)$ \\
            \midrule
            \multicolumn{5}{c}{\makecell{\textbf{Unsupervised CRL} (i.i.d. tuples)}} \\
            \midrule
            \citet{article:arora19} & I.I.D-ERM & \multicolumn{3}{c}{$\mathfrak{C}_N^2(\mathcal{H})k^2$} \\
            \midrule
            \citet{article:lei2023} & I.I.D-ERM & \multicolumn{3}{c}{$\mathfrak{C}_N^2(\mathcal{H})k$} \\
            \midrule
            \multicolumn{5}{c}{\makecell{\textbf{Supervised CRL} (tuples constructed from pool of i.i.d. samples)}} \\
            \midrule
            \citet{article:hieu2025} & $U_N^\mathrm{hl}$ (cf. Eqn.~\eqref{eq:hieu_ledent_estimator}) & \multicolumn{3}{c}{$[\mathfrak{C}^2_N(\mathcal{H})+1]\max\left[\rho_{\min}^{-1}, (1-\rho_{\max})^{-1}\right]$} \\
            \midrule
            This Work & \makecell{$U_N^\mathrm{hl}$ (cf. Eqn.~\eqref{eq:hieu_ledent_estimator}) \\(For $\mathrm{L}_\phi$)} & \makecell{$\mathfrak{C}^2_N(\mathcal{H})[\widehat\theta_{k+2}R+(1-\widehat\theta_{k+2})^2k] \ +$\\$[R+k^2]$\\(Theorem \ref{thm:refined_hieu_ledent_general}/\ref{thm:refined_hieu_ledent_general_appendix})}  & \multicolumn{2}{c}{\makecell{$\mathfrak{C}^2_N(\mathcal{H})R +[R+k^2]$\\(Theorem \ref{thm:refined_hieu_ledent_small_probabilities})}} \\
            \midrule
            This Work & \makecell{$\bar U_N/U_N$ (cf. Eqn.~\eqref{eq:combined_aux_estimators})\\(For $\mathrm{\widetilde L}_\phi$)} & \makecell{$\mathfrak{C}^2_N(\mathcal{H})(1-\tau)^{-2}k \ +$\\$[R(1-\tau)^{-4} + Rk(1-\tau)^{-1}+k^2]$\\(Theorem \ref{thm:main_result_more_general}/\ref{thm:main_result_more_general_appendix})} & \multicolumn{2}{c}{\makecell{$\mathfrak{C}^2_N(\mathcal{H})k + [R + Rk]$\\(Theorem \ref{thm:second_main_result_appendix})}} \\
            \midrule
            \textcolor{blue}{This Work} & \textcolor{blue}{ \makecell{$\bar U_N^\star/U_N^\star$ (cf. Eqn.~\eqref{eq:exact_estimators_class_weighted})\\(For $\mathrm{L}_\phi$, cf. Erratum~\ref{err:erratum-for-estimator})} } & \multicolumn{2}{c}{\makecell{$\mathfrak{C}^2_N(\mathcal{H})(1-\rho_{\max})^{-2k}k \ +$\\$[R(1-\rho_{\max})^{-2k} + Rk(1-\rho_{\max})^{-k}+k^2]$\\(\textcolor{blue}{Theorem} \ref{thm:exact_class_weighted}, \textcolor{blue}{Remark} \ref{rem:natural_exact_estimator})}} & \multicolumn{1}{|c}{\makecell{$\mathfrak{C}_N^2(\mathcal{H})k+[R+Rk]$\\(\textcolor{blue}{Theorem} \ref{thm:exact_class_weighted})}}  \\
            \bottomrule
            \end{tabular}
        \end{small}
    \end{center}
    \vskip -0.1in
\end{table*}

\begin{remark}[\textbf{on Theorems \ref{thm:main_result_more_general}/\ref{thm:main_result_more_general_appendix}}]
\label{rem:remark_on_natural_estimator} 
We note that the sample complexity of $\bar U_N$ in Theorem \ref{thm:main_result_more_general} and Theorem \ref{thm:main_result_more_general_appendix} also applies for the natural proposed estimator $U_N$. Specifically, the bias of $U_N$ incurred a cost of at most $\mathcal{O}(Rk)$ in sample complexity (cf. Proposition \ref{prop:diff_between_natural_and_nonatural}), which is already present in the result for $\bar U_N$.
\end{remark}

\begin{remark}[\textbf{on Difference between Theorems \ref{thm:main_result_more_general}/\ref{thm:main_result_more_general_appendix} and \citet{article:lei2023}}]
At first glance, \citet{article:lei2023} seems to yield a better sample complexity (independent of $\tau$). However, we note that the sample complexity of their work is for estimating the much easier collision-allowed risk, which we denote as $\Lnorm_\Omega$. In the i.i.d.-sample regime investigated by this work, the result for $\Lnorm_\Omega$ is also provided in Theorem \ref{thm:concentration_of_bar_Uomega} (which yields $\mathcal{\widetilde O}(\mathfrak{C}_N^2(\mathcal{H})k + Rk)$ in sample complexity). However, this only serves as an intermediary step towards the result for the more interesting collision-free contrastive risk.
\end{remark}

\begin{remark}[\textbf{on Difference between i.i.d.-\emph{sample} and i.i.d.-\emph{tuple} Regimes}]
As explained in the main text, the dominant component in the above results is the one which includes a factor of the complexity term $\mathfrak{C}_N^2(\mathcal{H})$: the other additive terms (such as $R+k^2$ in Theorem~\ref{thm:refined_hieu_ledent_small_probabilities}) arise from the need for the empirical pool of samples to be non-degenerate w.r.t. the construction of valid tuples, and are \textit{independent of the hypothesis class}. 
 
These results only include contributions towards the speed of \textit{concentration of the contrastive risk}: as explained in the main text, there are many works~\cite{article:arora19, article:chuang2020, article:liliu2021, bao2022surrogate, article:huang2023} which analyze the relationship between the minimization of the population-level (collided) unsupervised risk and the performance at the downstream classification task under various augmentation strategies. Whilst extremely valuable, this direction is an orthogonal consideration to our work, which studies the concentration of the collision-free contrastive risk in the supervised case. These works often consider different generation and augmentation regimes, and the concentration of the unsupervised risk component of the error usually \textit{follows variants of the i.i.d. tuples regime}. For instance, \citet{article:cui2025} assumes that the positive and anchor positive samples are both generated with an augmentation strategy from the same sample (which elegantly matches practical applications). The negative samples are independently generated, similarly to the i.i.d. tuple regime. Similarly, in~\citet{cui2025inclusive}, study a variant of supervised CRL with non-collided i.i.d. tuples (as in~\citet{article:arora19} and \citet{article:lei2023}) under both augmentation and label noise.  

We also note that in the i.i.d. tuples regime, the \textit{collision-free} contrastive risk can only be estimated if the learner can generate i.i.d. tuples with positive and negative samples following the correct class distribution, which essentially requires access to supervised label information. On the other hand, the collided contrastive risk can often be estimated under credible (class-agnostic) i.i.d. tuples generation assumptions. The finite-sample concentration of both collision-free and collided contrastive risks is also much easier to prove in the i.i.d. tuples regime as there is no need for Hoeffding-type arguments. Accordingly, most works in this branch of the literature focus on other considerations such as downstream classification performance and augmentation strategy.
\end{remark}

\begin{erratum}
    \label{err:erratum-for-estimator}
    \textcolor{blue}{
    The original text proves uniform deviation bound for the surrogate risk $\mathrm{\widetilde L}_\phi$ rather than the desired $\mathrm{L}_\phi$. \textbf{However, the tool developed in Theorem \ref{thm:concentration_of_bar_Uomega} (with Generalization in Theorem \ref{lem:general_tuple_functions}) readily extends to prove uniform deviation bound for the natural estimator (denoted $U_N^\star$, cf. Eqn.~\eqref{eq:exact_estimators_class_weighted}) of $\mathrm{L}_\phi$ (cf. Eqn.~\eqref{eq:population_contrastive_risk}) with minimal additional arguments and no new results are needed (cf. Appendix \ref{sec:exact-debias}, Theorem \ref{thm:exact_class_weighted})}. The claim about improved generalization behavior in extreme multi-class scenarios still holds. Specifically, when $\rho_r\le\frac{1}{k+1}$ for all $r\in[R]$ and $N\ge k\max\left\{64\ln\left(\frac{6R}{\Delta}\right), 9R\right\}$, for all $\Delta\in(0,1)$, we have, in $\mathcal{O}$-notation, the following uniform deviation bound:
    \begin{align}
        \sup_{f\in\mathcal{F}}\left|\bar U_N^\star(f)-\Lun(f)\right|\le\mathcal{O}\Bigg[\mathfrak{C}_N(\mathcal{H})\sqrt{\frac{k}{N}}+\mathcal{B}\sqrt{\frac{(R+k)\ln R/\Delta}{N}}+\frac{\mathcal{B}R(k+\ln R/\Delta)}{N}\Bigg],
    \end{align}
    \noindent with probability of at least $1-\Delta$. The proof for estimating the exact debiasing risk of $\mathrm{L}_\phi$ (cf. Eqn.~\eqref{eq:population_contrastive_risk}) requires only two invocations of the pre-built machinery in Theorem \ref{lem:general_tuple_functions} (cf. Theorem \ref{thm:exact_class_weighted}) and Proposition \ref{prop:natural_exact_estimator_direct} to pass from $\bar U_N^\star$ to $U_N^\star$. The proofs for previous results on estimating the surrogate risk $\mathrm{\widetilde L}_\phi$ (Theorems \ref{thm:main_result_appendix}, \ref{thm:second_main_result_appendix}, \ref{thm:main_result_more_general_appendix}, \ref{thm:main_result_more_general}) remain unaffected.
    }
\end{erratum}


\section{Auxiliary Estimators}
\label{sec:aux_estimators_construction}
\subsection{Construction}
\begin{figure}
    \centering
    \includegraphics[width=\linewidth]{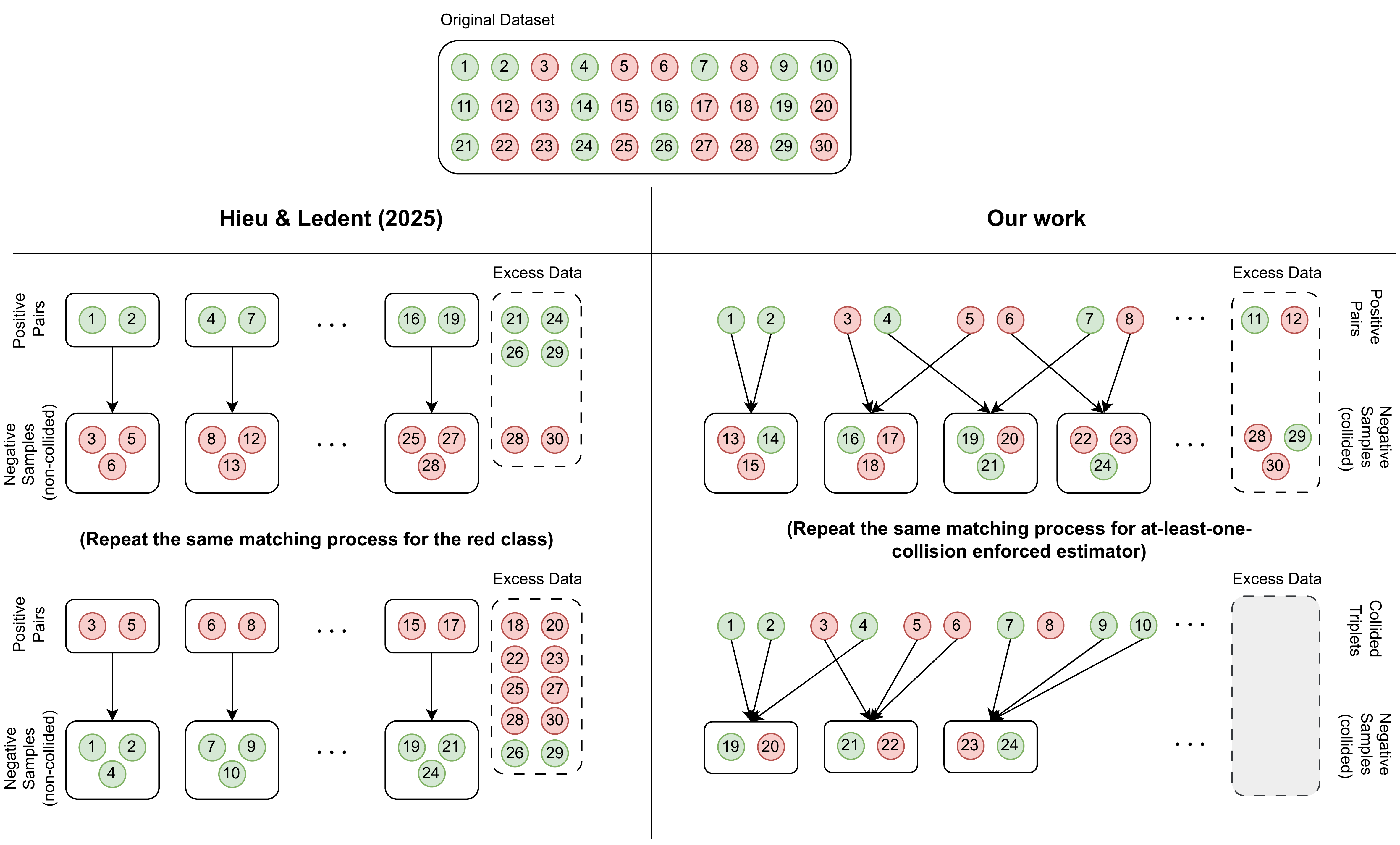}
    \caption{Comparison of the tuple selection process for each permutation of the labeled dataset $S$. For simplicity, only two classes are considered in this illustration (green and red). \textbf{Left}: Construction of the estimator $U_N^\mathrm{hl}$ from~\citet{article:hieu2025}: Tuples are selected class-wise for each component estimator defined in \citet{article:hieu2025}. $U_N^\mathrm{hl}$ is expressed as an average of this construction over all permutations of the dataset, which is equivalent to a naive class-wise tuple selection from the empirical distribution (averaging all valid tuples for each class). \textbf{Right}: Construction of the auxiliary estimator $\bar{U}_N$  in our work. The tuples are selected to estimate both arbitrary-collision risk $\bar{U}_{\Omega}$ (top-right) and at-least-one-collision risk $\bar{U}_{\Lambda}$ (bottom-right). $\bar{U}_N$ is defined as a weighted combination of the averages of $\bar{U}_{\Omega}$ and  $\bar{U}_{\Lambda}$ over all permutations of the dataset (cf. Eqns~\eqref{eq:bar_Uomega_Ulambda}, \eqref{eq:combined_aux_estimators_appendix}). In the final (natural) estimator $U_N$, the arbitrary and at-least-one-collision components are both constructed naturally by averaging all valid tuples. Unlike in the class-wise construction from~\citet{article:hieu2025}, the natural and auxiliary estimators $\bar{U}_N$ and $U_N$ are not equal to each other (due to a mismatch between the empirical and population collision probabilities), but they can be shown to be asymptotically close (cf. Proposition~\ref{prop:diff_between_natural_and_nonatural}). }
    \label{fig:tuples_selection}
\end{figure}
In this section, we construct (asymptotically) unbiased estimators of $\Lnorm_\Omega(f)$ and $\Lnorm_\Lambda(f)$ for $f\in\mathcal{F}$ that are close to our proposed natural estimators $U_\Omega(f), U_\Lambda(f)$. For a bijective map $\pi:[N]\to[N]$, we define $S_\pi=\{X_{\pi(j)}\}_{j=1}^N$, i.e., the original dataset $S$ shuffled by permuting positions of data-points according to $\pi$. Let $n=2\Big\lfloor \frac{N}{k+2}\Big\rfloor$ and $m=3\Big\lfloor\frac{N}{k+2}\Big\rfloor$, we define $S_\pi^{(n)}=\set{X_{\pi(j)}}_{j=1}^n$ and $S_\pi^{(m)}=\set{X_{\pi(j)}}_{j=1}^m$, i.e., the first $n$ and $m$ data-points of the shuffled dataset $S_\pi$. Then, we define the following collections of tuples:
\begin{align}
    \Omega_{r}^\pi :=\Big\{\Big(X, X^+, \set{X_i^-}_{i=1}^k\Big)&:X, X^+\in S_r \cap S_\pi^{(n)}, 
    	\set{X_i^-}_{i=1}^k\subseteq S\setminus S_\pi^{(n)}\Big\}.  \label{eq:omega_r_pi} \\
    \Lambda_r^\pi := \Big\{\Big(X, X^+, \bar X, \set{X_i^-}_{i=1}^{k-1}\Big)&: X, X^+,\bar X\in S_r\cap S_\pi^{(m)}, 
    	\set{X_i^-}_{i=1}^{k-1}\subseteq S\setminus S_\pi^{(m)}\Big\}. \nonumber
\end{align}
Essentially, Eqn.~\eqref{eq:omega_r_pi} denotes the collection of tuples whose anchor-positive pairs are selected from class-$r$ data-points from the first partition of $S_\pi$ while negative samples are selected from the other partition. Similarly, $\Lambda_r^\pi$ denotes the collection of tuples where the anchor-positive-collided triplets are drawn from the first partition. Let $\Pi$ denote the set of all bijective maps $\pi:[N]\to[N]$, we define the auxiliary estimators $\bar U_\Omega(f)$ and $\bar U_\Lambda(f)$ as follows:
\begin{align}
    \bar U_{\Omega}(f) &:= \widehat\E_\Pi \left[\sum_{r=1}^R \omega_r^\pi U_{\Omega_r^\pi}(f)\right], \qquad 
    \bar U_{\Lambda}(f) := \widehat\E_\Pi \left[\sum_{r=1}^R \lambda_r^\pi U_{\Lambda_r^\pi}(f)\right].\label{eq:bar_Uomega_Ulambda}
\end{align}
Where for each $\pi\in\Pi$ and $r\in[R]$, we define:
\begin{alignat}{2}
    U_{\Omega_r^\pi}(f) &:=\widehat\E_{\Omega_r^\pi}\left[\ell_{\phi,f}\right], \quad \omega_r^\pi &&:= \frac{\lfloor n_r^\pi/2\rfloor}{\sum_{q=1}^R\lfloor n_q^\pi/2 \rfloor}. \\
    U_{\Lambda_r^\pi}(f) &:=\widehat\E_{\Lambda_r^\pi}\left[\ell_{\phi,f}\right],
     \quad \lambda_r^\pi &&:= \frac{\lfloor m_r^\pi/3\rfloor}{\sum_{q=1}^R\lfloor m_q^\pi/3 \rfloor}.
\end{alignat}
\noindent And we define $n_r^\pi:=|S_r\cap S_\pi^{(n)}|$ and $m_r^\pi:=|S_r\cap S_\pi^{(m)}|$, i.e., the number of class-$r$ data points in $S_\pi^{(n)}$ and $S_\pi^{(m)}$, respectively. Then, we can combine the auxiliary estimators as follows:
\begin{align}
    \label{eq:combined_aux_estimators_appendix}
    \bar U_N(f) := \frac{1}{1-\widehat\tau}\bar U_\Omega(f) - \frac{\widehat \tau}{1-\widehat\tau}\bar U_\Lambda(f).
\end{align}
\noindent In the following section, we will prove that the auxiliary estimators $\bar U_\Omega(f)$ and $\bar U_\Lambda(f)$ are indeed (asymptotically) unbiased. Then, we will provide formal proof for our claim that for large $N$, the difference between $\bar U_\Omega(f)$ and $U_\Omega(f)$ (as well as the difference between $\bar U_\Lambda(f)$ and $U_\Lambda(f)$) is negligible (i.e., Proposition \ref{prop:diff_between_natural_and_nonatural}).

\noindent\textbf{Standing Assumptions}: Throughout the Appendix, we assume that $n>R$ and $m>2R$ (which holds as soon as $N\ge9Rk$). Then, $\sum_{q=1}^R\lfloor n_q^\pi/2\rfloor\ge\frac{n-R}{2}>0$ and $\sum_{q=1}^R\lfloor m_r^\pi/3\rfloor\ge\frac{m-2R}{3}>0$ for every $\pi\in\Pi$, so that the weights $\omega_r^\pi, \lambda_r^\pi$ are well-defined and the quantities $\frac{1}{n-R}$ and $\frac{1}{m-2R}$ in the results below are positive. Moreover, we use the convention that an empirical average over an empty collection of tuples is zero.

\subsection{Discussion of the Proofs}
\label{sec:proof_discussion}
As illustrated in Figure \ref{fig:tuples_selection}, the construction of the new estimators $U_N$ and $\bar{U}_N$ differs enormously from that of the previous work. To the best of our knowledge, this is necessary to achieve our improved sample complexity. This radical change of direction inevitably leads to a variety of difficulties in the analysis:
\begin{enumerate}[label=(\roman*)]
    \item The first difficulty has already been mentioned in the main text: the selection of negative samples to compute the arbitrary-collision estimator $U_\Omega$ (as well as the at-least-one-collision estimator $U_\Lambda$) follows the \textbf{wrong distribution}. Specifically, if we are given a dataset $S$ drawn i.i.d. from $\mathcal{\bar D}^{\otimes N}$ (where $\mathcal{\bar D}$ denotes the full mixture of class-conditional distributions), then randomly hold out two data points from class $r\in[R]$ (to be used as positive and anchor positive), the remaining $N-2$ data points \textbf{no longer follow $\mathcal{\bar D}$}. Intuitively, this is because every time we hold out data points in class $r$, the probability of re-selecting another data point from class $r$ slightly decreases while probabilities of other classes slightly increase. In other words, the weights of all member distributions in the full mixture $\mathcal{\bar D}$ have been distorted. The more serious problem is that the subsequently selected negatives now depend on the pre-selected anchor-positive pair, leading to \textbf{in-tuple dependence} (between samples). Note that this is a different and more challenging dependence type from what we and the work before us are dealing with. Simpler Hoeffding-type arguments are aimed at \textbf{between-tuple} dependence (present in all estimators including $U_N$, $\bar{U}_N$ and $U_N^{\mathrm{hl}}$), which usually means that while the tuples are not independent of each other due to \textit{repeating samples}, all samples of every tuple are independent. In contrast, only $U_N$ exhibits \textit{in-tuple dependence}, which is why we must construct the auxiliary estimator $\bar{U}_N$.

    \item This issue directly motivates our decision to partition the dataset as illustrated in Figure~\ref{fig:tuples_selection} when constructing the auxiliary estimator $\bar{U}_N$. When we partition the dataset $S$ deterministically, i.e., shuffle $S$ by a permutation $\pi\in\Pi$ then separate $S$ into two partitions at a fixed index, we can treat each partition as a dataset drawn i.i.d. from $\mathcal{\bar D}$. However, one question might arise for many readers: \textbf{why do we have to construct $\bar U_\Omega$ (and $\bar U_\Lambda$) by averaging over all possible splits}? It is true that we could attempt to build $\bar U_\Omega$ as simply ${\widetilde{U}}_\Omega:=\sum_{r=1}^R\omega_r^\pi U_{\Omega_r^\pi}(f)$ based on a particular split $\pi\in\Pi$. However, let us investigate how this would affect bounding the uniform bias $\|{\widetilde{U}}_\Omega- U_\Omega\|_\mathcal{F}$. Applying a similar argument as Lemma \ref{lem:supporting_lemma}, we have:
    \begin{align*}
        \left\|U_\Omega-{\widetilde U}_\Omega\right\|_\mathcal{F} &\le \sum_{r=1}^R\widehat\rho_r\sup_{f\in\mathcal{F}}\Big|U_{\Omega_r}(f)-U_{\Omega_r^\pi}(f)\Big| + \mathcal{B}\sum_{r=1}^R|\omega_r^\pi-\widehat\rho_r|.
    \end{align*}
    \noindent This already reveals a significant problem: for each $r\in[R]$, the averages $U_{\Omega_r}$ and $U_{\Omega_r^\pi}$ are not supported on the same set of tuples. Furthermore, conditionally given the sample sizes $\set{N_r}_{r\in[R]}$, both $\Omega_r^\pi$ and $\omega_r^\pi$ are random due to the randomness of $\pi$, making it hard to bound either $|U_{\Omega_r}(f)-U_{\Omega_r^\pi}(f)|$ or $|\omega_r^\pi-\widehat\rho_r|$. On the other hand, if we average $U_{\Omega_r^\pi}$ over all permutations in $\Pi$, the task of controlling the uniform bias reduces to bounding the difference between two probability mass functions corresponding to the $\mathrm{Hypergeometric}(N-1,N_r-1,k)$ and $\mathrm{Hypergeometric}(N-2,N_r-2,k)$ distributions (see Proposition \ref{prop:diff_between_natural_and_nonatural_extra_assumption}).

    \item The most technically cumbersome challenge arises in the concentration analysis of the auxiliary arbitrary-collision estimator $\bar U_\Omega$ (and the auxiliary at-least-one-collision estimator $\bar U_\Lambda$). Because we constructed $\bar U_\Omega$ as an average over all $\pi\in\Pi$, we cannot directly apply the classic decoupling techniques in \citet{article:ginezinn1984,article:arconesgine1993, article:clemencon2008} as done in the previous work. In \citet{article:hieu2025}, each class-wise estimator is treated as a second-order U-statistic with a random sample size. While this randomness introduces some technical overhead, conditioning on the class counts immediately permits the use of standard decoupling arguments. After conditioning, the concentration of the empirical class probabilities can be handled separately and does not interact with the complexity of the function class. Concretely, upon conditioning on $\set{N_r}_{r\in[R]}$, each class-wise U-Statistic $U_{\Theta_r}$ concentrates to the correct class-wise risk $\mathrm{L}^{r}_\phi(f)$, thereby enabling symmetrization and subsequent U-statistics decoupling. In contrast, in our setting, conditioning on the sample sizes alone does not guarantee that $\bar U_\Omega$ concentrates around its target risk $\Lnorm_\Omega$. As a result, our analysis requires a more delicate argument based on a sequence of auxiliary risks to bridge this gap.
    \end{enumerate}
\section{Bridging Natural Estimators and Auxiliary Estimators}
\label{appendix:proof_of_prop1}
Recall from the main text that for a representation function $f\in\mathcal{F}$ we proposed the U-Statistic formulation $U_N(f)$ (Eqn.~\eqref{eq:natural_estimator}), which is a combination of two estimators $U_\Omega(f)$ and $U_\Lambda(f)$ that estimate $\Lnorm_\Omega(f)$ and $\Lnorm_\Lambda(f)$, respectively.

Unfortunately, both $U_\Omega(f)$ and $U_\Lambda(f)$ are biased estimators. Therefore, we proposed $\bar U_\Omega(f)$ and $\bar U_\Lambda(f)$ as two auxiliary estimators and we claim that:
\begin{enumerate}
    \item Both $\bar U_\Omega(f)$ and $\bar U_\Lambda(f)$ are asymptotically unbiased estimators of $\Lnorm_\Omega(f)$ and $\Lnorm_\Lambda(f)$.

    \item They are different from $U_\Omega(f)$ and $U_\Lambda(f)$ by negligible factors.
\end{enumerate}

Prior to presenting the proof for Proposition \ref{prop:diff_between_natural_and_nonatural}, we first start by proving the asymptotic unbiasedness of $\bar U_\Omega(f)$ and $\bar U_\Lambda(f)$. Before that, for completeness of definitions, we define the class-wise risks $\Lnorm^r_\Omega$ and $\Lnorm^r_\Lambda$ as follows:
\begin{align}
	\Lnorm_\Omega^r(f) &= \underset{\substack{X,X^+\sim\mathcal{D}_r^{\otimes 2}\\\set{X_i^-}_{i=1}^k\sim\mathcal{\bar D}^{\otimes k}}}{\E}\left[\ell_{\phi, f}\left(X, X^+, \set{X_i^-}_{i=1}^k\right)\right], \label{eq:lomega_classwise} \\
	\Lnorm_\Lambda^r(f) &= \underset{\substack{X,X^+,\bar X\sim\mathcal{D}_r^{\otimes 3}\\\set{X_i^-}_{i=1}^{k-1}\sim\mathcal{\bar D}^{\otimes k-1}}}{\E}\left[\ell_{\phi, f}\left(X, X^+, \bar X,  \set{X_i^-}_{i=1}^{k-1}\right)\right] \label{eq:llambda_classwise}.
\end{align}

\begin{proposition}
    \label{prop:unbiasedness_aux_estimators}
    For $f\in\mathcal{F}$, let $\bar U_\Omega(f)$ be defined as in Eqn.~\eqref{eq:bar_Uomega_Ulambda}. Then, we have:
    \begin{align}
        \roundbrac{1-\frac{2}{n\rho_{\min}}}\Lnorm_\Omega(f)\le\E[\bar U_\Omega(f)]\le\left(1+\frac{R}{n-R}\right) \Lnorm_\Omega(f).
    \end{align}
    \noindent Where $\rho_{\min}=\min_{r\in[R]}\rho_r$, i.e., the probability of the rarest class.
\end{proposition}

\begin{proof}
    Firstly, we decompose the risk $\Lnorm_\Omega(f)$ to a weighted sum of class-wise risk:
    \begin{align}
        \label{eq:decomposed_Lomega}
        \Lnorm_\Omega(f) &= \sum_{r=1}^R\rho_r\Lnorm_\Omega^r(f).
    \end{align}
    \noindent We observe that when we shuffle the dataset $S$ according to a permutation $\pi\in\Pi$ then select the first $n$ elements, this process is equivalent to selecting $n$ data points from $S$ without replacement. Therefore, for any $\pi_1, \pi_2\in\Pi$ and $\pi_1\ne\pi_2$, $S_{\pi_1}^{(n)} \overset{d}{=} S_{\pi_2}^{(n)}$ and $S\setminus S_{\pi_1}^{(n)}\overset{d}{=} S\setminus S_{\pi_2}^{(n)}$ \footnote{Where $X\overset{d}{=}Y$ means ``$X$ is identically distributed as $Y$".}. Hence, for any distinct permutations $\pi,\bar\pi\in\Pi$, we have:
    \begin{align*}
        \E\squarebrac{\sum_{r=1}^R\omega_r^{\pi}U_{\Omega_r^{\pi}}(f)}=\E\squarebrac{\sum_{r=1}^R\omega_r^{\bar\pi}U_{\Omega_r^{\bar\pi}}(f)}.
    \end{align*}
    \noindent Now, fixing an arbitrary permutation $\theta\in\Pi$, we have:
    \begin{align*}
        \E[\bar U_\Omega(f)] &= \E\squarebrac{\frac{1}{N!}\sum_{\pi\in\Pi}\sum_{r=1}^R\omega_r^\pi U_{\Omega_r^\pi}(f)} = \frac{1}{N!}\sum_{\pi\in\Pi}\E\squarebrac{\sum_{r=1}^R\omega_r^\pi U_{\Omega_r^\pi}(f)} \\
        &= \E\squarebrac{\sum_{r=1}^R\omega_r^{\theta}U_{\Omega_r^{\theta}}(f)}=\sum_{r=1}^R\E\squarebrac{\omega_r^{\theta}U_{\Omega_r^{\theta}}(f)}.
    \end{align*}
    \noindent For each $r\in[R]$, we have:
    \begin{align*}
        \E\squarebrac{\omega_r^\theta U_{\Omega_r^\theta}(f)}&= \underbrace{\E[\omega_r^\theta U_{\Omega_r^\theta}(f)|n_r^\theta\le 1]}_{=0\text{ since }\Omega_r^\theta=\emptyset}\P(n_r^\theta\le 1) + \E[\omega_r^\theta U_{\Omega_r^\theta}(f)|n_r^\theta\ge 2]\P(n_r^\theta\ge 2)\\
        &= \E[\omega_r^\theta U_{\Omega_r^\theta}(f)|n_r^\theta\ge 2]\P(n_r^\theta\ge 2).
    \end{align*}
    \noindent Since $\omega_r^\theta$ and $U_{\Omega_r^\theta}(f)$ are conditionally independent given $n_r^\theta$, we can write:
    \begin{align*}
        \E\squarebrac{\omega_r^\theta U_{\Omega_r^\theta}(f)|n_r^\theta\ge2} &= \E\squarebrac{\omega_r^\theta|n_r^\theta\ge2}\E\squarebrac{U_{\Omega_r^\theta}(f)|n_r^\theta\ge2} = \E\squarebrac{\omega_r^\theta|n_r^\theta\ge2}\Lnorm_\Omega^r(f).
    \end{align*}
    \noindent Furthermore, we have:
    \begin{align*}
        \E[\omega_r^\theta] &= \E\squarebrac{\omega_r^\theta|n_r^\theta\ge2}\P(n_r^\theta\ge2) + \underbrace{\E\squarebrac{\omega_r^\theta|n_r^\theta\le 1}}_{=0}\P(n_r^\theta\le 1) \\
        &= \E\squarebrac{\omega_r^\theta|n_r^\theta\ge2}\P(n_r^\theta\ge2).
    \end{align*}
    \noindent Therefore:
    \begin{align*}
        \E\squarebrac{\omega_r^\theta U_{\Omega_r^\theta}(f)} &= \E[\omega_r^\theta U_{\Omega_r^\theta}(f)|n_r^\theta\ge 2]\P(n_r^\theta\ge 2) = \P(n_r^\theta\ge 2)\E\squarebrac{\omega_r^\theta|n_r^\theta\ge2}\Lnorm_\Omega^r(f) \\
        &= \P(n_r^\theta\ge 2)\cdot\frac{\E[\omega_r^\theta]}{\P(n_r^\theta\ge2)}\Lnorm_\Omega^r(f) = \E[\omega_r^\theta]\Lnorm_\Omega^r(f).
    \end{align*}
    \noindent Using the following facts:
    \begin{alignat*}{3}
        \frac{n_r^\theta-1}{2}\le \lfloor n_r^\theta/2\rfloor \le \frac{n_r^\theta}{2} \quad\text{and}\quad\frac{n-R}{2} = \sum_{q=1}^R\squarebrac{\frac{n_r}{2}-\frac{1}{2}}\le \sum_{q=1}^R\lfloor n_q^\theta/2\rfloor \le \sum_{q=1}^R\frac{n_r^\theta}{2}=\frac{n}{2} .
    \end{alignat*}
    \noindent We have:
    \begin{align*}
        \E[\omega_r^\theta] = \E\squarebrac{\frac{\lfloor n_r^\theta/2\rfloor}{\sum_{q=1}^R\lfloor n_q^\theta/2\rfloor}} &\le \E\squarebrac{\frac{n_r^\theta/2}{n/2-R/2}} = \E\squarebrac{\frac{n_r^\theta}{n-R}} \\
        &= \E\squarebrac{\frac{\frac{n_r^\theta}{n}(n-R)+\frac{Rn_r^\theta}{n}}{n-R}} = \E\squarebrac{\frac{n_r^\theta}{n}\roundbrac{1+\frac{R}{n-R}}} \\
        &= \rho_r\roundbrac{1+\frac{R}{n-R}} \qquad (n_r^\theta\sim\mathrm{Binom}(n, \rho_r)).
    \end{align*}
    \noindent And:
    \begin{align*}
        \E[\omega_r^\theta] = \E\squarebrac{\frac{\lfloor n_r^\theta/2\rfloor}{\sum_{q=1}^R\lfloor n_q^\theta/2\rfloor}} &\ge \E\squarebrac{\frac{(n_r^\theta-2)/2}{n/2}} = \rho_r - \frac{2}{n} \ge \rho_r\roundbrac{1-\frac{2}{n\rho_{\min}}}.
    \end{align*}
    \noindent Combining the upper and lower bounds of $\E[\omega_r^\theta]$, we have:
    \begin{align*}
        \rho_r\roundbrac{1-\frac{2}{n\rho_{\min}}}\Lnorm_\Omega^r(f)\le \E\squarebrac{\omega_r^\theta U_{\Omega_r^\theta}(f)} \le \rho_r\roundbrac{1+\frac{R}{n-R}}\Lnorm_\Omega^r(f).
    \end{align*}
    \noindent Finally, we have:
    \begin{align*}
        \roundbrac{1-\frac{2}{n\rho_{\min}}}\sum_{r=1}^R\rho_r\Lnorm_\Omega^r(f) &\le \sum_{r=1}^R\E\squarebrac{\omega_r^\theta U_{\Omega_r^\theta}(f)} \le \roundbrac{1+\frac{R}{n-R}}\sum_{r=1}^R\rho_r\Lnorm_\Omega^r(f).
   \end{align*} 
    \noindent In other words, we have:
    \begin{align*}
        \roundbrac{1-\frac{2}{n\rho_{\min}}}\Lnorm_\Omega(f) &\le \E[\bar U_\Omega(f)] \le \roundbrac{1+\frac{R}{n-R}}\Lnorm_\Omega(f),
    \end{align*}
    \noindent as desired. As $N\to\infty$, we have $n\to\infty$. Hence, $\lim_{N\to\infty}\E[\bar U_\Omega(f)]=\lim_{n\to\infty}\E[\bar U_\Omega(f)]$ and:
    \begin{align*}
        \Lnorm_\Omega(f)=\lim_{n\to\infty}\roundbrac{1-\frac{2}{n\rho_{\min}}}\Lnorm_\Omega(f) \le \lim_{n\to\infty}\E[\bar U_\Omega(f)] \le \lim_{n\to\infty}\roundbrac{1+\frac{R}{n-R}}\Lnorm_\Omega(f)=\Lnorm_\Omega(f).
    \end{align*}
    \noindent Therefore, $\lim_{n\to\infty}\E[\bar U_\Omega(f)]=\Lnorm_\Omega(f)$, making $\bar U_\Omega(f)$ asymptotically unbiased.
\end{proof}

\begin{lemma}
    \label{lem:supporting_lemma}
    For $f\in\mathcal{F}$, let $U_\Omega(f)$ be defined as in Eqn.~\eqref{eq:natural_estimator} and $\bar U_\Omega(f)$ be defined as in Eqn.~\eqref{eq:bar_Uomega_Ulambda}. Suppose that $|\ell_{\phi, f}|\le \mathcal{B}$ for all $f\in\mathcal{F}$ and that $n>R$. Then, we have:
    \begin{align}
        \sup_{f\in\mathcal{F}}\left|U_\Omega(f)-\bar U_\Omega(f)\right|\le \frac{2\mathcal{B}R}{n-R} + \sum_{r=1}^R\widehat\rho_r\sup_{f\in\mathcal{F}}\left|U_{\Omega_r}(f)-\widehat{\E}_\Pi\left[\frac{\widehat\rho_r^\pi}{\widehat\rho_r}U_{\Omega_r^\pi}(f)\right]\right|,
    \end{align}
    \noindent where for all $r\in[R]$ and $\pi\in\Pi$, we define $\widehat\rho_r:=\frac{N_r}{N}$, $\widehat\rho_r^\pi:=\frac{n_r^\pi}{n}$, and $\widehat{\E}_\Pi[\cdot]$ denotes the average over permutations $\pi\in\Pi$. If $n_r^\pi\le1$ (so that $\Omega_r^\pi=\emptyset$ and $\omega_r^\pi=0$), $U_{\Omega_r^\pi}(f)$ may take any value in $[-\mathcal{B},\mathcal{B}]$.
\end{lemma}

\begin{proof}
We have:
\begin{align*}
    &\sup_{f\in\mathcal{F}}\Big|U_\Omega(f)-\bar U_\Omega(f)\Big| \\
    &= \sup_{f\in\mathcal{F}}\Bigg|\sum_{r=1}^R\widehat\rho_r U_{\Omega_r}(f) - \widehat\E_\Pi\left[\sum_{r=1}^R\omega_r^\pi U_{\Omega_r^\pi}(f)\right] \Bigg| = \sup_{f\in\mathcal{F}}\Bigg|\sum_{r=1}^R\widehat\rho_r U_{\Omega_r}(f)-\frac{1}{N!}\sum_{\pi\in\Pi}\sum_{r=1}^R \omega_r^\pi U_{\Omega_r^\pi}(f)\Bigg| \\
    &= \sup_{f\in\mathcal{F}}\Bigg|\sum_{r=1}^R\widehat\rho_rU_{\Omega_r}(f)-\sum_{r=1}^R\frac{1}{N!}\sum_{\pi\in\Pi}\omega_r^\pi U_{\Omega_r^\pi}(f)\Bigg|.
\end{align*}
\noindent Using triangle inequality, we have:
\begin{align*}
    \sup_{f\in\mathcal{F}}\Big|U_\Omega(f)-\bar U_\Omega(f)\Big|
    &\le \sum_{r=1}^R\sup_{f\in\mathcal{F}}\Bigg|\widehat\rho_rU_{\Omega_r}(f)-\frac{1}{N!}\sum_{\pi\in\Pi}\omega_r^\pi U_{\Omega_r^\pi}(f)\Bigg| \\
    &= \sum_{r=1}^R\sup_{f\in\mathcal{F}}\Bigg|\widehat\rho_r U_{\Omega_r}(f) - \frac{1}{N!}\sum_{\pi\in\Pi}\widehat\rho_r^\pi U_{\Omega_r^\pi}(f) + \frac{1}{N!}\sum_{\pi\in\Pi}\left[\omega_r^\pi - \widehat\rho_r^\pi\right]U_{\Omega_r^\pi}(f)\Bigg| \\
    &\le \sum_{r=1}^R\widehat\rho_r\sup_{f\in\mathcal{F}}\Bigg|U_{\Omega_r}(f) - \frac{1}{N!}\sum_{\pi\in\Pi}\frac{\widehat\rho_r^\pi}{\widehat\rho_r}U_{\Omega_r^\pi}(f)\Bigg| + \sum_{r=1}^R\sup_{f\in\mathcal{F}}\Bigg|\frac{1}{N!}\sum_{\pi\in\Pi}\left[\omega_r^\pi - \widehat\rho_r^\pi\right]U_{\Omega_r^\pi}(f)\Bigg| \\
    &= \sum_{r=1}^R\widehat\rho_r\sup_{f\in\mathcal{F}}\left|U_{\Omega_r}(f)-\widehat\E_\Pi\left[\frac{\widehat\rho_r^\pi}{\widehat\rho_r}U_{\Omega_r^\pi}(f)\right]\right|+\sum_{r=1}^R\sup_{f\in\mathcal{F}}\Bigg|\frac{1}{N!}\sum_{\pi\in\Pi}\left[\omega_r^\pi - \widehat\rho_r^\pi\right]U_{\Omega_r^\pi}(f)\Bigg|.
\end{align*}
\noindent Now, for each permutation $\pi\in\Pi$, we define $\widehat\rho_r^\pi:=\frac{n_r^\pi}{n}$ where $n_r^\pi := |S_r\cap S_\pi^{(n)}|$, i.e., the number of data points of class $r$ from $S_\pi^{(n)}$. When we select a permutation $\pi\in\Pi$ uniformly, we can think of $n_r^\pi$ as a hypergeometric random variable ($n$ draws from a population of size $N$ with $N_r$ special items). Therefore, we can write:
\begin{align*}
    \forall r\in[R]:\quad\widehat\rho_r:=\frac{N_r}{N} = \frac{1}{n}\E[n_r^\pi] = \frac{1}{n}\widehat\E_\Pi[n_r^\pi]= \frac{1}{N!}\sum_{\pi\in\Pi} \frac{n_r^\pi}{n}.
\end{align*}
\noindent Since $|U_{\Omega_r^\pi}(f)|\le\mathcal{B}$ for all $r\in[R]$, $\pi\in\Pi$ and $f\in\mathcal{F}$, we have:
\begin{align*}
    \sum_{r=1}^R\sup_{f\in\mathcal{F}}\Bigg|\frac{1}{N!}\sum_{\pi\in\Pi}\left[\omega_r^\pi - \widehat\rho_r^\pi\right]U_{\Omega_r^\pi}(f)\Bigg| &\le \mathcal{B}\sum_{r=1}^R\frac{1}{N!}\sum_{\pi\in\Pi}\left|\widehat\rho_r^\pi-\omega_r^\pi\right|.
\end{align*}
\noindent For each $r\in[R]$ and $\pi\in\Pi$, we have:
\begin{align*}
    \absbrac{\widehat\rho_r^\pi-\omega_r^\pi} = \absbrac{\frac{n_r^\pi}{n} - \frac{\lfloor n_r^\pi/2\rfloor}{\sum_{q=1}^R\lfloor n_q^\pi/2\rfloor}} 
    &\le \absbrac{\frac{n_r^\pi}{n} - \frac{2\lfloor n_r^\pi/2\rfloor}{n}} + \absbrac{\frac{2\lfloor n_r^\pi/2\rfloor}{n}-\frac{\lfloor n_r^\pi/2\rfloor}{\sum_{q=1}^R\lfloor n_q^\pi/2\rfloor}} \\
    &= \frac{1}{n}(n_r^\pi-2\lfloor n_r^\pi/2\rfloor) + \lfloor n_r^\pi/2\rfloor\roundbrac{\frac{1}{\sum_{q=1}^R\lfloor n_q^\pi/2\rfloor}-\frac{2}{n}}.
\end{align*}
\noindent Now, use the fact that $n_r^\pi-2\lfloor n_r^\pi/2\rfloor\le 1$ and $\sum_{q=1}^R\lfloor n_q^\pi/2\rfloor\ge\frac{n-R}{2}$ (cf. Proposition \ref{prop:unbiasedness_aux_estimators}), we have:
\begin{align*}
    \absbrac{\widehat\rho_r^\pi-\omega_r^\pi} &\le \frac{1}{n} + 2\lfloor n_r^\pi/2\rfloor \roundbrac{\frac{1}{n-R}-\frac{1}{n}} \le \frac{1}{n}+\frac{2R\lfloor n_r^\pi/2\rfloor}{n(n-R)}.
\end{align*}
\noindent Therefore:
\begin{align*}
    \sum_{r=1}^R|\widehat\rho_r^\pi-\omega_r^\pi| &\le \frac{R}{n}+\frac{2R}{n(n-R)}{\sum_{r=1}^R\lfloor n_r^\pi/2\rfloor} \le \frac{R}{n}+\frac{R}{n-R}\le\frac{2R}{n-R}\qquad \Bigg(\sum_{r=1}^R\lfloor n_r^\pi/2\rfloor {\le n/2}\Bigg).
\end{align*}
\noindent Finally, using the above inequality, we have:
\begin{align*}
    \sup_{f\in\mathcal{F}}\absbrac{U_\Omega(f)-\bar U_\Omega(f)} &\le \frac{\mathcal{B}}{N!}\sum_{\pi\in\Pi}{\sum_{r=1}^R\left|\widehat\rho_r^\pi-\omega_r^\pi\right|} + \sum_{r=1}^R\widehat\rho_r\sup_{f\in\mathcal{F}}\left|U_{\Omega_r}(f)-\widehat{\E}_\Pi\left[\frac{\widehat\rho_r^\pi}{\widehat\rho_r}U_{\Omega_r^\pi}(f)\right]\right| \\
    &\le \frac{2\mathcal{B}R}{n-R} + \sum_{r=1}^R\widehat\rho_r\sup_{f\in\mathcal{F}}\left|U_{\Omega_r}(f)-\widehat{\E}_\Pi\left[\frac{\widehat\rho_r^\pi}{\widehat\rho_r}U_{\Omega_r^\pi}(f)\right]\right|,
\end{align*}
\noindent as desired.
\end{proof}

Before continuing on to the proof of Proposition \ref{prop:diff_between_natural_and_nonatural}, we introduce the formal definition of class collision below.
\begin{definition}[Class-collision]
    Let $T_k$ be the map that returns the negative samples of a given tuple (i.e., the last $k$ elements in the tuple). Specifically:
    \begin{align}
        T_k: \mathcal{X}^{k+2}&\to\mathcal{X}^k:\left(X, X^+, \set{X_i^-}_{i=1}^k\right) \mapsto (X_1^-,\dots,X_k^-).
    \end{align}
    \noindent Let the map $\sharp^r_\mathrm{col}:\mathcal{X}^{k+2}\to\{0,1,\dots, k\}$ be defined as follows:
    \begin{align}
        \forall r\in[R]:\quad \sharp_\mathrm{col}^r(t) := |S_r \cap T_k(t)|.
    \end{align}
    \noindent Then, for each $t=\left(X, X^+, \set{X_i^-}_{i=1}^k\right)\in\mathcal{X}^{k+2}$ such that $X, X^+\sim\mathcal{D}_r^{\otimes2}$,  $\sharp_\mathrm{col}^r(t)$ denotes the number of negative samples in $t$ that belong to class $r$, i.e., number of \textbf{class-collisions}. 
\end{definition}

\begin{definition}[Split-induced tuple distribution]
    \label{def:split_induced_distribution}
    Suppose that $N\ge k+2$, and let $r\in[R]$ with $N_r\ge2$ and $\pi\in\Pi$ with $n_r^\pi\ge1$. The \textbf{split-induced tuple distribution} $\mu_r^\pi$ is the probability mass function over $\Omega_r$ defined as follows:
    \begin{itemize}
        \item If $n_r^\pi\ge2$, $\mu_r^\pi$ is the uniform distribution over $\Omega_r^\pi$, i.e., $\mu_r^\pi(t):=\mathds{1}_{\Omega_r^\pi}(t)|\Omega_r^\pi|^{-1}$ for all $t\in\Omega_r$.
        \item If $n_r^\pi=1$, i.e., $S_r\cap S_\pi^{(n)}=\{x\}$, $\mu_r^\pi$ is the law of the tuple $(x,X^+,\set{X_i^-}_{i=1}^k)$ obtained by drawing $\set{X_i^-}_{i=1}^k$ uniformly among the $k$-subsets of $S\setminus S_\pi^{(n)}$, then $X^+$ uniformly from $S_r\setminus(\{x\}\cup\set{X_i^-}_{i=1}^k)$. If the latter set is empty, the whole tuple is instead drawn uniformly from $\Omega_r$.
    \end{itemize}
\end{definition}

\begin{lemma}
    \label{lem:supporting_lemma_2}
    For $N\ge k+2$ and a fixed $r\in[R]$ with $N_r\ge 2$, let $P_r, Q_r$ be probability mass functions over $\Omega_r$ such that:
    \begin{align}
        \label{eq:pmf_two_procedures}
        \forall t\in\Omega_r: \quad P_r(t) := \frac{1}{|\Omega_r|}, \quad Q_r(t) := \frac{1}{N!}\sum_{\pi\in\Pi}\frac{\widehat\rho_r^\pi}{\widehat\rho_r}\mu_r^\pi(t),
    \end{align}
    \noindent where $\mu_r^\pi$ is the split-induced tuple distribution (Definition \ref{def:split_induced_distribution}). Then, for any $\barkappa\in\{0,1,\dots, k\}$, we have:
    \begin{align}
        \E_{t\sim P_r}\squarebrac{\ell_{\phi, f}(t)\Big|\sharp_\mathrm{col}^r(t)=\barkappa} = \E_{t\sim Q_r}\squarebrac{\ell_{\phi, f}(t)\Big|\sharp_\mathrm{col}^r(t)=\barkappa}.
    \end{align}
\end{lemma}

\begin{proof}
    Note first that $Q_r$ is indeed a probability mass function, since $\mu_r^\pi$ is one whenever $n_r^\pi\ge1$, $\widehat\rho_r^\pi=0$ whenever $n_r^\pi=0$, and $\frac{1}{N!}\sum_{\pi\in\Pi}\widehat\rho_r^\pi=\widehat\rho_r$ (see the proof of Lemma \ref{lem:supporting_lemma}). For each $0\le \barkappa \le k$, we define the set of tuples $\Omega_{r,\barkappa}\subset\Omega_r$ as:
    \begin{align}
        \Omega_{r,\barkappa} := \Big\{t\in\Omega_r: \sharp_\mathrm{col}^r(t)=\barkappa\Big\},
    \end{align}
    \noindent i.e, the set of tuples with exactly $\barkappa$ class-collisions. Then, we have:
    \begin{align*}
        \E_{t\sim P_r}\squarebrac{\ell_{\phi, f}(t)\Big|\sharp_\mathrm{col}^r(t)=\barkappa} &= \frac{1}{P_r(\Omega_{r,\barkappa})}\sum_{t\in\Omega_{r,\barkappa}}P_r(t)\ell_{\phi, f}(t) = \frac{1}{|\Omega_{r,\barkappa}|/|\Omega_r|}\sum_{t\in\Omega_{r,\barkappa}}\frac{\ell_{\phi, f}(t)}{|\Omega_r|} \\
        &= \frac{1}{|\Omega_{r, \barkappa}|}\sum_{t\in\Omega_{r,\barkappa}}\ell_{\phi,f}(t).
    \end{align*}
    \noindent Now, we claim that for all $t_1,t_2\in\Omega_{r, \barkappa}$ for $0\le \barkappa\le k$, we have $Q_r(t_1)=Q_r(t_2)$. This can be proven using a simple symmetry argument. Suppose that we have two tuples:
    \begin{align*}
        t_1 &= \left(X_{u_1}, X_{u_2}, \set{X_{u_\ell}}_{\ell=3}^{\barkappa+2}, \set{X_{u_j}}_{j=\barkappa+3}^{k+2}\right), \qquad t_2 = \left(X_{w_1}, X_{w_2}, \set{X_{w_\ell}}_{\ell=3}^{\barkappa+2}, \set{X_{w_j}}_{j=\barkappa+3}^{k+2}\right).
    \end{align*}
    \noindent Where $\set{u_j}_{j=1}^{k+2},\set{w_j}_{j=1}^{k+2}\subset [N]$ s.t. $\sharp_\mathrm{col}^r(t_1)=\sharp_\mathrm{col}^r(t_2)=\barkappa$, i.e., they share the same number of $\barkappa$ class-$r$ negatives, and:
    \begin{itemize}
        \item $X_{u_1}, X_{u_2}, X_{w_1}, X_{w_2} \in S_r$.
        \item $X_{u_\ell}, X_{w_\ell}\in S_r$ for all $3\le \ell \le \barkappa+2$.
        \item $X_{u_j}, X_{w_j}\notin S_r$ for all $\barkappa+3\le j\le k+2$.
    \end{itemize}    
    \noindent To show that $Q_r(t_1)=Q_r(t_2)$, we define:
    \begin{align}
        \forall r\in[R]:\quad\mathcal{I}_r :=\Big\{j\in[N]: X_j\in S_r\Big\}.
    \end{align}
    \noindent Since $u_j\in\mathcal{I}_r$ if and only if $w_j\in\mathcal{I}_r$ for all $1\le j\le k+2$, there exists a permutation $\sigma$ of $[N]$ such that $\sigma(u_j)=w_j$ for all $1\le j\le k+2$ and $\sigma(\mathcal{I}_r)=\mathcal{I}_r$. The map $\sigma$ is role-preserving: the anchor, the positive and each negative sample of $t_1$ are mapped to the anchor, the positive and the corresponding negative sample of $t_2$. Relabelling the data points by $\sigma$ (i.e., $X_i\mapsto X_{\sigma(i)}$) maps $t_1$ to $t_2$, $\Omega_r$ onto itself and $S_\pi^{(n)}=\set{X_{\pi(j)}}_{j=1}^n$ onto $S_{\sigma\circ\pi}^{(n)}$ for every $\pi\in\Pi$. Since $\sigma(\mathcal{I}_r)=\mathcal{I}_r$, we have $|S_{\sigma\circ\pi}^{(n)}\cap S_r|=|S_\pi^{(n)}\cap S_r|$, i.e., $\widehat\rho_r^{\sigma\circ\pi}=\widehat\rho_r^{\pi}$. Moreover, since each step in the definition of $\mu_r^\pi$ is a uniform choice determined by $S_\pi^{(n)}$ and $S_r$ only, we have $\mu_r^{\sigma\circ\pi}(t_2)=\mu_r^\pi(t_1)$. As $\pi\mapsto\sigma\circ\pi$ is a bijection of $\Pi$, we obtain:
    \begin{align*}
        Q_r(t_2) = \frac{1}{N!}\sum_{\pi\in\Pi}\frac{\widehat\rho_r^{\sigma\circ\pi}}{\widehat\rho_r}\mu_r^{\sigma\circ\pi}(t_2) = \frac{1}{N!}\sum_{\pi\in\Pi}\frac{\widehat\rho_r^{\pi}}{\widehat\rho_r}\mu_r^{\pi}(t_1) = Q_r(t_1).
    \end{align*}
    \noindent Therefore, we have:
    \begin{align*}
        \E_{t\sim Q_r}\squarebrac{\ell_{\phi, f}(t)\Big|\sharp_\mathrm{col}^r(t)=\barkappa} &= \frac{1}{Q_r(\Omega_{r, \barkappa})}\sum_{t\in\Omega_{r,\barkappa}}Q_r(t)\ell_{\phi, f}(t) = \frac{1}{|\Omega_{r,\barkappa}|}\sum_{t\in\Omega_{r,\barkappa}}\ell_{\phi, f}(t)\\&= \E_{t\sim P_r}\squarebrac{\ell_{\phi, f}(t)\Big|\sharp_\mathrm{col}^r(t)=\barkappa},
    \end{align*}
    \noindent as desired.
\end{proof}
\begin{proposition}
    \label{prop:diff_between_natural_and_nonatural_extra_assumption}
    Let $\mathcal{F}$ denote a class of representation functions and $U_\Omega(f)$, $\bar U_\Omega(f)$ be U-Statistics formulations in Eqn.~\eqref{eq:natural_estimator} and Eqn.~\eqref{eq:bar_Uomega_Ulambda}, respectively. Suppose that $|\ell_{\phi, f}|\le \mathcal{B}$ for all $f\in\mathcal{F}$ and $N_r\ge 2(k+1)$ for all $r\in[R]$. Then:
    \begin{align}
        \sup_{f\in\mathcal{F}}\left|U_\Omega(f)-\bar U_\Omega(f)\right| \le \frac{2\mathcal{B}R}{n-R} +\frac{6\mathcal{B}Rk}{N} + \frac{2\mathcal{B}k}{N-k-1}.
    \end{align}
\end{proposition}
\begin{proof}
    From Lemma \ref{lem:supporting_lemma} (which applies, since $N\ge 2R(k+1)$ gives $n\ge 2R$), where for $\pi\in\Pi$ with $n_r^\pi=1$ we take $U_{\Omega_r^\pi}(f):=\E_{t\sim\mu_r^\pi}[\ell_{\phi,f}(t)]$ where $\mu_r^\pi$ is the split-induced tuple distribution (Definition \ref{def:split_induced_distribution}), our main task now is to bound the following weighted sum of uniform differences:
    \begin{align*}
        \sum_{r=1}^R\widehat\rho_r\sup_{f\in\mathcal{F}}\left|U_{\Omega_r}(f)-\widehat{\E}_\Pi\left[\frac{\widehat\rho_r^\pi}{\widehat\rho_r}U_{\Omega_r^\pi}(f)\right]\right|.
    \end{align*}
    \noindent We can write $U_{\Omega_r}(f)=\E_{t\sim P_r}\left[\ell_{\phi, f}(t)\right]$ and $\widehat{\E}_\Pi\left[\frac{\widehat\rho_r^\pi}{\widehat\rho_r}U_{\Omega_r^\pi}(f)\right]=\E_{t\sim Q_r}\squarebrac{\ell_{\phi, f}(t)}$ where $P_r, Q_r$ are probability mass functions defined in Eqn.~\eqref{eq:pmf_two_procedures}. For each $0\le\barkappa\le k$, by Lemma \ref{lem:supporting_lemma_2}, we can write:
    \begin{align}
        \E_{t\sim P_r}\squarebrac{\ell_{\phi, f}(t)\Big|\sharp_\mathrm{col}^r(t)=\barkappa} = \E_{t\sim Q_r}\squarebrac{\ell_{\phi, f}(t)\Big|\sharp_\mathrm{col}^r(t)=\barkappa}:=E_{r,\barkappa}(f).
    \end{align}
    \noindent Then, we have:
    \begin{align*}
        &\sum_{r=1}^R\widehat\rho_r\sup_{f\in\mathcal{F}}\left|U_{\Omega_r}(f)-\widehat{\E}_\Pi\left[\frac{\widehat\rho_r^\pi}{\widehat\rho_r}U_{\Omega_r^\pi}(f)\right]\right| = \sum_{r=1}^R\widehat\rho_r\sup_{f\in\mathcal{F}}\left|\E_{t\sim P_r}\left[\ell_{\phi, f}(t)\right]- \E_{t\sim Q_r}\left[\ell_{\phi, f}(t)\right]\right| \\
        &=\sum_{r=1}^R\widehat\rho_r\sup_{f\in\mathcal{F}}\left|\sum_{\barkappa=0}^k \P_{t\sim P_r}(\sharp_\mathrm{col}^r(t)=\barkappa)E_{r,\barkappa}(f) - \sum_{\barkappa=0}^k \P_{t\sim Q_r}(\sharp_\mathrm{col}^r(t)=\barkappa)E_{r,\barkappa}(f)\right| \\
        &= \sum_{r=1}^R\widehat\rho_r  \sup_{f\in\mathcal{F}}\left|\sum_{\barkappa=0}^k E_{r,\barkappa}(f)\Big[\P_{t\sim P_r}(\sharp_\mathrm{col}^r(t)=\barkappa)-\P_{t\sim Q_r}(\sharp_\mathrm{col}^r(t)=\barkappa)\Big]\right| \\
        &\le \sum_{r=1}^R\widehat\rho_r\sum_{\barkappa=0}^k \sup_{f\in\mathcal{F}}\left|E_{r,\barkappa}(f)\Big[\P_{t\sim P_r}(\sharp_\mathrm{col}^r(t)=\barkappa)-\P_{t\sim Q_r}(\sharp_\mathrm{col}^r(t)=\barkappa)\Big]\right| \\
        &\le \mathcal{B}\sum_{r=1}^R\widehat\rho_r\sum_{\barkappa=0}^k\Big|\P_{t\sim P_r}(\sharp_\mathrm{col}^r(t)=\barkappa)-\P_{t\sim Q_r}(\sharp_\mathrm{col}^r(t)=\barkappa)\Big|.
    \end{align*}
    \noindent Now, the problem reduces to comparing class-collision probabilities under two probability mass functions $P_r, Q_r$ for each $r\in[R]$. Firstly, we claim that for each $r\in[R]$:
    \begin{itemize}
        \item For $t\sim P_r$, $\sharp_\mathrm{col}^r(t)\sim\mathrm{Hypergeometric}(N-2, N_r-2, k)$.
        \item For $t\sim Q_r$, $\sharp_\mathrm{col}^r(t)\sim\mathrm{Hypergeometric}(N-1, N_r-1, k)$ (as $k\le N_r-2$).
    \end{itemize}
    \noindent Both of these claims will be elaborated further in Remark \ref{remark:hypergeometric_dist_arguments}. As a result, for each $r\in[R]$ and $0\le\barkappa\le k$, we have:
    \begin{align*}
        &\Big|\P_{t\sim P_r}(\sharp_\mathrm{col}^r(t)=\barkappa)-\P_{t\sim Q_r}(\sharp_\mathrm{col}^r(t)=\barkappa)\Big| \\&= \left|\frac{\binom{N_r-2}{\barkappa}\binom{N-N_r}{k-\barkappa}}{\binom{N-2}{k}} - \frac{\binom{N_r-1}{\barkappa}\binom{N-N_r}{k-\bar\kappa}}{\binom{N-1}{k}}\right| = \P_{t\sim P_r}(\sharp_\mathrm{col}^r(t)=\barkappa)\left|1-\frac{\binom{N_r-1}{\barkappa}}{\binom{N_r-2}{\barkappa}}\left[\frac{\binom{N-1}{k}}{\binom{N-2}{k}}\right]^{-1}\right| \\
        &= \P_{t\sim P_r}(\sharp_\mathrm{col}^r(t)=\barkappa)\left|1 - \frac{N_r-1}{N_r-\barkappa-1}\left[\frac{N-1}{N-k-1}\right]^{-1}\right|.
    \end{align*}
    \noindent Now, for notational brevity, we denote $\alpha, \beta$ as the following terms:
    \begin{align}
        \alpha &= \frac{N_r-1}{N_r-\barkappa-1},\quad\textup{and}\quad  \beta=\frac{N-1}{N-k-1}. 
    \end{align}
    \noindent Note that $\frac{N_r-1}{N_r-\barkappa-1}\le \frac{N_r-1}{N_r-k-1}$. Therefore:
    \begin{align*}
        \alpha\beta^{-1} &= \beta^{-1}\left[\frac{N_r-1}{N_r-\barkappa-1}\right] \le \beta^{-1}\left[\frac{N_r-1}{N_r-k-1}\right] = \beta^{-1}\left[1+\frac{k}{N_r-k-1}\right] \\
        &\le \beta^{-1}\left[1+\frac{3k}{N_r-k-1}\right] \overset{(a)}{\le} 1 +\frac{3k}{N_r-k-1} \qquad (\beta^{-1}\le 1).
    \end{align*}
    \noindent Similarly, we have:
    \begin{align*}
        \alpha\beta^{-1} &\ge \beta^{-1} \ge \left[\frac{N-1}{N-k-1}\right]^{-2} = \left[\frac{1}{1+\frac{k}{N-k-1}}\right]^2 \ge \left[1-\frac{k}{N-k-1}\right]^2 \\
        &\overset{(b)}{\ge} 1 - \frac{2k}{N-k-1}\qquad((1-\delta)^2\ge 1-2\delta,\quad\forall\delta\in[0,1]).
    \end{align*}   
    \noindent By plugging inequalities $(a), (b)$ back to the collision probabilities difference bound, we have:
    \begin{align*}
        &\Big|\P_{t\sim P_r}(\sharp_\mathrm{col}^r(t)=\barkappa)-\P_{t\sim Q_r}(\sharp_\mathrm{col}^r(t)=\barkappa)\Big|\le \P_{t\sim P_r}(\sharp_\mathrm{col}^r(t)=\barkappa)\left|1 - \alpha\beta^{-1}\right| \\
        &\le \P_{t\sim P_r}(\sharp_\mathrm{col}^r(t)=\barkappa)\left(\frac{3k}{N_r-k-1}+\frac{2k}{N-k-1}\right).
    \end{align*}
    \noindent Therefore, we have:
    \begin{align*}
        &\sum_{r=1}^R\widehat\rho_r\sum_{\barkappa=0}^k\Big|\P_{t\sim P_r}(\sharp_\mathrm{col}^r(t)=\barkappa)-\P_{t\sim Q_r}(\sharp_\mathrm{col}^r(t)=\barkappa)\Big|\\ 
        &\le \sum_{r=1}^R\widehat\rho_r\sum_{\barkappa=0}^k \P_{t\sim P_r}(\sharp_\mathrm{col}^r(t)=\barkappa)\left(\frac{3k}{N_r-k-1}+\frac{2k}{N-k-1}\right) \\
        &= \sum_{r=1}^R\widehat\rho_r\left(\frac{3k}{N_r-k-1}+\frac{2k}{N-k-1}\right)\underbrace{\left[\sum_{\barkappa=0}^k \P_{t\sim P_r}(\sharp_\mathrm{col}^r(t)=\barkappa)\right]}_{=1} \\
        &\le \sum_{r=1}^R\widehat\rho_r \left(\frac{6k}{N_r}+\frac{2k}{N-k-1}\right)\qquad\left(k+1\le \frac{N_r}{2}\right) \\ 
        &= \frac{6Rk}{N} + \frac{2k}{N-k-1}.
    \end{align*} 
    \noindent Finally, we have:
    \begin{align*}
        \sup_{f\in\mathcal{F}}\left|U_\Omega(f)-\bar U_\Omega(f)\right| &\le \frac{2\mathcal{B}R}{n-R} + \sum_{r=1}^R\widehat\rho_r\sup_{f\in\mathcal{F}}\left|U_{\Omega_r}(f)-\widehat{\E}_\Pi\left[\frac{\widehat\rho_r^\pi}{\widehat\rho_r}U_{\Omega_r^\pi}(f)\right]\right|\qquad(\textup{Lemma \ref{lem:supporting_lemma}}) \\
        &\le \frac{2\mathcal{B}R}{n-R} +\frac{6\mathcal{B}Rk}{N} + \frac{2\mathcal{B}k}{N-k-1},
    \end{align*}
    \noindent as desired.
\end{proof}

\begin{proposition}
    \label{prop:diff_between_natural_and_nonatural}
    Let $\mathcal{F}$ denote a class of representation functions. Suppose that $|\ell_{\phi, f}|\le \mathcal{B}$ for all $f\in\mathcal{F}$ and that $N\ge 9Rk$. Then:
    \begin{align}
        \|U_\Omega-\bar U_\Omega\|_\mathcal{F} &\le \mathcal{O}\left(\mathcal{B}\left[
            \frac{R}{\lfloor\frac{N}{k}\rfloor-R} + \frac{Rk}{N-k}
        \right]\right), \quad\|U_\Lambda-\bar U_\Lambda\|_\mathcal{F} \le \mathcal{O}\left(\mathcal{B}\left[
            \frac{R}{\lfloor \frac{N}{k}\rfloor-R} + \frac{Rk}{N-k}
        \right]\right).
    \end{align}
\end{proposition}
\begin{proof}
    Now, the goal is to remove the assumption that $N_r\ge2(k+1)$ for all $r\in[R]$. We first use the following abbreviations for each $r\in[R]$ and $0\le\barkappa\le k$:
    \begin{align*}
        p_{r, \barkappa} &= \P_{t\sim P_r}(\sharp_\mathrm{col}^r(t)=\barkappa), \qquad
        q_{r,\barkappa} = \P_{t\sim Q_r}(\sharp_\mathrm{col}^r(t)=\barkappa).
    \end{align*}
    \noindent Then, we have:
    \begin{align*}
        \sum_{r=1}^R\widehat\rho_r \sum_{\barkappa=0}^k|p_{r,\barkappa} - q_{r,\barkappa}|  &= \sum_{r\in[R]:N_r\ge3} \widehat\rho_r\sum_{\barkappa=0}^k |p_{r,\barkappa}-q_{r,\barkappa}| + \sum_{r\in[R]:N_r\le2} \widehat\rho_r\sum_{\barkappa=0}^k |p_{r,\barkappa}-q_{r,\barkappa}| \\
        &\le \sum_{r\in[R]:N_r\ge3} \widehat\rho_r\sum_{\barkappa=0}^k |p_{r,\barkappa}-q_{r,\barkappa}| + \frac{2}{N}\sum_{r\in[R]:N_r\le 2}\sum_{\barkappa=0}^k|p_{r,\barkappa} - q_{r,\barkappa}| \\
        &= \sum_{r\in[R]:N_r\ge3} \widehat\rho_r\sum_{\barkappa=0}^k |p_{r,\barkappa}-q_{r,\barkappa}| + \frac{2}{N}\sum_{r\in[R]:N_r\le 2}|p_{r,0}-q_{r, 0}| \\
        &\overset{(*)}{\le} \sum_{r\in[R]:N_r\ge3} \widehat\rho_r\sum_{\barkappa=0}^k |p_{r,\barkappa}-q_{r,\barkappa}| + \frac{2R}{N}.
    \end{align*}
    \noindent Where the inequality in $(*)$ is due to the fact that $|p_{r,0}-q_{r,0}|\le 1$ for all minority classes $r\in[R]$ such that $N_r\le 2$ (a class with $N_r=1$ contributes nothing to the sum in Lemma \ref{lem:supporting_lemma}, as $\Omega_r=\emptyset$ and we may take $U_{\Omega_r^\pi}=0$). We can decompose the remaining sum as follows:
    \begin{align*}
        \sum_{r\in[R]:N_r\ge3}\widehat\rho_r\sum_{\barkappa=0}^k|p_{r,\barkappa}-q_{r,\barkappa} | &= \underbrace{\sum_{r\in[R]:N_r\ge3}\widehat\rho_r\sum_{\substack{0\le\barkappa\le k\\ N_r\ge 2(\barkappa+1)}}|p_{r,\barkappa} - q_{r,\barkappa}|}_{\text{(A)}} + \underbrace{\sum_{r\in[R]:N_r\ge3}\widehat\rho_r\sum_{\substack{0\le\barkappa\le k\\ N_r\le 2\barkappa+1}}|p_{r,\barkappa}-q_{r,\barkappa}|}_{\text{(B)}}.
    \end{align*}
    \noindent \textbf{Bounding (A)}: By a more refined approach than the proof of Proposition \ref{prop:diff_between_natural_and_nonatural_extra_assumption}, for each $r\in[R]$ and $0\le\barkappa\le k$ such that $N_r\ge 2(\barkappa+1)$, we have
    \begin{align*}
        \alpha\beta^{-1} &= \beta^{-1}\left[\frac{N_r-1}{N_r-\barkappa-1}\right] = \beta^{-1}\left[1+\frac{\barkappa}{N_r-\barkappa-1}\right] \\
        &\le 1 + \frac{3\barkappa}{N_r-\barkappa-1}.
    \end{align*}
    \noindent By Remark \ref{remark:hypergeometric_dist_arguments}, $q_{r,\barkappa}=h_{r,\barkappa}+e_rp_{r,\barkappa}$ for $\barkappa\le N_r-2$, and $h_{r,\barkappa}=p_{r,\barkappa}\alpha\beta^{-1}$ (cf. the proof of Proposition \ref{prop:diff_between_natural_and_nonatural_extra_assumption}). As a result, for $r\in[R]$ and $0\le\barkappa\le k$ such that $N_r\ge 2(\barkappa+1)$:
    \begin{align*}
        |p_{r,\barkappa} - q_{r,\barkappa}| &\le p_{r,\barkappa}|1-\alpha\beta^{-1}| + e_rp_{r,\barkappa} \\
        &\le p_{r,\barkappa}\left(\frac{3\barkappa}{N_r-\barkappa-1} + \frac{2k}{N-k-1} + e_r\right).
    \end{align*}
    \noindent Therefore, we have:
    \begin{align*}
        &\sum_{r\in[R]:N_r\ge3}\widehat\rho_r\sum_{\substack{0\le\barkappa\le k\\ N_r\ge 2(\barkappa+1)}}|p_{r,\barkappa} - q_{r,\barkappa}| - \sum_{r\in[R]:N_r\ge3}\widehat\rho_re_r \\
        &\le \sum_{r\in[R]:N_r\ge3}\widehat\rho_r\sum_{\substack{0\le\barkappa\le k\\ N_r\ge 2(\barkappa+1)}}p_{r,\barkappa}\left(\frac{3\barkappa}{N_r-\barkappa-1} + \frac{2k}{N-k-1}\right) \\
        &\le \sum_{r\in[R]:N_r\ge3}\widehat\rho_r\sum_{\substack{0\le\barkappa\le k\\ N_r\ge 2(\barkappa+1)}}p_{r,\barkappa}\left(\frac{6\barkappa}{N_r} + \frac{2k}{N-k-1}\right) \qquad \left(\textup{Since }\barkappa+1\le \frac{N_r}{2}\right) \\
        &\le \sum_{r\in[R]:N_r\ge3}\widehat\rho_r\sum_{\substack{0\le\barkappa\le k\\ N_r\ge 2(\barkappa+1)}}p_{r,\barkappa}\left(\frac{6k}{N_r} + \frac{2k}{N-k-1}\right) \\
        &= \sum_{r\in[R]:N_r\ge3}\widehat\rho_r \left(\frac{6k}{N_r} + \frac{2k}{N-k-1} \right)\underbrace{\sum_{\substack{0\le\barkappa\le k\\ N_r\ge 2(\barkappa+1)}}p_{r,\barkappa}}_{\le 1}\\
        &\le \frac{2k}{N-k-1} + \sum_{r\in[R]:N_r\ge3}\widehat\rho_r \frac{6k}{N_r} \\
        &\le \frac{2k}{N-k-1} + \frac{6Rk}{N}.
    \end{align*}
    \noindent Moreover, $e_r\le\P\big(\mathrm{Hypergeometric}(N-1,N_r-1,k)\ge N_r-1\big)\le\frac{k}{N-1}$ by Markov's inequality. As a result, we have (A) $\le\frac{2k}{N-k-1}+\frac{6Rk}{N}+\frac{k}{N-1}$.
    
    \noindent\textbf{Bounding (B)}: Using the triangle inequality, we can write
    \begin{align*}
        &\sum_{r\in[R]:N_r\ge3}\widehat\rho_r\sum_{\substack{0\le\barkappa\le k\\ N_r\le 2\barkappa+1}}|p_{r,\barkappa}-q_{r,\barkappa}| \\
        &\le \sum_{r\in[R]:N_r\ge3}\widehat\rho_r\sum_{\substack{0\le\barkappa\le k\\ N_r\le 2\barkappa+1}}p_{r,\barkappa} + \sum_{r\in[R]:N_r\ge3}\widehat\rho_r\sum_{\substack{0\le\barkappa\le k\\ N_r\le 2\barkappa+1}}q_{r,\barkappa} \\
        &= \sum_{r\in[R]:N_r\ge3}\widehat\rho_r\P_{t\sim P_r}\left(2\sharp_\mathrm{col}^r(t)+1\ge N_r\right) +
        \sum_{r\in[R]:N_r\ge3}\widehat\rho_r\P_{t\sim Q_r}\left(2\sharp_\mathrm{col}^r(t)+1\ge N_r\right).
    \end{align*}
    \noindent Under $P_r$, $\E[\sharp_\mathrm{col}^r(t)]=\frac{k(N_r-2)}{N-2}$. Under $Q_r$, $\E[\sharp_\mathrm{col}^r(t)]\le\frac{k(N_r-1)}{N-1}$, since $Q_r$ only moves the mass $e_r$ of $\mathrm{Hypergeometric}(N-1,N_r-1,k)$ at $N_r-1$ to the law of $\sharp_\mathrm{col}^r(t)$ under $P_r$, which is supported on $\{0,\dots,N_r-2\}$ (Remark \ref{remark:hypergeometric_dist_arguments}). Hence, by Markov's inequality, for every $r\in[R]$ with $N_r\ge3$:
    \begin{align*}
        \P_{t\sim P_r}\left(2\sharp_\mathrm{col}^r(t)+1\ge N_r\right)&\le\frac{2k(N_r-2)}{(N-2)(N_r-1)}\le\frac{2k}{N-2}, \\
        \P_{t\sim Q_r}\left(2\sharp_\mathrm{col}^r(t)+1\ge N_r\right)&\le\frac{2k(N_r-1)}{(N-1)(N_r-1)}=\frac{2k}{N-1}.
    \end{align*}
    \noindent As a result:
    \begin{align*}
        &\sum_{r\in[R]:N_r\ge3}\widehat\rho_r\sum_{\substack{0\le\barkappa\le k\\ N_r\le 2\barkappa+1}}|p_{r,\barkappa}-q_{r,\barkappa}| \\&\le \sum_{r\in[R]:N_r\ge3}\widehat\rho_r\P_{t\sim P_r}\left(2\sharp_\mathrm{col}^r(t)+1\ge N_r\right) +
        \sum_{r\in[R]:N_r\ge3}\widehat\rho_r\P_{t\sim Q_r}\left(2\sharp_\mathrm{col}^r(t)+1\ge N_r\right) \\
        &\le \sum_{r\in[R]:N_r\ge3}\widehat\rho_r\left(\frac{2k}{N-2}+\frac{2k}{N-1}\right) \le \frac{4k}{N-2}.
    \end{align*}
    \noindent Combining every inequalities, we have:
    \begin{align*}
        \sum_{r=1}^R\widehat\rho_r \sum_{\barkappa=0}^k|p_{r,\barkappa} - q_{r,\barkappa}| &\le \frac{2R}{N} + \frac{2k}{N-k-1}+\frac{6Rk}{N} + \frac{k}{N-1} + \frac{4k}{N-2}.
    \end{align*}
    \noindent Hence:
    \begin{align*}
        \sup_{f\in\mathcal{F}}\left|U_\Omega(f)-\bar U_\Omega(f)\right| &\le \frac{2\mathcal{B}R}{n-R} + \mathcal{B}\sum_{r=1}^R\widehat\rho_r \sum_{\barkappa=0}^k|p_{r,\barkappa} - q_{r,\barkappa}|  \\
        &\le \mathcal{B}\left[\frac{2R}{n-R} + \frac{2R}{N} + \frac{2k}{N-k-1}+\frac{6Rk}{N} + \frac{k}{N-1} + \frac{4k}{N-2}\right] \\
        &= \mathcal{O}\left[\mathcal{B}\left(
            \frac{R}{n-R} + Rk\left[\frac{1}{N-k}+\frac{1}{(N-k)^2}\right]
        \right)\right],
    \end{align*}
    \noindent Since $N\ge 9Rk$, we have $n-R\ge\frac{2N}{k+2}-2-R\ge\frac{N}{k+2}$, so that Lemma \ref{lem:supporting_lemma} applies (with $U_{\Omega_r^\pi}$ chosen as in the proof of Proposition \ref{prop:diff_between_natural_and_nonatural_extra_assumption}) and $\frac{R}{n-R}\le\frac{3Rk}{N}$. Hence, all the terms above are $\mathcal{O}\big(\frac{\mathcal{B}Rk}{N-k}\big)$, as desired. Of course, the proof is extensible to $\sup_{f\in\mathcal{F}}\left|U_\Lambda(f)-\bar U_\Lambda(f)\right|$: when $m_r^\pi\in\{1,2\}$ (so that $\lambda_r^\pi=0$), $U_{\Lambda_r^\pi}(f)$ is taken as the expectation of $\ell_{\phi,f}$ under the analogue of the split-induced tuple distribution (Definition \ref{def:split_induced_distribution}) for $\Lambda_r$, in which the anchor is drawn uniformly from $S_r\cap S_\pi^{(m)}$, the $k-1$ negatives uniformly from $S\setminus S_\pi^{(m)}$, and $(X^+,\bar X)$ as a uniform ordered pair of class-$r$ points that are neither the anchor nor negatives (or the whole tuple uniformly from $\Lambda_r$ if fewer than two such points remain). Then, the term $\frac{2\mathcal{B}R}{n-R}$ becomes $\frac{4\mathcal{B}R}{m}$ (as $\sum_{r=1}^R|\frac{m_r^\pi}{m}-\lambda_r^\pi|\le\frac{2}{m}(m-3\sum_{q=1}^R\lfloor m_q^\pi/3\rfloor)\le\frac{4R}{m}$ for $m>2R$), classes with $N_r\le2$ contribute nothing, and the number of collisions follows $\mathrm{Hypergeometric}(N-3,N_r-3,k-1)$ under $P_r$ and, up to moving its mass on $\{\barkappa\ge N_r-2\}$ to the law under $P_r$, $\mathrm{Hypergeometric}(N-1,N_r-1,k-1)$ under $Q_r$. For $N_r\ge 2\barkappa+3$, the ratio of these two probability mass functions at $\barkappa$ is $\frac{(N_r-1)(N_r-2)}{(N_r-\barkappa-1)(N_r-\barkappa-2)}\cdot\frac{(N-k)(N-k-1)}{(N-1)(N-2)}\in\big[1-\frac{2k}{N-2},1+\frac{9k}{N_r}\big]$, and the remaining terms are bounded by Markov's inequality as above, which gives:
    \begin{align*}
        \sup_{f\in\mathcal{F}}\left|U_\Lambda(f)-\bar U_\Lambda(f)\right|\le\mathcal{B}\left[\frac{4R}{m}+\frac{9Rk}{N}+\frac{2k}{N-2}+\frac{6k}{N-1}+\frac{2k}{N-3}\right].
    \end{align*}
    \noindent Since $m-2R\ge\frac{N}{k+2}$ when $N\ge 9Rk$, we have $\frac{R}{m}\le\frac{3Rk}{N}$, which gives the bound on $\|U_\Lambda-\bar U_\Lambda\|_\mathcal{F}$.
\end{proof}

\begin{remark}
    \label{remark:hypergeometric_dist_arguments}
    For each $r\in[R]$ with $N_r\ge2$, the probability mass functions $P_r, Q_r$ (Eqn.~\eqref{eq:pmf_two_procedures}) correspond to two selection procedures.
    \begin{enumerate}
        \item \textbf{Procedure $(1)$ (Corresponding to $P_r$)}: From the population of $N$ data points
        \begin{itemize}
            \item We first choose two instances from $S_r$ as anchor-positive.
            \item Then, we choose $k$ elements from the remaining $N-2$ data points as negatives.
        \end{itemize}

        \item \textbf{Procedure $(2)$ (Corresponding to $Q_r$)}: From the population of $N$ data points
        \begin{itemize}
            \item We first choose the anchor $X$ uniformly from $S_r$ (this is the effect of the weights $\widehat\rho_r^\pi/\widehat\rho_r$ in Eqn.~\eqref{eq:pmf_two_procedures}), then $n-1$ data points uniformly without replacement from $S\setminus\{X\}$; together with $X$, they form the set $S^{(n)}$.
            \item Then, we choose the positive uniformly from $(S^{(n)}\cap S_r)\setminus\{X\}$.
            \item Then, we choose $k$ elements from the $N-n$ data points, excluding everything in $S^{(n)}$, as negatives.
            \item If $(S^{(n)}\cap S_r)\setminus\{X\}=\emptyset$, the positive is instead chosen (after the negatives) uniformly from the class-$r$ points that are neither $X$ nor negatives; if there is none, the whole tuple is chosen by Procedure $(1)$.
        \end{itemize}
    \end{enumerate}
    \noindent Obviously, the number of class-collisions in $(1)$ follows the $\mathrm{Hypergeometric}(N-2,N_r-2,k)$ distribution because after the anchor-positive is chosen, the negatives are selected from a population of $N-2$ items with $N_r-2$ ``special items" (excluding the chosen anchor-positive).

    \noindent On the other hand, given the anchor $X$, the negatives of procedure $(2)$ are $k$ elements drawn uniformly without replacement from $S\setminus\{X\}$, i.e., from a population of $N-1$ items with $N_r-1$ special items, and, by symmetry, the positive is uniform among the class-$r$ points that are neither $X$ nor negatives. Hence, writing $h_{r,\barkappa}$ for the probability that a $\mathrm{Hypergeometric}(N-1,N_r-1,k)$ variable equals $\barkappa$ and $e_r:=h_{r,N_r-1}$ (which is zero when $k\le N_r-2$), we have $\P_{t\sim Q_r}(\sharp_\mathrm{col}^r(t)=\barkappa)=h_{r,\barkappa}+e_r\P_{t\sim P_r}(\sharp_\mathrm{col}^r(t)=\barkappa)$ for $\barkappa\le N_r-2$. That is, apart from the removal of the anchor, the ``split-first" step does not affect the distribution of class-collisions.
\end{remark}

\begin{remark}(Stochastic Dominance of Hypergeometric Variables)
    \label{remark:hypergeometric_coupling}
    If we have two hypergeometric variables $X\sim\mathrm{Hypergeometric}(N, M, k)$ and $Y\sim\mathrm{Hypergeometric}(N-2, M-2, k)$ where $M\ge 3$. Then, we can show $X\succcurlyeq_\mathrm{st}Y$ via a simple coupling:
    \begin{enumerate}
        \item We start with a bag of $N$ balls with $M$ special balls and we mark two random special balls as ``redundant". Next, we draw $k+2$ balls without replacement from the $N$ balls.
        \item To sample $X$, we take the first $k$ balls in the previously selected $k+2$ balls.
        \item To sample $Y$, we follow the following procedure:
        \begin{itemize}
            \item If the first $k$ balls did not contain the ``redundant" ones, it is also a valid draw of $Y$.
            \item If one ``redundant" showed up in the first $k$ draws, replace that with one of the last two balls (note that there must be one non-redundant balls among the last two).
            \item If two ``redundant showed up in the first $k$ draws, replace both with the last two balls.
        \end{itemize}
    \end{enumerate}
    \noindent Via this coupling, we will always have $X\ge Y$. Therefore, $X$ stochastically dominates $Y$.
\end{remark}

\section{Concentration of the Auxiliary Estimators}
In this section, we will prove that the proposed auxiliary estimators $\bar U_\Omega(f)$ and $\bar U_\Lambda(f)$ concentrate well to their respective targets. To do so, we introduce some definitions of auxiliary risks corresponding to $\Lnorm_\Omega(f)$ and $\Lnorm_\Lambda(f)$.

\subsection{Auxiliary Risks for \texorpdfstring{$\Lnorm_\Omega(f)$}{}}
Firstly, we introduce the following notations:
\begin{align}
    S^{(n)} &= \set{X_j}_{j=1}^n, \textup{ and } \forall r\in[R]: \begin{cases}
    n_r &:= \left|S_r\cap S^{(n)}\right|, \\ 
    \omega_r &:=\lfloor n_r/2\rfloor\left[\sum_{q=1}^R\lfloor n_q/2\rfloor\right]^{-1}, \\ \overline{\omega}_r &:=\E[\omega_r].
    \end{cases}
\end{align}
\noindent Furthermore, for every $\pi\in\Pi$, we denote:
\begin{align}
    S_\pi^{(n)} &= \set{X_{\pi(j)}}_{j=1}^n, \textup{ and }\forall r\in[R]:\begin{cases}
        n_r^\pi &:= |S_r\cap S_\pi^{(n)}|, \\
        \omega_r^\pi &:= \lfloor n_r^\pi/2\rfloor\left[\sum_{q=1}^R\lfloor n_q^\pi/2\rfloor\right]^{-1}, \\
        \omega_r^* &:= \widehat{\E}_\Pi[\omega_r^\pi].
    \end{cases}
\end{align}
\noindent Then, we define the following auxiliary risks:
\begin{align}
    \Lnorm_\Omega^*(f) &:= \sum_{r=1}^R \omega_r^*\Lnorm^r_\Omega(f), \quad \overline{\Lnorm}_\Omega(f):=\sum_{r=1}^R\overline{\omega}_r\Lnorm^r_\Omega(f). \label{eq:Lbar_Omega} \\
    \widehat{\Lnorm}_\Omega(f|\pi) &:= \sum_{r=1}^R\omega_r^\pi\Lnorm_\Omega^r(f), \quad\widehat{\Lnorm}_\Omega(f) := \sum_{r=1}^R\omega_r\Lnorm_\Omega^r(f). \label{eq:Lhat_Omega}
\end{align}
\noindent From the definitions, we can observe that  $\Lnorm^*_\Omega(f) = \widehat{\E}_\Pi\left[\widehat\Lnorm_\Omega(f|\pi)\right]$, $\overline{\Lnorm}_\Omega(f) = \E[\widehat{\Lnorm}_\Omega(f)]=\E\left[\widehat{\Lnorm}_\Omega(f|\pi)\right]$ for any fixed permutation $\pi\in\Pi$ and $\widehat{\Lnorm}_\Omega(f)=\widehat\Lnorm_\Omega(f|\mathrm{id})$ where $\mathrm{id}\in\Pi$ is the identity permutation (i.e., $\mathrm{id}(j)=j,\forall j\in[N]$). Our proof strategy can be outlined in the below diagram:
\begin{align}
    \label{eq:proof_strategy}
    \boxed{\underset{\textup{Eqn.~\eqref{eq:bar_Uomega_Ulambda}}}{\bar U_\Omega(f)} \xrightarrow{\textup{c.t.}} 
    \underset{\textup{Eqn.~\eqref{eq:Lbar_Omega}}}{\Lnorm^*_\Omega(f)}
    \xrightarrow{\textup{c.t.}}
    \underset{\textup{Eqn.~\eqref{eq:Lbar_Omega}}}{\overline{\Lnorm}_\Omega(f)}
    \xrightarrow{\textup{c.t.}}
    \underset{\textup{Eqn.~\eqref{eq:lomega}}}{\Lnorm_\Omega(f)}}.
\end{align}
\noindent Where the arrows $\xrightarrow{\textup{c.t.}}$ depicts the ``concentrates to" relationship between risks/estimators. 

\subsection{Auxiliary Risks for \texorpdfstring{$\Lnorm_\Lambda(f)$}{}}
Similar to the previous section, we introduce the following notations:
\begin{align}
    S^{(m)} &= \set{X_j}_{j=1}^m, \textup{ and } \forall r\in[R]: \begin{cases}
    m_r &:= \left|S_r\cap S^{(m)}\right|, \\ 
    \lambda_r &:=\lfloor m_r/3\rfloor\left[\sum_{q=1}^R\lfloor m_q/3\rfloor\right]^{-1}, \\ 
    \overline{\lambda}_r &:=\E[\lambda_r].
    \end{cases}
\end{align}
\noindent Furthermore, for every $\pi\in\Pi$, we denote:
\begin{align}
    S_\pi^{(m)} &= \set{X_{\pi(j)}}_{j=1}^m, \textup{ and }\forall r\in[R]:\begin{cases}
        m_r^\pi &:= |S_r\cap S_\pi^{(m)}|, \\
        \lambda_r^\pi &:= \lfloor m_r^\pi/3\rfloor\left[\sum_{q=1}^R\lfloor m_q^\pi/3\rfloor\right]^{-1}, \\
        \lambda_r^* &:= \widehat{\E}_\Pi[\lambda_r^\pi].
    \end{cases}
\end{align}
\noindent Then, we define the following auxiliary risks:
\begin{align}
    \Lnorm_\Lambda^*(f) &:= \sum_{r=1}^R \lambda_r^*\Lnorm^r_\Lambda(f), \quad \overline{\Lnorm}_\Lambda(f):=\sum_{r=1}^R\overline{\lambda}_r\Lnorm^r_\Lambda(f). \label{eq:Lbar_Lambda} \\
    \widehat{\Lnorm}_\Lambda(f|\pi) &:= \sum_{r=1}^R\lambda_r^\pi\Lnorm_\Lambda^r(f), \quad\widehat{\Lnorm}_\Lambda(f) := \sum_{r=1}^R\lambda_r\Lnorm_\Lambda^r(f). \label{eq:Lhat_Lambda}
\end{align}
\noindent Similarly, we have $\Lnorm^*_\Lambda(f) = \widehat{\E}_\Pi\left[\widehat\Lnorm_\Lambda(f|\pi)\right]$, $\overline{\Lnorm}_\Lambda(f) = \E[\widehat{\Lnorm}_\Lambda(f)]=\E\left[\widehat{\Lnorm}_\Lambda(f|\pi)\right]$ for any $\pi\in\Pi$ and $\widehat{\Lnorm}_\Lambda(f)=\widehat\Lnorm_\Lambda(f|\mathrm{id})$ where $\mathrm{id}\in\Pi$ is the identity permutation (i.e., $\mathrm{id}(j)=j,\forall j\in[N]$). We will prove that $\bar U_\Lambda(f)$ concentrates to $\Lnorm_\Lambda(f)$ by proving the following chain of concentration:
\begin{align}
    \label{eq:proof_strategy_2}
    \boxed{\underset{\textup{Eqn.~\eqref{eq:bar_Uomega_Ulambda}}}{\bar U_\Lambda(f)} \xrightarrow{\textup{c.t.}} 
    \underset{\textup{Eqn.~\eqref{eq:Lbar_Lambda}}}{\Lnorm^*_\Lambda(f)}
    \xrightarrow{\textup{c.t.}}
    \underset{\textup{Eqn.~\eqref{eq:Lbar_Lambda}}}{\overline{\Lnorm}_\Lambda(f)}
    \xrightarrow{\textup{c.t.}}
    \underset{\textup{Eqn.~\eqref{eq:llambda}}}{\Lnorm_\Lambda(f)}}.
\end{align}
\noindent Where the arrows $\xrightarrow{\textup{c.t.}}$ depicts the ``concentrates to" relationship between risks/estimators. 

\subsection{Concentration of \texorpdfstring{$\Lnorm^*_\Omega(f)$}{} to \texorpdfstring{$\Lnorm_\Omega(f)$}{}}
\begin{lemma}
    \label{lem:supporting_lemma_3}
    Let $\mathcal{F}$ be a class of representation functions. For each $f\in\mathcal{F}$, let $\overline{\Lnorm}_\Omega(f)$ be defined in Eqn.~\eqref{eq:Lbar_Omega} and $\Lnorm_\Omega(f)$ be defined in Eqn.~\eqref{eq:lomega}. Then, we have:
    \begin{align}
        \sup_{f\in\mathcal{F}}\left|\overline{\Lnorm}_\Omega(f)-\Lnorm_\Omega(f)\right| &\le \frac{2\mathcal{B}R}{n-R}.
    \end{align}
\end{lemma}

\begin{proof}
    We have:
    \begin{align*}
        \sup_{f\in\mathcal{F}}\left|\overline{\Lnorm}_\Omega(f)-\Lnorm_\Omega(f)\right| &\le \sup_{f\in\mathcal{F}}\left|\sum_{r=1}^R\overline{\omega}_r\Lnorm_\Omega^r(f)-\sum_{r=1}^R\rho_r\Lnorm_\Omega^r(f)\right| \le \sum_{r=1}^R\sup_{f\in\mathcal{F}}\left|\Lnorm_\Omega^r(f)(\overline{\omega}_r-\rho_r)\right| \\
        &\le \sum_{r=1}^R|\overline{\omega}_r-\rho_r|\cdot\sup_{f\in\mathcal{F}}\left|\Lnorm_\Omega^r(f)\right|\\
        &\le \mathcal{B}\sum_{r=1}^R|\overline{\omega}_r-\rho_r|.
    \end{align*}
    \noindent Note that $\overline{\omega}_r = \E[\omega_r^\pi]$ for any $\pi\in\Pi$\footnote{Because for any $\pi\in\Pi$, $n_r^\pi\sim\mathrm{Bin}(n, \rho_r)$. In other words, $n_r^\pi$ has the same distribution no matter which permutation $\pi$ we choose.}. Therefore, we can write $\overline{\omega}_r=\E\left[\widehat\E_\Pi[\omega_r^\pi]\right]$. As a result, for each $r\in[R]$, we have:
    \begin{align*}
        |\overline{\omega}_r-\rho_r| &= \left|\E\left[\widehat{\E}_\Pi[\omega_r^\pi]\right] - \E\left[\widehat\rho_r\right]\right| \le \E\left|\widehat\E_\Pi[\omega_r^\pi] - \widehat\rho_r\right|.
    \end{align*}
    \noindent From Lemma \ref{lem:supporting_lemma}, we showed that $\sum_{r=1}^R |\omega_r^\pi - \widehat\rho_r^\pi|\le \frac{2R}{n-R}$. Therefore:
    \begin{align*}
        \sum_{r=1}^R\left|\widehat\E_\Pi[\omega_r^\pi] - \widehat\rho_r\right| &= \sum_{r=1}^R\left|\widehat\E_\Pi[\omega_r^\pi] - \widehat\E_\Pi[\widehat\rho_r^\pi]\right| \le \sum_{r=1}^R\widehat\E_\Pi\left[\left|\omega_r^\pi - \widehat\rho_r^\pi\right|\right] = \widehat\E_\Pi\left[\sum_{r=1}^R|\omega_r^\pi-\widehat\rho_r^\pi|\right] \le \frac{2R}{n-R}.
    \end{align*}
    \noindent Finally, we have:
    \begin{align*}
        \sup_{f\in\mathcal{F}}\left|\overline{\Lnorm}_\Omega(f)-\Lnorm_\Omega(f)\right| &\le \mathcal{B}\sum_{r=1}^R|\overline{\omega}_r-\rho_r| \le \mathcal{B}\sum_{r=1}^R\E\left|\widehat\E_\Pi[\omega_r^\pi] -\widehat\rho_r\right| \\
        &= \mathcal{B}\E\left[\sum_{r=1}^R\Big|\widehat\E_\Pi[\omega_r^\pi]-\widehat\rho_r\Big|\right]\le \frac{2\mathcal{B}R}{n-R},
    \end{align*}
    \noindent as desired.
\end{proof}

\begin{lemma}
    \label{lem:supporting_lemma_4}
    Let $\mathcal{F}$ be a class of representation functions. For each $f\in\mathcal{F}$, let $\overline{\Lnorm}_\Omega(f), \Lnorm^*_\Omega(f)$ be defined in Eqn.~\eqref{eq:Lbar_Omega}. Then, for any $\Delta\in(0,1)$ we have:
    \begin{align}
        \sup_{f\in\mathcal{F}}\left|\overline{\Lnorm}_\Omega(f)-\Lnorm^*_\Omega(f)\right| &\le \mathcal{B}\left[\frac{4R}{n-R}+\frac{8R}{3N}\ln2R/\Delta + \sqrt{\frac{8(R-1)\ln 2R/\Delta}{3N}}\right],
    \end{align}
    \noindent with probability of at least $1-\Delta$.
\end{lemma}

\begin{proof}
    Similar to the initial analysis of Lemma \ref{lem:supporting_lemma_3}, we have:
    \begin{align*}
        \sup_{f\in\mathcal{F}}\left|\overline{\Lnorm}_\Omega(f)-\Lnorm^*_\Omega(f)\right| &\le \sup_{f\in\mathcal{F}}\left|\sum_{r=1}^R\overline{\omega}_r\Lnorm_\Omega^r(f)-\sum_{r=1}^R\omega_r^*\Lnorm_\Omega^r(f)\right| \le \sum_{r=1}^R\sup_{f\in\mathcal{F}}\left|\Lnorm_\Omega^r(f)(\overline{\omega}_r-\omega_r^*)\right| \\ &\le \sum_{r=1}^R|\overline{\omega}_r-\omega^*_r|\cdot\sup_{f\in\mathcal{F}}\left|\Lnorm_\Omega^r(f)\right|\\
        &\le \mathcal{B}\sum_{r=1}^R|\overline{\omega}_r-\omega_r^*|.
    \end{align*}
    \noindent Then, for each $r\in[R]$, we have:
    \begin{align*}
        |\overline{\omega}_r-\omega_r^*|  &\le |\omega_r^*-\widehat\rho_r|+|\widehat\rho_r-\rho_r|+|\overline{\omega}_r-\rho_r| \\
        &= \left|\widehat\E_\Pi[\omega_r^\pi] - \widehat\rho_r\right| + \left|\E\left[\widehat{\E}_\Pi[\omega_r^\pi]\right] - \E\left[\widehat\rho_r\right]\right| + |\widehat\rho_r-\rho_r| \\
        &\le \left|\widehat\E_\Pi[\omega_r^\pi] - \widehat\rho_r\right| + \E\left|\widehat\E_\Pi[\omega_r^\pi] - \widehat\rho_r\right| + |\widehat\rho_r-\rho_r|.
    \end{align*}
    \noindent As a result, we have:
    \begin{align*}
      	\sum_{r=1}^R|\overline{\omega}_r-\omega_r^*|&\le \sum_{r=1}^R\left|\widehat\E_\Pi[\omega_r^\pi] - \widehat\rho_r\right| + \sum_{r=1}^R\E\left|\widehat\E_\Pi[\omega_r^\pi] - \widehat\rho_r\right| + \sum_{r=1}^R|\widehat\rho_r-\rho_r| \\ 
      	&\le \frac{2R}{n-R}+\frac{2R}{n-R}+\sum_{r=1}^R|\widehat\rho_r-\rho_r| = \frac{4R}{n-R}+\sum_{r=1}^R|\widehat\rho_r-\rho_r|.
    \end{align*}
    \noindent The second-to-last inequality comes from $\sum_{r=1}^R\left|\widehat\E_\Pi[\omega_r^\pi] - \widehat\rho_r\right|\le\frac{2R}{n-R}$, which we proved in Lemma \ref{lem:supporting_lemma} and Lemma \ref{lem:supporting_lemma_3}. For each $r\in[R]$, we can write:
    \begin{align*}
        \widehat\rho_r:=\frac{N_r}{N} = \frac{1}{N}\sum_{j=1}^N B_{r, j},\quad\textup{where } B_{r,j}\sim\mathrm{Bern}(\rho_r),\forall j\in[N].
    \end{align*}
    \noindent Then, by the high-probability Bernstein bound, for all $\delta\in(0,1)$ we have the following inequality with probability of at least $1-\delta$:
    \begin{align*}
        |\widehat\rho_r-\rho_r|&=\left|\frac{1}{N}\sum_{j=1}^N[B_{r,j} - \rho_r]\right| \le \frac{8}{3N}\ln 2/\delta + \sigma_r\sqrt{\frac{8\ln 2/\delta}{3N}},
    \end{align*}
    \noindent where $\sigma_r^2$ is the following mean of second moments:
    \begin{align*}
        \sigma_r^2 &= \frac{1}{N}\sum_{j=1}^N \E\left[(B_{r,j}-\rho_r)^2\right]= \frac{1}{N}\sum_{j=1}^N\mathrm{Var}(B_{r,j}) = \rho_r(1-\rho_r).
    \end{align*}
    \noindent Then, by the union bound, with probability of at least $1-R\delta$, we have:
    \begin{align*}
        \sum_{r=1}^R|\widehat\rho_r-\rho_r| &\le \frac{8R}{3N}\ln 2/\delta + \sqrt{\frac{8\ln 2/\delta}{3N}}\sum_{r=1}^R\sqrt{\rho_r(1-\rho_r)} \\
        &\le \frac{8R}{3N}\ln 2/\delta + \sqrt{\frac{8\ln 2/\delta}{3N}}\sqrt{\left(\sum_{r=1}^R\rho_r\right)\left(\sum_{r=1}^R(1-\rho_r)\right)} \quad\textup{(Cauchy-Schwarz)} \\
        &\le \frac{8R}{3N}\ln 2/\delta + \sqrt{\frac{8(R-1)\ln 2/\delta}{3N}}.
    \end{align*}
    \noindent For a fixed $\Delta\in(0,1)$, by setting $\delta=\Delta/R$, we have:
    \begin{align*}
        \sum_{r=1}^R|\widehat\rho_r-\rho_r| \le \frac{8R}{3N}\ln2R/\Delta + \sqrt{\frac{8(R-1)\ln 2R/\Delta}{3N}},
    \end{align*}
    \noindent with probability of at least $1-\Delta$. As a result:
    \begin{align*}
        \sup_{f\in\mathcal{F}}\left|\overline{\Lnorm}_\Omega(f)-\Lnorm^*_\Omega(f)\right| &\le \mathcal{B}\sum_{r=1}^R|\overline{\omega}_r-\omega_r^*| \le \mathcal{B}\left[\frac{4R}{n-R} + \sum_{r=1}^R|\widehat\rho_r-\rho_r|\right] \\
        &\le \mathcal{B}\left[\frac{4R}{n-R}+\frac{8R}{3N}\ln2R/\Delta + \sqrt{\frac{8(R-1)\ln 2R/\Delta}{3N}}\right],
    \end{align*}
    \noindent with probability of at least $1-\Delta$, as desired.
\end{proof}

\begin{proposition}
    \label{prop:supporting_proposition_1}
    Let $\mathcal{F}$ be a class of representation functions. For each $f\in\mathcal{F}$, let $\Lnorm_\Omega(f)$ be defined in Eqn.~\eqref{eq:lomega} and $\Lnorm^*_\Omega(f)$ be defined in Eqn.~\eqref{eq:Lbar_Omega}. Then, for any $\Delta\in(0,1)$, we have:
    \begin{align}
        \sup_{f\in\mathcal{F}}\left|{\Lnorm}^*_\Omega(f)-\Lnorm_\Omega(f)\right| &\le \mathcal{B}\left[\frac{6R}{n-R}+\frac{8R}{3N}\ln2R/\Delta + \sqrt{\frac{8(R-1)\ln 2R/\Delta}{3N}}\right],
    \end{align}    
    \noindent with probability of at least $1-\Delta$.
\end{proposition}

\begin{proof}
    This is a simple consequence of Lemma \ref{lem:supporting_lemma_3} and Lemma \ref{lem:supporting_lemma_4}. For any $\Delta\in(0,1)$, we have:
    \begin{align*}
        \sup_{f\in\mathcal{F}}\left|{\Lnorm}^*_\Omega(f)-\Lnorm_\Omega(f)\right| 
        &\le \underbrace{\sup_{f\in\mathcal{F}}\left|{\Lnorm}^*_\Omega(f)-\overline{\Lnorm}_\Omega(f)\right|}_{\textup{Lemma \ref{lem:supporting_lemma_4}}} + \underbrace{\sup_{f\in \mathcal{F}}\left|\overline{\Lnorm}_\Omega(f)-\Lnorm_\Omega(f)\right|}_{\textup{Lemma \ref{lem:supporting_lemma_3}}} \\
        &\le \mathcal{B}\left[\frac{4R}{n-R}+\frac{8R}{3N}\ln2R/\Delta + \sqrt{\frac{8(R-1)\ln 2R/\Delta}{3N}}\right] + \frac{2\mathcal{B}R}{n-R} \\
        &= \mathcal{B}\left[\frac{6R}{n-R}+\frac{8R}{3N}\ln2R/\Delta + \sqrt{\frac{8(R-1)\ln 2R/\Delta}{3N}}\right],
    \end{align*}
    \noindent with probability of at least $1-\Delta$, as desired.
\end{proof}

\subsection{Concentration of \texorpdfstring{$\bar U_\Omega(f)$}{} to \texorpdfstring{$\Lnorm_\Omega^*(f)$}{}}
\begin{lemma}
    \label{lem:supporting_lemma_5}
    Let $\mathcal{F}$ be a class of representation functions. For each $f\in\mathcal{F}$, let $\Lnorm^*_\Omega(f)$ be defined in Eqn.~\eqref{eq:Lbar_Omega} and $\bar U_\Omega(f)$ be defined in Eqn.~\eqref{eq:bar_Uomega_Ulambda}. Then, for any convex and non-decreasing $\varphi:\R\to\R$, we have:
    \begin{align}
        \E\varphi\left(\sup_{f\in\mathcal{F}}\left|\bar U_\Omega(f)-\Lnorm^*_\Omega(f)\right|\right) \le \E\varphi\left(\sup_{f\in\mathcal{F}}\left|\sum_{r=1}^R\omega_r U_{\Omega_r^\mathrm{id}}(f)-\widehat{\Lnorm}_\Omega(f)\right|\right).
    \end{align}
\end{lemma}

\begin{proof}
    For notational brevity, we denote:
    \begin{align*}
        \forall\pi\in\Pi:\quad\widehat U_\Omega(f|\pi) &= \sum_{r=1}^R \omega_r^\pi U_{\Omega_r^\pi}(f),\quad \text{and}\quad\widehat{\Lnorm}_\Omega(f|\pi) := \sum_{r=1}^R\omega_r^\pi\Lnorm_\Omega^r(f).
    \end{align*}
    \noindent Equivalently, we have $\bar U_\Omega(f) :=\widehat\E_\Pi\left[\widehat U_\Omega(f|\pi)\right]$ and $\Lnorm^*_\Omega(f):=\widehat\E_\Pi\left[\widehat{\Lnorm}_\Omega(f|\pi)\right]$. Then, we have:
    \begin{align*}
        \E\varphi\Big(\sup_{f\in\mathcal{F}}\Big|\bar U_\Omega(f) -\Lnorm^*_\Omega(f)\Big|\Big) 
        &= \E\varphi\left(\sup_{f\in\mathcal{F}}\left|\widehat\E_\Pi\left[\widehat U_\Omega(f|\pi)\right]-\widehat\E_\Pi\left[\widehat{\Lnorm}_\Omega(f|\pi)\right]\right|\right) \\
        &= \E\varphi\left(\sup_{f\in\mathcal{F}}\left|\widehat\E_\Pi\left[\widehat U_\Omega(f|\pi)-\widehat{\Lnorm}_\Omega(f|\pi)\right]\right|\right) \\
        &\le \E\varphi\left(\widehat\E_\Pi\left[\sup_{f\in\mathcal{F}}\left|\widehat{U}_\Omega(f|\pi)-\widehat{\Lnorm}_\Omega(f|\pi)\right|\right]\right) \quad \left(\sup\left|\widehat\E_\Pi[ \cdot ]\right| \le \widehat\E_\Pi\left[\sup|  \cdot  |\right]\right) \\
        &\le \widehat\E_\Pi\left[\E\varphi\left(\sup_{f\in\mathcal{F}}\left|\widehat{U}_\Omega(f|\pi)-\widehat{\Lnorm}_\Omega(f|\pi)\right|\right)\right] \quad\textup{(Jensen's Inequality)}.
    \end{align*}
    \noindent For any two distinct permutations $\pi_1,\pi_2\in\Pi$, we always have:
    \begin{align*}
        S_{\pi_1}^{(n)} \overset{d}{=} S_{\pi_2}^{(n)}, \quad\textup{and}\quad S\setminus S_{\pi_1}^{(n)} \overset{d}{=}S\setminus S_{\pi_2}^{(n)}.
    \end{align*}
    \noindent Where $X\overset{d}{=}Y$ means ``$X$ is identically distributed as $Y$". Therefore:
    \begin{align*}
        \forall\pi_1,\pi_2\in\Pi: \widehat U_\Omega(f|\pi_1)\overset{d}{=} \widehat U_\Omega(f|\pi_2)\quad\textup{and}\quad \widehat\Lnorm_\Omega(f|\pi_1)\overset{d}{=} \widehat \Lnorm_\Omega(f|\pi_2).
    \end{align*}
    \noindent In other words, for any $\pi_1,\pi_2\in\Pi$:
    \begin{align*}
        \E\varphi\left(\sup_{f\in\mathcal{F}}\left|\widehat{U}_\Omega(f|\pi_1)-\widehat{\Lnorm}_\Omega(f|\pi_1)\right|\right)=\E\varphi\left(\sup_{f\in\mathcal{F}}\left|\widehat{U}_\Omega(f|\pi_2)-\widehat{\Lnorm}_\Omega(f|\pi_2)\right|\right).
    \end{align*}
    \noindent Therefore, we have:
    \begin{align*}
        \E\varphi\left(\sup_{f\in\mathcal{F}}\Big|\bar U_\Omega(f) -\Lnorm^*_\Omega(f)\Big|\right) &\le \widehat\E_\Pi\left[\E\varphi\left(\sup_{f\in\mathcal{F}}\left|\widehat{U}_\Omega(f|\pi)-\widehat{\Lnorm}_\Omega(f|\pi)\right|\right)\right] \\ &= \E\varphi\left(\sup_{f\in\mathcal{F}}\left|\widehat{U}_\Omega(f|\mathrm{id})-\widehat{\Lnorm}_\Omega(f|\mathrm{id})\right|\right) \\
        &= \E\varphi\left(\sup_{f\in\mathcal{F}}\left|\sum_{r=1}^R\omega_r U_{\Omega_r^\mathrm{id}}(f)-\widehat{\Lnorm}_\Omega(f)\right|\right),
    \end{align*}
    \noindent as desired.
\end{proof}

\noindent From the above lemma, we can reduce the concentration analysis problem of $\bar U_\Omega(f)$ around $\Lnorm^*_\Omega(f)$ to analyzing the concentration of $\sum_{r=1}^R\omega_rU_{\Omega_r^\mathrm{id}}(f)$ around $\widehat\Lnorm_\Omega(f)$. From here on, we denote the truncated datasets $S^{(n)}=\set{X_j}_{j=1}^n$ and $\bar S^{(n)} = \set{Z_j}_{j=1}^{N-n}$ where $Z_j=X_{n+j}$ for all $1\le j\le N-n$.

\noindent For notational brevity, we first denote:
\begin{align}
    \label{eq:Uhat_Omega}
    \widehat U_\Omega(f) := \sum_{r=1}^R\omega_r U_{\Omega_r^\mathrm{id}}(f).
\end{align}
Then, we note that $\widehat U_\Omega(f)$ can also be re-written as follows:
\begin{align}
    \label{eq:Uhat_Omega_alternative_form} 
    \widehat U_\Omega(f) &= \frac{1}{n!(N-n)!}\sum_{\substack{\xi\in\Pi[n]\\ \zeta\in\Pi[N-n]}}\frac{1}{\bar n}\sum_{r=1}^R\sum_{j=1}^{\bar n_r}\ell_{\phi,f}\left(X_{2j-1}^{r,\xi},X_{2j}^{r, \xi},\set{Z^\zeta_{k(j+ n_r^*-1)+i}}_{i=1}^k \right),
\end{align}
\noindent where $\bar n_r=\lfloor n_r/2\rfloor$, $\bar n=\sum_{r=1}^R\bar n_r$, $n_r^*=\sum_{q=1}^{r-1}\bar n_q$ and $\Pi[n], \Pi[N-n]$ denotes the sets of all bijections $[n]\to[n]$ and $[N-n]\to[N-n]$, respectively. We also use the following notations for brevity:
\begin{enumerate}[label=(\roman*)]
    \item $X^r_u$ denotes the $u^\mathrm{th}$ data-point that belongs to class $r\in[R]$ in $\set{X_{j}}_{j=1}^n$. 
    \item $X^{r, \xi}_u$ denotes the $u^\mathrm{th}$ data-point that belongs to class $r\in[R]$ in $\set{X_{\xi(j)}}_{j=1}^n$. 
    \item $Z^{\zeta}_{u}=Z_{\zeta(u)}$, i.e., $\big\{Z_{ku+i}^\zeta\big\}_{i=1}^k$ is the $(u+1)$-th block of $k$-elements extracted in order from the shuffled set $\set{Z_{\zeta(j)}}_{j=1}^{N-n}$.
\end{enumerate}

\noindent Furthermore, for the proofs of subsequent results, we also define the following dataset of tuples:
\begin{align}
    \label{eq:independent_tuples_dataset} 
    S_\mathrm{tup} := \bigcup_{r=1}^R\set{\left(X_{2j-1}^r, X_{2j}^r, \set{Z_{k(j+n_r^*-1)+i}}_{i=1}^k\right)}_{j=1}^{\bar n_r}.
\end{align}

\begin{lemma}
    \label{lem:supporting_lemma_6}
    Let $\mathcal{F}$ be a class of representation functions. For each $f\in\mathcal{F}$, let  $\widehat U_\Omega(f)$ be the estimator defined in Eqn.~\eqref{eq:Uhat_Omega} (or Eqn.~\eqref{eq:Uhat_Omega_alternative_form}) and $\widehat{\Lnorm}_\Omega(f)$ be defined in Eqn.~\eqref{eq:Lhat_Omega}. Then, for any convex and non-decreasing $\varphi:\R\to\R$, we have:
    \begin{align}
        \E_S\varphi\left(\sup_{f\in\mathcal{F}}\left|\widehat U_\Omega(f)-\widehat{\Lnorm}_\Omega(f)\right|\right) &\le \E_{S_\mathrm{tup}, \Rad_{\bar n}}\varphi\left(2\mathcal{R}_\mathcal{H}^{S_\mathrm{tup}, \Rad_{\bar n}}\right),
    \end{align}
    \noindent where $\Rad_{\bar n}=\set{\sigma_j}_{j=1}^{\bar n}$ is a sequence of independent Rademacher variables, $\mathcal{H}$ is defined in Eqn.~\eqref{eq:loss_class} and the random variable $\mathcal{R}_\mathcal{H}^{S_\mathrm{tup}, \Rad_{\bar n}}$ is defined as follows:
    \begin{align}
        \mathcal{R}_\mathcal{H}^{S_\mathrm{tup}, \Rad_{\bar n}} :=\sup_{f\in\mathcal{F}}\left|\frac{1}{\bar n}\sum_{r=1}^R\sum_{j=1}^{\bar n_r}\sigma_{j+n_r^*}\ell_{\phi, f}\left(X_{2j-1}^r, X_{2j}^r,\set{Z_{k(j+n_r^*-1)+i}}_{i=1}^k\right)\right|.
    \end{align}
\end{lemma}

\begin{proof}
    For brevity, we denote $C=n!(N-n)!$ and $\Pi_{n, N}=\Pi[n]\times\Pi[N-n]$. Then, we have:    
    \begin{align*}
        &\E_S\varphi\left(\sup_{f\in\mathcal{F}}\left|\widehat U_\Omega(f)-\widehat{\Lnorm}_\Omega(f)\right|\right) \\ 
        &= \E_S\varphi\left(\sup_{f\in\mathcal{F}}\left|\frac{1}{C}\sum_{\xi, \zeta\in\Pi_{n,N}}\frac{1}{\bar n}\sum_{r=1}^R\sum_{j=1}^{\bar n_r}\ell_{\phi,f}\left(X_{2j-1}^{r,\xi},X_{2j}^{r, \xi},\set{Z^\zeta_{k(j+ n_r^*-1)+i}}_{i=1}^k \right)-\widehat{\Lnorm}_\Omega(f)\right|\right) \\
        &\le \E_S\varphi\left(\frac{1}{C}\sum_{\xi,\zeta\in\Pi_{n,N}}\sup_{f\in\mathcal{F}}\left|\frac{1}{\bar n}\sum_{r=1}^R\sum_{j=1}^{\bar n_r}\ell_{\phi,f}\left(X_{2j-1}^{r,\xi},X_{2j}^{r, \xi},\set{Z^\zeta_{k(j+ n_r^*-1)+i}}_{i=1}^k \right)-\widehat{\Lnorm}_\Omega(f)\right|\right) \\
        &\le \frac{1}{C}\sum_{\xi,\zeta\in\Pi_{n,N}}\E_S\varphi\left(\sup_{f\in\mathcal{F}}\left|\frac{1}{\bar n}\sum_{r=1}^R\sum_{j=1}^{\bar n_r}\ell_{\phi,f}\left(X_{2j-1}^{r,\xi},X_{2j}^{r, \xi},\set{Z^\zeta_{k(j+ n_r^*-1)+i}}_{i=1}^k \right) -\widehat{\Lnorm}_\Omega(f)\right|\right)\quad\textup{(Jensen's Inequality)}.
    \end{align*}
    \noindent For each $\xi\in\Pi[n]$ and $\zeta\in\Pi[N-n]$, we define $S_\mathrm{tup}^{\xi,\zeta}$ as the set of tuples:
    \begin{align}
        S_\mathrm{tup}^{\xi,\zeta} &:= \bigcup_{r=1}^R\set{\left(X_{2j-1}^{r,\xi}, X_{2j}^{r,\xi}, \set{Z^\zeta_{k(j+n_r^*-1)+i}}_{i=1}^k\right)}_{j=1}^{\bar n_r}, \\
        \textup{and } \Psi\left(S_\mathrm{tup}^{\xi,\zeta}\right) &= \frac{1}{\bar n}\sum_{r=1}^R\sum_{j=1}^{\bar n_r}\ell_{\phi,f}\left(X_{2j-1}^{r,\xi},X_{2j}^{r, \xi},\set{Z^\zeta_{k(j+ n_r^*-1)+i}}_{i=1}^k \right).
    \end{align}
    \noindent Then, for any $\xi_1,\xi_2\in\Pi[n]$ and $\zeta_1,\zeta_2\in\Pi[N-n]$, we will always have $S_\mathrm{tup}^{\xi_1,\zeta_1}\overset{d}{=}S_\mathrm{tup}^{\xi_2,\zeta_2}$. Therefore, we have:
    \begin{align*}
        \E_S\left[\Psi\left(S_\mathrm{tup}^{\xi_1,\zeta_1}\right)\right] = \E_S\left[\Psi\left(S_\mathrm{tup}^{\xi_2,\zeta_2}\right)\right].
    \end{align*}
    \noindent Therefore:
    \begin{align*}
        \E_S\varphi\left(\sup_{f\in\mathcal{F}}\left|\widehat U_\Omega(f)-\widehat{\Lnorm}_\Omega(f)\right|\right) &\le \frac{1}{C}\sum_{\xi,\zeta\in\Pi_{n,N}}\E_S\varphi\left(\sup_{f\in\mathcal{F}}\left|\Psi\left(S_\mathrm{tup}^{\xi, \zeta}\right)-\widehat{\Lnorm}_\Omega(f)\right|\right) \\
        &= \E_{S_\mathrm{tup}}\varphi\left(\sup_{f\in\mathcal{F}}\left|\Psi\left(S_\mathrm{tup}\right)-\widehat{\Lnorm}_\Omega(f)\right|\right).
    \end{align*}
    \noindent Now, for each $r\in[R]$, recall that $n_r=|S_r\cap S^{(n)}|$ and $n_r\sim\mathrm{Bin}(n,\rho_r)$. Let ${\bf n}=\set{n_r}_{1\le r\le R}$ be the vector of class-wise sample sizes in the first partition $S^{(n)}$. We observe that:
    \begin{align*}
        \E_{S_\mathrm{tup}|{\bf n}}\left[\Psi\left(S_\mathrm{tup}\right)\right] &= \E_{S_\mathrm{tup}|{\bf n}}\left[\frac{1}{\bar n}\sum_{r=1}^R\sum_{j=1}^{\bar n_r}\ell_{\phi,f}\left(X_{2j-1}^{r,\xi},X_{2j}^{r, \xi},\set{Z^\zeta_{k(j+ n_r^*-1)+i}}_{i=1}^k \right)\right] \\
        &= \frac{1}{\bar n}\sum_{r=1}^R\sum_{j=1}^{\bar n_r}\E_{S_\mathrm{tup}|{\bf n}}\left[\ell_{\phi,f}\left(X_{2j-1}^{r,\xi},X_{2j}^{r, \xi},\set{Z^\zeta_{k(j+ n_r^*-1)+i}}_{i=1}^k \right)\right] \\
        &= \sum_{r=1}^R\frac{\bar n_r}{\bar n} \Lnorm_\Omega^r(f) 
        = \sum_{r=1}^R\omega_r\Lnorm_\Omega^r(f) \\
        &= \widehat{\Lnorm}_\Omega(f).
    \end{align*}
    \noindent In other words, given the sample sizes $\bf n$, $\Psi(S_\mathrm{tup})$ has expectation $\widehat\Lnorm_\Omega(f)$. Therefore, we can apply the symmetrization trick (given ${\bf n}$) here. Specifically:
    \begin{align*}
        &\E_S\varphi\left(\sup_{f\in\mathcal{F}}\left|\widehat U_\Omega(f)-\widehat{\Lnorm}_\Omega(f)\right|\right) =\E_{\bf n}\E_{S|\bf n}\varphi\left(\sup_{f\in\mathcal{F}}\left|\widehat U_\Omega(f)-\widehat{\Lnorm}_\Omega(f)\right|\right) \\
        &\le \E_{\bf n}\E_{S_\mathrm{tup}|\bf n}\varphi\left(\sup_{f\in\mathcal{F}}\left|\Psi\left(S_\mathrm{tup}\right)-\widehat{\Lnorm}_\Omega(f)\right|\right) \\
        &\le \E_{\bf n}\E_{S_\mathrm{tup}, \Rad_{\bar n}|\bf n}\varphi\left(\sup_{f\in\mathcal{F}}\left|\frac{2}{\bar n}\sum_{r=1}^R\sum_{j=1}^{\bar n_r}\sigma_{j+n_r^*}\ell_{\phi, f}\left(X_{2j-1}^r, X_{2j}^r,\set{Z_{k(j+n_r^*-1)+i}}_{i=1}^k\right)\right|\right)\quad\textup{(Lemma \ref{lem:symmetrization_inequality})} \\
        &= \E_{S_\mathrm{tup}, \Rad_{\bar n}}\varphi\left(\sup_{f\in\mathcal{F}}\left|\frac{2}{\bar n}\sum_{r=1}^R\sum_{j=1}^{\bar n_r}\sigma_{j+n_r^*}\ell_{\phi, f}\left(X_{2j-1}^r, X_{2j}^r,\set{Z_{k(j+n_r^*-1)+i}}_{i=1}^k\right)\right|\right) \\
        &= \E_{S_\mathrm{tup}, \Rad_{\bar n}}\varphi\left(2\mathcal{R}_\mathcal{H}^{S_\mathrm{tup}, \Rad_{\bar n}}\right),
    \end{align*}
    \noindent as desired.
\end{proof}

\begin{proposition}
    \label{prop:supporting_proposition_2}
    Let $\mathcal{F}$ be a class of representation functions. For each $f\in\mathcal{F}$, let  $\bar U_\Omega(f)$ be defined in Eqn.~\eqref{eq:bar_Uomega_Ulambda} and ${\Lnorm}^*_\Omega(f)$ be defined in Eqn.~\eqref{eq:Lbar_Omega}. Then, for any $\Delta\in(0,1)$, we have:
    \begin{align}
        \sup_{f\in\mathcal{F}}\left|\bar U_\Omega(f)-\Lnorm^*_\Omega(f)\right| \le 2\RC_{\mu_*}(\mathcal{H}) + 8\mathcal{B}\sqrt{\frac{\ln 1/\Delta}{n-R}},
    \end{align}
    \noindent with probability of at least $1-\Delta$ where $\mu_*$ is a distribution of tuples set defined as:
    \begin{align}
        \label{eq:mu_star}
        \mu_* &:= \arg\max_{\mu\in\mathcal{U}}\RC_{\mu}(\mathcal{H}), \\
        \mathcal{U}=\Bigg\{
            \bigotimes_{r=1}^R\Big[\mathcal{D}_r^{\otimes 2}\otimes\mathcal{\bar D}^{\otimes k}\Big]^{\otimes\lfloor q_r/2\rfloor} &: q_r\in \mathbb{Z}_{\ge0}, \forall r\in[R]\textup{ and } \sum_{r=1}^Rq_r = n
        \Bigg\}.
    \end{align}
\end{proposition}

\begin{proof}
    Let $\mathcal{R}_\mathcal{H}^{S_\mathrm{tup},\Rad_{\bar n}}$ be defined in Lemma \ref{lem:supporting_lemma_6} and let ${\bf n}=\set{n_r}_{1\le r\le R}$ be the random vector of class-wise sample sizes in $S^{(n)}$. Applying Lemma \ref{lem:supporting_lemma_5} and Lemma \ref{lem:supporting_lemma_6} with $\varphi(x)=e^{\lambda x}$ (for $\lambda>0$), we have:     
    \begin{align*}
        \E_S\exp\left(\lambda \sup_{f\in\mathcal{F}}\left|\bar U_\Omega(f)-\Lnorm^*_\Omega(f)\right|\right) &\le \E_S\exp\left(\lambda \sup_{f\in\mathcal{F}}\left|\widehat U_\Omega(f)-\widehat\Lnorm_\Omega(f)\right|\right) \qquad \textup{(Lemma \ref{lem:supporting_lemma_5})} \\ 
        &\le \E_{S_\mathrm{tup, \Rad_{\bar n}}}\exp\left(2\lambda\mathcal{R}_\mathcal{H}^{S_\mathrm{tup},\Rad_{\bar n}}\right) \qquad\qquad\textup{(Lemma \ref{lem:supporting_lemma_6})} \\
        &= \E_{\bf n}\E_{S_\mathrm{tup, \Rad_{\bar n}|\bf n}}\exp\left(2\lambda\mathcal{R}_\mathcal{H}^{S_\mathrm{tup},\Rad_{\bar n}}\right) \\
        &\le \E_{\bf n}\exp\left(2\lambda\E_{S_\mathrm{tup}, \Rad_{\bar n}|\bf n}\left[\mathcal{R}_\mathcal{H}^{S_\mathrm{tup},\Rad_{\bar n}}\right]+\frac{4\lambda^2\mathcal{B}^2}{\bar n}\right)\quad\textup{(Lemma \ref{lem:subgaussianity_of_rad_complexity})} \\
        &= \E_{\bf n}\exp\left(2\lambda\RC_{\bar \mu_{\bf n}}(\mathcal{H})+\frac{4\lambda^2\mathcal{B}^2}{\bar n}\right).
    \end{align*}
    \noindent where we define the tuples distribution $\bar\mu_{\bf n}$ as:
    \begin{align*}
        \bar\mu_{\bf n} := \bigotimes_{r=1}^R\Big[\mathcal{D}_r^{\otimes 2}\otimes\mathcal{\bar D}^{\otimes k}\Big]^{\otimes \lfloor n_r/2\rfloor}.
    \end{align*}
    \noindent Which is the distribution of $S_\mathrm{tup}$ given the observations of the sample sizes $\bf n$. Obviously, $\bar\mu_{\bf n}\in\mathcal{U}$ for every observation of ${\bf n}$. Hence, we have:
    \begin{align*}
        \E_S\exp\left(\lambda \sup_{f\in\mathcal{F}}\left|\bar U_\Omega(f)-\Lnorm^*_\Omega(f)\right|\right) &\le \E_{\bf n}\exp\left(2\lambda\RC_{\bar \mu_{\bf n}}(\mathcal{H})+\frac{4\lambda^2\mathcal{B}^2}{\bar n}\right) \\
        &\le \sup_{\mu\in\mathcal{U}}\exp\left(2\lambda\RC_{\mu}(\mathcal{H})+\frac{8\lambda^2\mathcal{B}^2}{n-R}\right) \qquad\Bigg(\text{Since }\bar n\ge\frac{n-R}{2}\Bigg)\\
        &= \exp\left(2\lambda\RC_{\mu_*}(\mathcal{H})+\frac{8\lambda^2\mathcal{B}^2}{n-R}\right).
    \end{align*}
    \noindent Using the Chernoff bound, for all $\varepsilon, \lambda>0$, we have:
    \begin{align*}
        \P\left(\sup_{f\in\mathcal{F}}\left|\bar U_\Omega(f)-\Lnorm^*_\Omega(f)\right|\ge\varepsilon\right) &\le e^{-\lambda\varepsilon}\E_S\exp\left(\lambda \sup_{f\in\mathcal{F}}\left|\bar U_\Omega(f)-\Lnorm^*_\Omega(f)\right|\right) \\
        &\le \exp\left(2\lambda\RC_{\mu_*}(\mathcal{H})+\frac{8\lambda^2\mathcal{B}^2}{n-R}-\lambda\varepsilon\right).
    \end{align*}
    \noindent Setting $\Delta=\exp\left(2\lambda\RC_{\mu_*}(\mathcal{H})+\frac{8\lambda^2\mathcal{B}^2}{n-R} -\lambda\varepsilon\right)$ and solve for $\varepsilon$, we have:
    \begin{align*}
        \varepsilon = 2\RC_{\mu_*}(\mathcal{H})+\frac{8\lambda \mathcal{B}^2}{n-R} +\lambda^{-1}\ln 1/\Delta.
    \end{align*}
    \noindent Solving for the optimal value $\lambda$, we have $\lambda^{-1}=\mathcal{B}\sqrt{\frac{8}{(n-R)\ln (1/\Delta)}}$. Plugging this value back to $\varepsilon$ yields:
    \begin{align*}
        \sup_{f\in\mathcal{F}}\left|\bar U_\Omega(f)-\Lnorm^*_\Omega(f)\right| &\le 2\RC_{\mu_*}(\mathcal{H})+4\sqrt{2}\mathcal{B}\sqrt{\frac{\ln 1/\Delta}{n-R}} \\
        &\le 2\RC_{\mu_*}(\mathcal{H})+8\mathcal{B}\sqrt{\frac{\ln 1/\Delta}{n-R}},
    \end{align*}
    \noindent with probability of at least $1-\Delta$, as desired.
\end{proof}

\begin{theorem}
    \label{thm:concentration_of_bar_Uomega}
    Let $\mathcal{F}$ be a class of representation functions. For each $f\in\mathcal{F}$, let  $\bar U_\Omega(f)$ be defined in Eqn.~\eqref{eq:bar_Uomega_Ulambda} and ${\Lnorm}_\Omega(f)$ be defined in Eqn.~\eqref{eq:lomega}. Then, for any $\Delta\in(0,1)$, we have:
    \begin{align}
        &\sup_{f\in\mathcal{F}}\left|\bar U_\Omega(f)-\Lnorm_\Omega(f)\right| \\
        &\le 2\RC_{\mu_*}(\mathcal{H}) + \mathcal{B}\left[\frac{6R}{n-R}+\frac{8R}{3N}\ln4R/\Delta + \sqrt{\frac{8(R-1)\ln 4R/\Delta}{3N}} + 8\sqrt{\frac{\ln 2/\Delta}{n-R}}\right], \nonumber
    \end{align}
    \noindent with probability of at least $1-\Delta$.
\end{theorem}

\begin{proof}
    We have:
    \begin{align*}
        \sup_{f\in\mathcal{F}}\left|\bar U_\Omega(f)-\Lnorm_\Omega(f)\right| 
        \le \sup_{f\in\mathcal{F}}\left|\bar U_\Omega(f)-\Lnorm^*_\Omega(f)\right|  + \sup_{f\in\mathcal{F}}\left|\Lnorm^*_\Omega(f)-\overline{\Lnorm}_\Omega(f)\right| + \sup_{f\in\mathcal{F}}\left|\overline{\Lnorm}_\Omega(f)-\Lnorm_\Omega(f)\right|.
    \end{align*}
    \noindent For a given $\delta\in(0,1)$, by Lemma \ref{lem:supporting_lemma_3}, Proposition \ref{prop:supporting_proposition_1} and Proposition \ref{prop:supporting_proposition_2}, we have:
    \begin{align*}
        \sup_{f\in\mathcal{F}}\left|\overline{\Lnorm}_\Omega(f)-\Lnorm_\Omega(f)\right| &\le \frac{2\mathcal{B}R}{n-R},\\
        \sup_{f\in\mathcal{F}}\left|\Lnorm^*_\Omega(f)-\overline{\Lnorm}_\Omega(f)\right| &\le \mathcal{B}\left[\frac{4R}{n-R}+\frac{8R}{3N}\ln2R/\delta + \sqrt{\frac{8(R-1)\ln 2R/\delta}{3N}}\right] \textup{ with prob. } 1-\delta,\\
        \sup_{f\in\mathcal{F}}\left|\bar U_\Omega(f)-\Lnorm^*_\Omega(f)\right|&\le 2\RC_{\mu_*}(\mathcal{H})+8\mathcal{B}\sqrt{\frac{\ln 1/\delta}{n-R}} \textup{ with prob. } 1-\delta.
    \end{align*}
    \noindent Combining everything with the union bound, we have:
    \begin{align*}
        \sup_{f\in\mathcal{F}}\left|\bar U_\Omega(f)-\Lnorm_\Omega(f)\right| 
        \le 2\RC_{\mu_*}(\mathcal{H}) + \mathcal{B}\left[\frac{6R}{n-R}+\frac{8R}{3N}\ln2R/\delta + \sqrt{\frac{8(R-1)\ln 2R/\delta}{3N}} + 8\sqrt{\frac{\ln 1/\delta}{n-R}}\right],
    \end{align*}
    \noindent with probability of at least $1-2\delta$. Setting $\delta=\Delta/2$ yields the desired bound.
\end{proof}

\subsection{Concentration of \texorpdfstring{$\Lnorm^*_\Lambda(f)$}{} to \texorpdfstring{$\Lnorm_\Lambda(f)$}{}}
The concentration analysis of $\bar U_\Lambda(f)$ is almost verbatim to the previous sub-sections. Therefore, for readers who have grasped the overall strategy, we recommend skipping this sub-section and the next sub-section.
\begin{lemma}
    \label{lem:supporting_lemma_7}
    Let $\mathcal{F}$ be a class of representation functions. For each $f\in\mathcal{F}$, let $\overline{\Lnorm}_\Lambda(f)$ be defined in Eqn.~\eqref{eq:Lbar_Lambda} and $\Lnorm_\Lambda(f)$ be defined in Eqn.~\eqref{eq:llambda}. Then, we have:
    \begin{align}
        \sup_{f\in\mathcal{F}}\left|\overline{\Lnorm}_\Lambda(f)-\Lnorm_\Lambda(f)\right| &\le \frac{4\mathcal{B}R}{m-2R}.
    \end{align}
\end{lemma}

\begin{proof}
    Similar to Lemma \ref{lem:supporting_lemma_3}, we can easily show that:
    \begin{align*}
        \sup_{f\in\mathcal{F}}\left|\overline{\Lnorm}_\Lambda(f)-\Lnorm_\Lambda(f)\right| &\le \mathcal{B}\sum_{r=1}^R\left|\overline{\lambda}_r - \rho_r\right| \le \mathcal{B}\sum_{r=1}^R\E\left|\widehat{\E}_\Pi[\lambda_r^\pi]-\widehat\rho_r\right|.
    \end{align*}
    \noindent Then, using the fact that we can write $\widehat\rho_r=\widehat\E_\Pi[m_r^\pi/m]$, we have:
    \begin{align*}
        \left|\widehat{\E}_\Pi[\lambda_r^\pi]-\widehat\rho_r\right| &= \left|\widehat\E_\Pi[\lambda_r^\pi] - \widehat\E_\Pi[m_r^\pi/m]\right| \le \widehat\E_\Pi\Big|\lambda_r^\pi - m_r^\pi/m\Big|.
    \end{align*}
    \noindent Then, for each $r\in[R]$ and $\pi\in\Pi$, we have:
    \begin{align*}
        \Big|m_r^\pi/m-\lambda_r^\pi\Big| &= \left|\frac{m_r^\pi}{m} - \frac{\lfloor m_r^\pi /3\rfloor}{\sum_{q=1}^R\lfloor m_q^\pi/3\rfloor}\right| \\
        &\le \left|\frac{m_r^\pi}{m}-\frac{3\lfloor m_r^\pi/3\rfloor}{m}\right| + \left|\frac{3\lfloor m_r^\pi/3\rfloor}{m}-\frac{\lfloor m_r^\pi/3\rfloor}{\sum_{q=1}^R\lfloor m_q^\pi/3\rfloor}\right| \\
        &= \frac{1}{m}(m_r^\pi-3\lfloor m_r^\pi/3\rfloor) + \lfloor m_r^\pi/3\rfloor\left(\frac{1}{\sum_{q=1}^R\lfloor m_q^\pi/3\rfloor}-\frac{3}{m}\right).
    \end{align*}
    \noindent Using the fact that $m_r^\pi-3\lfloor m_r^\pi/3\rfloor\le 2$ and $\sum_{q=1}^R\lfloor m_q^\pi/3\rfloor\ge\frac{m-2R}{3}$, we have:
    \begin{align*}
        \Big|m_r^\pi/m-\lambda_r^\pi\Big| &\le \frac{2}{m}+3\lfloor m_r^\pi/3\rfloor\left(\frac{1}{m-2R}-\frac{1}{m}\right) \le\frac{2}{m}+\frac{6R\lfloor m_r^\pi/3\rfloor}{m(m-2R)}.
    \end{align*}
    \noindent Hence, we have:
    \begin{align*}
        \sup_{f\in\mathcal{F}}\left|\overline{\Lnorm}_\Lambda(f)-\Lnorm_\Lambda(f)\right| &\le \mathcal{B}\sum_{r=1}^R\E\left|\widehat{\E}_\Pi[\lambda_r^\pi]-\widehat\rho_r\right| \le \mathcal{B}\sum_{r=1}^R\E\left[\widehat\E_\Pi\Big|\lambda_r^\pi-m_r^\pi/m\Big|\right] \\
        &\le \frac{2\mathcal{B}R}{m} + \frac{6\mathcal{B}R}{m(m-2R)}\E\Bigg[\widehat\E_\Pi\Bigg[\underbrace{\sum_{r=1}^R\lfloor m_r^\pi/3\rfloor}_{\le m/3}\Bigg]\Bigg] \\
        &\le \frac{2\mathcal{B}R}{m} + \frac{2\mathcal{B}R}{m-2R} 
        \le\frac{4\mathcal{B}R}{m-2R},
    \end{align*}
    \noindent as desired.
\end{proof}

\begin{lemma}
    \label{lem:supporting_lemma_8}
    Let $\mathcal{F}$ be a class of representation functions. For each $f\in\mathcal{F}$, let $\overline{\Lnorm}_\Lambda(f), \Lnorm^*_\Lambda(f)$ be defined in Eqn.~\eqref{eq:Lbar_Lambda}. Then, for any $\Delta\in(0,1)$ we have:
    \begin{align}
        \sup_{f\in\mathcal{F}}\left|\overline{\Lnorm}_\Lambda(f)-\Lnorm^*_\Lambda(f)\right| &\le \mathcal{B}\left[\frac{8R}{m-2R}+\frac{8R}{3N}\ln2R/\Delta + \sqrt{\frac{8(R-1)\ln 2R/\Delta}{3N}}\right],
    \end{align}
    \noindent with probability of at least $1-\Delta$.
\end{lemma}

\begin{proof}
    Similar to Lemma \ref{lem:supporting_lemma_4}, we can easily obtain:
    \begin{align*}
        \sup_{f\in\mathcal{F}}\left|\overline{\Lnorm}_\Lambda(f)-\Lnorm^*_\Lambda(f)\right| &\le \mathcal{B}\sum_{r=1}^R|\overline{\lambda}_r-\lambda_r^*| = \mathcal{B}\sum_{r=1}^R\left|(\overline{\lambda}_r - \rho_r) + (\rho_r-\widehat\rho_r) + (\widehat\rho_r - \lambda_r^*)\right| \\
        &\le \mathcal{B}\left[\sum_{r=1}^R|\lambda_r^*-\widehat\rho_r|+\sum_{r=1}^R|\overline{\lambda}_r-\rho_r|+\sum_{r=1}^R|\widehat\rho_r-\rho_r|\right] \\
        &\le \mathcal{B}\left[\sum_{r=1}^R\left|\widehat\E_\Pi[\lambda_r^\pi] - \widehat\rho_r\right| + \sum_{r=1}^R\E\left|\widehat\E_\Pi[\lambda_r^\pi] - \widehat\rho_r\right| + \sum_{r=1}^R|\widehat\rho_r-\rho_r|\right] \\
        &\le \frac{4\mathcal{B}R}{m-2R}+\frac{4\mathcal{B}R}{m-2R} + \mathcal{B}\sum_{r=1}^R|\widehat\rho_r-\rho_r| \quad\textup{(Lemma \ref{lem:supporting_lemma_7})}\\
        &= \frac{8\mathcal{B}R}{m-2R}+\mathcal{B}\sum_{r=1}^R|\widehat\rho_r-\rho_r|.
    \end{align*}
    \noindent Using Bernstein's bound to analyze the concentration of $\widehat\rho_r$ around $\rho_r$ as usual, for any $\Delta\in(0,1)$, we have:
    \begin{align*}
        \sum_{r=1}^R|\widehat\rho_r-\rho_r|\le \frac{8R}{3N}\ln 2R/\Delta + \sqrt{\frac{8(R-1)\ln 2R/\Delta}{3N}},
    \end{align*}
    \noindent with probability of at least $1-\Delta$. As a result, we have:
    \begin{align*}
        \sup_{f\in\mathcal{F}}\left|\overline{\Lnorm}_\Lambda(f)-\Lnorm^*_\Lambda(f)\right|  &\le \frac{8\mathcal{B}R}{m-2R}+\mathcal{B}\sum_{r=1}^R|\widehat\rho_r-\rho_r| \\
        &\le \mathcal{B}\left[\frac{8R}{m-2R}+\frac{8R}{3N}\ln 2R/\Delta + \sqrt{\frac{8(R-1)\ln 2R/\Delta}{3N}}\right],
    \end{align*}
    \noindent with probability of at least $1-\Delta$, as desired.
\end{proof}

\subsection{Concentration of \texorpdfstring{$\bar U_\Lambda(f)$}{} to \texorpdfstring{$\Lnorm_\Lambda^*(f)$}{}}
Applying the same strategy as Lemma \ref{lem:supporting_lemma_5}, Lemma \ref{lem:supporting_lemma_6} and Proposition \ref{prop:supporting_proposition_2}, we have the following analogous result.
\begin{proposition}
    \label{prop:supporting_proposition_3}
    Let $\mathcal{F}$ be a class of representation functions. For each $f\in\mathcal{F}$, let  $\bar U_\Lambda(f)$ be defined in Eqn.~\eqref{eq:bar_Uomega_Ulambda} and ${\Lnorm}^*_\Lambda(f)$ be defined in Eqn.~\eqref{eq:Lbar_Lambda}. Then, for any $\Delta\in(0,1)$, we have:
    \begin{align}
        \sup_{f\in\mathcal{F}}\left|\bar U_\Lambda(f)-\Lnorm^*_\Lambda(f)\right| \le 2\RC_{\nu_*}(\mathcal{H}) + 8\mathcal{B}\sqrt{\frac{\ln 1/\Delta}{m-2R}},
    \end{align}
    \noindent with probability of at least $1-\Delta$ where $\nu_*$ is a distribution of tuples set defined as:
    \begin{align}
        \label{eq:nu_star}
        \nu_* &:= \arg\max_{\nu\in\mathcal{V}}\RC_{\nu}(\mathcal{H}), \\
        \mathcal{V}=\Bigg\{
            \bigotimes_{r=1}^R\Big[\mathcal{D}_r^{\otimes 3}\otimes\mathcal{\bar D}^{\otimes k-1}\Big]^{\otimes\lfloor q_r/3\rfloor} &: q_r\in \mathbb{Z}_{\ge0}, \forall r\in[R]\textup{ and } \sum_{r=1}^Rq_r = m
        \Bigg\}.
    \end{align}
\end{proposition}

\noindent Then, we also have the following result.
\begin{theorem}
    \label{thm:concentration_of_bar_Ulambda}
    Let $\mathcal{F}$ be a class of representation functions. For each $f\in\mathcal{F}$, let  $\bar U_\Lambda(f)$ be defined in Eqn.~\eqref{eq:bar_Uomega_Ulambda} and ${\Lnorm}_\Lambda(f)$ be defined in Eqn.~\eqref{eq:llambda}. Then, for any $\Delta\in(0,1)$, we have:
    \begin{align}
        &\sup_{f\in\mathcal{F}}\left|\bar U_\Lambda(f)-\Lnorm_\Lambda(f)\right| \\
        &\le 2\RC_{\nu_*}(\mathcal{H}) + \mathcal{B}\left[\frac{12R}{m-2R}+\frac{8R}{3N}\ln4R/\Delta + \sqrt{\frac{8(R-1)\ln 4R/\Delta}{3N}} + 8\sqrt{\frac{\ln 2/\Delta}{m-2R}}\right], \nonumber
    \end{align}
    \noindent with probability of at least $1-\Delta$.
\end{theorem}

\section{Concentration of the Collision Probability Estimator}
Our final ingredient for the proposed U-Statistic formulation is an estimator for the collision probability $\tau$ (Eqn.~\eqref{eq:collision_probability}). In this work, we propose the following plug-in estimator:
\begin{align}
    \label{eq:collision_probability_estimator}
    \widehat\tau := 1 - \sum_{r=1}^R\widehat\rho_r\left[1-\widehat\rho_r\right]^k, \quad \forall r\in[R]:\widehat\rho_r=\frac{N_r}{N}.
\end{align}

\begin{proposition}
    \label{prop:concentration_of_collision_probability_estimator}
    Let $\tau$ be the collision probability defined in Eqn.~\eqref{eq:collision_probability} and $\widehat\tau$ be the plug-in estimator in Eqn.~\eqref{eq:collision_probability_estimator}. For any $\Delta\in(0,1)$, we have:
    \begin{align}
        |\tau-\widehat\tau| &\le \sqrt{2R}\left[\frac{2\ln (R+1)/\Delta}{3N} + \sqrt{\frac{2(1-\|\rho\|_2^2)\ln (R+1)/\Delta}{N}}\right],
    \end{align}
    \noindent with probability of at least $1-\Delta$ where $\|\rho\|_2^2=\sum_{r=1}^R\rho_r^2$.
\end{proposition}

\begin{proof}
    Let $\Phi:\R^R \to [0, 1]$ be defined as follows:
    \begin{align*}
        \Phi(q_1, \dots,q_R) &= 1 - \sum_{r=1}^Rq_r\left[1-q_r\right]^k.
    \end{align*}
    \noindent Then, we have:
    \begin{align*}
        |\tau-\widehat\tau| &= \left|\Phi(\rho_1,\dots,\rho_R)-\Phi(\widehat\rho_1,\dots,\widehat\rho_R)\right| \\
        &\le \max_{\set{q_r}_{r=1}^R\in\Delta_R}\left\|\nabla\Phi(q_1,\dots,q_R)\right\|_2\cdot\|\u-\E[\u]\|_2,
    \end{align*}
    \noindent where $\u=\set{\widehat\rho_1,\dots,\widehat\rho_R}\in\Delta_R$ is the vector of probabilities and $\Delta_R$ is defined as the probability simplex with $R$ probabilities:
    \begin{align}
        \Delta_R:=\left\{\set{q_r}_{r=1}^R: q_r\ge0,\forall r\in[R] \textup{ and } \sum_{r=1}^Rq_r=1\right\}.
    \end{align}
    \noindent First, we bound the $\ell^2$-norm of the Jacobian $\left\|\nabla\Phi(q_1,\dots,q_R)\right\|_2$. For each $r\in[R]$, we have:
    \begin{align*}
        \frac{\partial\Phi}{\partial q_r} &= kq_r(1-q_r)^{k-1}-(1-q_r)^k.
    \end{align*}
    \noindent Equivalently, $\frac{\partial\Phi}{\partial q_r}=(1-q_r)^{k-1}[q_r(k+1)-1]$ and this quantity lies in $[-1,1]$ for all $q_r\in[0,1]$: If $q_r\le\frac{1}{k+1}$, it lies in $[-1,0]$. Otherwise, it is positive and is at most $kq_r(1-q_r)^{k-1}\le 1$ (indeed, $q_r(k+1)-1\le kq_r$ and $kq_r(1-q_r)^{k-1}$ is the probability that a $\mathrm{Bin}(k,q_r)$ random variable equals one). As a result, for all set of probabilities $\{q_r\}_{r=1}^R\in\Delta_R$, we have:
    \begin{align*}
        \|\nabla\Phi(q_1,\dots,q_R)\|_2^2 = \sum_{r=1}^R(1-q_r)^{2(k-1)}[q_r(k+1)-1]^2\le R.
    \end{align*}
    \noindent Now, we analyze the concentration of $\u=\{\widehat\rho_1,\dots,\widehat\rho_R\}$ around its mean. Let $\mathcal{B}_R=\set{e_1,\dots,e_R}$ be the canonical basis of $\R^R$. Then, we can think of $\rho$ as a distribution over $\mathcal{B}_R$ such that for a random vector $X\sim\rho$, we have:
    \begin{align*}
        \forall r\in[R]: \P(X=e_r) = \rho_r.
    \end{align*}
    \noindent Then, we can write $\u-\E[\u]$ as:
    \begin{align*}
        \u-\E[\u] = \frac{1}{N}\sum_{j=1}^N (X_j-\E[\u])\quad\textup{where}\quad X_j\overset{\mathrm{i.i.d.}}{\sim} \rho,\forall j\in[N].
    \end{align*}
    \noindent Let $Z_j=X_j-\E[\u]$ for all $j\in[N]$, so that $\|Z_j\|_2\le\sqrt{2}$ (since $\|e_r-\E[{\bf u}]\|_2^2=1-2\rho_r+\|\rho\|_2^2\le 2$ for all $r\in[R]$). Using the matrix Bernstein inequality (Proposition \ref{prop:bernstein_matrix}) for the vectors $Z_j/\sqrt{2}$, for all $\lambda>0$, we have:
    \begin{align*}
        \P\left(\|\u-\E[\u]\|_2\ge \lambda\sqrt{2}\right) &= \P\left(\left\|\frac{1}{N}\sum_{j=1}^N Z_j\right\|_2 \ge \lambda\sqrt{2}\right) \le (R+1)\exp\left(-\frac{N^2\lambda^2/2}{\sum_{j=1}^N\sigma_j^2+N\lambda/3}\right),
    \end{align*}
    \noindent where $\sigma_j^2=\max\left(\|\E[Z_jZ_j^\top]\|_\sigma, \E[Z_j^\top Z_j]\right)$ for all $j\in[N]$ (which upper bounds the corresponding quantity for $Z_j/\sqrt{2}$). From here, we use the notation $\rho=\E[\u]$ for brevity. Now, for all $j\in[N]$ we have:
    \begin{align*}
        \E[Z_j^\top Z_j] &= \E[\|Z_j\|_2^2] = \sum_{r=1}^R\rho_r \|e_r-\rho\|_2^2 \\
        &= \sum_{r=1}^R\rho_r\left[(1-\rho_r)^2 + \sum_{q\ne r}^R \rho_q^2\right] = \sum_{r=1}^R\rho_r\left[(1-\rho_r)^2 + \|\rho\|_2^2-\rho_r^2\right] \\
        &= \sum_{r=1}^R\rho_r\left[\|\rho\|_2^2+1-2\rho_r\right] = \|\rho\|_2^2 + 1 - 2\|\rho\|_2^2\\
        &= 1 - \|\rho\|_2^2.
    \end{align*}
    \noindent Since the matrix $\E[Z_jZ_j^\top]$ is positive semi-definite, for all $j\in[N]$, we have:
    \begin{align*}
        \|\E[Z_jZ_j^\top]\|_\sigma &\le \mathrm{trace}(\E[Z_jZ_j^\top]) = \mathrm{trace}\left(\sum_{r=1}^R\rho_r(e_r-\rho)(e_r-\rho)^\top\right) \\
        &= \mathrm{trace}\left(\mathrm{diag}(\rho) - \rho\rho^\top\right) = 1 - \|\rho\|_2^2.
    \end{align*}
    \noindent From the above, we have $\sigma_j^2=1-\|\rho\|_2^2$ for all $j\in[N]$. As a result, we have:
    \begin{align*}
        \P\left(\left\|\frac{1}{N}\sum_{j=1}^N Z_j\right\|_2 \ge \lambda\sqrt{2}\right) &\le (R+1)\exp\left(-\frac{N^2\lambda^2/2}{N-N\|\rho\|_2^2+N\lambda/3}\right) \\
        &= (R+1)\exp\left(-\frac{N\lambda^2/2}{1-\|\rho\|_2^2+\lambda/3}\right).
    \end{align*}
    \noindent Setting the right-hand-side of the above inequality to $\Delta\in(0,1)$, we have:
    \begin{align*}
        \frac{N\lambda^2}{2(1-\|\rho\|_2^2+\lambda/3)} = \ln\frac{R+1}{\Delta} \implies N\lambda^2-\frac{2\ln (R+1)/\Delta}{3}\lambda - 2(1-\|\rho\|_2^2)\ln\frac{R+1}{\Delta} = 0.
    \end{align*}
    \noindent Solving the above quadratic equation yields:
    \begin{align*}
        \lambda &= \frac{\frac{2\ln (R+1)/\Delta}{3} +\sqrt{\frac{4\ln^2(R+1)/\Delta}{9} +8N(1-\|\rho\|_2^2)\ln\frac{R+1}{\Delta}}}{2N} \\
        &= \frac{\ln(R+1)/\Delta}{3N} + \frac{1}{N}\sqrt{\frac{\ln^2(R+1)/\Delta}{9} + 2N(1-\|\rho\|_2^2)\ln\frac{R+1}{\Delta}}.
    \end{align*}
    \noindent In other words, we have:
    \begin{align*}
        \frac{1}{\sqrt{2}}\|\u-\rho\|_2 &\le \frac{\ln(R+1)/\Delta}{3N} + \frac{1}{N}\sqrt{\frac{\ln^2(R+1)/\Delta}{9} + 2N(1-\|\rho\|_2^2)\ln\frac{R+1}{\Delta}} \\
        &\le \frac{\ln(R+1)/\Delta}{3N} + \frac{1}{N}\left[\frac{\ln (R+1)/\Delta}{3} + \sqrt{2N(1-\|\rho\|_2^2)\ln\frac{R+1}{\Delta}}\right] \\
        &= \frac{2\ln (R+1)/\Delta}{3N} + \sqrt{\frac{2(1-\|\rho\|_2^2)\ln (R+1)/\Delta}{N}},
    \end{align*}
    \noindent with probability of $1-\Delta$. Combine this with the bound on $\|\nabla\Phi(q_1,\dots,q_R)\|_2$, we have:
    \begin{align*}
        |\tau-\widehat\tau| 
        &\le \sqrt{2R}\left[\frac{2\ln (R+1)/\Delta}{3N} + \sqrt{\frac{2(1-\|\rho\|_2^2)\ln (R+1)/\Delta}{N}}\right],
    \end{align*}
    \noindent with probability of at least $1-\Delta$, as desired.
\end{proof}

\begin{proposition}[Multiplicative Concentration of $1-\widehat\tau$]
    \label{prop:first_multiplicative_control_of_tauhat}
    Let $\tau$ be the collision probability defined in Eqn.~\eqref{eq:collision_probability} and $\widehat\tau$ be the plug-in estimator in Eqn.~\eqref{eq:collision_probability_estimator}. For any $\Delta\in(0,1)$, with probability of at least $1-\Delta$:
    \begin{align}
        1-\widehat\tau &\ge \frac{C}{2}\left[1+\frac{1-\gamma_{4k}}{\gamma_{4k}}\bar C\right]^{-1}(1-\tau) &&\textup{as long as }N\ge 8\gamma_{4k}^{-1}\ln\left(\frac{2}{\Delta}\right)+12k\ln\left(\frac{2R}{\Delta}\right), \\
        1-\tau&\ge\frac{C}{2}\left[1+\frac{1-\widehat\gamma_{4k}}{\widehat\gamma_{4k}}\bar C\right]^{-1}(1-\widehat\tau)&&\textup{as long as }N\ge 3\gamma_{2k}^{-1}\ln\left(\frac{2}{\Delta}\right)+ 16k\ln\left(\frac{2R}{\Delta}\right).
    \end{align}
    \noindent with probability of at least $1-\Delta$ where $\gamma_\alpha, \widehat\gamma_\alpha, C,\bar C$ (where $\alpha\ge1$) are defined as follows:
    \begin{align*}
        \gamma_\alpha &:= \sum_{r\in[R]:\rho_r\le\frac{1}{\alpha}}\rho_r, \qquad\widehat\gamma_\alpha:=\sum_{r\in[R]:\widehat\rho_r\le\frac{1}{\alpha}}\widehat\rho_r,\qquad C:=\inf_{k\ge1}\left(1-\frac{1}{2k}\right)^k,\qquad \bar C:=\sup_{k\ge1}\left(1-\frac{1}{4k}\right)^k.
    \end{align*}
\end{proposition}

\begin{proof}(Prove $1-\tau\lesssim 1-\widehat\tau$ whp.)
    We define the indices set $\Gamma$ and $\widehat\Gamma$ as follows:
    \begin{align*}
        \Gamma := \Big\{r\in[R]: \rho_r \le\frac{1}{4k}\Big\}, \qquad \widehat\Gamma:=\Big\{r\in[R]: \widehat\rho_r\le \frac{1}{2k}\Big\}.
    \end{align*}
    \noindent Then, by definition, we have $\gamma_{4k}:=\sum_{r\in\Gamma}\rho_r$. 

    \noindent\textbf{Claim (i)}: For $\delta\in(0,1)$, we have $\P(\Gamma\subseteq\widehat\Gamma)\ge1-\delta$ as long as $N\ge 12k\ln\left(\frac{R}{\delta}\right)$.

    Fix one class $r\in\Gamma$. We can write $\widehat\rho_r=\frac{1}{N}\sum_{j=1}^N X_j$ where $X_j\sim\mathrm{Bern}(\rho_r)$. Let $\{U_j\}_{j=1}^N$ be a sequence of uniform variables ($U_j\sim\mathrm{Uniform}(0,1),\forall j\in[N]$). Then, for each $j\in[N]$, we have $X_j\overset{d}{=}\1{U_j\le \rho_r}$ by coupling. Therefore, we can write $\widehat\rho_r$ as follows:
    \begin{align*}
        \forall r\in[R]:\quad\widehat\rho_r &= \frac{1}{N}\sum_{j=1}^N X_j \overset{d}{=} \frac{1}{N}\sum_{j=1}^N\1{U_j\le \rho_r}.
    \end{align*}
    \noindent Then, we have:
    \begin{align*}
        \P\left(\widehat\rho_r\ge\frac{1}{2k}\right) &= \P\left(\frac{1}{N}\sum_{j=1}^NX_j\ge\frac{1}{2k}\right) =\P\left(\frac{1}{N}\sum_{j=1}^N \1{U_j\le\rho_r}\ge\frac{1}{2k}\right) \\
        &=\P\left(\frac{\sharp\{j\in[N]:U_j\le\rho_r\}}{N}\ge\frac{1}{2k}\right) \\
        &\le \P\left(\frac{\sharp\{j\in[N]: U_j\le \frac{1}{4k}\}}{N}\ge\frac{1}{2k}\right)= \P\left(\frac{1}{N}\sum_{j=1}^N \1{U_j\le\frac{1}{4k}} \ge\frac{1}{2k}\right) \qquad \Big(\textup{Since }\rho_r\le\frac{1}{4k}\Big).
    \end{align*}
    \noindent For each $j\in[N]$, we have $\1{U_j\le\frac{1}{4k}}\sim\mathrm{Bern}(\frac{1}{4k})$ (by coupling). Therefore, we can use the multiplicative Chernoff (Proposition \ref{prop:multiplicative_chernoff}) bound as follows:
    \begin{align*}
        \P\left(\widehat\rho_r\ge\frac{1}{2k}\right) &\le \P\left(\frac{1}{N}\sum_{j=1}^N \1{U_j\le\frac{1}{4k}} \ge\frac{1}{2k}\right) \le \exp\left(-\frac{N}{12k}\right).
    \end{align*}
    \noindent Then, using the union bound, we have:
    \begin{align*}
        \P\left(\exists r\in\Gamma: \widehat\rho_r\ge\frac{1}{2k}\right) &\le |\Gamma|\exp\left(-\frac{N}{12k}\right) \le R\exp\left(-\frac{N}{12k}\right).
    \end{align*}
    \noindent Hence, if we let $Re^{-N/12k}\le \delta$, then as long as $N\ge 12k\ln\left(\frac{R}{\delta}\right)$, we have $\widehat\rho_r\le\frac{1}{2k}$ for all $r\in\Gamma$ wp. of at least $1-\delta$. In other words, as long as $N\ge12k\ln\left(\frac{R}{\delta}\right)$, $\Gamma\subseteq\widehat\Gamma$ wp. of at least $1-\delta$.

    \noindent\textbf{Claim (ii)}: Conditionally given that $\Gamma\subseteq\widehat\Gamma$ and for any constant $\delta\in(0,1)$, we have $1-\widehat\tau\ge\frac{C}{2}\gamma_{4k}$ and $1-\tau\le\gamma_{4k}\left[1+\frac{1-\gamma_{4k}}{\gamma_{4k}}\bar C\right]$ with probability of at least $1-\delta$ as long as $N\ge8\gamma_{4k}^{-1}\ln\left(\frac{1}{\delta}\right)$.
    
    Firstly, we note that given $\Gamma\subseteq\widehat\Gamma$, then for all $r\in\Gamma$, $\widehat\rho_r\le\frac{1}{2k}$. Therefore:
    \begin{align*}
        1-\widehat\tau &= \sum_{r=1}^R\widehat\rho_r(1-\widehat\rho_r)^k \ge \sum_{r\in\Gamma}\widehat\rho_r(1-\widehat\rho_r)^k \ge \left(1-\frac{1}{2k}\right)^k\sum_{r\in\Gamma}\widehat\rho_r\\
        &\ge \inf_{k\ge1}\left(1-\frac{1}{2k}\right)^k\sum_{r\in\Gamma}\widehat\rho_r= C\sum_{r\in\Gamma}\widehat\rho_r \quad (\textup{Since }\Gamma\subseteq\widehat\Gamma).
    \end{align*}
    \noindent Furthermore, we have:
    \begin{align*}
        \P\left(\sum_{r\in\Gamma}\widehat\rho_r \le\frac{\gamma_{4k}}{2}\right) &= \P\left(\sum_{r\in\Gamma}\widehat\rho_r\le\frac{1}{2}\E\left[\sum_{r\in\Gamma}\widehat\rho_r\right]\right) \le \exp\left(-\frac{N\gamma_{4k}}{8}\right)\qquad\textup{(Multiplicative Chernoff)}.
    \end{align*}
    \noindent Then, if we let $e^{-N\gamma_{4k}/8}\le \delta$, then we have $N\ge8\gamma_{4k}^{-1}\ln\left(\frac{1}{\delta}\right)$. Therefore, as long as $N\ge8\gamma_{4k}^{-1}\ln\left(\frac{1}{\delta}\right)$, we have $\sum_{r\in\Gamma}\widehat\rho_r\ge\frac{\gamma_{4k}}{2}$, making $1-\widehat\tau\ge\frac{C}{2}\gamma_{4k}$, wp. of at least $1-\delta$. Furthermore:
    \begin{align*}
        1-\tau &= \sum_{r\in\Gamma} \rho_r(1-\rho_r)^k + \sum_{\bar r\notin\Gamma}\rho_{\bar r}(1-\rho_{\bar r})^k \\
        &\le \gamma_{4k} + \sum_{\bar r:\rho_{\bar r}>\frac{1}{4k}}\rho_{\bar r}(1-\rho_{\bar r})^k \qquad ((1-\rho_r)^k\le 1,\forall r\in[R])  \\
        &\le \gamma_{4k} + \left(1-\frac{1}{4k}\right)^k\sum_{\bar r:\rho_{\bar r}>\frac{1}{4k}}\rho_{\bar r} = \gamma_{4k}+(1-\gamma_{4k})\left(1-\frac{1}{4k}\right)^k \\
        &\le \gamma_{4k} +(1-\gamma_{4k})\sup_{k\ge 1}\left(1-\frac{1}{4k}\right)^k = \gamma_{4k}+(1-\gamma_{4k})\bar C \\
        &= \gamma_{4k}\left[1+\frac{1-\gamma_{4k}}{\gamma_{4k}}\bar C\right].
    \end{align*}
    \noindent Combining claims \textbf{(i)} and \textbf{(ii)}: As long as we have $N\ge 8\gamma_{4k}^{-1}\ln\left(\frac{1}{\delta}\right)+12k\ln\left(\frac{R}{\delta}\right)$, then $1-\widehat\tau\ge\frac{C}{2}\gamma_{4k}$ and $\gamma_{4k}\ge(1-\tau)\left[1+\frac{1-\gamma_{4k}}{\gamma_{4k}}\bar C\right]^{-1}$ hold simultaneously with probability of at least $1-2\delta$. Therefore:
    \begin{align*}
        1-\widehat\tau &\ge\frac{C\gamma_{4k}}{2}\ge\frac{C}{2}\left[1+\frac{1-\gamma_{4k}}{\gamma_{4k}}\bar C\right]^{-1}(1-\tau),
    \end{align*}
    \noindent with probability of at least $1-2\delta$. Setting $\delta=\Delta/2$ yields the desired sample complexity.
\end{proof}

\begin{proof}(Prove $1-\widehat\tau\lesssim 1-\tau$ whp.)
    Now, we define the following indices sets:
    \begin{align*}
        \Gamma_c :=\Big\{r\in[R]:\rho_r>\frac{1}{2k}\Big\}, \qquad\widehat\Gamma_c:=\Big\{r\in[R]: \widehat\rho_r>\frac{1}{4k}\Big\}.
    \end{align*}
    \noindent Then, by definition, $\widehat\gamma_{4k} := \sum_{r\notin\widehat\Gamma_c}\widehat\rho_r$ and $\gamma_{2k}=\sum_{r\notin\Gamma_c}\rho_r$.
    
    \noindent\textbf{Claim (i)}: For $\delta\in(0,1)$, $\P(\Gamma_c\subseteq\widehat\Gamma_c)\ge1-\delta$ as long as $N\ge16k\ln(\frac{R}{\delta})$. Fix $r\in\Gamma_c$, then we have $\rho_r\ge\frac{1}{2k}$. Then, we can write $\widehat\rho_r=\frac{1}{N}\sum_{j=1}^NX_j$ where $X_j\sim\mathrm{Bern}(\rho_r)$ for all $j\in[N]$. Using the same coupling strategy as Proposition \ref{prop:first_multiplicative_control_of_tauhat}, i.e., writing $\widehat\rho_r\overset{d}{=} \frac{1}{N}\sum_{j=1}^N\1{U_j\le \rho_r}$, we have:
    \begin{align*}
        \P\left(\widehat\rho_r \le \frac{1}{4k}\right) &= \P\left(\frac{1}{N}\sum_{j=1}^NX_j\le\frac{1}{4k}\right) = \P\left(\frac{1}{N}\sum_{j=1}^N\1{U_j\le \rho_r}\le\frac{1}{4k}\right) \\
        &\le \P\left(\frac{1}{N}\sum_{j=1}^N\1{U_j\le\frac{1}{2k}}\le \frac{1}{4k}\right) \qquad\textup{(By coupling)} \\
        &\le \exp\left(-\frac{N}{16k}\right) \qquad\textup{(Multiplicative Chernoff)}.
    \end{align*}
    \noindent Therefore, by the union bound:
    \begin{align*}
        \P\left(\exists r\in\Gamma_c:\widehat\rho_r\le \frac{1}{4k}\right)\le |\Gamma_c|\exp\left(-\frac{N}{16k}\right)\le R\exp\left(-\frac{N}{16k}\right).
    \end{align*}
    \noindent Therefore, if $Re^{-N/16k}\le\delta$ then as long as $N\ge 16k\ln\left(\frac{R}{\delta}\right)$, we have $\Gamma_c\subseteq\widehat\Gamma_c$ wp. of at least $1-\delta$.
    \noindent\textbf{Claim (ii)}: Conditionally given $\Gamma_c\subseteq\widehat\Gamma_c$, for any $\delta\in(0,1)$, we have  $1-\widehat\tau\le \widehat\gamma_{4k}\left[1+\frac{1-\widehat\gamma_{4k}}{\widehat\gamma_{4k}}\bar C\right]$ and $1-\tau\ge\frac{C}{2}\widehat\gamma_{4k}$ wp. of at least $1-\delta$ as long as $N\ge3\gamma_{2k}^{-1}\ln\left(\frac{1}{\delta}\right)$. We note that:
    \begin{align*}
        1-\tau &= \sum_{r=1}^R\rho_r(1-\rho_r)^k \ge \sum_{r\notin\Gamma_c}\rho_r(1-\rho_r)^k \\
        &\ge\left(1-\frac{1}{2k}\right)^k \sum_{r\notin\Gamma_c}\rho_r \ge \inf_{k\ge1}\left(1-\frac{1}{2k}\right)^k \sum_{r\notin\Gamma_c}\rho_r \\
        &= C\sum_{r\notin\Gamma_c}\rho_r=C\gamma_{2k}.
    \end{align*}
    \noindent By the multiplicative Chernoff bound (Proposition \ref{prop:multiplicative_chernoff}), we have:
    \begin{align*}
        \P\left(\gamma_{2k}\le\frac{1}{2}\sum_{r\notin\Gamma_c}\widehat\rho_r\right) = \P\left(\sum_{r\notin\Gamma_c}\widehat\rho_r \ge 2\E\left[\sum_{r\notin\Gamma_c}\widehat\rho_r\right]\right) \le \exp\left(-\frac{N\gamma_{2k}}{3}\right).
    \end{align*}
    \noindent Therefore, as long as $N\ge3\gamma^{-1}_{2k}\ln\left(\frac{1}{\delta}\right)$, we have $\gamma_{2k}\ge\frac{1}{2}\sum_{r\notin\Gamma_c}\widehat\rho_r$ and:
    \begin{align*}
        1 - \tau \ge \frac{C}{2}\sum_{r\notin\Gamma_c}\widehat\rho_r \ge \frac{C}{2}\sum_{r\notin\widehat\Gamma_c}\widehat\rho_r = \frac{C}{2}\widehat\gamma_{4k} \qquad ([R]\setminus \Gamma_c\supseteq[R]\setminus\widehat\Gamma_c).
    \end{align*}
    \noindent Furthermore, we have:
    \begin{align*}
        1-\widehat\tau &= \sum_{r\notin\widehat\Gamma_c}\widehat\rho_r(1-\widehat\rho_r)^k + \sum_{\bar r\in\widehat\Gamma_c}\widehat\rho_{\bar r}(1-\widehat\rho_{\bar r})^k \\
        &\le \sum_{r\notin\widehat\Gamma_c}\widehat\rho_r + \left(1-\frac{1}{4k}\right)^k\sum_{\bar r\in\widehat\Gamma_c}\widehat\rho_{\bar r} = \widehat\gamma_{4k} + \left(1-\frac{1}{4k}\right)^k(1-\widehat\gamma_{4k}) \qquad \Big(\widehat\gamma_{4k}:=\sum_{r\notin\widehat\Gamma_c}\widehat\rho_r\Big)\\
        &\le \widehat\gamma_{4k} + (1-\widehat\gamma_{4k})\sup_{k\ge1}\left(1-\frac{1}{4k}\right)^k \qquad \Big(\bar C := \sup_{k\ge1}\Big(1-\frac{1}{4k}\Big)^k\Big) \\
        &= \widehat\gamma_{4k}\left[1+\frac{1-\widehat\gamma_{4k}}{\widehat\gamma_{4k}}\bar C\right].
    \end{align*}
    \noindent Combining claims \textbf{(i)} and \textbf{(ii)} via the union bound: As long as we have $N\ge 3\gamma_{2k}^{-1}\ln\left(\frac{1}{\delta}\right) + 16k\ln\left(\frac{R}{\delta}\right)$, then $1-\tau\ge\frac{C}{2}\widehat\gamma_{4k}$ and $1-\widehat\tau\le \widehat\gamma_{4k}\left[1+\frac{1-\widehat\gamma_{4k}}{\widehat\gamma_{4k}}\bar C\right]$ holds simultaneously with probability of at least $1-2\delta$. Therefore, we have:
    \begin{align*}
        1-\tau \ge\frac{C}{2}\widehat\gamma_{4k}\ge\frac{C}{2}\left[1+\frac{1-\widehat\gamma_{4k}}{\widehat\gamma_{4k}}\bar C\right]^{-1}(1-\widehat\tau),
    \end{align*}
    \noindent with probability of at least $1-2\delta$. Setting $\delta=\Delta/2$ yields the desired sample complexity.
\end{proof}

\begin{remark}
    From the above result, if we have $\gamma_{4k}\in\Omega(1)$, then as long as $N\ge\mathcal{O}\left(k\ln\left(\frac{R}{\delta}\right)\right)$, we have a ``well-behaved" multiplicative control for $1-\widehat\tau$. This means that if the probability $\gamma_{4k}$ is not vanishingly small, meaning the majority of the classes have small probabilities, the above result is sufficient. In practical extreme multi-class settings, this is a reasonable assumption especially when $k\ll R$. However, if we want to get rid of the prior assumption that the majority of the classes are small, we can rely on the result that follows.
\end{remark}

\begin{proposition}[Multiplicative Concentration of $1-\widehat\tau$]
    \label{prop:second_multiplicative_control_of_tauhat}
    Let $\tau$ be the collision probability defined in Eqn.~\eqref{eq:collision_probability} and $\widehat\tau$ be the plug-in estimator in Eqn.~\eqref{eq:collision_probability_estimator}. Suppose $\sum_{r:\rho_r\le\frac{1}{4k^2}}\rho_r(1-\rho_r)^k\le \frac{1-\tau}{2}$. Then, for any $\Delta\in(0,1)$, with probability of at least $1-\Delta$, we have:
    \begin{align}
        \textup{When }\rho_{\max}\le \frac{1}{2}:1-\widehat\tau &\ge \frac{1-\tau}{4}\quad\textup{ as long as } N\ge 132k^2\ln\left(\frac{2R}{\Delta}\right). \\
        \textup{When }\rho_{\max}>\frac{1}{2}:1-\widehat\tau &\ge \frac{1-\tau}{8}\quad\textup{ as long as } N\ge 164k^2\ln\left(\frac{4R}{\Delta}\right)+\frac{25\ln\left(\frac{4}{\Delta}\right)}{1-\rho_{\max}}.
    \end{align}
\end{proposition}

\begin{proof}(Prove $1-\tau\lesssim 1-\widehat\tau$)
    We define the set $\Gamma_*,\Gamma\subseteq[R]$ as follows:
    \begin{align*}
        \bar\Gamma_* :=\Big\{r\in[R]:\frac{1}{4k^2}<\rho_r\le\frac{1}{2}\Big\},\qquad\Gamma:=\Big\{r\in[R]:\rho_r\le\frac{1}{4k^2}\Big\}.
    \end{align*}
    \noindent We first make the observation that $\Gamma\cup\bar\Gamma_*=[R]$ if $\rho_{\max}=\max_{r\in[R]}\rho_r\le \frac{1}{2}$. Otherwise, we have $\Gamma\cup\bar\Gamma_*=[R]\setminus\{r_*\}$ where $r_*=\arg\max_{r\in[R]}\rho_r$.

    \noindent\textbf{Claim (i)}: For any $\delta\in(0,1)$, as long as we have $N\ge 32k^2\ln\left(\frac{R}{\delta}\right)$, then $(1-\widehat\rho_r)^k\ge\frac{1}{\sqrt{2}}(1-\rho_r)^k$ with probability of at least $1-\delta$ for all $r\in\bar\Gamma_*$. For all $r\in[R]$ such that $\rho_r\le \frac{1}{2}$, we have:
    \begin{align*}
        \P\left([1-\widehat\rho_r]^k\le\frac{[1-\rho_r]^k}{\sqrt{2}}\right) &= \P\left([1-\widehat\rho_r]^k \le \left[2^{-\frac{1}{2k}}(1-\rho_r)\right]^k\right)= \P\left(1-\widehat\rho_r \le 2^{-\frac{1}{2k}}(1-\rho_r)\right) \\
        &\le \exp\left(-2\left(1-\frac{1}{2^\frac{1}{2k}}\right)^2N(1-\rho_r)^2\right) \qquad\textup{(Hoeffding's Inequality, Theorem \ref{thm:hoeffding_second_thm})}\\
        &\le \exp\left(-\frac{N}{2}\left(1-\frac{1}{2^\frac{1}{2k}}\right)^2\right)\qquad\Big(\textup{Since }\rho_r\le\frac{1}{2}\Big) \\
        &\le \exp\left(-\frac{N}{2}\left[\frac{1-2^{-1/2}}{k}\right]^2\right) \le \exp\left(-\frac{N}{32k^2}\right).
    \end{align*}
    \noindent The last inequality uses $1-2^{-x}\ge2(1-2^{-1/2})x$ for all $0\le x\le\frac{1}{2}$. Therefore, by the union bound:
    \begin{align*}
        \P\left(\exists r\in\bar\Gamma_*:[1-\widehat\rho_r]^k\le\frac{[1-\rho_r]^k}{\sqrt{2}}\right) \le |\bar\Gamma_*|\exp\left(-\frac{N}{32k^2}\right)\le R\exp\left(-\frac{N}{32k^2}\right).
    \end{align*}
    \noindent Hence, if $\delta\ge Re^{-N/32k^2}$ then as long as $N\ge 32k^2\ln\left(\frac{R}{\delta}\right)$, we have $[1-\widehat\rho_r]^k\ge\frac{[1-\rho_r]^k}{\sqrt{2}}$ for all $r\in\bar\Gamma_*$.

    \noindent\textbf{Claim (ii)}: For any $\delta\in(0,1)$, as long as $N\ge 100k^2\ln\left(\frac{R}{\delta}\right)$, then $\widehat\rho_r\ge\frac{\rho_r}{\sqrt{2}}$ with probability of at least $1-\delta$ for all $r\in\bar\Gamma_*$. By multiplicative Chernoff bound, we have:
    \begin{align*}
        \P\left(\widehat\rho_r\le \frac{1}{\sqrt{2}}\rho_r\right) &\le \exp\left(-\frac{1}{2}\left(1-\frac{1}{\sqrt{2}}\right)^2N\rho_r\right)\le \exp\left(-\frac{1}{8}\left(1-\frac{1}{\sqrt{2}}\right)^2\frac{N}{k^2}\right) \qquad \Big(\textup{Since }\rho_r>\frac{1}{4k^2}\Big) \\
        &\le \exp\left(-\frac{N}{100k^2}\right).
    \end{align*}
    \noindent Then, by the union bound:
    \begin{align*}
        \P\left(\exists r\in\bar\Gamma_*:\widehat\rho_r\le\frac{\rho_r}{\sqrt{2}}\right) \le |\bar\Gamma_*|\exp\left(-\frac{N}{100k^2}\right)\le R\exp\left(-\frac{N}{100k^2}\right).
    \end{align*}
    \noindent As a result, as long as $N\ge 100k^2\ln\left(\frac{R}{\delta}\right)$, then $\widehat\rho_r\ge\frac{\rho_r}{\sqrt{2}}$ for all $r\in\bar\Gamma_*$.

    \noindent\newline \textbf{Case 1 ($\rho_{\max}\le\frac{1}{2}$)}: In this case, we have $1-\tau=\sum_{r\in\bar\Gamma_*}\rho_r(1-\rho_r)^k+\sum_{r\in\Gamma}\rho_r(1-\rho_r)^k$. Combining claims \textbf{(i)} and \textbf{(ii)}: as long as we have $N\ge 132k^2\ln\left(\frac{R}{\delta}\right)$, with probability of at least $1-2\delta$, the following event holds:
    \begin{align*}
        1-\widehat\tau &\ge \sum_{r\in\bar\Gamma_*}\widehat\rho_r(1-\widehat\rho_r)^k \ge \frac{1}{2}\sum_{r\in\bar\Gamma_*}\rho_r(1-\rho_r)^k \qquad \textup{(By claims \textbf{(i)}, \textbf{(ii)})} \\
        &\ge\frac{1}{2}\left[(1-\tau) - \sum_{r\in\Gamma}\rho_r(1-\rho_r)^k\right]\ge \frac{1-\tau}{4}\qquad\Big(\sum_{r\in\Gamma}\rho_r(1-\rho_r)^k\le\frac{1-\tau}{2}\Big).
    \end{align*}
    \noindent Finally, setting $\delta=\Delta/2$ yields the desired sample complexity for $\rho_{\max}\le\frac{1}{2}$ case.

    \noindent \textbf{Case 2 ($\rho_{\max}>\frac{1}{2}$)}: Denote $r_*=\arg\max_{r\in[R]}\rho_r$. Using the strategy in claim \textbf{(i)}, for all $r\in[R]$ s.t. $r\ne r_*$, we have:
    \begin{align*}
        \P\left([1-\widehat\rho_{r}]^k\le\frac{\rho_{\max}^k}{\sqrt{2}}\right) &\le \P\left([1-\widehat\rho_{r}]^k \le \frac{[1-\rho_{r}]^k}{\sqrt{2}}\right) \qquad (1-\rho_r\ge\rho_{\max},\forall r\ne r_*) \\
        &\le \exp\left(-\frac{N}{32k^2}\right)\qquad\textup{(Hoeffding's Inequality)}.
    \end{align*}
    \noindent Therefore, by the union bound, we have:
    \begin{align*}
        \P\left(\exists r\ne r_*:[1-\widehat\rho_{r}]^k\le\frac{\rho_{\max}^k}{\sqrt{2}}\right) \le (R-1)\exp\left(-\frac{N}{32k^2}\right).
    \end{align*}
    \noindent Hence, as long as $N\ge 32k^2\ln\left(\frac{R-1}{\delta}\right)$, with probability of at least $1-\delta$, we have:
    \begin{align*}
        \sum_{r\ne r_*}\widehat\rho_r(1-\widehat\rho_r)^k \ge \frac{1}{\sqrt{2}}\rho_{\max}^k\sum_{r\ne r_*}\widehat\rho_r = \frac{1}{\sqrt{2}}\rho_{\max}^k(1-\widehat\rho_{r_*}).
    \end{align*}
    \noindent Using multiplicative Chernoff's bound again on $1-\widehat\rho_{r_*}$, we have:
    \begin{align*}
        \P\left(1-\widehat\rho_{r_*}\le\frac{1-\rho_{\max}}{\sqrt{2}}\right) &\le \exp\left(-\frac{N}{2}\left(1-\frac{1}{\sqrt 2}\right)^2(1-\rho_{\max})\right) \le \exp\left(-\frac{N(1-\rho_{\max})}{25}\right).
    \end{align*}
    \noindent Therefore, as long as $N\ge 25(1-\rho_{\max})^{-1}\ln\left(\frac{1}{\delta}\right)$, with probability of at least $1-\delta$, we have $1-\widehat\rho_{r_*}\ge\frac{1-\rho_{\max}}{\sqrt{2}}$. Combine this with the above, as long as $N\ge 32k^2\ln\left(\frac{R-1}{\delta}\right)+25(1-\rho_{\max})^{-1}\ln\left(\frac{1}{\delta}\right)$, with probability of at least $1-2\delta$, we have:
    \begin{align*}
        \sum_{r\ne r_*}\widehat\rho_r(1-\widehat\rho_r)^k\ge\frac{1}{2}\rho_{\max}^k(1-\rho_{\max})\ge \frac{1}{2}\rho_{\max}(1-\rho_{\max})^k \qquad(1-\rho_{\max}\le \rho_{\max}).
    \end{align*}
    \noindent Combine with claims \textbf{(i)}, \textbf{(ii)}, as long as $N\ge 132k^2\ln\left(\frac{R}{\delta}\right)+32k^2\ln\left(\frac{R-1}{\delta}\right)+\frac{25\ln\left(\frac{1}{\delta}\right)}{1-\rho_{\max}}$, wp. of at least $1-4\delta$, we have:
    \begin{align*}
        1-\widehat\tau &\ge\frac{1}{2}\left[\frac{1}{2}\rho_{\max}(1-\rho_{\max})^k+\frac{1}{2}\sum_{r\in\bar\Gamma_*}\rho_r(1-\rho_r)^k\right]= \frac{1}{4}\left[\rho_{\max}(1-\rho_{\max})^k+\sum_{r\in\bar\Gamma_*}\rho_r(1-\rho_r)^k\right] \ge \frac{1-\tau}{8}.
    \end{align*}
    \noindent Finally, setting $\delta=\Delta/4$ yields the desired sample complexity for case $\rho_{\max}>\frac{1}{2}$.
\end{proof}

\begin{proposition}
    \label{prop:third_multiplicative_control_of_tauhat}
    Let $\tau$ be the collision probability defined in Eqn.~\eqref{eq:collision_probability} and $\widehat\tau$ be the plug-in estimator in Eqn.~\eqref{eq:collision_probability_estimator}. Suppose $\sum_{r:\rho_r\le\frac{1}{4k^2}}\rho_r(1-\rho_r)^k\ge \frac{1-\tau}{2}$. Then, for any $\Delta\in(0,1)$, with probability of at least $1-\Delta$:
    \begin{align}
        1-\widehat\tau &\ge \frac{1-\tau}{8}\quad\textup{ as long as } N\ge 12k^2\ln\left(\frac{2R}{\Delta}\right) + \frac{16\ln\left(\frac{2}{\Delta}\right)}{1-\tau}. 
    \end{align}
\end{proposition}

\begin{proof}
    Let us define the following indices sets:
    \begin{align*}
        \Gamma := \Big\{r\in[R]:\rho_r \le \frac{1}{4k^2}\Big\}, \qquad\widehat\Gamma:=\Big\{r\in[R]:\widehat\rho_r\le \frac{1}{2k^2}\Big\}.
    \end{align*}
    \noindent Then, by the initial assumption, we have $\sum_{r\in\Gamma}\rho_r(1-\rho_r)^k\ge\frac{1-\tau}{2}$.

    \textbf{Claim}: For any $\delta\in(0,1)$, $\P(\Gamma\subseteq\widehat\Gamma)\ge1-\delta$ as long as $N\ge 12k^2\ln\left(\frac{R}{\delta}\right)$.

    Let us fix any $r\in\Gamma$. Then, we can write $\widehat\rho_r=\frac{1}{N}\sum_{j=1}^NX_j$ where $X_j\sim\mathrm{Bern}(\rho_r)$. Then, by coupling, we have $X_j\overset{d}{=}\1{U_j\le \rho_r}$ where $U_j\sim\mathrm{Uniform}(0,1)$ for all $j\in[N]$. Therefore:
    \begin{align*}
        \P\left(\widehat\rho_r\ge\frac{1}{2k^2}\right) &= \P\left(\frac{1}{N}\sum_{j=1}^NX_j\ge\frac{1}{2k^2}\right) = \P\left(\frac{1}{N}\sum_{j=1}^N\1{U_j\le\rho_r}\ge\frac{1}{2k^2}\right) \\
        &\underset{(\text{Coupling)}}{\le} \P\left(\frac{1}{N}\sum_{j=1}^N\1{U_j\le\frac{1}{4k^2}}\ge\frac{1}{2k^2}\right)\le \exp\left(-\frac{N}{12k^2}\right)\qquad\textup{(Multiplicative Chernoff)}.
    \end{align*}
    \noindent By the union bound:
    \begin{align*}
        \P\left(\exists r\in\Gamma:\widehat\rho_r\ge\frac{1}{2k^2}\right) &\le |\Gamma|\exp\left(-\frac{N}{12k^2}\right)\le R\exp\left(-\frac{N}{12k^2}\right).
    \end{align*}
    \noindent As a result, if we let $\delta\ge Re^{-N/12k^2}$ then as long as $N\ge12k^2\ln\left(\frac{R}{\delta}\right)$, we have $\widehat\rho_r\le\frac{1}{2k^2},\forall r\in\Gamma$ with probability of at least $1-\delta$. In other words, $\Gamma\subseteq\widehat\Gamma$ with probability of at least $1-\delta$.

    \noindent From the above claim, under the high probability event that $\Gamma\subseteq\widehat\Gamma$, we have:
    \begin{align*}
        1-\widehat\tau &\ge \sum_{r\in\widehat\Gamma}\widehat\rho_r(1-\widehat\rho_r)^k \ge \left(1-\frac{1}{2k^2}\right)^k\sum_{r\in\widehat\Gamma}\widehat\rho_r\ge\left(1-\frac{1}{2k^2}\right)^k\sum_{r\in\Gamma}\widehat\rho_r.
    \end{align*}
    \noindent Now, by the multiplicative Chernoff bound:
    \begin{align*}
        \P\left(\sum_{r\in\Gamma}\widehat\rho_r\le\frac{1}{2}\sum_{r\in\Gamma}\rho_r\right) &\le \exp\left(-\frac{N}{8}\sum_{r\in\Gamma}\rho_r\right) \le \exp\left(-\frac{N}{8}\sum_{r\in\Gamma}\rho_r(1-\rho_r)^k\right) \le \exp\left(-\frac{N(1-\tau)}{16}\right).
    \end{align*}
    \noindent Hence, if we let $\delta\ge e^{-N(1-\tau)/16}$ then as long as $N\ge 16(1-\tau)^{-1}\ln\left(\frac{1}{\delta}\right)$ we have $\sum_{r\in\Gamma}\widehat\rho_r\ge\frac{1}{2}\sum_{r\in\Gamma}\rho_r$ with probability of at least $1-\delta$. Combining the above with \textbf{claim (i)}, as long as we have $N\ge 12k^2\ln\left(\frac{R}{\delta}\right)+16(1-\tau)^{-1}\ln\left(\frac{1}{\delta}\right)$, we have:
    \begin{align*}
    	1-\widehat\tau 
        &\ge \left(1-\frac{1}{2k^2}\right)^k \sum_{r\in\Gamma}\widehat\rho_r \ge \frac{1}{2}\left(1-\frac{1}{2k^2}\right)^k\sum_{r\in\Gamma}\rho_r \qquad \Big(\sum_{r\in\Gamma}\widehat\rho_r\ge\frac{1}{2}\sum_{r\in\Gamma}\rho_r \textup{ wp. }\ge 1-\delta\Big) \\
	    &\ge \frac{1}{2}\left(1-\frac{1}{2k^2}\right)^k\sum_{r\in\Gamma}\rho_r(1-\rho_r)^k \\
	    &\ge \frac{1-\tau}{4}\left(1-\frac{1}{2k^2}\right)^k \ge \frac{1-\tau}{8},
    \end{align*}
    \noindent with probability of at least $1-2\delta$. Setting $\delta=\Delta/2$ yields the desired sample complexity.
\end{proof}

\begin{remark}
    In Table \ref{tab:multiplicative_concentration_summary}, we summarize the multiplicative concentration results for the plug-in estimator $\widehat\tau$. We note that under the assumptions of Proposition \ref{prop:second_multiplicative_control_of_tauhat}, we can easily show that $1-\tau\lesssim 1-\rho_{\max}$. Specifically, suppose $k\ge2$, we have:
    \begin{align*}
        \frac{1-\tau}{2} &\le \sum_{r:\rho_r>\frac{1}{4k^2}}\rho_r(1-\rho_r)^k = \rho_{\max}(1-\rho_{\max})^k+\sum_{\substack{r:\rho_r>\frac{1}{4k^2},\rho_r\ne\rho_{\max}}}\rho_r(1-\rho_r)^k \\
        &\le \frac{1}{4}(1-\rho_{\max})^{k-1} + \sum_{\substack{r:\rho_r>\frac{1}{4k^2},\rho_r\ne\rho_{\max}}}\rho_r(1-\rho_r)^k \qquad\Big(x(1-x)\le\frac{1}{4},\forall x\in[0,1]\Big) \\
        &\le \frac{1}{4}\cdot\frac{1-\rho_{\max}}{2^{k-2}} + \left(1-\frac{1}{4k^2}\right)^k\sum_{\substack{r:\rho_r>\frac{1}{4k^2},\rho_r\ne\rho_{\max}}}\rho_r  \qquad\Big(\text{Since } \rho_{\max}\ge\frac{1}{2}\Big) \\
        &\le \frac{1}{2^k}(1-\rho_{\max}) + \left(1-\frac{1}{4k^2}\right)^k(1-\rho_{\max}) \\
        &= (1-\rho_{\max})\left[\frac{1}{2^k}+\left(1-\frac{1}{4k^2}\right)^k\right].
    \end{align*}
    \noindent When $k=1$ then we have $\rho_{\max}(1-\rho_{\max})^k=\rho_{\max}(1-\rho_{\max})\le 1-\rho_{\max}$. Therefore, for any value of $k\ge 1$, we have the following general upper bound:
    \begin{align*}
        \frac{1-\tau}{2} &\le (1-\rho_{\max})\left[\frac{1}{2^{k-1}}+\left(1-\frac{1}{4k^2}\right)^k\right] \le 2(1-\rho_{\max}).
    \end{align*}
    \noindent Hence, \textbf{in all cases}, $1-\tau\le 4(1-\rho_{\max})$ or $\frac{1}{1-\tau}\gtrsim\frac{1}{1-\rho_{\max}}$. Therefore, in both Proposition \ref{prop:second_multiplicative_control_of_tauhat} and Proposition \ref{prop:third_multiplicative_control_of_tauhat}, the sample complexity can admit the following worst-case order:
    \begin{align*}
        N\ge\mathcal{O}\left(k^2\ln\left(\frac{R}{\Delta}\right) + \frac{\ln\left(\frac{1}{\Delta}\right)}{1-\tau}\right).
    \end{align*}
    \noindent Therefore, we have the following final general result for multiplicative concentration of $1-\widehat\tau$.
\end{remark} 

\begin{proposition}
	\label{prop:final_multiplicative_control_of_tauhat}
	Let $\tau$ be the collision probability defined in Eqn.~\eqref{eq:collision_probability} and $\widehat\tau$ be the plug-in estimator in Eqn.~\eqref{eq:collision_probability_estimator}. Then, for any $\Delta\in(0,1)$, with probability of at least $1-\Delta$:
	\begin{align}
		1-\widehat\tau &\ge \frac{1-\tau}{8}\quad\textup{ as long as } N \ge 164k^2\ln\left(\frac{4R}{\Delta}\right) + \frac{100\ln\left(\frac{2}{\Delta}\right)}{1-\tau}.
	\end{align}
\end{proposition}

\begin{proof}
	Combining Proposition \ref{prop:second_multiplicative_control_of_tauhat} and Proposition \ref{prop:third_multiplicative_control_of_tauhat}, we have $1-\widehat\tau\ge\frac{1-\tau}{8}$ with probability of at least $1-\Delta$ as long as $N$ is large enough to cover both cases when $\sum_{r:\rho_r\le\frac{1}{4k^2}}\rho_r(1-\rho_r)^k\ge\frac{1-\tau}{2}$ and $\sum_{r:\rho_r\le\frac{1}{4k^2}}\rho_r(1-\rho_r)^k\le\frac{1-\tau}{2}$. Taking the max of minimum required sample complexity for both cases, we have:
	\begin{align*}
		&\max\left\{164k^2\ln\left(\frac{4R}{\Delta}\right) + \frac{25\ln\left(\frac{4}{\Delta}\right)}{1-\rho_{\max}}, 12k^2\ln\left(\frac{2R}{\Delta}\right) + \frac{16\ln\left(\frac{2}{\Delta}\right)}{1-\tau}\right\} \\
		&\le 164k^2\ln\left(\frac{4R}{\Delta}\right) +  \max\left\{\frac{25\ln\left(\frac{4}{\Delta}\right)}{1-\rho_{\max}},\frac{16\ln\left(\frac{2}{\Delta}\right)}{1-\tau} \right\} \\
		&\le 164k^2\ln\left(\frac{4R}{\Delta}\right) + 50\ln\left(\frac{2}{\Delta}\right) \max\left\{\frac{1}{1-\rho_{\max}}, \frac{1}{1-\tau}\right\} \qquad \left(\ln\left(\frac{4}{\Delta}\right)\le 2\ln\left(\frac{2}{\Delta}\right)\right) \\
		&\le 164k^2\ln\left(\frac{4R}{\Delta}\right) + \frac{100\ln\left(\frac{2}{\Delta}\right)}{1-\tau}\qquad (1-\tau\le 1-\rho_{\max}^2\le 2(1-\rho_{\max})).
	\end{align*}
	\noindent Therefore, as long as $N\ge164k^2\ln\left(\frac{4R}{\Delta}\right) + \frac{100\ln\left(\frac{2}{\Delta}\right)}{1-\tau}$, we have $1-\widehat\tau\ge\frac{1-\tau}{8}$ with probability of at least $1-\Delta$.
\end{proof}

\begin{table}[t]
\centering
\caption{Multiplicative concentration results of the form $\P\Big(1-\widehat\tau\ge c(1-\tau)\Big)\ge 1-\Delta$.}
\label{tab:multiplicative_concentration_summary}
\begin{tabular}{lccc}
\toprule
\textbf{Prop.} & \textbf{Constant ($c$)} & \textbf{Assumption} & \textbf{Sample Complexity} \\
\midrule
\ref{prop:first_multiplicative_control_of_tauhat} & $\frac{C}{2}\left[1+\frac{1-\gamma_{4k}}{\gamma_{4k}}\bar C\right]^{-1}$ & N/A & $\mathcal{O}\left(\gamma_{4k}^{-1}\ln\left(\frac{1}{\Delta}\right)+k\ln\left(\frac{R}{\Delta}\right)\right)$ \\
\midrule
\ref{prop:final_multiplicative_control_of_tauhat} & $\frac{1}{8}$ & N/A & $\mathcal{O}\left(k^2\ln\left(\frac{R}{\Delta}\right)+\frac{\ln\left(\frac{1}{\Delta}\right)}{1-\tau}\right)$ \\
\midrule
\ref{prop:second_multiplicative_control_of_tauhat} & $\frac{1}{8}$ & $\sum_{r:\rho_r\le\frac{1}{4k^2}}\rho_r(1-\rho_r)^k\le\frac{1-\tau}{2}$ & $\mathcal{O}\left(k^2\ln\left(\frac{R}{\Delta}\right)+\frac{\ln\left(\frac{1}{\Delta}\right)}{1-\rho_{\max}}\right)$ \\
\midrule
\ref{prop:third_multiplicative_control_of_tauhat} & $\frac{1}{8}$ & $\sum_{r:\rho_r\le\frac{1}{4k^2}}\rho_r(1-\rho_r)^k\ge\frac{1-\tau}{2}$ & $\mathcal{O}\left(k^2\ln\left(\frac{R}{\Delta}\right)+\frac{\ln\left(\frac{1}{\Delta}\right)}{1-\tau}\right)$ \\
\bottomrule
\end{tabular}
\end{table}

\section{Proof of the Main Results}
\begin{remark}[On the Boundedness of worst-case Rademacher complexities]
    \label{remark:bound_on_worst_case_rc}
    Before presenting the proofs, we would like to make some remarks about the boundedness of $\RC_{\mu_*}(\mathcal{H})$ and $\RC_{\nu_*}(\mathcal{H})$, which are both important objects to the main results. Recall that $\mathfrak{R}_{\mu_*}(\mathcal{H})$ denotes the worst-case expected Rademacher complexity over the distribution set $\mathcal{U}$ defined as follows:
    \begin{align*}
        \mathcal{U}=\Bigg\{
            \bigotimes_{r=1}^R\Big[\mathcal{D}_r^{\otimes 2}\otimes\mathcal{\bar D}^{\otimes k}\Big]^{\otimes\lfloor q_r/2\rfloor} &: q_r\in \mathbb{Z}_{\ge0}, \forall r\in[R]\textup{ and } \sum_{r=1}^Rq_r = n
        \Bigg\}.
    \end{align*}
    \noindent Essentially, $\mathcal{U}$ denotes the space of distributions over $(k+2)$-tuples collections with at least $\frac{n-R}{2}$ and at most $\frac{n}{2}$ tuples. For any $\mu\in\mathcal{U}$ and $S_\mathrm{tup}\sim\mu$. Suppose that $|S_\mathrm{tup}|=\widetilde n$, we have $\frac{n-R}{2}\le\widetilde n\le \frac{n}{2}$. By Dudley entropy integral (Theorem \ref{thm:dudley}):
    \begin{align*}
        \ERC_{S_\mathrm{tup}}(\mathcal{H}) &\le \frac{4}{n} + \frac{12}{\sqrt{\widetilde n}}\int_{\frac{1}{n}}^\mathcal{B}\sqrt{\ln 2\mathcal{N}(\mathcal{H}, \epsilon, \Lnorm_2(S_\mathrm{tup}))}d\epsilon \\
        &\le \frac{4}{n}+\frac{12\sqrt{2}}{\sqrt{n-R}}\int_\frac{1}{n}^\mathcal{B}\sqrt{\ln 2\mathcal{N}(\mathcal{H}, \epsilon, \Lnorm_2(S_\mathrm{tup}))}d\epsilon \\
        &\le \frac{4}{n} + \frac{12\sqrt{2}}{\sqrt{n-R}}\sup_{S^*_\mathrm{tup}\in(\X^{k+2})^{\widetilde n}}\int_\frac{1}{n}^\mathcal{B}\sqrt{\ln 2\mathcal{N}(\mathcal{H}, \epsilon, \Lnorm_2(S^*_\mathrm{tup}))}d\epsilon \\
        &\le \frac{4}{n} + \frac{12\sqrt{2}}{\sqrt{n-R}}\sup_{S^*_\mathrm{tup}\in(\X^{k+2})^{n}}\int_\frac{1}{n}^\mathcal{B}\sqrt{\ln 2\mathcal{N}(\mathcal{H}, \epsilon, \Lnorm_2(S^*_\mathrm{tup}))}d\epsilon = \frac{4}{n} + \frac{12\sqrt{2}\mathfrak{C}_n(\mathcal{H})}{\sqrt{n-R}}\\
        &\le \frac{4}{n} + \frac{12\sqrt{2}\mathfrak{C}_N(\mathcal{H})}{\sqrt{n-R}}.
    \end{align*}
    \noindent Therefore, we have the data-independent bound as follows:
    \begin{align}
        \forall \mu\in\mathcal{U}:\mathfrak{R}_\mu(\mathcal{H}) &= \E_{S_\mathrm{tup}\sim\mu}\left[\ERC_{S_\mathrm{tup}}(\mathcal{H}) \right]\le \frac{4}{n} + \frac{12\sqrt{2}\mathfrak{C}_N(\mathcal{H})}{\sqrt{n-R}}, \\
        \forall \nu\in\mathcal{V}:\mathfrak{R}_\nu(\mathcal{H}) &= \E_{S_\mathrm{tup}\sim\nu}\left[\ERC_{S_\mathrm{tup}}(\mathcal{H}) \right]\le \frac{4}{m} + \frac{12\sqrt{3}\mathfrak{C}_N(\mathcal{H})}{\sqrt{m-2R}} \qquad \text{(Similar arguments)}.
    \end{align}
    \noindent We can use the tighter complexity terms $\mathfrak{C}_n(\mathcal{H}), \mathfrak{C}_m(\mathcal{H})$ to bound $\RC_{\mu_*}(\mathcal{H})$ and $\RC_{\nu_*}(\mathcal{H})$. However, the improvement is at most logarithmic in $k$. Hence, we use $\mathfrak{C}_N(\mathcal{H})$ for both terms so that the subsequent results simplify elegantly.
\end{remark}
\begin{remark}
    We note that all main results of this work can be made tighter via \textbf{worst-case Rademacher complexity}. Specifically, given a class $\mathcal{G}$ of functions $g:\mathcal{Z}\to\R$, the worst-case Rademacher complexity $\mathfrak{R}_m^\mathrm{wc}(\mathcal{G})$ is defined as:
    \begin{align}
        \mathfrak{R}^\mathrm{wc}_m(\mathcal{G}) := \sup_{S\in\mathcal{Z}^m}\mathfrak{\widehat R}_S(\mathcal{G}).
    \end{align}
    \noindent Then, for any $\mu\in\mathcal{U}$ and $S_\mathrm{tup}\sim\mu$, $|S_\mathrm{tup}|=\widetilde n$, we have:
    \begin{align}
        \forall \mu\in\mathcal{U}:\mathfrak{R}_\mu(\mathcal{H}) &= \E_{S_\mathrm{tup}\sim\mu}\big[\ERC_{S_\mathrm{tup}}(\mathcal{H}) \big]\le \mathfrak{R}_{\widetilde n}^\mathrm{wc}(\mathcal{H}) \le \mathfrak{R}_{\lceil(n-R)/2\rceil}^\mathrm{wc}(\mathcal{H}) \qquad\text{(By monotonicity)}.
    \end{align}
    \noindent Then, the worst-case complexity $\mathfrak{R}_{\lceil (n-R)/2\rceil}^\mathrm{wc}(\mathcal{H})$ can be bounded by metric entropy (as we do in this work) or using the contraction principle \citep{article:maurer2016,foster2019ellinftyvectorcontractionrademacher, article:lei2023, article:lei2026, arxiv:lei2026}.
\end{remark}
\begin{theorem}
    \label{thm:main_result_appendix}
    Let $\mathcal{F}$ be a class of representation functions and let $\bar U_N(f)$ be defined for each $f\in\mathcal{F}$ in Eqn.~\eqref{eq:combined_aux_estimators}. Then, for any $\Delta\in(0,1)$, we have:
    \begin{align}
        &\sup_{f\in\mathcal{F}}\left|\bar U_N(f)-\mathrm{\widetilde L}_\phi(f)\right| \le \frac{2\RC_{\mu_*}(\mathcal{H})+2\widehat\tau\RC_{\nu_*}(\mathcal{H})}{1-\widehat\tau} +  \frac{8\mathcal{B}}{1-\widehat\tau}\left[\sqrt{\frac{\ln 12/\Delta}{n-R}}+\sqrt{\frac{\ln 12/\Delta}{m-2R}}\right] \nonumber \\
        &\qquad\qquad + \frac{2\mathcal{B}}{1-\widehat\tau}\left[\frac{3R}{n-R}+\frac{6R}{m-2R}+\frac{8R}{3N}\ln\left(\frac{24R}{\Delta}\right) + \sqrt{\frac{8(R-1)\ln \left(\frac{24R}{\Delta}\right)}{3N}}\right] \\
        &\qquad\qquad + \frac{4\mathcal{B}\sqrt{2R}}{C(1-\widehat\tau)^2}\left[1+\frac{1-\widehat\gamma_{4k}}{\widehat\gamma_{4k}}\bar C\right]\left[\frac{2\ln \left(\frac{6(R+1)}{\Delta}\right)}{3N} + \sqrt{\frac{2(1-\|\rho\|_2^2)\ln \left(\frac{6(R+1)}{\Delta}\right)}{N}}\right], \nonumber
    \end{align}
    \noindent with probability of at least $1-\Delta$ as long as $N\ge 3\gamma_{2k}^{-1}\ln\left(\frac{12}{\Delta}\right)+16k\ln\left(\frac{12R}{\Delta}\right)$  where the constants $C, \bar C,\gamma_{2k}, \widehat\gamma_{4k}$ are defined in Proposition \ref{prop:first_multiplicative_control_of_tauhat}. Furthermore, when $R\ge k$, as long as $N\ge k\times\max\left\{\frac{3\gamma_{2k}^{-1}\ln(12/\Delta)}{k}+16\ln(12R/\Delta), 9R\right\}$:
    \begin{align}
        \sup_{f\in\mathcal{F}}\left|\bar U_N(f)-\mathrm{\widetilde L}_\phi(f)\right| \le\mathcal{O}\Bigg[\frac{\mathfrak{C}_N(\mathcal{H})}{1-\widehat\tau}\sqrt{\frac{k}{N}}+\frac{\mathcal{B}\sqrt{R}}{(1-\widehat\tau)^2}\left[1+\frac{1-\widehat\gamma_{4k}}{\widehat\gamma_{4k}}\right]\sqrt{\frac{\ln R/\Delta}{N}}+\frac{\mathcal{B}R(k+\ln R/\Delta)}{N(1-\widehat\tau)}\Bigg],
    \end{align}
    \noindent with probability of at least $1-\Delta$.
\end{theorem}

\begin{proof}
    We have:
    \begin{align*}
        \sup_{f\in\mathcal{F}}\left|\bar U_N(f)-\mathrm{\widetilde L}_\phi(f)\right| &= \sup_{f\in\mathcal{F}}\left|\frac{1}{1-\widehat\tau}\bar U_\Omega(f)-\frac{\widehat\tau}{1-\widehat\tau}\bar U_\Lambda(f) - \frac{1}{1-\tau} \Lnorm_\Omega(f)+\frac{\tau}{1-\tau}\Lnorm_\Lambda(f)\right| \\
        &\le \underbrace{\sup_{f\in\mathcal{F}}\left|\frac{1}{1-\widehat\tau}\bar U_\Omega(f)-\frac{1}{1-\tau}\Lnorm_\Omega(f)\right|}_{\alpha} + \underbrace{\sup_{f\in\mathcal{F}}\left|\frac{\widehat\tau}{1-\widehat\tau}\bar U_\Lambda(f)-\frac{\tau}{1-\tau}\Lnorm_\Lambda(f)\right|}_{\beta}.
    \end{align*}
    \noindent Now, we proceed to bound the $\alpha$ and $\beta$ terms. Firstly, we have:
    \begin{align*}
        \alpha &\le \frac{1}{1-\widehat\tau} \sup_{f\in\mathcal{F}}\left|\bar U_\Omega(f) - \Lnorm_\Omega(f)\right| + \sup_{f\in\mathcal{F}}|\Lnorm_\Omega(f)|\cdot\left|\frac{1}{1-\widehat\tau} - \frac{1}{1-\tau}\right| 
        \\
        &\le \frac{1}{1-\widehat\tau}\sup_{f\in\mathcal{F}}\left|\bar U_\Omega(f) - \Lnorm_\Omega(f)\right| + \frac{\mathcal{B}|\widehat\tau-\tau|}{(1-\tau)(1-\widehat\tau)}.
    \end{align*}
    \noindent By Proposition \ref{prop:concentration_of_collision_probability_estimator} (absolute error of $\widehat\tau$) and Theorem \ref{thm:concentration_of_bar_Uomega} (absolute error of $\bar U_\Omega(f)$), we have: 
    \begin{align*}
        \alpha &\le \frac{1}{1-\widehat\tau}\left\{2\RC_{\mu_*}(\mathcal{H}) + \mathcal{B}\left[\frac{6R}{n-R}+\frac{8R}{3N}\ln4R/\delta + \sqrt{\frac{8(R-1)\ln 4R/\delta}{3N}} + 8\sqrt{\frac{\ln 2/\delta}{n-R}}\right]\right\} \\
        &\qquad + \frac{\mathcal{B}\sqrt{2R}}{(1-\tau)(1-\widehat\tau)}\left[\frac{2\ln (R+1)/\delta}{3N} + \sqrt{\frac{2(1-\|\rho\|_2^2)\ln (R+1)/\delta}{N}}\right],
    \end{align*}
    \noindent with probability of at least $1-2\delta$ (by the union bound). Furthermore, by Proposition \ref{prop:first_multiplicative_control_of_tauhat} (relative error of $\widehat\tau$), we have $1-\tau\ge\frac{C}{2}\left[1+\frac{1-\widehat\gamma_{4k}}{\widehat\gamma_{4k}}\bar C\right]^{-1}(1-\widehat\tau),$ with probability of at least $1-\delta$ as long as $N\ge 3\gamma_{2k}^{-1}\ln\left(\frac{2}{\delta}\right)+16k\ln\left(\frac{2R}{\delta}\right)$. As a result, we have:
    \begin{align*}
        \alpha &\le \frac{1}{1-\widehat\tau}\left\{2\RC_{\mu_*}(\mathcal{H}) + \mathcal{B}\left[\frac{6R}{n-R}+\frac{8R}{3N}\ln4R/\delta + \sqrt{\frac{8(R-1)\ln 4R/\delta}{3N}} + 8\sqrt{\frac{\ln 2/\delta}{n-R}}\right]\right\} \\
        &\qquad + \frac{2\mathcal{B}\sqrt{2R}}{C(1-\widehat\tau)^2}\left[1+\frac{1-\widehat\gamma_{4k}}{\widehat\gamma_{4k}}\bar C\right]\left[\frac{2\ln (R+1)/\delta}{3N} + \sqrt{\frac{2(1-\|\rho\|_2^2)\ln (R+1)/\delta}{N}}\right],
    \end{align*}
    \noindent with probability of at least $1-3\delta$ (by union bound) as long as $N\ge3\gamma_{2k}^{-1}\ln\left(\frac{2}{\delta}\right)+16k\ln\left(\frac{2R}{\delta}\right)$. Similarly, we can bound $\beta$ using the concentration result of $\bar U_\Lambda(f)$ in Theorem \ref{thm:concentration_of_bar_Ulambda} and Proposition \ref{prop:concentration_of_collision_probability_estimator}, Proposition \ref{prop:first_multiplicative_control_of_tauhat}. Specifically, we have:
    \begin{align*}
        \beta &\le \frac{\widehat\tau}{1-\widehat\tau}\sup_{f\in\mathcal{F}}\left|\bar U_\Lambda(f) - \Lnorm_\Lambda(f)\right| + \frac{\mathcal{B}|\widehat\tau-\tau|}{(1-\tau)(1-\widehat\tau)} \\
        &\le \frac{\widehat\tau}{1-\widehat\tau}\left\{2\RC_{\nu_*}(\mathcal{H}) + \mathcal{B}\left[\frac{12R}{m-2R}+\frac{8R}{3N}\ln4R/\delta + \sqrt{\frac{8(R-1)\ln 4R/\delta}{3N}} + 8\sqrt{\frac{\ln 2/\delta}{m-2R}}\right]\right\} \\
        & \qquad + \frac{2\mathcal{B}\sqrt{2R}}{C(1-\widehat\tau)^2}\left[1+\frac{1-\widehat\gamma_{4k}}{\widehat\gamma_{4k}}\bar C\right]\left[\frac{2\ln (R+1)/\delta}{3N} + \sqrt{\frac{2(1-\|\rho\|_2^2)\ln (R+1)/\delta}{N}}\right] \\
		&\le \frac{2\widehat\tau\RC_{\nu_*}(\mathcal{H})}{1-\widehat\tau} + \frac{\mathcal{B}}{1-\widehat\tau}\left[\frac{12R}{m-2R}+\frac{8R}{3N}\ln4R/\delta + \sqrt{\frac{8(R-1)\ln 4R/\delta}{3N}} + 8\sqrt{\frac{\ln 2/\delta}{m-2R}}\right] \quad (\widehat\tau\le 1) \\
		& \qquad + \frac{2\mathcal{B}\sqrt{2R}}{C(1-\widehat\tau)^2}\left[1+\frac{1-\widehat\gamma_{4k}}{\widehat\gamma_{4k}}\bar C\right]\left[\frac{2\ln (R+1)/\delta}{3N} + \sqrt{\frac{2(1-\|\rho\|_2^2)\ln (R+1)/\delta}{N}}\right],
    \end{align*}
    \noindent with probability of at least $1-3\delta$ as long as $N\ge 3\gamma_{2k}^{-1}\ln\left(\frac{2}{\delta}\right)+16k\ln\left(\frac{2R}{\delta}\right)$. Combining the bounds of $\alpha$ and $\beta$, we have:
    \begin{align*}
        \sup_{f\in\mathcal{F}}\left|\bar U_N(f)-\mathrm{\widetilde L}_\phi(f)\right| &\le \frac{2\RC_{\mu_*}(\mathcal{H})+2\widehat\tau\RC_{\nu_*}(\mathcal{H})}{1-\widehat\tau} +  \frac{8\mathcal{B}}{1-\widehat\tau}\left[\sqrt{\frac{\ln 2/\delta}{n-R}}+\sqrt{\frac{\ln 2/\delta}{m-2R}}\right] \\
        &+ \frac{2\mathcal{B}}{1-\widehat\tau}\left[\frac{3R}{n-R}+\frac{6R}{m-2R}+\frac{8R}{3N}\ln4R/\delta + \sqrt{\frac{8(R-1)\ln 4R/\delta}{3N}}\right] \\
        &+ \frac{4\mathcal{B}\sqrt{2R}}{C(1-\widehat\tau)^2}\left[1+\frac{1-\widehat\gamma_{4k}}{\widehat\gamma_{4k}}\bar C\right]\left[\frac{2\ln (R+1)/\delta}{3N} + \sqrt{\frac{2(1-\|\rho\|_2^2)\ln (R+1)/\delta}{N}}\right] \\
        &\le \mathcal{O}\Bigg[\frac{\RC_{\mu_*}(\mathcal{H})+\widehat\tau\RC_{\nu_*}(\mathcal{H})}{1-\widehat\tau}+\frac{\mathcal{B}}{1-\widehat\tau}\left[\sqrt{\frac{\ln 1/\delta}{n-R}}+\sqrt{\frac{\ln 1/\delta}{m-2R}}+\frac{R}{n-R}+\frac{R}{m-2R} +\frac{R\ln R/\delta}{N}\right]\\
        &\qquad\qquad\qquad\qquad +\frac{\mathcal{B}\sqrt{R}}{(1-\widehat\tau)^2}\left[1+\frac{1-\widehat\gamma_{4k}}{\widehat\gamma_{4k}}\right]\sqrt{\frac{(1-\|\rho\|_2^2)\ln R/\delta}{N}}\Bigg],
    \end{align*}
    \noindent with probability of at least $1-6\delta$ as long as $N\ge 3\gamma_{2k}^{-1}\ln\left(\frac{2}{\delta}\right)+16k\ln\left(\frac{2R}{\delta}\right)$. Setting $\delta=\Delta/6$ yields the desired bound. Furthermore, as long as $N\ge 9Rk$, we have $n-R\ge\frac{2N}{k+2}-2-R\ge\frac{N}{k+2}$ and therefore $\sqrt{\frac{1}{n-R}}\le\sqrt{\frac{k+2}{N}}\in\mathcal{O}\left(\sqrt{\frac{k}{N}}\right)$. Similarly, $\frac{1}{\sqrt{m-2R}}\le\mathcal{O}\Big(\sqrt{\frac{k}{N}}\Big)$. Therefore, as long as $N\ge k\times\max\left\{\frac{3\gamma_{2k}^{-1}\ln(12/\Delta)}{k}+16\ln(12R/\Delta), 9R\right\}$ (so that both $N\ge 9Rk$ and $N\ge3\gamma_{2k}^{-1}\ln\left(\frac{12}{\Delta}\right)+16k\ln\left(\frac{12R}{\Delta}\right)$ are satisfied), we have:
    \begin{align*}
        \RC_{\mu_*}(\mathcal{H}) &\le \mathcal{O}\left[\mathfrak{C}_N(\mathcal{H})\sqrt{\frac{k}{N}}\right],\qquad \RC_{\nu_*}(\mathcal{H}) \le \mathcal{O}\left[\mathfrak{C}_N(\mathcal{H})\sqrt{\frac{k}{N}}\right].
    \end{align*}
    \noindent As a result, as long as $N\ge k\times\max\left\{\frac{3\gamma_{2k}^{-1}\ln(12/\Delta)}{k}+16\ln(12R/\Delta), 9R\right\}$ and given that $R\ge k$, we have:
    \begin{align*}
        \sup_{f\in\mathcal{F}}\left|\bar U_N(f)-\mathrm{\widetilde L}_\phi(f)\right|&\le \mathcal{O}\Bigg[\frac{1+\widehat\tau}{1-\widehat\tau}\mathfrak{C}_N(\mathcal{H})\sqrt{\frac{k}{N}}+\frac{\mathcal{B}\sqrt{k}}{1-\widehat\tau}\left[\sqrt{\frac{\ln 1/\Delta}{N}}+\sqrt{\frac{\ln 1/\Delta}{N}}\right]+\frac{\mathcal{B}R(k+\ln R/\Delta)}{N(1-\widehat\tau)} \\
        &\qquad\qquad\qquad\qquad\qquad + \frac{\mathcal{B}\sqrt{R}}{(1-\widehat\tau)^2}\left[1+\frac{1-\widehat\gamma_{4k}}{\widehat\gamma_{4k}}\right]\sqrt{\frac{(1-\|\rho\|_2^2)\ln R/\Delta}{N}}\Bigg] \\
        &\le \mathcal{O}\Bigg[\frac{\mathfrak{C}_N(\mathcal{H})}{1-\widehat\tau}\sqrt{\frac{k}{N}}+\frac{\mathcal{B}\sqrt{R}}{(1-\widehat\tau)^2}\left[1+\frac{1-\widehat\gamma_{4k}}{\widehat\gamma_{4k}}\right]\sqrt{\frac{\ln R/\Delta}{N}}+\frac{\mathcal{B}R(k+\ln R/\Delta)}{N(1-\widehat\tau)}\Bigg],
    \end{align*}
    \noindent with probability of at least $1-\Delta$, as desired.
\end{proof}

\begin{theorem}[cf. Table \ref{tab:results_comparison}]
    \label{thm:second_main_result_appendix}
    Let $\mathcal{F}$ be a class of representation functions and let $\bar U_N(f)$ be defined for each $f\in\mathcal{F}$ in Eqn.~\eqref{eq:combined_aux_estimators}. Then, for any $\Delta\in(0,1)$, we have:
    \begin{align}
        &\sup_{f\in\mathcal{F}}\left|\bar U_N(f)-\mathrm{\widetilde L}_\phi(f)\right| \le \frac{2\RC_{\mu_*}(\mathcal{H})+2\tau\RC_{\nu_*}(\mathcal{H})}{1-\tau} +  \frac{8\mathcal{B}}{1-\tau}\left[\sqrt{\frac{\ln 12/\Delta}{n-R}}+\sqrt{\frac{\ln 12/\Delta}{m-2R}}\right] \nonumber \\
        &\qquad\qquad + \frac{2\mathcal{B}}{1-\tau}\left[\frac{3R}{n-R}+\frac{6R}{m-2R}+\frac{8R}{3N}\ln\left(\frac{24R}{\Delta}\right) + \sqrt{\frac{8(R-1)\ln \left(\frac{24R}{\Delta}\right)}{3N}}\right] \\
        &\qquad\qquad + \frac{4\mathcal{B}\sqrt{2R}}{C(1-\tau)^2}\left[1+\frac{1-\gamma_{4k}}{\gamma_{4k}}\bar C\right]\left[\frac{2\ln \left(\frac{6(R+1)}{\Delta}\right)}{3N} + \sqrt{\frac{2(1-\|\rho\|_2^2)\ln \left(\frac{6(R+1)}{\Delta}\right)}{N}}\right], \nonumber
    \end{align}
    \noindent with probability of at least $1-\Delta$ as long as $N\ge 8\gamma_{4k}^{-1}\ln\left(\frac{12}{\Delta}\right)+12k\ln\left(\frac{12R}{\Delta}\right)$  where the constants $C, \bar C,\gamma_{2k}, \gamma_{4k}$ are defined in Proposition \ref{prop:first_multiplicative_control_of_tauhat}. Furthermore, when $R\ge k$, as long as $N\ge k\times\max\left\{\frac{8\gamma_{4k}^{-1}\ln(12/\Delta)}{k}+12\ln(12R/\Delta), 9R\right\}$:
    \begin{align}
        \sup_{f\in\mathcal{F}}|\bar U_N(f)-\mathrm{\widetilde L}_\phi(f)|\le \mathcal{O}\Bigg[\frac{\mathfrak{C}_N(\mathcal{H})}{1-\tau}\sqrt{\frac{k}{N}}+\frac{\mathcal{B}\sqrt{R}}{(1-\tau)^2}\left[1+\frac{1-\gamma_{4k}}{\gamma_{4k}}\right]\sqrt{\frac{\ln R/\Delta}{N}}+\frac{\mathcal{B}R(k+\ln R/\Delta)}{N(1-\tau)}\Bigg],
    \end{align}
    \noindent with probability of at least $1-\Delta$.
\end{theorem}

\begin{proof}
    Similar to Theorem \ref{thm:main_result_appendix}, we have:
    \begin{align*}
        \sup_{f\in\mathcal{F}}\left|\bar U_N(f)-\mathrm{\widetilde L}_\phi(f)\right| 
        &\le \underbrace{\sup_{f\in\mathcal{F}}\left|\frac{1}{1-\widehat\tau}\bar U_\Omega(f)-\frac{1}{1-\tau}\Lnorm_\Omega(f)\right|}_{\alpha} + \underbrace{\sup_{f\in\mathcal{F}}\left|\frac{\widehat\tau}{1-\widehat\tau}\bar U_\Lambda(f)-\frac{\tau}{1-\tau}\Lnorm_\Lambda(f)\right|}_{\beta}.
    \end{align*}
    \noindent Bounding the $\alpha$ term, we have:
    \begin{align*}
        \alpha &\le \frac{1}{1-\tau} \sup_{f\in\mathcal{F}}\left|\bar U_\Omega(f) - \Lnorm_\Omega(f)\right| + \sup_{f\in\mathcal{F}}|\bar U_\Omega(f)|\cdot\left|\frac{1}{1-\widehat\tau} - \frac{1}{1-\tau}\right| 
        \\ &\le \frac{1}{1-\tau}\sup_{f\in\mathcal{F}}\left|\bar U_\Omega(f) - \Lnorm_\Omega(f)\right| + \frac{\mathcal{B}|\widehat\tau-\tau|}{(1-\tau)(1-\widehat\tau)}.
    \end{align*}
    \noindent Then, By Proposition \ref{prop:concentration_of_collision_probability_estimator} and Theorem \ref{thm:concentration_of_bar_Uomega}, we have: 
    \begin{align*}
        \alpha &\le \frac{1}{1-\tau}\left\{2\RC_{\mu_*}(\mathcal{H}) + \mathcal{B}\left[\frac{6R}{n-R}+\frac{8R}{3N}\ln4R/\delta + \sqrt{\frac{8(R-1)\ln 4R/\delta}{3N}} + 8\sqrt{\frac{\ln 2/\delta}{n-R}}\right]\right\} \\
        &\qquad + \frac{\mathcal{B}\sqrt{2R}}{(1-\tau)(1-\widehat\tau)}\left[\frac{2\ln (R+1)/\delta}{3N} + \sqrt{\frac{2(1-\|\rho\|_2^2)\ln (R+1)/\delta}{N}}\right],
    \end{align*}
    \noindent with probability of at least $1-2\delta$ (by the union bound). Furthermore, by Proposition \ref{prop:first_multiplicative_control_of_tauhat} (relative error of $\widehat\tau$), we have $1-\widehat\tau \ge \frac{C}{2}\left[1+\frac{1-\gamma_{4k}}{\gamma_{4k}}\bar C\right]^{-1}(1-\tau)$ with probability of at least $1-\delta$ as long as $N\ge 8\gamma_{4k}^{-1}\ln\left(\frac{2}{\delta}\right)+12k\ln\left(\frac{2R}{\delta}\right)$. Therefore, we have:
    \begin{align*}
        \alpha &\le \frac{1}{1-\tau}\left\{2\RC_{\mu_*}(\mathcal{H}) + \mathcal{B}\left[\frac{6R}{n-R}+\frac{8R}{3N}\ln4R/\delta + \sqrt{\frac{8(R-1)\ln 4R/\delta}{3N}} + 8\sqrt{\frac{\ln 2/\delta}{n-R}}\right]\right\} \\
        &\qquad + \frac{2\mathcal{B}\sqrt{2R}}{C(1-\tau)^2}\left[1+\frac{1-\gamma_{4k}}{\gamma_{4k}}\bar C\right]\left[\frac{2\ln (R+1)/\delta}{3N} + \sqrt{\frac{2(1-\|\rho\|_2^2)\ln (R+1)/\delta}{N}}\right],
    \end{align*}
    \noindent with probability of at least $1-3\delta$ (by union bound) as long as $N\ge 8\gamma_{4k}^{-1}\ln\left(\frac{2}{\delta}\right)+12k\ln\left(\frac{2R}{\delta}\right)$. Similarly, we have:
    \begin{align*}
        \beta &\le \frac{\tau}{1-\tau}\left\{2\RC_{\nu_*}(\mathcal{H}) + \mathcal{B}\left[\frac{12R}{m-2R}+\frac{8R}{3N}\ln4R/\delta + \sqrt{\frac{8(R-1)\ln 4R/\delta}{3N}} + 8\sqrt{\frac{\ln 2/\delta}{m-2R}}\right]\right\} \\
        & \qquad + \frac{2\mathcal{B}\sqrt{2R}}{C(1-\tau)^2}\left[1+\frac{1-\gamma_{4k}}{\gamma_{4k}}\bar C\right]\left[\frac{2\ln (R+1)/\delta}{3N} + \sqrt{\frac{2(1-\|\rho\|_2^2)\ln (R+1)/\delta}{N}}\right] \\
        &\le \frac{2\tau\RC_{\nu_*}(\mathcal{H})}{1-\tau} + \frac{\mathcal{B}}{1-\tau}\left[\frac{12R}{m-2R}+\frac{8R}{3N}\ln4R/\delta + \sqrt{\frac{8(R-1)\ln 4R/\delta}{3N}} + 8\sqrt{\frac{\ln 2/\delta}{m-2R}}\right] \quad (\tau \le 1) \\
        & \qquad + \frac{2\mathcal{B}\sqrt{2R}}{C(1-\tau)^2}\left[1+\frac{1-\gamma_{4k}}{\gamma_{4k}}\bar C\right]\left[\frac{2\ln (R+1)/\delta}{3N} + \sqrt{\frac{2(1-\|\rho\|_2^2)\ln (R+1)/\delta}{N}}\right],
    \end{align*}
    \noindent with probability of at least $1-3\delta$ as long as $N\ge 8\gamma_{4k}^{-1}\ln\left(\frac{2}{\delta}\right)+12k\ln\left(\frac{2R}{\delta}\right)$. Combining the bounds on $\alpha,\beta$ via the union bound, we have:
    \begin{align*}
        \sup_{f\in\mathcal{F}}\left|\bar U_N(f)-\mathrm{\widetilde L}_\phi(f)\right|  &\le \frac{2\RC_{\mu_*}(\mathcal{H})+2\tau\RC_{\nu_*}(\mathcal{H})}{1-\tau} +  \frac{8\mathcal{B}}{1-\tau}\left[\sqrt{\frac{\ln 2/\delta}{n-R}}+\sqrt{\frac{\ln 2/\delta}{m-2R}}\right] \\
        &+ \frac{2\mathcal{B}}{1-\tau}\left[\frac{3R}{n-R}+\frac{6R}{m-2R}+\frac{8R}{3N}\ln4R/\delta + \sqrt{\frac{8(R-1)\ln 4R/\delta}{3N}}\right] \\
        &+ \frac{4\mathcal{B}\sqrt{2R}}{C(1-\tau)^2}\left[1+\frac{1-\gamma_{4k}}{\gamma_{4k}}\bar C\right]\left[\frac{2\ln (R+1)/\delta}{3N} + \sqrt{\frac{2(1-\|\rho\|_2^2)\ln (R+1)/\delta}{N}}\right] \\
        &\le \mathcal{O}\Bigg[\frac{\RC_{\mu_*}(\mathcal{H})+\tau\RC_{\nu_*}(\mathcal{H})}{1-\tau}+\frac{\mathcal{B}}{1-\tau}\left[\sqrt{\frac{\ln 1/\delta}{n-R}}+\sqrt{\frac{\ln 1/\delta}{m-2R}}+\frac{R}{n-R} + \frac{R}{m-2R}+\frac{R\ln R/\delta}{N}\right]\\
        &\qquad\qquad\qquad\qquad +\frac{\mathcal{B}\sqrt{R}}{(1-\tau)^2}\left[1+\frac{1-\gamma_{4k}}{\gamma_{4k}}\right]\sqrt{\frac{(1-\|\rho\|_2^2)\ln R/\delta}{N}}\Bigg],
    \end{align*}
    \noindent with probability of at least $1-6\delta$ as long as $N\ge {8}{\gamma_{4k}^{-1}}\ln\left(\frac{2}{\delta}\right)+12k\ln\left(\frac{2R}{\delta}\right)$. Setting $\delta=\Delta/6$ yields the desired complete bound. Applying the same arguments to obtain $\mathcal{O}$-notation bound in Theorem \ref{thm:main_result_appendix}: When $R\ge k$, as long as $N\ge k\times\max\left\{\frac{8\gamma_{4k}^{-1}\ln(12/\Delta)}{k}+12\ln(12R/\Delta),9R\right\}$, we have:
    \begin{align*}
        \sup_{f\in\mathcal{F}}|\bar U_N(f)-\mathrm{\widetilde L}_\phi(f)|\le \mathcal{O}\Bigg[\frac{\mathfrak{C}_N(\mathcal{H})}{1-\tau}\sqrt{\frac{k}{N}}+\frac{\mathcal{B}\sqrt{R}}{(1-\tau)^2}\left[1+\frac{1-\gamma_{4k}}{\gamma_{4k}}\right]\sqrt{\frac{\ln R/\Delta}{N}}+\frac{\mathcal{B}R(k+\ln R/\Delta)}{N(1-\tau)}\Bigg],
    \end{align*}
    \noindent with probability of at least $1-\Delta$.
\end{proof}

\begin{theorem}[cf. Theorem \ref{thm:main_result_more_general}, Table \ref{tab:results_comparison}]
	\label{thm:main_result_more_general_appendix}
	Let $\mathcal{F}$ be a class of representation functions and let $\bar U_N(f)$ be defined for each $f\in\mathcal{F}$ in Eqn.~\eqref{eq:combined_aux_estimators}. Then, for any $\Delta\in(0,1)$, we have:
	\begin{align}
		&\sup_{f\in\mathcal{F}}\left|\bar U_N(f)-\mathrm{\widetilde L}_\phi(f)\right| \le \frac{2\RC_{\mu_*}(\mathcal{H})+2\tau\RC_{\nu_*}(\mathcal{H})}{1-\tau} +  \frac{8\mathcal{B}}{1-\tau}\left[\sqrt{\frac{\ln 12/\Delta}{n-R}}+\sqrt{\frac{\ln 12/\Delta}{m-2R}}\right] \nonumber \\
		&\qquad\qquad + \frac{2\mathcal{B}}{1-\tau}\left[\frac{3R}{n-R}+\frac{6R}{m-2R}+\frac{8R}{3N}\ln\left(\frac{24R}{\Delta}\right) + \sqrt{\frac{8(R-1)\ln \left(\frac{24R}{\Delta}\right)}{3N}}\right] \\
		&\qquad\qquad + \frac{16\mathcal{B}\sqrt{2R}}{(1-\tau)^2}\left[\frac{2\ln \left(\frac{6(R+1)}{\Delta}\right)}{3N} + \sqrt{\frac{2(1-\|\rho\|_2^2)\ln \left(\frac{6(R+1)}{\Delta}\right)}{N}}\right], \nonumber
	\end{align}
	\noindent with probability of at least $1-\Delta$ as long as $N\ge 164k^2\ln\left(\frac{24R}{\Delta}\right)+ \frac{100\ln\left(\frac{12}{\Delta}\right)}{1-\tau}$. Furthermore, when $R\ge k$, as long as $N\ge k\times \max\left\{164k\ln\left(\frac{24R}{\Delta}\right)+\frac{100\ln\left(\frac{12}{\Delta}\right)}{k(1-\tau)},9R\right\}$, we have:
    \begin{align*}
        \sup_{f\in\mathcal{F}}|\bar U_N(f)-\mathrm{\widetilde L}_\phi(f)|\le \mathcal{O}\Bigg[\frac{\mathfrak{C}_N(\mathcal{H})}{1-\tau}\sqrt{\frac{k}{N}}+\frac{\mathcal{B}\sqrt{R}}{(1-\tau)^2}\sqrt{\frac{\ln R/\Delta}{N}}+\frac{\mathcal{B}R(k+\ln R/\Delta)}{N(1-\tau)}\Bigg],
    \end{align*}
    \noindent with probability of at least $1-\Delta$.
\end{theorem}

\begin{proof}
	Let $\delta\in(0,1)$. Suppose that the minimum sample complexity $N\ge 164k^2\ln\left(\frac{4R}{\delta}\right)+\frac{100\ln\left(\frac{2}{\delta}\right)}{1-\tau}$ is satisfied. Let us reuse the quantities $\alpha,\beta$ as defined in Theorem \ref{thm:second_main_result_appendix}:
	\begin{align*}
		\alpha &= \sup_{f\in\mathcal{F}}\left|\frac{1}{1-\widehat\tau}\bar U_\Omega(f)-\frac{1}{1-\tau}\Lnorm_\Omega(f)\right|, 
        \qquad \beta = \sup_{f\in\mathcal{F}}\left|\frac{\widehat\tau}{1-\widehat\tau}\bar U_\Lambda(f) - \frac{\tau}{1-\tau}\Lnorm_\Lambda(f)\right|.
	\end{align*}
	\noindent Combine Proposition \ref{prop:concentration_of_collision_probability_estimator} (absolute error of $\widehat\tau$), Theorem \ref{thm:concentration_of_bar_Uomega} (absolute error of $\bar U_\Omega(f)$) and Proposition \ref{prop:final_multiplicative_control_of_tauhat} (generic multiplicative error of $\widehat\tau$) via the union bound, we have:
	\begin{align*}
		\alpha &\le \frac{1}{1-\tau}\left\{2\RC_{\mu_*}(\mathcal{H}) + \mathcal{B}\left[\frac{6R}{n-R}+\frac{8R}{3N}\ln4R/\delta + \sqrt{\frac{8(R-1)\ln 4R/\delta}{3N}} + 8\sqrt{\frac{\ln 2/\delta}{n-R}}\right]\right\} \\
		&\qquad + \frac{\mathcal{B}\sqrt{2R}}{(1-\tau)(1-\widehat\tau)}\left[\frac{2\ln (R+1)/\delta}{3N} + \sqrt{\frac{2(1-\|\rho\|_2^2)\ln (R+1)/\delta}{N}}\right] \\
		&\le \frac{1}{1-\tau}\left\{2\RC_{\mu_*}(\mathcal{H}) + \mathcal{B}\left[\frac{6R}{n-R}+\frac{8R}{3N}\ln4R/\delta + \sqrt{\frac{8(R-1)\ln 4R/\delta}{3N}} + 8\sqrt{\frac{\ln 2/\delta}{n-R}}\right]\right\} \\
		&\qquad + \frac{8\mathcal{B}\sqrt{2R}}{(1-\tau)^2}\left[\frac{2\ln (R+1)/\delta}{3N} + \sqrt{\frac{2(1-\|\rho\|_2^2)\ln (R+1)/\delta}{N}}\right],
	\end{align*}
	\noindent with probability of at least $1-3\delta$. Similarly, combining Proposition \ref{prop:concentration_of_collision_probability_estimator} (absolute error of $\widehat\tau$), Proposition \ref{prop:final_multiplicative_control_of_tauhat} (generic multiplicative error of $\widehat\tau$) and Theorem \ref{thm:concentration_of_bar_Ulambda} (absolute error of $\bar U_\Lambda(f)$) via union bound, we have:
	\begin{align*}
		\beta &\le  \frac{\tau}{1-\tau}\left\{2\RC_{\nu_*}(\mathcal{H}) + \mathcal{B}\left[\frac{12R}{m-2R}+\frac{8R}{3N}\ln4R/\delta + \sqrt{\frac{8(R-1)\ln 4R/\delta}{3N}} + 8\sqrt{\frac{\ln 2/\delta}{m-2R}}\right]\right\} \\
		& \qquad + \frac{8\mathcal{B}\sqrt{2R}}{(1-\tau)^2}\left[\frac{2\ln (R+1)/\delta}{3N} + \sqrt{\frac{2(1-\|\rho\|_2^2)\ln (R+1)/\delta}{N}}\right]\\
		&\le  \frac{2\tau\RC_{\nu_*}(\mathcal{H})}{1-\tau}+ \frac{\mathcal{B}}{1-\tau}\left[\frac{12R}{m-2R}+\frac{8R}{3N}\ln4R/\delta + \sqrt{\frac{8(R-1)\ln 4R/\delta}{3N}} + 8\sqrt{\frac{\ln 2/\delta}{m-2R}}\right] \\
		& \qquad + \frac{8\mathcal{B}\sqrt{2R}}{(1-\tau)^2}\left[\frac{2\ln (R+1)/\delta}{3N} + \sqrt{\frac{2(1-\|\rho\|_2^2)\ln (R+1)/\delta}{N}}\right],
	\end{align*}
	\noindent with probability of at least $1-3\delta$. Finally, combining the bounds of $\alpha,\beta$ via union bound, we have:
	\begin{align*}
		\sup_{f\in\mathcal{F}}\left|\bar U_N(f)-\mathrm{\widetilde L}_\phi(f)\right|  &\le \frac{2\RC_{\mu_*}(\mathcal{H})+2\tau\RC_{\nu_*}(\mathcal{H})}{1-\tau} +  \frac{8\mathcal{B}}{1-\tau}\left[\sqrt{\frac{\ln 2/\delta}{n-R}}+\sqrt{\frac{\ln 2/\delta}{m-2R}}\right] \\
		&+ \frac{2\mathcal{B}}{1-\tau}\left[\frac{3R}{n-R}+\frac{6R}{m-2R}+\frac{8R}{3N}\ln4R/\delta + \sqrt{\frac{8(R-1)\ln 4R/\delta}{3N}}\right] \\
		&+ \frac{16\mathcal{B}\sqrt{2R}}{(1-\tau)^2}\left[\frac{2\ln (R+1)/\delta}{3N} + \sqrt{\frac{2(1-\|\rho\|_2^2)\ln (R+1)/\delta}{N}}\right], \\
		&\le \mathcal{O}\Bigg[\frac{\RC_{\mu_*}(\mathcal{H})+\tau\RC_{\nu_*}(\mathcal{H})}{1-\tau}+\frac{\mathcal{B}}{1-\tau}\left[\sqrt{\frac{\ln 1/\delta}{n-R}}+\sqrt{\frac{\ln 1/\delta}{m-2R}}+\frac{R}{n-R}+\frac{R}{m-2R}+\frac{R\ln R/\delta}{N}\right]\\
		&\qquad\qquad\qquad\qquad\qquad +\frac{\mathcal{B}\sqrt{R}}{(1-\tau)^2}\sqrt{\frac{(1-\|\rho\|_2^2)\ln R/\delta}{N}}\Bigg],
	\end{align*}
	\noindent with probability of at least $1-6\delta$ as long as $N\ge 164k^2\ln\left(\frac{4R}{\delta}\right)+\frac{100\ln\left(\frac{2}{\delta}\right)}{1-\tau}$. Setting $\delta=\Delta/6$ yields the desired complete bound. Applying the same arguments to obtain $\mathcal{O}$-notation bound in Theorem \ref{thm:main_result_appendix}: When $R\ge k$, as long as $N\ge k\times \max\left\{164k\ln\left(\frac{24R}{\Delta}\right)+\frac{100\ln\left(\frac{12}{\Delta}\right)}{k(1-\tau)},9R\right\}$, we have:
    \begin{align*}
        \sup_{f\in\mathcal{F}}|\bar U_N(f)-\mathrm{\widetilde L}_\phi(f)|\le \mathcal{O}\Bigg[\frac{\mathfrak{C}_N(\mathcal{H})}{1-\tau}\sqrt{\frac{k}{N}}+\frac{\mathcal{B}\sqrt{R}}{(1-\tau)^2}\sqrt{\frac{\ln R/\Delta}{N}}+\frac{\mathcal{B}R(k+\ln R/\Delta)}{N(1-\tau)}\Bigg],
    \end{align*}
    \noindent with probability of at least $1-\Delta$.
\end{proof}

\newpage
\section{Refined Analysis for the Estimator $U^\mathrm{hl}_N(f)$ from \texorpdfstring{\citet{article:hieu2025}}{Hieu-Ledent}}
To re-iterate from the main text, in \citet{article:hieu2025}, the estimator of the population contrastive risk is constructed for each $f\in\mathcal{F}$ as follows:
\begin{align}
    \label{eq:hieu_ledent_estimator_appendix}
    U^\mathrm{hl}_N(f) :=\sum_{r=1}^R\widehat\rho_r U_\Theta^r(f), \qquad\forall r\in[R]:
        U_\Theta^r(f) =\widehat\E_{\Theta_r}\left[\ell_{\phi, f}\right],
\end{align}
\noindent where $\Theta_r$ is the set of collision-free tuples corresponding to class $r\in[R]$:
\begin{align}
    \Theta_r &:= \set{\left(X, X^+, \set{X_i^-}_{i=1}^k\right): X, X^+\in S_r \textup{ \& } \set{X_i^-}_{i=1}^k\subseteq S\setminus S_r}. \label{eq:no_collision_tuples_appendix}
\end{align}
\noindent Basically, for every class $r\in[R]$, the class-wise estimator $U_\Theta^r(f)$ is the (asymptotically) unbiased estimator for the class-wise collision-free risk defined as follows:
\begin{align}
	\label{eq:collision_free_classwise_risk}
	\forall r\in[R]:\quad\mathrm{L}^r_\phi(f) := \underset{\substack{X, X^+\sim\mathcal{D}_r^{\otimes 2}\\ \set{X_i^-}_{i=1}^k\sim\mathcal{\bar D}_r^{\otimes k}}}{\E}\left[\ell_{\phi, f}\left(X, X^+, \set{X_i^-}_{i=1}^k\right)\right].
\end{align}
\noindent In this section, we will refine the analysis from \citet{article:hieu2025} and mitigate the $\mathcal{O}\left(\frac{1}{\sqrt{\rho_{\min}}}\right)$ dependency of the generalization gap. Firstly, we re-state their following result:

\begin{proposition}{(\citet[Proposition D.5]{article:hieu2025})}
    \label{prop:hieu_ledent_supporting_prop}
    Let $U_\Theta^r(f)$ be the class-wise risk estimate defined in Eqn.~\eqref{eq:no_collision_tuples_appendix}. Then, for any $\Delta\in(0,1)$, we have:
    \begin{align}
        \P\left(\sup_{f\in\mathcal{F}}\Big|U_\Theta^r(f)-\mathrm{L}^r_\phi(f)\Big|\ge 2\mathfrak{\widehat R}_{T^\mathrm{iid}_r}(\mathcal{H})+10\mathcal{B}\sqrt{\frac{\ln 4/\Delta}{2(1\vee \bar N_r)}}\Bigg|N_r\right)\le \Delta,
    \end{align}
    \noindent where $\bar N_r = \left\lfloor\min\left(\frac{N_r}{2}, \frac{N-N_r}{k}\right)\right\rfloor$ and $T^\mathrm{iid}_r$ is the set of $\bar N_r$ independent tuples selected ``greedily" by parallel matching positives to negatives in the labeled dataset $S$. Specifically:
    \begin{align}
        \forall r\in[R]:\quad T_r^\mathrm{iid} &:= \set{\left(X^r_{2j-1}, X^r_{2j}, \set{\bar X_{jk-k+i}^r}_{i=1}^k\right)}_{j=1}^{\bar N_r},
    \end{align}
    \noindent where for $u\le N_r$ and $v\le N-N_r$:
    \begin{enumerate}[label=\textup{(\roman*)}]
        \item $X_u^r$ denotes the $u$-th data-point in $S_r$.
        \item $\bar X_v^r$ denotes the $v$-th data-point in $\bar S_r=S\setminus S_r$.
    \end{enumerate}
\end{proposition}

\begin{remark}
    For the sake of definition completeness, we define the empirical Rademacher complexity $\mathfrak{\widehat R}_{T^\mathrm{iid}_r}(\mathcal{H})=0$ when $T^\mathrm{iid}_r=\emptyset$. In other words, the complexity is $0$ when evaluated on an empty dataset.
\end{remark}

\begin{theorem}[cf. Table \ref{tab:results_comparison}]
    \label{thm:refined_hieu_ledent_small_probabilities}
    Let $\mathcal{F}$ be a class of representation functions and $U_N^\mathrm{hl}(f)$ be defined for each $f\in\mathcal{F}$ in Eqn.~\eqref{eq:hieu_ledent_estimator_appendix}. Suppose that we have $\rho_r\le\frac{1}{k+2}$ for all $r\in[R]$. Then, for any $\Delta\in(0,1)$, as long as $N\ge 3(k+2)^2\ln\frac{3R}{\Delta}$, we have:
    \begin{align}
        &\sup_{f\in\mathcal{F}}\Big|U_N^\mathrm{hl}(f)-\Lun(f)\Big| \\
        &\le 48\mathfrak{C}_N(\mathcal{H})\sqrt{\frac{R}{N}} + \frac{\mathcal{B}R}{N} +\frac{8}{N} + \frac{8\mathcal{B}R}{3N}\ln \left(\frac{6R^2}{\Delta}\right) + \mathcal{B}\sqrt{\frac{8(R-1)\ln 6R^2/\Delta}{3N}} + 10\mathcal{B}\sqrt{\frac{2R\ln 12R/\Delta}{N}}\nonumber \\
        &\le \mathcal{O}\left[\mathfrak{C}_N(\mathcal{H})\sqrt{\frac{R}{N}}+\mathcal{B}\sqrt{\frac{R\ln(R/\Delta)}{N}}\right],\nonumber
    \end{align}
    \noindent with probability of at least $1-\Delta$ where $\mathfrak{C}_N(\mathcal{H})$ is a complexity measure of $\mathcal{H}$ defined as:
    \begin{align}
        \label{eq:worst_case_dudley_integral}\mathfrak{C}_N(\mathcal{H}):=\sup_{S_\mathrm{tup}\in(\mathcal{X}^{k+2})^N}\int_{\frac{1}{N}}^\mathcal{B}\sqrt{\ln 2\mathcal{N}(\mathcal{H},\epsilon, \Lnorm_2(S_\mathrm{tup}))}d\epsilon,
    \end{align}
    \noindent i.e., the worst-case Dudley integral over all possible $(k+2)$-tuple datasets of size $N$.
\end{theorem}

\begin{proof}
    Firstly, we have:
    \begin{align*}
        \sup_{f\in\mathcal{F}}\Big|U_N^\mathrm{hl}(f)-\Lun(f)\Big| &= \sup_{f\in\mathcal{F}}\Bigg|\sum_{r=1}^R\widehat\rho_r U_\Theta^r(f)-\sum_{r=1}^R\rho_r\mathrm{L}^r_\phi(f)\Bigg| \\
        &= \sup_{f\in\mathcal{F}}\left|\sum_{r=1}^R\widehat\rho_r\left[U_\Theta^r(f)-\mathrm{L}^r_\phi(f)\right] + \sum_{r=1}^R\mathrm{L}^r_\phi(f)\left[\widehat\rho_r-\rho_r\right]\right| \\
        &\le \sum_{r=1}^R\widehat\rho_r\sup_{f\in\mathcal{F}}\left|U_\Theta^r(f)-\mathrm{L}^r_\phi(f)\right| + \sum_{r=1}^R\sup_{f\in\mathcal{F}}|\mathrm{L}^r_\phi(f)|\cdot\left|\widehat\rho_r-\rho_r\right| \\
        &\le \sum_{r=1}^R\widehat\rho_r\sup_{f\in\mathcal{F}}\left|U_\Theta^r(f)-\mathrm{L}^r_\phi(f)\right| + \mathcal{B}\sum_{r=1}^R|\widehat\rho_r-\rho_r|.
    \end{align*}
    \noindent From the proof of Lemma \ref{lem:supporting_lemma_4}, we already know that for all $\delta\in(0,1)$, we have:
    \begin{align*}
        \sum_{r=1}^R|\widehat\rho_r-\rho_r| &\le \frac{8R}{3N}\ln 2R/\delta + \sqrt{\frac{8(R-1)\ln 2R/\delta}{3N}},
    \end{align*}
    \noindent with probability of at least $1-\delta$. Now, we proceed to analyze the weighted sum of absolute deviation. From Proposition \ref{prop:hieu_ledent_supporting_prop}, we have:
    \begin{align*}
        \sum_{r=1}^R\widehat\rho_r\sup_{f\in\mathcal{F}}|U_\Theta^r(f)-\mathrm{L}^r_\phi(f)| &= \sum_{\bar r:N_{\bar r}=1}\frac{\sup_{f\in\mathcal{F}}|\mathrm{L}^{\bar r}_\phi(f)|}{N} + \sum_{r:N_r\ge2}\widehat\rho_r\sup_{f\in\mathcal{F}}|U_\Theta^r(f)-\mathrm{L}^r_\phi(f)| \\
        &\le \frac{\mathcal{B}R}{N} + \sum_{r:N_r\ge2}\widehat\rho_r\sup_{f\in\mathcal{F}}|U_\Theta^r(f)-\mathrm{L}^r_\phi(f)|.    
    \end{align*}

    \noindent\textbf{Claim}: As long as $N\ge 3(k+2)^2\ln\frac{1}{\delta}$ then $\widehat\rho_r\le \frac{2}{k+2}$ for all $r\in[R]$ with probability of at least $1-R\delta$. Firstly, we define $\set{U_j}_{j=1}^N$ as the sequence of i.i.d. random variables where $U_j\sim\mathrm{Uniform}(0,1)$ for all $j\in[N]$. For all $r\in[R]$, we have:
    \begin{align*}
        \P\left(\widehat\rho_r\ge\frac{2}{k+2}\right) &= \P\left(\frac{N_r}{N}\ge\frac{2}{k+2}\right) = \P\left(\frac{1}{N}\sum_{j=1}^N\1{U_j\le \rho_r}\ge\frac{2}{k+2}\right) \\
        &\le \P\left(\frac{1}{N}\sum_{j=1}^N\1{U_j\le \frac{1}{k+2}}\ge\frac{2}{k+2}\right) \qquad\textup{(By coupling)} \\
        &\le \exp\left(-\frac{N}{3(k+2)^2}\right)\qquad\text{(Multiplicative Chernoff)}.
    \end{align*}
    \noindent Therefore, by the union bound:
    \begin{align*}
        \P\left(\exists r\in[R]: \widehat\rho_r\ge\frac{2}{k+2}\right) &\le R\exp\left(-\frac{N}{3(k+2)^2}\right).
    \end{align*}
    \noindent Therefore, if we let $\delta\ge \exp\left(-\frac{N}{3(k+2)^2}\right)$, then we have $N\ge 3(k+2)^2\ln\frac{1}{\delta}$, as desired. Combine the above claim with Proposition \ref{prop:hieu_ledent_supporting_prop}, the following events:
    \begin{align*}
        \begin{cases}
            \widehat\rho_r \le \frac{2}{k+2}, &\qquad\forall r\in[R], \\
            \sup_{f\in\mathcal{F}}\Big|U_\Theta^r(f)-\mathrm{L}^r_\phi(f)\Big|\le 2\mathfrak{\widehat R}_{T^\mathrm{iid}_r}(\mathcal{H})+10\mathcal{B}\sqrt{\frac{\ln 4/\delta}{2\bar N_r}},&\qquad\forall r\in[R].
        \end{cases}
    \end{align*}
    \noindent hold simultaneously with probability of at least $1-2R\delta$. We also note that $\widehat\rho_r\le\frac{2}{k+2}$ is equivalent to $\frac{N_r}{2}\le \frac{N-N_r}{k}$. Therefore, given that $\widehat\rho_r\le\frac{2}{k+2}$ for all $r\in[R]$, we have $|T_r^\mathrm{iid}|=\bar N_r=\lfloor N_r/2\rfloor$. Therefore, with probability of at least $1-2R\delta$, we have:
    \begin{align*}
        &\sum_{r=1}^R\widehat\rho_r\sup_{f\in\mathcal{F}}|U_\Theta^r(f)-\mathrm{L}^r_\phi(f)| \\
        &\le \frac{\mathcal{B}R}{N} + \sum_{r:N_r\ge2}\widehat\rho_r\sup_{f\in\mathcal{F}}|U_\Theta^r(f)-\mathrm{L}^r_\phi(f)| \\
        &\le \frac{\mathcal{B}R}{N} + \sum_{r:N_r\ge 2}\widehat\rho_r\left[2\mathfrak{\widehat R}_{T^\mathrm{iid}_r}(\mathcal{H})+10\mathcal{B}\sqrt{\frac{\ln 4/\delta}{2\lfloor N_r/2\rfloor}}\right] \\
        &\overset{(*)}{\le} \frac{\mathcal{B}R}{N} + \sum_{r:N_r\ge2}\widehat\rho_r\left[\frac{8}{N} + \frac{24}{\lfloor N_r/2\rfloor^\frac{1}{2}}\int_{\frac{1}{N}}^\mathcal{B}\sqrt{\ln 2\mathcal{N}(\mathcal{H},\epsilon, \Lnorm_2(T_r^\mathrm{iid}))}d\epsilon + 10\mathcal{B}\sqrt{\frac{\ln 4/\delta}{2\lfloor N_r/2\rfloor}}\right] \\
        &\le \frac{\mathcal{B}R}{N}+\frac{8}{N}+\sum_{r:N_r\ge2}\widehat\rho_r\left[\frac{24}{\sqrt{N_r/4}}\int_{\frac{1}{N}}^\mathcal{B}\sqrt{\ln 2\mathcal{N}(\mathcal{H},\epsilon, \Lnorm_2(T_r^\mathrm{iid}))}d\epsilon+10\mathcal{B}\sqrt{\frac{2\ln 4/\delta}{N_r}}\right] \\
        &= \frac{\mathcal{B}R}{N}+\frac{8}{N}+\sum_{r:N_r\ge2}\widehat\rho_r\left[\frac{48}{\sqrt{N_r}}\int_{\frac{1}{N}}^\mathcal{B}\sqrt{\ln 2\mathcal{N}(\mathcal{H},\epsilon, \Lnorm_2(T_r^\mathrm{iid}))}d\epsilon+10\mathcal{B}\sqrt{\frac{2\ln4/\delta}{N_r}}\right] \\
        &\le \frac{\mathcal{B}R}{N}+\frac{8}{N}+\sum_{r:N_r\ge2}\widehat\rho_r\left[\frac{48\mathfrak{C}_N(\mathcal{H})}{\sqrt{N_r}}+10\mathcal{B}\sqrt{\frac{2\ln4/\delta}{N_r}}\right] .
    \end{align*}
    \noindent In $(*)$, we used the Dudley entropy bound (Theorem \ref{thm:dudley}) with $\alpha=\frac{1}{N}$. Combine the above bound with the high probability bound on $\sum_{r=1}^R|\widehat\rho_r-\rho_r|$, as long as we have $N\ge 3(k+2)^2\ln\frac{1}{\delta}$, we have:
    \begin{align*}
        \sup_{f\in\mathcal{F}}\Big|U_N^\mathrm{hl}(f)-\mathrm{L}_\phi(f)\Big| &\le \left[48\mathfrak{C}_N(\mathcal{H})+10\mathcal{B}\sqrt{2\ln\frac{4}{\delta}}\right]\sum_{r:N_r\ge2}\frac{\widehat\rho_r}{\sqrt{N_r}}+\frac{\mathcal{B}R}{N}+\frac{8}{N} \\
        &\qquad\qquad + \frac{8\mathcal{B}R}{3N}\ln \frac{2R}{\delta} + \mathcal{B}\sqrt{\frac{8(R-1)\ln 2R/\delta}{3N}},
    \end{align*}
    \noindent with probability of at least $1-(2R+1)\delta$. Furthermore, we have:
    \begin{align*}
        \sum_{r:N_r\ge2}\frac{\widehat\rho_r}{\sqrt{N_r}} &= \sum_{r:N_r\ge2}\sqrt{\frac{\widehat\rho_r}{N}} \le \left(\sum_{r:N_r\ge2}\widehat\rho_r\right)^\frac{1}{2}\left(\sum_{r:N_r\ge2}\frac{1}{N}\right)^\frac{1}{2}\le\sqrt{\frac{R}{N}}.
    \end{align*}
    \noindent Therefore, as long as $N\ge3(k+2)^2\ln\frac{1}{\delta}$, we have:
    \begin{align*}
        &\sup_{f\in\mathcal{F}}\Big|U_N^\mathrm{hl}(f)-\mathrm{L}_\phi(f)\Big| \\&\le \sum_{r=1}^R\widehat\rho_r\sup_{f\in\mathcal{F}}\left|U_\Theta^r(f)-\mathrm{L}^r_\phi(f)\right| + \mathcal{B}\sum_{r=1}^R|\widehat\rho_r-\rho_r| \\
        &\le\left[48\mathfrak{C}_N(\mathcal{H})+10\mathcal{B}\sqrt{2\ln\frac{4}{\delta}}\right]\sqrt{\frac{R}{N}} +\frac{\mathcal{B}R}{N}+\frac{8}{N} + \frac{8\mathcal{B}R}{3N}\ln \frac{2R}{\delta} + \mathcal{B}\sqrt{\frac{8(R-1)\ln 2R/\delta}{3N}} \\
        &= 48\mathfrak{C}_N(\mathcal{H})\sqrt{\frac{R}{N}} + \frac{\mathcal{B}R}{N} +\frac{8}{N} + \frac{8\mathcal{B}R}{3N}\ln \frac{2R}{\delta} + \mathcal{B}\sqrt{\frac{8(R-1)\ln 2R/\delta}{3N}} + 10\mathcal{B}\sqrt{\frac{2R\ln 4/\delta}{N}},
    \end{align*}
    \noindent with probability of at least $1-(2R+1)\delta$. Crudely setting $\delta=\Delta/3R$ yields the desired bound.
\end{proof}

\begin{remark}
    The complexity measure $\mathfrak{C}_N(\mathcal{H})$ is easily bounded for linear classes or neural networks with bounded spectral norms and bounded activation Lipschitz constants. For example, if $\mathcal{F}$ is a linear class defined as follows:
    \begin{align}
        \mathcal{F}=\Big\{\x\mapsto A\x: A\in\R^{d\times m}, \|A^\top\|_{2, 1}\le a, \|A\|_\sigma\le s, \|\x\|_2\le b\Big\}.
    \end{align}
    \noindent Then, from \citet{article:hieu2025}, we have:
    \begin{align}
        \mathfrak{C}_N(\mathcal{H}) &\lesssim \eta sab^2\ln^\frac{1}{2}\left(\eta sab^2N(k+2)d\right)\ln(N\mathcal{B})\le \mathcal{\widetilde O}\left(\eta sab^2\right),
    \end{align}
    \noindent where $\eta>0$ is the $\ell_\infty$-Lipschitz constant of the contrastive loss $\phi$. On the other hand, if $\mathcal{F}$ is denoted as a class of deep neural networks of weights with bounded spectral norms:
    \begin{align}
        \mathcal{F} &= \Big\{\x\mapsto A_L\varphi_{L-1}\left(A_{L-1}\varphi_{L-2}(\dots\varphi_1(A_1\x)\dots)\right): A_\ell\in\R^{d_\ell\times d_{\ell-1}}, \|A_\ell\|_\sigma\le s_\ell,\forall \ell\in[L]\Big\},
    \end{align}
    \noindent Then, we have the following metric entropy bound \citep{article:longsedghi2020,article:graf2023,article:hieu2024}:
    \begin{align}
        \mathfrak{C}_N(\mathcal{H}) \lesssim {W}^\frac{1}{2}\ln^\frac{1}{2}\left(1+\eta LNb\prod_{\ell=1}^Ls_\ell^2\right) \le \mathcal{\widetilde O}([L{W}]^\frac{1}{2}),
    \end{align}
    \noindent where ${W}:=\sum_{\ell=1}^Ld_\ell\times d_{\ell-1}$, i.e., total number of specified parameters.
\end{remark}

\begin{theorem}[Theorem \ref{thm:refined_hieu_ledent_general}, Table \ref{tab:results_comparison}]
    \label{thm:refined_hieu_ledent_general_appendix}
    Let $\mathcal{F}$ be a class of representation functions and let $U_N^\mathrm{hl}(f)$ be defined for each $f\in\mathcal{F}$ as in Eqn.~\eqref{eq:hieu_ledent_estimator_appendix}. Suppose that $\rho_r\le\frac{1}{2}$ for all $r\in[R]$. Then, for any $\Delta\in(0,1)$, as long as $N\ge 6(k+1)^2\ln\frac{3R}{\Delta}$:
    \begin{align}
        \sup_{f\in\mathcal{F}}\Big|U_N^\mathrm{hl}(f)-\Lun(f)\Big| &\le \left[48\mathfrak{C}_N(\mathcal{H})+10\mathcal{B}\sqrt{2\ln\left(\frac{12R}{\Delta}\right)}\right]\left[ \widehat\theta_{k+2}^\frac{1}{2}\sqrt{\frac{2R}{N}} + (1-\widehat\theta_{k+2})\sqrt{\frac{2(k+1)}{N}}\right] \nonumber \\
        &\qquad\qquad +\frac{\mathcal{B}R}{N}+\frac{8}{N}+\frac{8\mathcal{B}R}{3N}\ln \left(\frac{6R^2}{\Delta}\right) + \mathcal{B}\sqrt{\frac{8(R-1)\ln 6R^2/\Delta}{3N}} \\
        &\le \mathcal{O}\left[\mathfrak{C}_N(\mathcal{H})\left[ \widehat\theta_{k+2}^\frac{1}{2}\sqrt{\frac{R}{N}} + (1-\widehat\theta_{k+2})\sqrt{\frac{k}{N}}\right] + \mathcal{B}\sqrt{\frac{\max(k,R)\ln(R/\Delta)}{N}}\right].
    \end{align}
    \noindent with probability of at least $1-\Delta$, where $\widehat\theta_{k+2}:=\sum_{r:\widehat\rho_r\le\frac{2}{k+2}}\widehat\rho_r$.
\end{theorem}

\begin{proof}
    Let $\delta\in(0,1)$. From the proof of Theorem \ref{thm:refined_hieu_ledent_small_probabilities}, we have:
    \begin{align*}
        &\sum_{r=1}^R|\widehat\rho_r-\rho_r| \le \frac{8R}{3N}\ln 2R/\delta + \sqrt{\frac{8(R-1)\ln 2R/\delta}{3N}},&&\text{wp. }\ge 1-\delta, \\
        &\sup_{f\in\mathcal{F}}\Big|U_\Theta^r(f)-\mathrm{L}^r_\phi(f)\Big| \le 2\mathfrak{\widehat R}_{T^\mathrm{iid}_r}(\mathcal{H})+10\mathcal{B}\sqrt{\frac{\ln 4/\delta}{2\bar N_r}},\quad\forall r\in[R]&& \text{wp. }\ge 1-R\delta.
    \end{align*}
    \noindent Therefore, by the union bound, with probability of at least $1-(R+1)\delta$, we have:
    \begin{align*}
        &\sup_{f\in\mathcal{F}}\Big|U_N^\mathrm{hl}(f)-\Lun(f)\Big| \\ 
        &\le \sum_{r=1}^R\widehat\rho_r\sup_{f\in\mathcal{F}}\left|U_\Theta^r(f)-\mathrm{L}^r_\phi(f)\right| + \mathcal{B}\sum_{r=1}^R|\widehat\rho_r-\rho_r| \\
        &\le \frac{\mathcal{B}R}{N}+\sum_{r:N_r\ge2}\widehat\rho_r\left[2\mathfrak{\widehat R}_{T_r^\mathrm{iid}}(\mathcal{H})+10\mathcal{B}\sqrt{\frac{\ln 4/\delta}{2\bar N_r}}\right]+\frac{8\mathcal{B}R}{3N}\ln 2R/\delta + \mathcal{B}\sqrt{\frac{8(R-1)\ln 2R/\delta}{3N}} \\
        &\le \frac{\mathcal{B}R}{N} + \frac{8}{N}+\left[48\mathfrak{C}_N(\mathcal{H})+10\mathcal{B}\sqrt{2\ln\frac{4}{\delta}}\right]\sum_{r:N_r\ge2}\frac{\widehat\rho_r}{\sqrt{\bar N_r}}+\frac{8\mathcal{B}R}{3N}\ln 2R/\delta + \mathcal{B}\sqrt{\frac{8(R-1)\ln 2R/\delta}{3N}}.
    \end{align*}
    \noindent Then, we have:
    \begin{align*}
        \sum_{r:N_r\ge2}\frac{\widehat\rho_r}{\sqrt{\bar N_r}} &= \sum_{r:\frac{2}{N}\le\widehat\rho_r\le\frac{2}{k+2}}\frac{\widehat\rho_r}{\sqrt{N_r/2}} + \sum_{\bar r:\frac{2}{k+2}<\widehat\rho_{\bar r}\le1}\frac{\widehat\rho_{\bar r}}{\sqrt{(N-N_{\bar r})/k}}.
    \end{align*}
    \noindent Now, let $\alpha>1$ be a real number and $\set{U_j}_{j=1}^N$ be a sequence of independent random variables such that $U_j\sim\mathrm{Uniform}(0,1)$ for all $j\in[N]$. Then, by coupling:
    \begin{align*}
    	\forall r\in[R]:\quad\widehat\rho_r = \frac{N_r}{N} \overset{d}{=} \frac{1}{N}\sum_{j=1}^N \mathds{1}_{\{U_j\le\rho_r\}}.
    \end{align*}
    \noindent Hence, for all $r\in[R]$, we have:
    \begin{align*}
        \P\left(\widehat\rho_r \ge \frac{1}{2}\left(1+\frac{1}{\alpha}\right)\right) &= \P\left(\frac{N_r}{N}\ge \frac{1}{2}\left(1+\frac{1}{\alpha}\right)\right) \\
        &= \P\left(\frac{1}{N}\sum_{j=1}^N\1{U_j\le \rho_r}\ge \frac{1}{2}\left(1+\frac{1}{\alpha}\right)\right) \\
        &\le \P\left(\frac{1}{N}\sum_{j=1}^N\1{U_j\le \frac{1}{2}}\ge \frac{1}{2}\left(1+\frac{1}{\alpha}\right)\right) \qquad\text{(By Coupling)}\\
        &\le\exp\left(-\frac{N}{6\alpha^2}\right) \qquad\text{(Multiplicative Chernoff)}.
    \end{align*}
    \noindent Therefore, by the union bound:
    \begin{align*}
        \P\left(\exists r\in[R]:\widehat\rho_r\ge\frac{1}{2}\left(1+\frac{1}{\alpha}\right)\right) &\le R\exp\left(-\frac{N}{6\alpha^2}\right).
    \end{align*}
    \noindent Hence, as long as $N\ge 6\alpha^2\ln\frac{1}{\delta}$, we have $\widehat\rho_r\le\frac{1}{2}\left(1+\frac{1}{\alpha}\right)$ for all $r\in[R]$ with probability of at least $1-R\delta$. Therefore, as long as $N\ge6\alpha^2\ln\frac{1}{\delta}$, we have:
    \begin{align*}
        \sum_{r:N_r\ge2}\frac{\widehat\rho_r}{\sqrt{\bar N_r}} &= \sum_{r:\frac{2}{N}\le\widehat\rho_r\le\frac{2}{k+2}}\frac{\widehat\rho_r}{\sqrt{N_r/2}} + \sum_{\bar r:\frac{2}{k+2}<\widehat\rho_{\bar r}\le1}\frac{\widehat\rho_{\bar r}}{\sqrt{(N-N_{\bar r})/k}} \\
        &\le \sum_{r:\frac{2}{N}\le\widehat\rho_r\le\frac{2}{k+2}}\sqrt{\frac{2\widehat\rho_r}{N}} + \sum_{\bar r:\frac{2}{k+2}<\widehat\rho_{\bar r}\le1}\widehat\rho_{\bar r}\sqrt{\frac{k}{N-\left(1+\frac{1}{\alpha}\right)\frac{N}{2}}}\\
        &= \sum_{r:\frac{2}{N}\le\widehat\rho_r\le\frac{2}{k+2}}\sqrt{\frac{2\widehat\rho_r}{N}} + \sum_{\bar r:\frac{2}{k+2}<\widehat\rho_{\bar r}\le1}\widehat\rho_{\bar r}\sqrt{\frac{k}{\frac{N}{2}(1-\frac{1}{\alpha})}} \\
        &= \sum_{r:\frac{2}{N}\le\widehat\rho_r\le\frac{2}{k+2}}\sqrt{\frac{2\widehat\rho_r}{N}} + \sum_{\bar r:\frac{2}{k+2}<\widehat\rho_{\bar r}\le1}\widehat\rho_{\bar r}\sqrt{\frac{k}{N}\cdot\frac{2\alpha}{\alpha-1}}\\
        &= \sum_{r:\frac{2}{N}\le\widehat\rho_r\le\frac{2}{k+2}}\sqrt{\frac{2\widehat\rho_r}{N}} + (1-\widehat\theta_{k+2})\sqrt{\frac{2\alpha k}{N(\alpha-1)}} \\
        &\le \left(\sum_{r:\frac{2}{N}\le\widehat\rho_r\le\frac{2}{k+2}}\widehat\rho_r\right)^\frac{1}{2}\left(\sum_{r:\frac{2}{N}\le\widehat\rho_r\le\frac{2}{k+2}}\frac{2}{N}\right)^\frac{1}{2} + (1-\widehat\theta_{k+2})\sqrt{\frac{2\alpha k}{N(\alpha-1)}} \ \ (\text{Cauchy-Schwarz}) \\
        &\le \sqrt{\widehat\theta_{k+2}}\cdot\sqrt{\frac{2R}{N}} + (1-\widehat\theta_{k+2})\sqrt{\frac{2\alpha k}{N(\alpha-1)}},
    \end{align*}
    \noindent with probability of at least $1-R\delta$. Choosing $\alpha=k+1$, then as long as $N\ge 6(k+1)^2\ln\frac{1}{\delta}$, with probability of at least $1-R\delta$, we have:
    \begin{align*}
        \sum_{r:N_r\ge2}\frac{\widehat\rho_r}{\sqrt{\bar N_r}} &\le \widehat\theta_{k+2}^\frac{1}{2}\sqrt{\frac{2R}{N}} + (1-\widehat\theta_{k+2})\sqrt{\frac{2(k+1)}{N}}.
    \end{align*}
    \noindent Combine with the bound on $\sup_{f\in\mathcal{F}}|U_N^\mathrm{hl}(f)-\Lun(f)|$ via union bound, as long as $N\ge6(k+1)^2\ln\frac{1}{\delta}$, with probability of at least $1-(2R+1)\delta$, we have:
    \begin{align*}
        &\sup_{f\in\mathcal{F}}\Big|U_N^\mathrm{hl}(f)-\Lun(f)\Big| \\
        &\le \sum_{r=1}^R\widehat\rho_r\sup_{f\in\mathcal{F}}\left|U_\Theta^r(f)-\mathrm{L}^r_\phi(f)\right| + \mathcal{B}\sum_{r=1}^R|\widehat\rho_r-\rho_r| \\
        &\le \frac{\mathcal{B}R}{N} + \frac{8}{N}+\left[48\mathfrak{C}_N(\mathcal{H})+10\mathcal{B}\sqrt{2\ln\frac{4}{\delta}}\right]\sum_{r:N_r\ge2}\frac{\widehat\rho_r}{\sqrt{\bar N_r}}+\frac{8\mathcal{B}R}{3N}\ln 2R/\delta + \mathcal{B}\sqrt{\frac{8(R-1)\ln 2R/\delta}{3N}} \\
        &\le \frac{\mathcal{B}R}{N}+\frac{8}{N}+\left[48\mathfrak{C}_N(\mathcal{H})+10\mathcal{B}\sqrt{2\ln\frac{4}{\delta}}\right]\left[ \widehat\theta_{k+2}^\frac{1}{2}\sqrt{\frac{2R}{N}} + (1-\widehat\theta_{k+2})\sqrt{\frac{2(k+1)}{N}}\right] \\
        &\qquad\qquad +\frac{8\mathcal{B}R}{3N}\ln 2R/\delta + \mathcal{B}\sqrt{\frac{8(R-1)\ln 2R/\delta}{3N}}.
    \end{align*}
    \noindent Crudely setting $\delta=\Delta/3R$ yields the desired bound. It is straightforward to get rid of the assumption that $\rho_r\le\frac{1}{2}$. We have the following result.
\end{proof}

\begin{theorem}[Theorem \ref{thm:refined_hieu_ledent_general_appendix} without $\rho_r\le\frac{1}{2}$ assumption]
	\label{thm:refined_hieu_ledent_more_general_appendix}
	Let $\mathcal{F}$ be a class of representation functions and let $U_N^\mathrm{hl}(f)$ be defined for each $f\in\mathcal{F}$ as in Eqn.~\eqref{eq:hieu_ledent_estimator_appendix}. Then, for any $\Delta\in(0,1)$, as long as $N\ge \frac{2(k+1)^2}{1-\rho_{\max}}\ln\frac{3R}{\Delta}$, we have:
	\begin{align}
		\sup_{f\in\mathcal{F}}\Big|U_N^\mathrm{hl}(f)-\Lun(f)\Big| &\le \left[48\mathfrak{C}_N(\mathcal{H})+10\mathcal{B}\sqrt{2\ln\left(\frac{12R}{\Delta}\right)}\right]\left[ \widehat\theta_{k+2}^\frac{1}{2}\sqrt{\frac{2R}{N}} + (1-\widehat\theta_{k+2})\sqrt{\frac{k+1}{N(1-\rho_{\max})}}\right] \nonumber \\
		&\qquad\qquad +\frac{\mathcal{B}R}{N}+\frac{8}{N}+\frac{8\mathcal{B}R}{3N}\ln \left(\frac{6R^2}{\Delta}\right) + \mathcal{B}\sqrt{\frac{8(R-1)\ln 6R^2/\Delta}{3N}} \\
        &\le \mathcal{O}\left[\mathfrak{C}_N(\mathcal{H})\left[ \widehat\theta_{k+2}^\frac{1}{2}\sqrt{\frac{R}{N}} + (1-\widehat\theta_{k+2})\sqrt{\frac{k}{N(1-\rho_{\max})}}\right] + \mathcal{B}\sqrt{\frac{R_*\ln(R/\Delta)}{N}}\right].
	\end{align}
	\noindent with probability of at least $1-\Delta$, where $\widehat\theta_{k+2}:=\sum_{r:\widehat\rho_r\le\frac{2}{k+2}}\widehat\rho_r$ and $R_*:=\max\left(R, \frac{k}{1-\rho_{\max}}\right)$.
\end{theorem}

\begin{proof}
	Using similar arguments in Theorem \ref{thm:refined_hieu_ledent_general_appendix}, for all $\delta\in(0,1)$, with probability of at least $1-(R+1)\delta$, we have:
	\begin{align*}
		&\sup_{f\in\mathcal{F}}\left|U_N^\mathrm{hl}(f)-\Lun(f)\right| \\
        &\le \sum_{r=1}^R\widehat\rho_r\sup_{f\in\mathcal{F}}\left|U_\Theta^r(f)-\mathrm{L}^r_\phi(f)\right| + \mathcal{B}\sum_{r=1}^R|\widehat\rho_r-\rho_r| \\
		&\le \frac{\mathcal{B}R}{N} + \frac{8}{N}+\left[48\mathfrak{C}_N(\mathcal{H})+10\mathcal{B}\sqrt{2\ln\frac{4}{\delta}}\right]\sum_{r:N_r\ge2}\frac{\widehat\rho_r}{\sqrt{\bar N_r}}+\frac{8\mathcal{B}R}{3N}\ln 2R/\delta + \mathcal{B}\sqrt{\frac{8(R-1)\ln 2R/\delta}{3N}}.
	\end{align*}
	\noindent Then, let $\rho_{\max}=\max_{r\in[R]}\rho_r$, we split the sum $\sum_{r:N_r\ge2}\frac{\widehat\rho_r}{\sqrt{\bar N_r}}$ as follows:
	\begin{align*}
		\sum_{r:N_r\ge2}\frac{\widehat\rho_r}{\sqrt{\bar N_r}} &= \sum_{r:\frac{2}{N}\le\widehat\rho_r\le\frac{2}{k+2}}\frac{\widehat\rho_r}{\sqrt{N_r/2}} + \sum_{\bar r: \frac{2}{k+2}<\widehat\rho_{\bar r}\le 1}\frac{\widehat\rho_{\bar r}}{\sqrt{(N-N_{\bar r})/k}}
	\end{align*}
	\noindent Then, using the same coupling trick. Let $\set{U_j}_{j=1}^N$ be a sequence of independent uniform random variables where $U_j\sim\mathrm{Uniform}(0,1)$ for all $j\in[N]$. Then, for $\alpha>1$, we have:
	\begin{align*}
		\P\left(1-\widehat\rho_r \le (1-\rho_{\max})\left(1-\frac{1}{\alpha}\right)\right) &\le \P\left(1-\widehat\rho_r \le (1-\rho_r)\left(1-\frac{1}{\alpha}\right)\right) \qquad (1-\rho_r\ge 1-\rho_{\max})\\
		&\le \P\left(\frac{1}{N}\sum_{j=1}^N\mathds{1}_{\{U_j \le 1 - \rho_r\}} \le(1-\rho_r)\left(1-\frac{1}{\alpha}\right)\right) \quad\text{(By Coupling)} \\
		&\le \exp\left(-\frac{N(1-\rho_r)}{2\alpha^2}\right) \le \exp\left(-\frac{N(1-\rho_{\max})}{2\alpha^2}\right) \\
		&= \exp\left(-\frac{N}{2\alpha^2(1-\rho_{\max})^{-1}}\right).
	\end{align*}
	\noindent Therefore, by the union bound:
	\begin{align*}
		\P\left(\exists r\in[R]: 1-\widehat\rho_r \le (1-\rho_{\max})\left(1-\frac{1}{\alpha}\right)\right) &\le R\exp\left(-\frac{N}{2\alpha^2(1-\rho_{\max})^{-1}}\right).
	\end{align*}
	\noindent Therefore, as long as $N\ge\frac{2\alpha^2}{1-\rho_{\max}}\ln\frac{1}{\delta}$, we have $1-\widehat{\rho}_r\ge (1-\rho_{\max})\left(1-\frac{1}{\alpha}\right)$ for all $r\in[R]$ with probability of at least $1-R\delta$. As a result, we have:
	\begin{align*}
		\sum_{r:N_r\ge2}\frac{\widehat\rho_r}{\sqrt{\bar N_r}} &=  \sum_{r:\frac{2}{N}\le\widehat\rho_r\le\frac{2}{k+2}}\frac{\widehat\rho_r}{\sqrt{N_r/2}} + \sum_{\bar r: \frac{2}{k+2}<\widehat\rho_{\bar r}\le 1}\frac{\widehat\rho_{\bar r}}{\sqrt{(N-N_{\bar r})/k}} \\
		&\le  \sum_{r:\frac{2}{N}\le\widehat\rho_r\le\frac{2}{k+2}}\sqrt{\frac{2\widehat\rho_r}{N}} +  \sum_{\bar r: \frac{2}{k+2}<\widehat\rho_{\bar r}\le 1}\widehat\rho_{\bar r}\sqrt{\frac{k}{N(1-\widehat\rho_{\bar r})}} \\
		&\le \widehat\theta_{k+2}^\frac{1}{2}\sqrt{\frac{2R}{N}} +  \sum_{\bar r: \frac{2}{k+2}<\widehat\rho_{\bar r}\le 1}\widehat\rho_{\bar r}\sqrt{\frac{k}{N(1-\widehat\rho_{\bar r})}} \qquad\textup{(Cauchy-Schwarz)} \\
		&\le \widehat\theta_{k+2}^\frac{1}{2}\sqrt{\frac{2R}{N}} + \sum_{\bar r: \frac{2}{k+2}<\widehat\rho_{\bar r}\le 1}\widehat\rho_{\bar r}\sqrt{\frac{k}{N(1-\rho_{\max})\left(1-\frac{1}{\alpha}\right)}} \\
		&=  \widehat\theta_{k+2}^\frac{1}{2}\sqrt{\frac{2R}{N}} + (1 - \widehat\theta_{k+2})\sqrt{\frac{\alpha k}{N(1-\rho_{\max})(\alpha - 1)}},
	\end{align*}
	\noindent with probability of at least $1-R\delta$ as long as $N\ge\frac{2\alpha^2}{1-\rho_{\max}}\ln\frac{1}{\delta}$. Choosing $\alpha=k+1$ yields:
	\begin{align*}
		\sum_{r:N_r\ge2}\frac{\widehat\rho_r}{\sqrt{\bar N_r}} &\le  \widehat\theta_{k+2}^\frac{1}{2}\sqrt{\frac{2R}{N}} + (1 - \widehat\theta_{k+2})\sqrt{\frac{k+1}{N(1-\rho_{\max})}},
	\end{align*}
	\noindent with probability of at least $1-R\delta$ as long as $N\ge\frac{2(k+1)^2}{1-\rho_{\max}}\ln\frac{1}{\delta}$. Combining this with the very first high-probability event, we have:
	\begin{align*}
		&\sup_{f\in\mathcal{F}}\left|U_N^\mathrm{hl}(f)-\Lun(f)\right| \\
		&\le \frac{\mathcal{B}R}{N} + \frac{8}{N}+\left[48\mathfrak{C}_N(\mathcal{H})+10\mathcal{B}\sqrt{2\ln\frac{4}{\delta}}\right]\sum_{r:N_r\ge2}\frac{\widehat\rho_r}{\sqrt{\bar N_r}}+\frac{8\mathcal{B}R}{3N}\ln 2R/\delta + \mathcal{B}\sqrt{\frac{8(R-1)\ln 2R/\delta}{3N}} \\
		&\le  \frac{\mathcal{B}R}{N} + \frac{8}{N}+\left[48\mathfrak{C}_N(\mathcal{H})+10\mathcal{B}\sqrt{2\ln\frac{4}{\delta}}\right]\left[\widehat\theta_{k+2}^\frac{1}{2}\sqrt{\frac{2R}{N}} + (1 - \widehat\theta_{k+2})\sqrt{\frac{k+1}{N(1-\rho_{\max})}}\right] \\
		&\qquad\qquad+\frac{8\mathcal{B}R}{3N}\ln 2R/\delta + \mathcal{B}\sqrt{\frac{8(R-1)\ln 2R/\delta}{3N}}.
	\end{align*}
	\noindent Setting $\delta=\Delta/3R$ yields the desired bound.
\end{proof}

\newpage
\section{Exact Debiasing in General Cases}
\label{sec:exact-debias}
Before analyzing the exact debiased form of $\Lnorm_\phi$, we first state the following generalization/refinement of Theorem \ref{thm:concentration_of_bar_Uomega} and Proposition \ref{prop:diff_between_natural_and_nonatural} for any class of bounded tuples loss.
\begin{theorem}[Generalization of Thm. \ref{thm:concentration_of_bar_Uomega}]
    \label{lem:general_tuple_functions}
    Let $\mathcal{G}$ be a class of functions which map tuples $\big(X,X^+,\set{X_i^-}_{i=1}^k\big)$ of $k+2$ labelled data points to $[0,\mathcal{B}_\mathcal{G}]$ for some $\mathcal{B}_\mathcal{G}\in\R_+$. For any $g\in\mathcal{G}$, we define:
    \begin{align*}
        \Lnorm_\Omega[g]:=\underset{r\sim\rho}{\E}\underset{\substack{X,X^+\sim\mathcal{D}_r^{\otimes 2}\\\set{X_i^-}_{i=1}^k\sim\mathcal{\bar D}^{\otimes k}}}{\E}\left[g\left(X, X^+, \set{X_i^-}_{i=1}^k\right)\right],\qquad\bar U_\Omega[g]:=\widehat\E_\Pi\left[\sum_{r=1}^R\omega_r^\pi\,\widehat\E_{\Omega_r^\pi}[g]\right].
    \end{align*}
    \noindent Then, for any $\Delta\in(0,1)$, with probability of at least $1-\Delta$:
    \begin{align}
        \sup_{g\in\mathcal{G}}\Big|\bar U_\Omega[g]-\Lnorm_\Omega[g]\Big|&\le2\sup_{\mu\in\mathcal{U}}\RC_\mu(\mathcal{G})+\mathcal{B}_\mathcal{G}\,b_N(\Delta), \\
        \textup{where}\quad b_N(\Delta)&:=\frac{6R}{n-R}+\frac{8R}{3N}\ln4R/\Delta+\sqrt{\frac{8(R-1)\ln4R/\Delta}{3N}}+8\sqrt{\frac{\ln2/\Delta}{n-R}},\label{eq:b_N}
    \end{align}
    \noindent and $\mathcal{U}$ is the set of distributions defined in Eqn.~\eqref{eq:mu_star}. Moreover, if $\mathcal{G}=\set{wh:h\in\mathcal{H}}$ for a fixed function $w$ which maps tuples of $k+2$ labelled data points to $[0,W]$ (for some $W\in\R_+$), then $\RC_\mu(\mathcal{G})\le W\,\RC_\mu(\mathcal{H})$ for all $\mu\in\mathcal{U}$, where $\mathcal{H}$ is the loss class defined in Eqn.~\eqref{eq:loss_class}.
\end{theorem}

\begin{proof}
    The first claim follows from the proof of Theorem \ref{thm:concentration_of_bar_Uomega} (i.e., Lemmas \ref{lem:supporting_lemma_3}, \ref{lem:supporting_lemma_4}, \ref{lem:supporting_lemma_5}, \ref{lem:supporting_lemma_6} and Propositions \ref{prop:supporting_proposition_1}, \ref{prop:supporting_proposition_2}), in which the functions $\ell_{\phi,f}$ are only used through the following two properties:
    \begin{enumerate}[label=(\alph*)]
        \item They take values in $[0,\mathcal{B}]$. This is used to bound the class-wise risks in Lemmas \ref{lem:supporting_lemma_3} and \ref{lem:supporting_lemma_4}, and to apply Lemma \ref{lem:subgaussianity_of_rad_complexity} in the proof of Proposition \ref{prop:supporting_proposition_2}.
        \item Conditionally on the vector ${\bf n}$ of class-wise sample sizes of the first partition $S^{(n)}$, the tuples of the collection $S_\mathrm{tup}$ (Eqn.~\eqref{eq:independent_tuples_dataset}) are independent, and those whose anchor-positive pair belongs to class $r$ are distributed according to $\mathcal{D}_r^{\otimes2}\otimes\mathcal{\bar D}^{\otimes k}$. This is used in Lemmas \ref{lem:supporting_lemma_5} and \ref{lem:supporting_lemma_6}.
    \end{enumerate}
    \noindent Both properties hold for every $g\in\mathcal{G}$, with $\mathcal{B}_\mathcal{G}$ in place of $\mathcal{B}$ and with the expectation of $g$ under $\mathcal{D}_r^{\otimes2}\otimes\mathcal{\bar D}^{\otimes k}$ in place of $\Lnorm_\Omega^r(f)$. Regarding (b), note that the negative samples of the tuples in $S_\mathrm{tup}$ are elements of the second partition, which is an i.i.d. sample of labelled data points from $\mathcal{\bar D}$ that is independent of the first partition. Therefore, the functions $g\in\mathcal{G}$ are allowed to depend on the classes of the negative samples. 
    
    \noindent For the second claim, fix $(t_j)_{j=1}^{\widetilde n}$ and define $\Xi({\bf a}):=\E_{\Rad_{\widetilde n}}\Big[\sup_{h\in\mathcal{H}}\Big|\frac{1}{\widetilde n}\sum_{j=1}^{\widetilde n}\sigma_ja_jh(t_j)\Big|\Big]$ for ${\bf a}=(a_j)_{j=1}^{\widetilde n}\in[-W,W]^{\widetilde n}$. The function $\Xi$ is convex. Hence, its maximum over the cube $[-W,W]^{\widetilde n}$ is attained at some vertex ${\bf a}\in\{-W,W\}^{\widetilde n}$, for which $(\sigma_ja_j)_{j=1}^{\widetilde n}\overset{d}{=}(W\sigma_j)_{j=1}^{\widetilde n}$. Therefore 
    \begin{align}
        \Xi({\bf a})\le W\,\Xi(1,\dots,1)=W\E_{\Rad_{\widetilde n}}\left[\sup_{h\in\mathcal{H}}\left|\frac{1}{\widetilde n}\sum_{j=1}^{\widetilde n}\sigma_jh(t_j)\right|\right] \quad\text{for all}\quad{\bf a}\in[-W,W]^{\widetilde n}
    \end{align}
    \noindent Choosing $a_j=w(t_j)$ and taking the expectation with respect to $(t_1,\dots,t_{\widetilde n})\sim\mu$ yields $\RC_\mu(\mathcal{G})\le W\,\RC_\mu(\mathcal{H}),\forall\mu\in\mathcal{U}$.
\end{proof}

\begin{proposition}[Refinement of Proposition \ref{prop:diff_between_natural_and_nonatural}]
    \label{prop:natural_exact_estimator_direct}
    Suppose that $n>R$ and let $g:\mathcal{X}^{k+2}\to[0,\mathcal{B}_g]$ be a tuple loss bounded by some $\mathcal{B}_g\in\R_+$. With the convention that an empirical average over an empty set is zero, we have:
    \begin{align}
        \left|\sum_{r=1}^R\widehat\rho_r\widehat\E_{\Omega_r}[g]-\widehat\E_\Pi\left[\sum_{r=1}^R\omega_r^\pi\widehat\E_{\Omega_r^\pi}[g]\right]\right|\le\mathcal{B}_g\left[\frac{2R}{n-R}+\frac{k}{N-1}\right].
    \end{align}
\end{proposition}

\begin{proof}
    Denote $U_\Omega(g)$, $\bar U_\Omega(g)$ and an auxiliary estimator $U'_\Omega(g)$ as follows:
    \begin{align}
        U_\Omega(g):=\sum_{r=1}^R\widehat\rho_r\widehat\E_{\Omega_r}[g],\quad \bar U_\Omega(g):=\widehat\E_\Pi\left[\sum_{r=1}^R\omega_r^\pi\widehat\E_{\Omega_r^\pi}[g]\right], \quad U'_\Omega[g]:= \widehat\E_\Pi\left[\sum_{r=1}^R\widehat\rho_r^\pi\widehat\E_{\Omega_r^\pi}[g]\right].
    \end{align}
    \noindent Let $\zeta:\mathcal{X}\to[R]$ be the map that returns the class of data points in $S$. Let $T_0$ be the default zero-loss tuple such that $g(T_0)=0$. Given $S$, we denote the random tuples $T,\bar T_\pi$ and $T_\pi'$ sampled as follows:
    \begin{enumerate}[label=(\roman*)]
        \item \textbf{For $T$}: Draw $X$ uniformly from $S$, $r=\zeta(X)$. If $N_r<2$, let $T=T_0$. Otherwise, draw $X^+$ uniformly from $S_r\setminus\{X\}$ and $(X_1^-,\dots,X_k^-)$ uniformly without replacement from $S\setminus\{X,X^+\}$. Then, $T=(X,X^+,\{X_i^-\}_{i=1}^k)$. For $t\in\Omega_r$, $\P(T=t)=\widehat\rho_r/|\Omega_r|$, which means $\E[g(T)]=U_\Omega(g)$.
        \item \textbf{For $T'_\pi$}: Given $\pi\in\Pi$, draw $X$ uniformly from $S_\pi^{(n)}$, let $r=\zeta(X)$ and draw $(X_1^-,\dots,X_k^-)$ uniformly without replacement from $S\setminus S_\pi^{(n)}$. Then:
        \begin{itemize}
            \item If $n_r^\pi=1$, let $T'_\pi=T_0$.
            \item Otherwise, draw $X^+$ uniformly from $(S_r\cap S_\pi^{(n)})\setminus\{X\}$ and let $T'_\pi=(X,X^+,\{X_i^-\}_{i=1}^k)$.
        \end{itemize}
        \noindent For $t\in\Omega_r^\pi$, we have $\P(T'_\pi=t|\pi)=\widehat\rho_r^\pi/|\Omega_r^\pi|$. Hence, $\E[g(T'_\pi)]=U_\Omega'(g)$.

        \item \textbf{For $\bar T_\pi$}: Given $\pi$, draw $r\in[R]$ with probability $\omega_r^\pi$. Then, draw $\bar T_\pi$ uniformly from $\Omega_r^\pi$. Then, $\E[g(\bar T_\pi)]=\bar U_\Omega(g)$. 
    \end{enumerate}
    \noindent\textbf{(i) Bounding $|\bar U_\Omega(g)-U'_\Omega(g)|$}: We know that when $n_r^\pi<2$ given a permutation $\pi\in\Pi$, we have $\Omega_r^\pi=\emptyset$, equivalently $\widehat\E_{\Omega_r^\pi}[g]=0$. Therefore, we have:
    \begin{align*}
        |\bar U_\Omega(g)-U'_\Omega(g)| &\le \widehat\E_\Pi\left[\sum_{r=1}^R\widehat\E_{\Omega_r^\pi}[g]\cdot|\widehat\rho_r^\pi-\omega_r^\pi|\right] = \widehat\E_\Pi\left[\sum_{r:n_r^\pi\ge 2}\widehat\E_{\Omega_r^\pi}[g]\cdot|\widehat\rho_r^\pi-\omega_r^\pi|\right] \le \mathcal{B}_g\widehat\E_\Pi\left[\sum_{r:n_r^\pi\ge2}|\widehat\rho_r^\pi-\omega_r^\pi|\right].
    \end{align*}
    \noindent\textbf{(ii) Bounding $|U_\Omega(g)-U'_\Omega(g)|$}: We couple both $T, T'_\pi$ to another tuple $Z'=(X,X^+,X_1^-,\dots,X_k^-)$ sampled as follows:
    \begin{itemize}
        \item Draw $X$ uniformly from $S$ and let $r=\zeta(X)$. If $N_r<2$, set $Z'=T_0$.
        \item Given $X$, draw $(X_1^-,\dots,X_k^-)$ uniformly without replacement from $S\setminus\{X\}$.
        \item Given $X$, draw $X^+$ uniformly from $S_r\setminus\{X\}$. If event $\mathcal{E}^c:=\{N_r\ge 2, X^+\in\{X_i^-\}_{i=1}^k\}$ (given $X$) occurs, resample $X^+$ uniformly from $S_r\setminus\{X, X_1^-,\dots,X_k^-\}$.
    \end{itemize}
    \noindent Note that given a fixed anchor $X$, we have $T\overset{d}{=} Z'|\mathcal{E}$. Furthermore:
    \begin{align*}
        \P(\mathcal{E}^c) &\le \P\Big(X^+\in\{X_i^-\}_{i=1}^k|X\Big) = \frac{1}{|S_{\zeta(X)}|-1}\sum_{x^+\in S_{\zeta(X)}\setminus \{X\}}\P\Big(x^+\in\{X_i^-\}_{i=1}^k|X\Big).
    \end{align*}
    \noindent The total number of $k$-tuples containing $x^+$ is $k(N-2)_{k-1}$: $x^+$ has to appear in at least one position of the tuple ($k$ choices) and the remaining positions hold an ordered list of distinct elements of $S\setminus\{x^+,X\}$ ($(N-2)_{k-1}$ choices). Therefore, given a positive $x^+\in S_{\zeta(X)}\setminus\{X\}$, we have:
    \begin{align*}
        \P\Big(x^+\in\{X_i^-\}_{i=1}^k|X\Big) = \frac{k(N-2)_{k-1}}{(N-1)_k} = \frac{k(N-2)_{k-1}}{(N-1)(N-2)_{k-1}}=\frac{k}{N-1}.
    \end{align*}
    \noindent As a result, $\P(\mathcal{E}^c)\le\frac{k}{N-1}$ and:
    \begin{align*}
        |\E[g(Z')] - U_\Omega(g)| &= \left|\Big(\E[g(Z')\mathds{1}_\mathcal{E}] + \E[g(Z')\mathds{1}_{\mathcal{E}^c}]\Big) - U_\Omega(g)\right| \\
        &=\left|\E[g(Z')\mathds{1}_{\mathcal{E}^c}] - \frac{k}{N-1}U_\Omega(g)\right| \qquad \left(\E[g(Z')\mathds{1}_\mathcal{E}] = \left(1-\frac{k}{N-1}\right)U_\Omega(g)\right)\\
        &= \left|\P(\mathcal{E}^c)\E[g(Z')|\mathcal{E}^c] - \frac{k}{N-1}U_\Omega(g)\right| \\
        &\le\frac{k\mathcal{B}_g}{N-1}
    \end{align*}
    \noindent Indeed, given an anchor $X$ of a class with $N_{\zeta(X)}\ge 2$, the first draw of $X^+$ avoids the negative samples with probability $1-\frac{k}{N-1}$ and on this event, $Z'$ is distributed as $T$. When $N_r<2$, we have $g(Z')=g(T)=0$. Furthermore, both $\E[g(Z')\mathds{1}_{\mathcal{E}^c}]$ and $\frac{k}{N-1}U_\Omega(g)$ are bounded in $\left[0,\frac{k\mathcal{B}_g}{N-1}\right]$. Finally, note that given $X$ (and its class) $T_\pi'\overset{d}{=}Z'|K^c$ where the event $K=\{N_r\ge 2, n_r^\pi=1\}$ via a simple coupling (Remark \ref{rem:simple-coupling}) (Also, note that on $K$, $T_\pi'=T_0$, making $g(T_\pi')=0$). As a result, we have:
    \begin{align*}
        |\E[g(Z')]-U_\Omega'(g)| &\le \mathcal{B}_g\P(K)=\mathcal{B}_g\widehat\E_\Pi\left[\sum_{r:N_r\ge 2, n_r^\pi=1}\frac{1}{n}\right] \le \mathcal{B}_g\widehat\E_\Pi\left[\sum_{r:n_r^\pi<2}|\widehat\rho_r^\pi-\omega_r^\pi|\right].
    \end{align*}
    \noindent Therefore, we have:
    \begin{align*}
        |U_\Omega'(g)-U_\Omega(g)| &\le |U_\Omega'(g)-\E[g(Z')]| + |U_\Omega(g)-\E[g(Z')]| \\
        &\le \mathcal{B}_g\left(\frac{k}{N-1}+\widehat\E_\Pi\left[\sum_{r:n_r^\pi<2}|\widehat\rho_r^\pi-\omega_r^\pi|\right]\right).
    \end{align*}
    \noindent \textbf{(iii) Combining Bounds}: Finally, we have:
    \begin{align*}
        |U_\Omega(g)-\bar U_\Omega(g)|
        &\le \mathcal{B}_g\left(\frac{k}{N-1}+\widehat\E_\Pi\left[\sum_{r:n_r^\pi\ge2}|\widehat\rho_r^\pi-\omega_r^\pi|+\sum_{r:n_r^\pi<2}|\widehat\rho_r^\pi-\omega_r^\pi|\right]\right) \\
        &= \mathcal{B}_g\left(\frac{k}{N-1}+\widehat\E_\Pi\left[\sum_{r=1}^R|\widehat\rho_r^\pi-\omega_r^\pi|\right]\right) \\
        &\le \mathcal{B}_g\left[\frac{k}{N-1}+\frac{2R}{n-R}\right],
    \end{align*}
    \noindent where the last step uses $\sum_{r=1}^R|\widehat\rho_r^\pi-\omega_r^\pi|\le\frac{2R}{n-R}$ for all $\pi\in\Pi$ (cf. Proof of Lemma \ref{lem:supporting_lemma}), as desired.
\end{proof}

\begin{remark}
    \label{rem:simple-coupling}
    We can select $T_\pi'$ and $Z'$ via the following coupled sampling scheme:
    \begin{enumerate}[label=(\roman*)]
        \item Draw a permutation $\pi$ uniformly from $\Pi$.
        \item Given $\pi$, draw $X$ uniformly from $S_\pi^{(n)}$ and, independently draw $\{X_1^-,\dots,X_k^-\}$ without replacement from $S\setminus S_\pi^{(n)}$ and let $r=\zeta(X)$.
        \item If $n_r^\pi\ge2$ (event $K^c$ occurs): Draw $X^+$ uniformly from $(S_r\cap S_\pi^{(n)})\setminus\{X\}$ and let $T_\pi'=Z'=(X,X^+,X_1^-,\dots,X_k^-)$.
        \item If $n_r^\pi<2$ (event $K$ occurs): Set $T_\pi'=T_0$. Draw $X^+_\star$ uniformly from $F:=S_r\setminus\{X,X_1^-,\dots,X_k^-\}$ and let $Z'=(X,X^+_\star,X_1^-,\dots,X_k^-)$. If $F$ is empty, also set $Z'=T_0$.
    \end{enumerate}
\end{remark}

\subsection*{Exact Debiasing of $\mathrm{L}_\phi$}
In order to debias $\mathrm{L}_\phi$ (cf. Eqn.~\eqref{eq:population_contrastive_risk}) exactly, we need to debias each class. Specifically, writing $\zeta(x)\in[R]$ for the class of a labeled data point $x$, for all $r\in[R]$, we have:
\begin{align}
    \label{eq:masked_identity}
    \underset{\substack{X,X^+\sim\mathcal{D}_r^{\otimes 2}\\\set{X_i^-}_{i=1}^k\sim\mathcal{\bar D}^{\otimes k}}}{\E}\left[\ell_{\phi, f}\left(X, X^+, \set{X_i^-}_{i=1}^k\right)\prod_{i=1}^k\1{\zeta(X_i^-)\ne r}\right]=(1-\rho_r)^k\Lnorm^r_\phi(f).
\end{align}
\noindent Hence, we have the natural and auxiliary \textbf{exact estimators of} $\mathrm{L}_\phi$ (cf. Eqn.~\eqref{eq:population_contrastive_risk}) defined as follows:
\begin{align}
    \label{eq:exact_estimators_class_weighted}
    U_N^\star(f):=\sum_{r=1}^R\frac{\widehat\rho_r}{(1-\widehat\rho_r)^k}\,\widehat\E_{\Omega_r}\Big[\ell_{\phi,f}\mathds{1}_{\Theta_r}\Big],
    \qquad\bar U_N^\star(f):=\widehat\E_\Pi\left[\sum_{r=1}^R\frac{\omega_r^\pi}{(1-\widehat\rho_r)^k}\,\widehat\E_{\Omega_r^\pi}\Big[\ell_{\phi,f}\mathds{1}_{\Theta_r}\Big]\right].
\end{align}
\noindent Here, $\mathds{1}_{\Theta_r}$ denotes the indicator function of the set $\Theta_r$, which includes collision-free tuples (cf. Eqn.~\eqref{eq:no_collision_tuples}): For a tuple $t=(X,X^+, \{X_i^-\}_{i=1}^k)$ of $\Omega_r$ or $\Omega_r^\pi$, $\mathds{1}_{\Theta_r}(t)=\prod_{i=1}^k\mathds{1}_{\zeta(X_i^-)\ne r}$. As before, we use the convention that an empirical average over an empty set of tuples is zero.

\begin{theorem}[Exact estimator of $\Lun$]
    \label{thm:exact_class_weighted}
    Let $\mathcal{F}$ be a class of representation functions, and let $\bar U_N^\star(f)$ be defined for each $f\in\mathcal{F}$ in Eqn.~\eqref{eq:exact_estimators_class_weighted}. Let $W_k:=(1-\rho_{\max})^{-k}$ and, for $\Delta\in(0,1)$ and $r\in[R]$:
    \begin{align}
        \label{eq:e_r_eta_N}
        e_r:=\frac{8}{3N}\ln\left(\frac{6R}{\Delta}\right)+\sqrt{\frac{8\rho_r(1-\rho_r)\ln\left(\frac{6R}{\Delta}\right)}{3N}},\qquad\eta_N:=2k\max_{r\in[R]}\frac{e_r}{1-\rho_r-e_r}.
    \end{align}
    \noindent Suppose that $e_r<1-\rho_r$ for all $r\in[R]$ and that $\eta_N\le2$. Then, with probability of at least $1-\Delta$:
    \begin{align}
        \label{eq:exact_class_weighted_bound}
        \sup_{f\in\mathcal{F}}\left|\bar U_N^\star(f)-\Lun(f)\right|\le W_k\Big[2\RC_{\mu_*}(\mathcal{H})+\mathcal{B}\,b_N(\Delta/3)\Big]+\mathcal{B}\eta_N\left[1+W_k\left(2\sqrt{\frac{2}{n-R}}+b_N(\Delta/3)\right)\right],
    \end{align}
    \noindent where $b_N$ is defined in Eqn.~\eqref{eq:b_N}. Furthermore, when $\rho_r\le\frac{1}{k+1}$ for all $r\in[R]$ (cf. Example 4.1), as long as $N\ge k\times\max\left\{64\ln\left(\frac{6R}{\Delta}\right),9R\right\}$, the assumptions above are satisfied, $W_k\le e$ and:
    \begin{align*}
        \sup_{f\in\mathcal{F}}\left|\bar U_N^\star(f)-\Lun(f)\right|\le\mathcal{O}\Bigg[\mathfrak{C}_N(\mathcal{H})\sqrt{\frac{k}{N}}+\mathcal{B}\sqrt{\frac{(R+k)\ln R/\Delta}{N}}+\frac{\mathcal{B}R(k+\ln R/\Delta)}{N}\Bigg],
    \end{align*}
    \noindent with probability of at least $1-\Delta$.
\end{theorem}

\begin{proof}
    For $r\in[R]$, let $w_r:=(1-\rho_r)^{-k}\in[1,W_k]$ and $\widehat w_r:=(1-\widehat\rho_r)^{-k}$. For a vector ${\bf v}=(v_1,\dots,v_R)$ of non-negative weights and $f\in\mathcal{F}$, we write $\bar U_\Theta^{\bf v}(f):=\widehat\E_\Pi\big[\sum_{r=1}^R\omega_r^\pi v_r\,\widehat\E_{\Omega_r^\pi}[\ell_{\phi,f}\mathds{1}_{\Theta_r}]\big]$, so that $\bar U_N^\star=\bar U_\Theta^{\widehat{\bf w}}$ where $\widehat{\bf w}=\{\widehat w_1,\dots,\widehat w_R\}$. Let ${\bf w}=\{w_1,\dots,w_R\}$, we have:
    \begin{align*}
        \sup_{f\in\mathcal{F}}|\bar U^\star_N(f)-\mathrm{L}_\phi(f)| &\le \sup_{f\in\mathcal{F}}|\bar U_\Theta^{\widehat{\bf w}}(f)-\bar U_\Theta^{{\bf w}}(f)| + \sup_{f\in\mathcal{F}}|\bar U_\Theta^{{\bf w}}(f)-\mathrm{L}_\phi(f)|.
    \end{align*}
    \noindent Let $\delta=\Delta/3$. We proceed to bound each uniform deviation.

    \noindent\textbf{(i) Uniform Bound on $|\bar U_\Theta^{\bf w}(f)-\Lun(f)|$}: We apply Theorem \ref{lem:general_tuple_functions} to the class $\mathcal{G}$ defined as follows: 
    \begin{align}
        \mathcal{G}:=\Bigg\{(X, X^+,X_1^-,\dots,X_k^-)\mapsto w_{\zeta(X)}\ell_{\phi,f}\Big(X, X^+, \{X_i^-\}_{i=1}^k\Big)\prod_{i=1}^k\1{\zeta(X_i^-)\ne\zeta(X)}:f\in\mathcal{F}\Bigg\}.
    \end{align}
    \noindent Here, we have $\mathcal{B}_\mathcal{G}=W_k\mathcal{B}$, $\mathfrak{R}_\mu(\mathcal{G})\le W_k\mathfrak{R}_\mu(\mathcal{H})$ for all $\mu\in\mathcal{U}$ (second claim of Theorem \ref{lem:general_tuple_functions}) and $\bar U_\Omega[g]=\bar U_\Theta^{\bf w}(f)$. Furthermore, by Eqn.~\eqref{eq:masked_identity}, we have:
    \begin{align*}
        \Lnorm_\Omega[g]=\sum_{r=1}^R\rho_rw_r(1-\rho_r)^k\Lnorm_\phi^r(f)=\sum_{r=1}^R\rho_r\Lnorm_\phi^r(f)=\Lun(f).
    \end{align*}
    \noindent Hence, with probability of at least $1-\delta$:
    \begin{align*}
        \sup_{f\in\mathcal{F}}\left|\bar U_\Theta^{\bf w}(f)-\Lun(f)\right|\le W_k\Big[2\RC_{\mu_*}(\mathcal{H})+\mathcal{B}\,b_N(\delta)\Big].
    \end{align*}
    \noindent\textbf{(ii) Uniform Bound on $|\bar U_\Theta^{\bf w}(f)-\bar U_\Theta^{\widehat{\bf w}}(f)|$}: We apply Theorem \ref{lem:general_tuple_functions} to the singleton $\mathcal{G}=\{g_0\}$ where:
    \begin{align}
        g_0(X,X^+,X_1^-,\dots,X_k^-) := w_{\zeta(X)}\prod_{i=1}^k\1{\zeta(X_i^-)\ne\zeta(X)}.
    \end{align}
    \noindent We have $\mathcal{B}_\mathcal{G}=W_k$ and $\Lnorm_\Omega[g_0]=\sum_{r=1}^R\rho_rw_r(1-\rho_r)^k=1$. Moreover, for any collection of $\widetilde n\ge\frac{n-R}{2}$ tuples $t_1,\dots,t_{\widetilde n}$, we have $\E_{\Rad_{\widetilde n}}\big|\frac{1}{\widetilde n}\sum_{j=1}^{\widetilde n}\sigma_jg_0(t_j)\big|\le\frac{W_k}{\sqrt{\widetilde n}}\le W_k\sqrt{\frac{2}{n-R}}$. Hence, with probability of at least $1-\delta$:
    \begin{align*}
        \bar U_\Omega[g_0]\le1+W_k\left[2\sqrt{\frac{2}{n-R}}+b_N(\delta)\right].
    \end{align*}
    \noindent By the high-probability Bernstein bound (Proposition \ref{prop:bernstein_bound_high_prob}, applied as in the proof of Lemma \ref{lem:supporting_lemma_4}) and the union bound, with probability of at least $1-\delta$, we have $|\widehat\rho_r-\rho_r|\le e_r$ for all $r\in[R]$. On this event, $1-\widehat\rho_r\ge1-\rho_r-e_r>0$ and $\frac{\widehat w_r}{w_r}=(1+x_r)^k$ with $x_r:=\frac{\widehat\rho_r-\rho_r}{1-\widehat\rho_r}$, so that $|x_r|\le\bar x_r:=\frac{e_r}{1-\rho_r-e_r}$. Since $k\bar x_r\le\frac{\eta_N}{2}\le1$, we have $(1+x_r)^k\le e^{k\bar x_r}\le1+2k\bar x_r$ and $(1+x_r)^k\ge(1-\bar x_r)^k\ge1-k\bar x_r$. Therefore:
    \begin{align*}
        \max_{r\in[R]}\left|\frac{\widehat w_r}{w_r}-1\right|\le 2k\max_{r\in[R]}\bar x_r := \eta_N.
    \end{align*}
    \noindent The weights $\widehat w_r$ do not depend on the permutation $\pi\in\Pi$. Hence, on the event that $|\rho_r-\widehat\rho_r|\le e_r,\forall r\in[R]$, for all $f\in\mathcal{F}$:
    \begin{align*}
        \left|\bar U_\Theta^{\widehat{\bf w}}(f)-\bar U_\Theta^{\bf w}(f)\right|&\le\mathcal{B}\,\widehat\E_\Pi\left[\sum_{r=1}^R\omega_r^\pi w_r\left|\frac{\widehat w_r}{w_r}-1\right|\widehat\E_{\Omega_r^\pi}[\mathds{1}_{\Theta_r}]\right]\le\mathcal{B}\eta_N\,\bar U_\Omega[g_0]\\
        &\le \mathcal{B}\eta_N\left[1+W_k\left(2\sqrt{\frac{2}{n-R}}+b_N(\delta)\right)\right].
    \end{align*}
    \noindent Combining (i) and (ii) via the triangle inequality and the union bound yields Eqn.~\eqref{eq:exact_class_weighted_bound}. 
    
    \noindent\textbf{(iii) On $\mathcal{O}$-Notation}: Suppose now that $\rho_r\le\frac{1}{k+1}$ for all $r\in[R]$ and $N\ge k\times\max\left\{64\ln\left(\frac{6R}{\Delta}\right),9R\right\}$. Then: 
    \begin{align*}
        W_k\le\left(1+\frac{1}{k}\right)^k\le e,\quad 1-\rho_r\ge\frac{1}{2}\quad\text{and}\quad\rho_r(1-\rho_r)\le\frac{1}{k+1}
    \end{align*}
    \noindent Writing $u:=\frac{k}{N}\ln\left(\frac{6R}{\Delta}\right)\le\frac{1}{64}$ (by assumption that $N\ge 64k\ln(6R/\Delta)$), we obtain $4ke_r\le\frac{32u}{3}+4\sqrt{\frac{8u}{3}}\le1$. Hence, $e_r\le\frac{1}{4}<1-\rho_r$, $\bar x_r\le4e_r$ and:
    \begin{align*}
        \eta_N:=2k\max_{r\in[R]}\bar x_r\le8k\max_{r\in[R]}e_r\le\frac{64u}{3}+8\sqrt{\frac{8u}{3}}\le2
    \end{align*}
    \noindent In particular, $\eta_N\le\mathcal{O}\Big[\sqrt{\frac{k\ln R/\Delta}{N}}\Big]$. The $\mathcal{O}$-notation bound then follows from Eqn.~\eqref{eq:exact_class_weighted_bound} with the same arguments as in the proof of Theorem \ref{thm:main_result_appendix}: Since $N\ge9Rk$, we have $n-R\ge\frac{N}{k+2}$, hence $\RC_{\mu_*}(\mathcal{H})\le\mathcal{O}\Big[\mathfrak{C}_N(\mathcal{H})\sqrt{k/N}\Big]$ (Remark \ref{remark:bound_on_worst_case_rc}) and: 
    \begin{align*}
        b_N(\Delta/3)\le\mathcal{O}\left[\sqrt{\frac{(R+k)\ln (R/\Delta)}{N}}+\frac{R(k+\ln (R/\Delta))}{N}\right]. 
    \end{align*}
    \noindent Finally, we use $W_k\le e$ and $\eta_N\le2$ to bound the product of $\eta_N$ with $W_k\Big(2\sqrt{\frac{2}{n-R}}+b_N(\Delta/3)\Big)$.
\end{proof}

\begin{remark}[Natural exact estimator]
    \label{rem:natural_exact_estimator}
    The guarantees of Theorem \ref{thm:exact_class_weighted} extend to the natural estimator $U_N^\star$ of Eqn.~\eqref{eq:exact_estimators_class_weighted}. Indeed, on the event $\{|\widehat\rho_r-\rho_r|\le e_r,\forall r\in[R]\}$ used in the proof above, we have $\widehat w_r\le(1+\eta_N)w_r\le3W_k$ for all $r\in[R]$. For each $f\in\mathcal{F}$, the tuple function $g_f:(X,X^+,X_1^-,\dots,X_k^-)\mapsto\widehat w_{\zeta(X)}\ell_{\phi,f}\big(X,X^+,\{X_i^-\}_{i=1}^k\big)\prod_{i=1}^k\1{\zeta(X_i^-)\ne\zeta(X)}$ then takes values in $[0,3W_k\mathcal{B}]$, and it satisfies $\sum_{r=1}^R\widehat\rho_r\widehat\E_{\Omega_r}[g_f]=U_N^\star(f)$ and $\widehat\E_\Pi\big[\sum_{r=1}^R\omega_r^\pi\widehat\E_{\Omega_r^\pi}[g_f]\big]=\bar U_N^\star(f)$. Hence, by Proposition \ref{prop:natural_exact_estimator_direct}:
    \begin{align*}
        \sup_{f\in\mathcal{F}}\left|U_N^\star(f)-\bar U_N^\star(f)\right|\le3W_k\mathcal{B}\left[\frac{2R}{n-R}+\frac{k}{N-1}\right].
    \end{align*}
    \noindent When $\rho_r\le\frac{1}{k+1}$ for all $r\in[R]$ and $N\ge9Rk$, the right-hand side is at most of order $\frac{\mathcal{B}Rk}{N}$, a term which is already present in the $\mathcal{O}$-notation bound of Theorem \ref{thm:exact_class_weighted}.
\end{remark}

\newpage
\section{Experiment Details}
\label{sec:experiment_details}
\subsection{Sub-sampled Estimation of U-Statistics}
In the main text, we have established that the proposed U-Statistic of this work ($U_N$) and that of prior work ($U_N^\mathrm{hl}$) require massive numbers of evaluations, making it infeasible to calculate either of them once let alone repeating such computations over many training iterations. Fortunately, we can construct sub-sampled estimators that converge to the full U-Statistic with probability one.

\noindent Let us recall the general construction. Given a finite population $\{\z_j\}_{j=1}^K\subset\mathcal{Z}$, a function $h:\mathcal{Z}\to\R$, a weighting function $w:\mathcal{Z}\to\R$ and $A^w_K(h)$ be defined as follows:
\begin{align}
    \label{eq:full_average}
    A^w_K(h) := \sum_{j=1}^Kw(\z_j)h(\z_j),\qquad \sum_{j=1}^Kw(\z_j)=1.
\end{align}
\noindent Then, given an i.i.d. draw $\set{\z^*_\ell}_{\ell=1}^M$ from a distribution $q$ over $\set{\z_j}_{j=1}^K$, the sub-sampled average:
\begin{align}
    \label{eq:sub-sampled_average}
    \widehat A^q_M(h) = \frac{1}{M}\sum_{\ell=1}^M \frac{w(\z^*_\ell)}{q(\z^*_\ell)}h(\z^*_\ell),
\end{align}
\noindent is an unbiased estimator of $A^w_K(h)$, which means $\widehat A_M^q(h)\to A_K^w(h)$ with probability one. 

\textbf{Sub-sampled Estimator of $U_N^\mathrm{hl}$}: Let $\Theta:=\bigcup_{r=1}^R\Theta_r$, i.e., the total collection of collision-free tuples and a representation function $f\in\mathcal{F}$. The U-Statistic $U_N^\mathrm{hl}(f)$ can be written in the form of Eqn.~\eqref{eq:full_average} as follows
\begin{align}
    U_N^\mathrm{hl}(f) := \sum_{t\in\Theta}w(t)\ell_{\phi, f}(t),\qquad \forall t\in\Theta_r:w(t)=\frac{\widehat\rho_r}{|\Theta_r|}.
\end{align}
\noindent Since $w$ itself is a distribution over $\Theta$, we can take $q=w$ and this is also the sampling distribution proposed in \citet{article:hieu2025}. Furthermore, because we choose $q=w$, the weightage of all tuples (in the empirical risk) drawn from the distribution $q$ are identical and all equal to $1$. The sub-sampling procedure of tuples using the distribution $q$ is given by Algorithm \ref{alg:subsampling_hieuledent}, which is also described in \citet[Section 3.3]{article:hieu2025}.

\begin{algorithm}
\caption{Sub-sampled Estimation for $U_N^\mathrm{hl}(f)$}
\begin{algorithmic}[1]
\STATE Given a representation function $f\in\mathcal{F}$
\FOR{$j$ in $1\dots M$}
    \STATE Select $r\in[R]$ with probability $\widehat\rho_r$
    \STATE Select $t_j\in\Theta_r$ with equal probabilities
\ENDFOR
\STATE $\mathrm{ERM}(f)\gets\frac{1}{M}\sum_{j=1}^M\ell_{\phi,f}(t_j)$.
\STATE \textbf{return} $\mathrm{ERM}(f)$
\end{algorithmic}
\label{alg:subsampling_hieuledent}
\end{algorithm}

\textbf{Sub-sampled Estimator of $U_N$}: Let $\Omega:=\bigcup_{r=1}^R\Omega_r$, i.e., the total collection of collision-allowed tuples. We can write $U_N(f)$ as follows
\begin{align}
    U_N(f) := \sum_{t\in\Omega}w(t)\ell_{\phi, f}(t), \qquad\forall t\in\Omega_r:w(t)=\frac{\widehat\rho_r}{|\Omega_r|}\left(\frac{1}{1-\widehat\tau}-\mathds{1}_{\Lambda_r}(t)\frac{ |\Omega_r|}{|\Lambda_r|}\frac{\widehat\tau}{1-\widehat\tau}\right).
\end{align}
\noindent Let us calculate the ratio $|\Omega_r|/|\Lambda_r|$ explicitly:
\begin{align*}
    \frac{|\Omega_r|}{|\Lambda_r|} &= \frac{\binom{N_r}{2}\times\binom{N-2}{k}}{\binom{N_r}{3}\times\binom{N-3}{k-1}} = \frac{\frac{1}{2}N_r(N_r-1)}{\frac{1}{6}N_r(N_r-1)(N_r-2)}\times\frac{(N-2)!}{k!(N-k-2)!}\times\frac{(k-1)!(N-k-2)!}{(N-3)!} \\
    &= \frac{3}{N_r-2}\times\frac{(N-2)!(k-1)!}{(N-3)!k!}\\
    &=\frac{3(N-2)}{k(N_r-2)}\approx \frac{3}{k\widehat\rho_r}.
\end{align*}
\noindent Therefore, for a tuple $t\in\Omega_r$, we have:
\begin{align*}
    w(t) &= \frac{\widehat\rho_r}{|\Omega_r|}\left(\frac{1}{1-\widehat\tau}-\mathds{1}_{\Lambda_r}(t)\frac{3\widehat\tau(N-2)}{k(1-\widehat\tau)(N_r-2)}\right)\approx \frac{\widehat\rho_r}{|\Omega_r|}\left(\frac{1}{1-\widehat\tau}-\mathds{1}_{\Lambda_r}(t)\frac{3\widehat\tau}{k\widehat\rho_r(1-\widehat\tau)}\right).
\end{align*}
\noindent Now, we note that it is possible that $w(t)<0$ if we pick a collided tuple $t$ from a small class $r\in[R]$. Specifically, given that $t\in\Omega_r$ is a collided tuple (i.e., $t\in\Lambda_r$) where $N_r\le\frac{3\widehat\tau(N-2)}{k}+2$. Then:
\begin{align*}
    N_r-2 \le \frac{3\widehat\tau(N-2)}{k} \implies \frac{3\widehat\tau(N-2)}{k(N_r-2)}\ge 1 \implies w(t)=\frac{\widehat\rho_r\left(1-\frac{3\widehat\tau(N-2)}{k(N_r-2)}\right)}{(1-\widehat\tau)\times|\Omega_r|}\le 0.
\end{align*}
\noindent The negative weight can potentially cause optimization instability during training. One work-around is to design the sampling distribution $q$ in a way that avoids collided tuples from small class entirely. In this work, we propose $q$ as follows:
\begin{align}
    \label{eq:subsampling_distribution_thiswork}
    \forall t\in\Omega_r: q(t)=\widehat\rho_r\times\begin{cases}
        {\mathds{1}_{\Theta_r}(t)}{|\Theta_r|^{-1}} &\textup{if } N_r \le \frac{3\widehat\tau(N-2)}{k}+2, 
        \\ 
        \frac{1}{|\Omega_r|} &\textup{otherwise}.
    \end{cases}
\end{align}
\noindent With the above chosen sampling distribution, for a given tuple $t$ whose anchor-positive pair belongs to a class $r\in[R]$ such that $N_r>\frac{3\widehat\tau(N-2)}{k}+2$ (meaning that $r$ is one of the major classes), the weightage of $t$ in the sub-sampled empirical risk is:
\begin{align*}
    \frac{w(t)}{q(t)} = \frac{1}{1-\widehat\tau}-\frac{|\Omega_r|}{|\Lambda_r|}\frac{\widehat\tau}{1-\widehat\tau}\mathds{1}_{\Lambda_r}(t) = \frac{1}{1-\widehat\tau} - \frac{3\widehat\tau(N-2)}{k(1-\widehat\tau)(N_r-2)}\mathds{1}_{\Lambda_r}(t).
\end{align*}
\noindent On the other hand, if $N_r\le \frac{3\widehat\tau(N-2)}{k}+2$, we have:
\begin{align*}
    \frac{w(t)}{q(t)} &= \frac{|\Theta_r|}{|\Omega_r|}(1-\widehat\tau)^{-1} = \frac{\binom{N_r}{2}\binom{N-N_r}{k}}{\binom{N_r}{2}\binom{N-2}{k}}(1-\widehat\tau)^{-1} \\
    &= \frac{\binom{N-N_r}{k}}{\binom{N-2}{k}}(1-\widehat\tau)^{-1} = \frac{(N-N_r)!}{k!(N-N_r-k)!}\times\frac{k!(N-k-2)!}{(N-2)!}(1-\widehat\tau)^{-1} \\
    &= \frac{(N-N_r)!(N-k-2)!}{(N-2)!(N-N_r-k)!}(1-\widehat\tau)^{-1} \\
    &= \frac{1}{1-\widehat\tau}\prod_{\ell=0}^{k-1}\frac{N-\ell-N_r}{N-\ell-2}.
\end{align*}
\noindent With all of the above computations handled, the calculation of sub-sampled estimation for $U_N(f)$ can be described in Algorithm \ref{alg:subsampling_thiswork}.

\begin{figure}[ht!]
	\begin{center}
		\centerline{\includegraphics[width=0.9\linewidth]{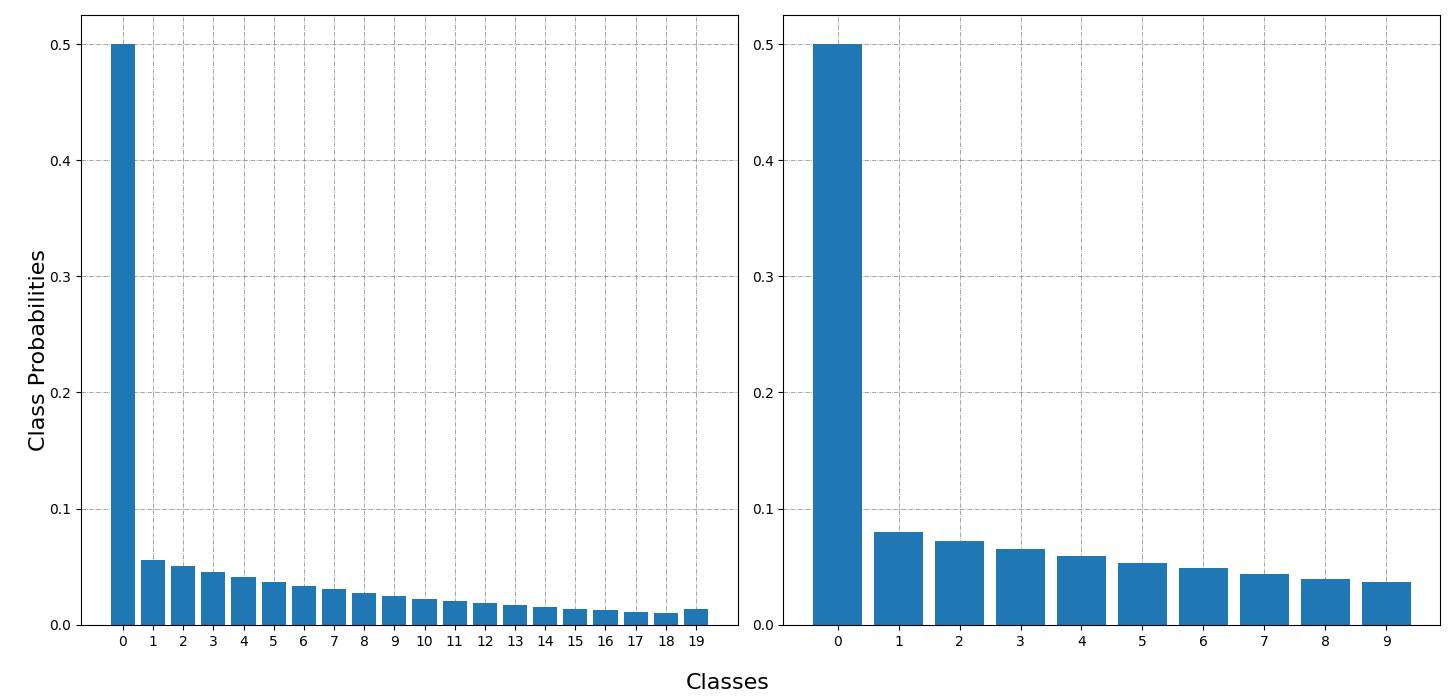}}
		\caption{Class distributions used for synthetic (left) and real (right) data experiments. For the real data experiments, we used the same class distribution structure as depicted in the figure above to create subsets from all real datasets (MNIST, FashionMNIST and CIFAR10).}
		\label{fig:class_distributions}
	\end{center}
\end{figure}

\begin{algorithm}[H]
\caption{Sub-sampled Estimation for $U_N(f)$}
\begin{algorithmic}[1]
\STATE Given a representation function $f\in\mathcal{F}$
\STATE Calculate $\widehat\tau:=1-\sum_{r=1}^R\widehat\rho_r(1-\widehat\rho_r)^k$
\FOR{$j$ in $1\dots M$}
    \STATE Select $r\in[R]$ with probability $\widehat\rho_r$
    \IF{$N_r\le\frac{3\widehat\tau(N-2)}{k}+2$}
        \STATE Select $t_j$ from $\Theta_r$ with equal probabilities \hfill\COMMENT{Avoid collision for small classes}
        \STATE $\omega_j\gets\frac{1}{1-\widehat\tau}\prod_{\ell=0}^{k-1}\frac{N-\ell-N_r}{N-\ell-2}$
    \ELSE
        \STATE Select $t_j$ from $\Omega_r$ with equal probabilities \hfill\COMMENT{Allow collision for large classes}
        \STATE $\omega_j\gets\frac{1}{1-\widehat\tau} - \frac{3\widehat\tau(N-2)}{k(1-\widehat\tau)(N_r-2)}\mathds{1}_{\Lambda_r}(t)$
    \ENDIF
\ENDFOR
\STATE $\mathrm{ERM}(f)\gets\frac{1}{M}\sum_{j=1}^M\omega_j\ell_{\phi, f}(t_j)$
\STATE \textbf{return} $\mathrm{ERM}(f)$
\end{algorithmic}
\label{alg:subsampling_thiswork}
\end{algorithm}

\begin{remark}
    From Algorithm \ref{alg:subsampling_thiswork}, there is an interesting effect of the chosen sub-sampling distribution $q$ in Eqn.~\eqref{eq:subsampling_distribution_thiswork} on how the sub-sampled estimation of $U_N$ behave on tail classes. Specifically:
    \begin{enumerate}[label=(\roman*)]
        \item For tuples whose anchor-positive pairs come from major classes, by default, they will have the largest weights (of $\frac{1}{1-\widehat\tau}$) if no class collision is found in their negative samples. 
        \item However, since these major classes are so plentiful in the labeled dataset, they also have the highest probabilities of class collision. Therefore, even though these major-class-anchor-positive-pair tuples are frequently sub-sampled for ERM, a lot of them are penalized for class collision (by a subtraction term of $\frac{3\widehat\tau(N-2)}{k(1-\widehat\tau)(N_r-2)}$).
        \item On the other hand, tuples with rare classes as their anchor-positive pairs receive lower weights by default (of $\frac{1}{1-\widehat\tau}\prod_{\ell=0}^{k-1}\frac{N-\ell-N_r}{N-\ell-2}$). However, since we avoid class-collision for these classes altogether, they are not at all penalized and often possess higher weights than tuples with major-class anchor-positive pairs.
    \end{enumerate}
    \noindent These effects interestingly make the sub-sampled estimator of $U_N$ possess a class-rarity-aware balancing feature that makes it favor tail classes more than $U_N^\mathrm{hl}$. However, we affirm our position that we \textbf{do not claim novelty} for this behavior since there has been a strong line of work in importance sampling and weighted empirical risk minimization prior to this work. We are primarily interested in providing sharper excess risk bounds. 
\end{remark}

\subsection{Experiment Settings}
In this work, we experimented with both synthetic and real datasets to compare the generalization capabilities of representation functions on tail classes when trained using empirical risks computed in Algorithm \ref{alg:subsampling_hieuledent} and Algorithm \ref{alg:subsampling_thiswork}. For the synthetic dataset, we randomly draw data points from a mixture of Gaussian distributions $\mathcal{\bar D}=\sum_{r=1}^R\rho_r\mathcal{N}(\cdot|\mu_r, \sigma_r^2)$ where $\rho_{\max}=0.5$, $\sigma_r^2=1$ for all $r\in[R]$ and other minor class probabilities decay exponentially. For real data experiments, we take small subsets of MNIST, FashionMNIST and CIFAR10 datasets such that, just like the synthetic dataset, tailed-ness is enforced. The specific class distributions of both synthetic and real datasets are visualized in Figure \ref{fig:class_distributions}. For the full experiment configurations, please see Table \ref{tab:exp_settings}.

\noindent\textbf{Hardware}: For all experiments, we use an Ubuntu desktop (Intel(R) Xeon(R) W-2133 CPU) with an \textbf{NVIDIA GeForce RTX 2080} GPU, \textbf{32GB RAM} and \textbf{12GB of GPU Memory}.

\begin{table}[t]
  \caption{Experiment settings for real and synthetic data experiments presented in the main text.}
  \label{tab:exp_settings}
  \begin{center}
    \begin{small}
        \begin{tabular}{c|c|c|c}
            \toprule
            & \textbf{Description} & \textbf{Synthetic} & \textbf{Real} \\
            \midrule
            $R$ & Number of classes & 20 & 10 \\
            $N$ & Number of labeled data points & 5000 & 5000 \\
            $M$ & Number of sub-sampled tuples  & 3000 & 3000 \\
            $k$ & Number of negative samples    & 5 & 3, 5, 7 \\
            \bottomrule
        \end{tabular}
    \end{small}
  \end{center}
  \vskip -0.1in
\end{table}

\subsection{Additional Experiments}
In order to further verify that the proposed estimator $U_N$ outperforms the class-wise estimator $U_N^\mathrm{hl}$ in typical extreme multi-class scenarios, we conducted further experiments on MNIST, FashionMNIST, CIFAR10 and CIFAR100 with convolutional neural networks (CNNs). For MNIST, FashionMNIST and CIFAR10, we simulate the class-imbalanced distribution illustrated in Figure \ref{fig:class_distributions} as originally done for deep neural networks. For CIFAR100, we use the entire balanced dataset without modifying the class distribution. The CNN architecture used in this experiment comprises of three $(\text{Conv}\to\text{BN}\to\text{ReLU}\to\text{MaxPool})$ blocks with $32\to64\to128$ channels, respectively. Aside from the comparison to $U^\mathrm{hl}_N$, we also ran this experiment on SupCon \citep{article:khosla2020} to provide a more comprehensive comparative study. Overall, the debiased estimator $U_N$ proposed in this work outperforms the class-wise estimator proposed in \citet{article:hieu2025}, as expected, with only a few exceptions on CIFAR10. Notably, in typical extreme multi-class scenario (CIFAR100), the performance of $U_N$ is better than that of both $U_N^\mathrm{hl}$ and SupCon.

\begin{table*}[ht]
	\begin{center}
		\begin{small}
        \caption{Average (macro) precisions, recalls and F1-scores of a simple \emph{convolutional} neural network trained using sub-sampled estimations of $U_N$, $U_N^\mathrm{hl}$ and SupCon as empirical risks. The one with higher score between the two estimators $U_N$ and $U_N^\mathrm{hl}$ is highlighted in bold font. The metrics of SupCon is underlined if it outperforms both $U_N$, $U_N^\mathrm{hl}$.}
			\label{tab:experiment_results_cnn_appendix}
			\centering
			\resizebox{\textwidth}{!}{\begin{tabular}{ll|ccc|ccc|ccc|ccc}
				\toprule
				&
				& \multicolumn{3}{c}{\textbf{MNIST} (imbl)}
				& \multicolumn{3}{c}{\textbf{FashionMNIST} (imbl)}
				& \multicolumn{3}{c}{\textbf{CIFAR10} (imbl)} 
                & \multicolumn{3}{c}{\textbf{CIFAR100}} \\
				\cmidrule(lr){3-5}
				\cmidrule(lr){6-8}
				\cmidrule(lr){9-11}
                \cmidrule(lr){12-14}
				\textbf{Metr.} & \textbf{$\sharp$Neg.}
				& $U_N$ & $U_N^{\mathrm{hl}}$ & SupCon 
				& $U_N$ & $U_N^{\mathrm{hl}}$ & SupCon
				& $U_N$ & $U_N^{\mathrm{hl}}$ & SupCon 
                & $U_N$ & $U_N^{\mathrm{hl}}$ & SupCon \\
				\midrule
				\multirow{3}{*}{Prec.} & $k=3$  &
				{\bf 0.9835} & 0.9238 & \multirow{3}{*}{0.9874} & {\bf 0.4146} & 0.3487 & \multirow{3}{*}{\underline{0.9382}} & {\bf 0.8128} & 0.7794 & \multirow{3}{*}{0.8057} & {\bf 0.3490} & 0.2836 & \multirow{3}{*}{0.2772} \\
				& $k=5$ &
				{\bf 0.9898} & 0.5031 & & {\bf 0.9212} & 0.3560 & & {\bf 0.8078} & 0.8045 & & {\bf 0.3239} & 0.3196 \\
				& $k=7$ &
				{\bf 0.9903} & 0.2845 & & {\bf 0.9253} & 0.3765 & & {0.7977} & {\bf 0.8033} & & {\bf 0.3173} & 0.2778 \\
				\midrule
				\multirow{3}{*}{Rec.} & $k=3$ &
				{\bf 0.9809} & 0.7599 & \multirow{3}{*}{0.9801} & {\bf 0.4998} & 0.4076 & \multirow{3}{*}{0.8435} & {\bf 0.5649} & 0.5404 & \multirow{3}{*}{\underline{0.6045}} & {\bf 0.3298} & 0.2675 & \multirow{3}{*}{0.3095} \\
				& $k=5$ &
				{\bf 0.9869} & 0.4315 & & {\bf 0.8629} & 0.4071 & & {\bf 0.5507} & {0.5324} & & {\bf 0.3422} & 0.2710 \\
				& $k=7$ &
				{\bf 0.9791} & 0.3711 & & {\bf 0.8451} & 0.3966 & & {\bf 0.5623} & 0.5515 & & {\bf 0.2895} & 0.2544 \\
				\midrule
				\multirow{3}{*}{F1} & $k=3$ &
				{\bf 0.9821} & 0.7771 & \multirow{3}{*}{0.9837} & {\bf 0.4195} & 0.3222 & \multirow{3}{*}{0.8745} & {\bf 0.6592} & 0.6328 & \multirow{3}{*}{\underline{0.6813}} & {\bf 0.3335} & 0.2723 & \multirow{3}{*}{0.2919} \\
				& $k=5$ &
				{\bf 0.9883} & 0.4209 & & {\bf 0.8862} & 0.3212 & & {\bf 0.6406} & {0.6366} & & {\bf 0.3306} & 0.2929 \\
				& $k=7$ &
				{\bf 0.9846} & 0.3141 & & {\bf 0.8744} & 0.3111 & & {0.6481} & {\bf 0.6505} & & {\bf 0.2997} & 0.2640 \\
				\bottomrule
			\end{tabular}}
		\end{small}
	\end{center}
    \vskip -0.1in
\end{table*}

\newpage
\section{Concentration of Sums of Independent Variables}
\label{sec:start_of_toolkit}
In this section, we recall some concentration inequalities required for our analysis. The results in this section as well as Sections~\ref{sec:UVstat} and~\ref{sec:SLTbasic} are classic and no claim of originality is made for these last three sections, which are included for completeness and self-containedness.
\subsection{Hoeffding's Inequality}
Concentration inequalities are the central tools for most of the modern results in learning theory. One of the most popular results was dated back to the work of \citet{article:hoeffding1948}, which focused concentration inequalities for sum of bounded independent random variables. We restate the key result below.

\begin{theorem}{{(Hoeffding's Inequality - \citet[Theorem 2]{article:hoeffding1948})}}
    \label{thm:hoeffding_second_thm}
    Let $X_1, \dots, X_n$ be independent random variables with expected values $\mu_1, \dots, \mu_n$ and $a_i\le X_i\le b_i$ with probability one for all $1\le i \le n$. Let $\overline{X}=\frac{1}{n}\sum_{i=1}^n X_i$ and $\mu=\E[\overline{X}]$. Then, for any constant $\epsilon>0$:
    \begin{align}
        \P(\overline{X}-\mu\ge \epsilon) \le \exp\left(
            -\frac{2n^2\epsilon^2}{\sum_{i=1}^n(b_i-a_i)^2}
        \right).
    \end{align}
\end{theorem}
\noindent The central idea for proving Theorem \ref{thm:hoeffding_second_thm} is to use the Chernoff bound to obtain the initial upper bound, which is a product of moment generating functions. Specifically, let $S_n = X_1+\dots+ X_n= n\overline{X}$, we have:
\begin{align*}
    \P(\overline{X}-\mu\ge\epsilon) &= \P(S_n - n\mu \ge n\epsilon) \\ 
    &\le e^{-n\epsilon t} \E[e^{t(S_n-n\mu)}] \quad \textup{(Chernoff Bound)} \\ 
    &= e^{-n\epsilon t}\prod_{i=1}^n \E e^{t(X_i-\mu_i)} \quad \textup{(By independence of $X_i$'s)}.
\end{align*}
\noindent Then, we can proceed to bound each moment generating function $\E e^{t(X_i-\mu_i)}$ separately using the Hoeffding's lemma. Another counterpart of Hoeffding's inequality is the Bernstein bound, which tend to be tighter, especially when the variances of the random variables considered are small. We state and provide the necessary technical lemma to prove the Bernstein bound below.

\subsection{Bernstein's Bound}
\begin{lemma}
    \label{lem:bernstein_bound_lemma}
    Let $X$ be a random variable with $\E[X]=0$ and $\E[X^2]=\kappa^2$. Suppose that there is some $c>0$ such that $|X|\le c$. Then, for any $t>0$, we have:
    \begin{align}
        \E e^{tX} \le \exp\left(
            t^2\kappa^2\left(
                \frac{e^{tc} - 1 - tc}{t^2c^2}
            \right)
        \right).
    \end{align}
\end{lemma}

\begin{proof}
    Let $G=\sum_{k=2}^\infty \frac{t^{k-2}\E[X^k]}{k!\kappa^2}$. Using Taylor's expansion, we have:
    \begin{align*}
        \E e^{tX} &= \E\left[
            1 + tX + \sum_{k=2}^\infty \frac{t^k X^k}{k!}
        \right] = 1 + \E\left[
            t^2X^2\sum_{k=2}^\infty \frac{t^{k-2}X^k}{k!X^2}
        \right] \\
        &= 1 + t^2\kappa^2\sum_{k=2}^\infty \frac{t^{k-2}\E[X^k]}{k!\kappa^2} = 1+t^2\kappa^2G.
    \end{align*}
    \noindent For $k\ge2$, we have $\E[X^k]=\E[X^2 X^{k-2}]\le \kappa^2c^{k-2}$. Therefore,
    \begin{align*}
        G &\le \sum_{k=2}^\infty \frac{t^{k-2}c^{k-2}\kappa^2}{k!\kappa^2} = \sum_{k=2}^\infty \frac{t^{k-2}c^{k-2}}{k!} \\
        &=\frac{1}{t^2c^2}\left[\sum_{k=0}^\infty \frac{t^kc^k}{k!} - 1 - tc\right] = \frac{e^{tc} - 1 - tc}{t^2c^2}.
    \end{align*}
    \noindent As a result, using the inequality $1+x\le e^x$, we have:
    \begin{align*}
        \E e^{tX} &= 1 + t^2\kappa^2G \le e^{t^2\kappa^2G} \le \exp\left(
                t^2\kappa^2\left(
                    \frac{e^{tc} - 1 - tc}{t^2c^2}
                \right)
            \right),
    \end{align*}
    \noindent as desired.
\end{proof}

\begin{theorem}[Bernstein Bound (\citet{book:concentration_inequalities} or \citet{book:vershynin2018} - Theorem 2.8.4)]
    \label{thm:bernstein_bound}
    Let $X_1, \dots, X_n$ be independent random variables with $0$ means. For all $i\in[n]$, let $\E[X_i^2] = \sigma_i^2$ and assume that $|X_i|\le M$ for some $M>0$. Then, for all $\epsilon>0$, we have:
    \begin{align}
        \P(|\overline X| \ge \epsilon) \le 2\exp\left(
            -\frac{n\epsilon^2/2}{\sigma^2 + M\epsilon/3}
        \right),
    \end{align}
    \noindent where $\sigma^2=\frac{1}{n}\sum_{i=1}^n \sigma_i^2$ and $\overline X = \frac{1}{n}\sum_{i=1}^n X_i$.
\end{theorem}

\begin{proof}
    First, we prove that $\P(\overline X \ge \epsilon)\le \exp\left(-\frac{n\epsilon^2/2}{\sigma^2+M\epsilon/3}\right)$. Using the Chernoff bound as usual:
    \begin{align*}
        \P(\overline X \ge \epsilon) &= \P\left(\sum_{i=1}^n X_i \ge n\epsilon\right) \le e^{-n\epsilon t}\prod_{i=1}^n \E e^{tX_i} \qquad (t>0) \\
        &\le e^{-n\epsilon t}\prod_{i=1}^n \exp\left(
            t^2\sigma_i^2\left(
                \frac{e^{tM} - 1 - tM}{t^2M^2}
            \right)
        \right)\qquad \textup{(Lemma \ref{lem:bernstein_bound_lemma})} \\
        &= e^{-n\epsilon t} \exp\left(nt^2\sigma^2\left(
            \frac{e^{tM} - 1 - tM}{t^2M^2}
        \right)\right).
    \end{align*}
    \noindent Solving for the optimal value of $t$, we have $t=M^{-1}\ln(1+M\epsilon/\sigma^2)$. Let $f(x)=(1+x)\ln(1+x)-x$. Plugging the optimal value of $t$ back to the above right-hand-side, we have:
    \begin{align*}
        \P(\overline{X} \ge \epsilon) &\le \exp\left(
            -\frac{n\sigma^2}{M^2}f\left(\frac{M\epsilon}{\sigma^2}\right)
        \right) \le \exp\left(
            -\frac{n\epsilon^2/2}{\sigma^2+M\epsilon/3}
        \right),
    \end{align*}
    \noindent where the last inequality holds because $f(x)\ge \frac{x^2/2}{1+x/3}$ for all $x\ge0$. We repeat the same arguments to prove that $\P(-\overline X \ge \epsilon)\le \exp\left(-\frac{n\epsilon^2/2}{\sigma^2+M\epsilon/3}\right)$ and combine with the above result to obtain the desired two-sided inequality.
\end{proof}

\begin{proposition}{(High-probability Bernstein Bound - See \citet[Prop. F.21]{article:ledent2024matrixcomp})}
    \label{prop:bernstein_bound_high_prob}
    Let $X_1, \dots, X_n$ be independent random variables with means $0$. For $i\in[n]$, let $\E[X_i^2] = \sigma_i^2$ and assume that $|X_i|\le M$ for some $M>0$. Then, for any $\delta\in(0,1)$, with probability of at least $1-\delta$:
    \begin{align}
        |\overline X| \le \frac{8M}{3n}\ln 2/\delta + \sigma\sqrt{\frac{8\ln 2/\delta}{3n}},
    \end{align}
    \noindent where $\sigma^2=\frac{1}{n}\sum_{i=1}^n \sigma_i^2$ and $\overline X = \frac{1}{n}\sum_{i=1}^n X_i$.
\end{proposition}

\begin{proof}
    Fix $\delta\in(0,1)$ and let $\epsilon>0$. We divide the arguments into cases when $M\epsilon\ge\sigma^2$ and when $M\epsilon\le \sigma^2$. For the case $M\epsilon\ge \sigma^2$, we have $\P(|\overline X| \ge \epsilon)\le 2\exp\left(-\frac{n\epsilon^2/2}{4M\epsilon/3}\right)$. Setting the right-hand-side of the inequality to $\delta$ and solve for $\epsilon$, we have:
    \begin{align*}
        \frac{n\epsilon^2/2}{4M\epsilon/3} = \ln 2/\delta \implies \epsilon = \frac{8M}{3n}\ln 2/\delta.
    \end{align*}
    \noindent In other words, with probability of at least $1-\delta$, $|\overline X|\le \frac{8M}{3n}\ln 2/\delta$. Similarly, for the case $\sigma^2\ge M\epsilon$, we have $\P(|\overline{X}|\ge\epsilon)\le 2\exp\left(-\frac{n\epsilon^2/2}{4\sigma^2/3}\right)$. Setting the right-hand-side to $\delta$ yields:
    \begin{align*}
        \frac{n\epsilon^2/2}{4\sigma^2/3} = \ln 2/\delta \implies \epsilon = \sigma\sqrt{\frac{8\ln 2/\delta}{3n}}.
    \end{align*}
    \noindent Or, $|\overline{X}|\le \sigma\sqrt{\frac{8\ln 2/\delta}{3n}}$ with probability of at least $1-\delta$. Hence, for all cases, we have:
    \begin{align*}
        |\overline{X}| \le \frac{8M}{3n}\ln 2/\delta + \sigma\sqrt{\frac{8\ln 2/\delta}{3n}}
    \end{align*}
    \noindent with probability of at least $1-\delta$.
\end{proof}

\begin{proposition}{(Bernstein's Condition)}
    \label{prop:bernstein_condition}
    Let $X$ be a random variable with variance $\mathrm{Var}(X)=\sigma^2$ and $\E[X]=\mu$. Suppose that $|X-\mu| \le M$ with probability one. Then, we have:
    \begin{align*}
        \E e^{\lambda(X-\mu)} &\le \exp\Bigg(\frac{\lambda^2\sigma^2/2}{1-M|\lambda|/3}\Bigg),\quad \forall |\lambda| <\frac{3}{M}.
    \end{align*}
\end{proposition}

\begin{proof}
    Using Taylor expansion on the moment-generating function $\E e^{\lambda(X-\mu)}$, we have:
    \begin{align*}
        \E e^{\lambda(X-\mu)} &= 1 + \sum_{k=1}^\infty \frac{\lambda^k \E[(X-\mu)^k]}{k!} = 1 + \sum_{k=2}^\infty \frac{\lambda^k \E[(X-\mu)^k]}{k!}  \\
        &= 1 + \sum_{k=2}^\infty \frac{\lambda^k \E[(X-\mu)^2(X-\mu)^{k-2}]}{k!} \\
        &\le 1 + \sum_{k=2}^\infty \frac{|\lambda|^k \E[(X-\mu)^2|X-\mu|^{k-2}]}{k!} \\
        &\le 1 + \sum_{k=2}^\infty \frac{|\lambda|^k \E[(X-\mu)^2M^{k-2}]}{k!} \\
        &= 1 + \sum_{k=2}^\infty \frac{|\lambda|^k \sigma^2M^{k-2}}{k!} \le 1 + \lambda^2\sigma^2\sum_{k=2}^\infty \frac{(M|\lambda|)^{k-2}}{k!}.
    \end{align*}

    \noindent Using the inequality $k!\ge 2\cdot 3^{k-2}$ for any $k\ge2$, for any $|\lambda|<3/M$ we have:
    \begin{align*}
        \E e^{\lambda(X-\mu)} &\le 1 + \lambda^2\sigma^2\sum_{k=2}^\infty \frac{(M|\lambda|)^{k-2}}{k!} \\
        &\le 1 + \frac{\lambda^2\sigma^2}{2}\sum_{k=2}^\infty \Bigg(\frac{M|\lambda|}{3}\Bigg)^{k-2} = 1+\frac{\lambda^2\sigma^2/2}{1-M|\lambda|/3} \\
        &\le \exp\Bigg(\frac{\lambda^2\sigma^2/2}{1-M|\lambda|/3}\Bigg) \qquad (1+x\le e^x).
    \end{align*}

    \noindent Using the same approach, we can also show that:
    \begin{align*}
        \E e^{\lambda(\mu-X)} \le \exp\Bigg(\frac{\lambda^2\sigma^2/2}{1-M|\lambda|/3}\Bigg),\quad\forall |\lambda|<\frac{3}{M}.
    \end{align*}
    \noindent Hence, for all $|\lambda|<\frac{3}{M}$, we have:
    \begin{align*}
        \E e^{\lambda |X-\mu|} &\le \E e^{\lambda (X-\mu)} + \E e^{\lambda(\mu-X)} \le 2\exp\Bigg(\frac{\lambda^2\sigma^2/2}{1-M|\lambda|/3}\Bigg).
    \end{align*}
\end{proof}
\noindent The above inequality can be extended to the following matrix Bernstein inequality.
\begin{proposition}[Non-commutative Matrix Bernstein Inequality (cf. \citet{article:ledent2021}, Proposition F.3 or \citet{article:recht2009}, Theorem 4)]
    \label{prop:bernstein_matrix}
    Let $X_1,\dots, X_n$ be independent zero-mean matrices of dimension $m\times d$. For all $1\le k\le n$, assume that $\|X_k\|_\sigma\le M$\footnote{Where $\|\cdot\|_\sigma$ denotes the spectral norm.} almost surely. Denote $\rho_k^2=\max\left(\|\E[X_kX_k^\top]\|_\sigma,\|\E[X_k^\top X_k]\|_\sigma\right)$. Then, for any $\lambda>0$:
    \begin{align}
        \P\left(\left\|\sum_{k=1}^n X_k\right\|_\sigma\ge\lambda\right) \le (m+d)\exp\left(-\frac{\lambda^2/2}{\sum_{k=1}^n\rho_k^2+M\lambda/3}\right).
    \end{align}
\end{proposition}

\subsection{Chernoff's Bound}
\begin{proposition}[Multiplicative Chernoff Bound (cf. Lemma F.15, \citet{article:ledent2024matrixcomp})]
    \label{prop:multiplicative_chernoff}
    Let $X_1,\dots,X_N$ be independent random variables taking \textbf{values in $\{0, 1\}$}. Let $S_N=\sum_{j=1}^NX_j$ and $\mu=\E[S_N]$. Then, for any $\delta>0$:
    \begin{align}
        \label{eq:multiplicative_chernoff_I}
        \P(S_N\ge (1+\delta)\mu) &\le \left(\frac{e^\delta}{(1+\delta)^{1+\delta}}\right)^\mu.
    \end{align}
    \noindent Similarly, we can show that for $\delta\in(0,1)$:
    \begin{align}
        \label{eq:multiplicative_chernoff_II}
        \P(S_N\le(1-\delta)\mu)\le\left(\frac{e^{-\delta}}{(1-\delta)^{1-\delta}}\right)^\mu.
    \end{align}
\end{proposition}

\begin{proof}
    For $t>0$, we have:
    \begin{align*}
        \P(S_N\ge (1+\delta)\mu) &= \P\left(e^{tS_N}\ge e^{t\mu(1+\delta)}\right) \\
        &\le e^{-t\mu(1+\delta)}M_{S_N}(t)\qquad\textup{(Markov's Inequality)} \\
        &\le e^{-t\mu(1+\delta)}\prod_{j=1}^N M_{X_j}(t) \qquad \textup{(Independence)}.
    \end{align*}
    \noindent Then, suppose that for each $j\in[N]$, $\E[X_j]=p_j$. Then, we have:
    \begin{align*}
        M_{X_j}(t) = 1 + p_j(e^t-1).
    \end{align*}
    \noindent Plugging this back to the above inequality, we have:
    \begin{align*}
        \P(S_N\ge (1+\delta)\mu) &\le e^{-t\mu(1+\delta)}\prod_{j=1}^N[1 + p_j(e^t-1)] \le e^{-t\mu(1+\delta)}\prod_{j=1}^N e^{p_j(e^t-1)} \qquad (1+x\le e^x) \\
        &= \frac{\exp\left[(e^t-1)\sum_{j=1}^Np_j\right]}{\exp[t\mu(1+\delta)]}  \\
        &= \left(\frac{\exp[e^t-1]}{\exp[t(1+\delta)]}\right)^\mu \qquad \qquad(\sum_{j=1}^Np_j=\mu).
    \end{align*}
    \noindent Solving for the optimal value of $t$, we have $t=\ln(1+\delta)$. Plugging this to the inequality:
    \begin{align}
        \P(S_N\ge (1+\delta)\mu) &\le \left(\frac{e^\delta}{(1+\delta)^{1+\delta}}\right)^\mu,
    \end{align}
    \noindent as desired. We can use the same proof strategy to prove the other inequality for $\delta\in(0,1)$.
\end{proof}

\begin{corollary}[Multiplicative Chernoff Bound]
    Let $S_N$ be defined in Proposition \ref{prop:multiplicative_chernoff}, we have:
    \begin{align}
        \P(S_N\ge (1+\delta)\mu) &\le e^{-\delta^2\mu/(2+\delta)},&& 0\le \delta. \\
        \P(S_N\le (1-\delta)\mu) &\le e^{-\delta^2\mu/2},&& 0\le \delta\le 1.\\
        \P(|S_N-\mu|\ge\delta\mu) &\le 2e^{-\delta^2\mu/3}, &&0\le\delta\le1.
    \end{align}
\end{corollary}

\begin{proof}
    Using the logarithm inequality $\ln(1+\delta)\ge\frac{2\delta}{2+\delta}$, from Eqn.\ref{eq:multiplicative_chernoff_I}, we have:
    \begin{align*}
        \P(S_N\ge (1+\delta)\mu) &\le \left(\frac{e^\delta}{(1+\delta)^{1+\delta}}\right)^\mu= e^{\mu\delta}(1+\delta)^{-\mu(1+\delta)} \\
        &\le \exp\left(\mu\delta-\frac{2\mu\delta(1+\delta)}{2+\delta}\right) = e^{-\delta^2\mu/(2+\delta)}.
    \end{align*}
    \noindent Similarly, for $\delta\in(0,1)$, we have $(1-\delta)\ln(1-\delta)\ge-\delta+\frac{\delta^2}{2}$: The difference of the two sides vanishes at $\delta=0$ and has the non-negative derivative $-\ln(1-\delta)-\delta$. Hence, from Eqn.~\eqref{eq:multiplicative_chernoff_II}, we have:
    \begin{align*}
        \P(S_N\le(1-\delta)\mu)&\le\left(\frac{e^{-\delta}}{(1-\delta)^{1-\delta}}\right)^\mu = e^{-\mu\delta}(1-\delta)^{-\mu(1-\delta)} \\
        &= \exp\left(-\mu\Big[\delta+(1-\delta)\ln(1-\delta)\Big]\right) \\
        &\le e^{-\delta^2\mu/2}.
    \end{align*}
    \noindent Finally, for $\delta\in(0,1)$, we have:
    \begin{align*}
        \P(|S_N-\mu|\ge\delta\mu) &\le \P(S_N-\mu\ge\delta\mu) + \P(S_N-\mu\le-\delta\mu) \\
        &\le e^{-\delta^2\mu/(2+\delta)} + e^{-\delta^2\mu/2} \\
        &\le 2e^{-\delta^2\mu/3},
    \end{align*}
    \noindent as desired.
\end{proof}
\section{Concentration of U-Statistics \& V-Statistics}
\label{sec:UVstat}
In this section, we recall some classic results on U-statistics and V-statistics, which form key building blocks of our approach. 
One limitation of inequalities in \citet{article:hoeffding1948} or \citet[Theorem 2.8.4]{book:vershynin2018} is that the initial Chernoff bound can be decomposed elegantly to a product of moment generating functions thanks to the independence assumption. However, this approach is inapplicable to more complex sum structures involving dependence. To leverage the arguments of Theorem \ref{thm:hoeffding_second_thm}, we focus on the following specific class of sums of dependent random variables:
\begin{align}
    \label{eq:sums_of_dependent_variables}
    S = \sum_{i=1}^N \rho_i S_i, \textup{ where } \rho_i\ge0, \forall i\in[N]\textup{ and } \sum_{i=1}^N \rho_i=1,
\end{align}
\noindent where each $S_i$ is a \textbf{sum of (bounded) independent random variables} and these sums are \textbf{not mutually independent}. Suppose that for all $i\in[N]$, $S_i$ is a sum of $n_i$ independent random variables. Furthermore, for simplicity, suppose that all random variables in the sum $S_i$ have the same mean $\mu_i$ and $\E S_i = n_i\mu_i$, which means $\E S = \sum_{i=1}^N \rho_i n_i\mu_i$. By the Chernoff bound, for $t>0$:
\begin{align*}
    \P(S - \E S\ge \epsilon) &\le e^{-t\epsilon}\E \exp\Big(t(S-\E S)\Big) = e^{-t\epsilon}\E\exp\left(\sum_{i=1}^N \rho_i t(S_i - n_i\mu_i)\right) \\
        &\le e^{-t\epsilon}\E\Bigg[\sum_{i=1}^N \rho_i\exp\Big(t(S_i-n_i\mu_i)\Big)\Bigg] \qquad \textup{(Jensen's Inequality)} \\
        &= e^{-t\epsilon}\sum_{i=1}^N \rho_i \E\exp\Big(t(S_i-n_i\mu_i)\Big).
\end{align*}
\noindent Since for all $i\in[N]$, $S_i$ is a sum of independent random variables, the inequality in Theorem \ref{thm:hoeffding_second_thm} apply. Specifically, suppose that for $i\in[N]$, $S_i = \sum_{j=1}^{n_i} X_{ij}$, we have:
\begin{align*}
    \E\exp\left(t(S_i-n_i\mu_i)\right) &= \E\exp\left(
        \sum_{j=1}^{n_i}t(X_{ij} - \mu_i)
    \right) = \prod_{j=1}^{n_i} \E e^{t(X_{ij} - \mu_i)}.
\end{align*}
\noindent Therefore, for each $i\in[N]$, we can proceed to use the same arguments in Thm. \ref{thm:hoeffding_second_thm} to bound the expectation $\E e^{t(X_{ij} - \mu_i)}$. In the following section, we will consider some classes of statistics that takes similar form in Eqn.~\eqref{eq:sums_of_dependent_variables}.

\subsection{U-Statistics}
\textbf{One-sample U-Statistics}: Let $X_1, \dots, X_n$ be identically and independently distributed random variables on $\mathcal{X}$ and $g:\mathcal{X}^r\to\R$ be a function symmetric in its arguments, i.e., $g(x_1, \dots, x_r) = g(x_{i_1}, \dots, x_{i_r})$ for all $1\le i_1, \dots, i_r \le r$ and $i_j \ne i_k$ for all $j,k\in[r]$. The one-sample U-Statistic (of order $r$) used to estimate $\E g(X_1, \dots, X_r)$ is defined as:
\begin{equation}
    U_n(g) = \frac{1}{\binom{n}{r}}\sum_{i_1, \dots, i_r\in C_{n, r}} g(X_{i_1}, \dots, X_{i_r}),
\end{equation}
\noindent where $C_{n, r}$ denotes the set of all $r$-element tuples selected without replacement from the indices set $[n]$. Let $q=\lfloor n/r\rfloor$ and define $V(X_1, \dots, X_n)$ as follows:
\begin{align}
    \widetilde U(X_1, \dots, X_n) = \frac{1}{q}\sum_{j=1}^q g(X_{jr-r+1}, \dots, X_{jr}).
\end{align}
\noindent Then, clearly $\widetilde U(X_1, \dots, X_n)$ is an average over the values of $g$ applied on independent $r$-element blocks of the sequence of random variables $X_1, \dots, X_n$. Then, we can re-define $U_n(g)$ as follows:
\begin{align}
    U_n(g) &= \frac{1}{n!}\sum_{\pi\in \Pi[n]} \widetilde U(X_{\pi,1}, \dots, X_{\pi,n}), \quad 
    \Pi[n] = \Big\{\pi:[n]\to[n]\big| \pi \textup{ bijective}\Big\},
\end{align}
\noindent which is consistent with the form in Eqn.~\eqref{eq:sums_of_dependent_variables}. This decoupling technique was later rigorously utilized by later works \cite{article:arconesgine1993, book:pena1998} to derive concentration inequalities for U-Statistics. If $g(X_1, \dots, X_r)\in[a, b]$ with probability one for some $a,b\in\R$, applying Theorem \ref{thm:hoeffding_second_thm}, for all $\epsilon>0$, we have:
\begin{align}
    \P(U_n(g) - \E U_n(g) \ge \epsilon) \le e^{-2q\epsilon^2/(b-a)^2}.
\end{align}
\noindent\textbf{Two-sample (or Multi-sample) U-Statistics}: Let $X_1, \dots, X_n$ be a sample drawn i.i.d. from a distribution over $\mathcal{X}$ and $Y_1, \dots, Y_m$ be another sample drawn i.i.d. from another distribution over $\mathcal{X}$. Let $g:\mathcal{X}^{r+s}\to\R$ be a kernel that is symmetric in the first $r$ arguments and the subsequent $s$ arguments ($r\le n$ and $s\le m$). Then, we define the two-sample U-Statistic of orders $r$ and $s$ as follows:
\begin{align}
    U_{n,m}(g) = \frac{1}{\binom{n}{r}\times\binom{m}{s}}\sum_{\substack{i_1, \dots, i_r \in C_{n,r}\\j_1, \dots, j_s \in C_{m, s}}} g(X_{i_1}, \dots, X_{i_r}, Y_{j_1}, \dots, Y_{j_s}).
\end{align}
\noindent Let $q=\min(\lfloor n/r\rfloor, \lfloor m/s\rfloor)$. Using the same decoupling trick, we have:
\begin{align}
    \P(U_{n,m}(g) - \E U_{n,m}(g)\ge \epsilon) \le e^{-2q\epsilon^2/(b-a)^2},
\end{align}
\noindent if we know that $g(X_1, \dots, X_r, Y_1, \dots, Y_s)\in[a, b]$ with probability one for some $a,b\in \R$.

\subsection{V-Statistics}
\label{sec:VStats}
Let $X_1, \dots, X_n$ be identically and independently distributed random variables on $\mathcal{X}$ and $g:\mathcal{X}^r\to\R$ be a function symmetric in its arguments. The one-sample V-Statistics of order $r$ is defined as the average over all $r$-tuples sampled \emph{with replacement} from $X_1, \dots, X_n$:
\begin{align}
    V_{n,r}(g) &= \frac{1}{n^r}\sum_{i_1=1}^n \dots \sum_{i_r=1}^n g(X_{i_1}, \dots, X_{i_r}) \\
    &= \frac{1}{n^r} \sum_{t\in [n]^r} g({\bf X}_t),
\end{align}
\noindent where $[n]^r$ is the set of $r$-tuples sampled with replacement and we use ${\bf X}_t=(X_{t_1}, \dots, X_{t_r})$ for any tuple of indices $t=(t_1, \dots, t_n)$. The V-Statistics can be re-formulated as a U-Statistic, i.e., an average over the the $r$-tuples sampled \emph{without replacement} from $X_1, \dots, X_n$. Let $C_{n,r}$ denote the set of $r$-tuples selected with no replacement from the indices set $[n]$, we can re-write the V-Statistics as follows:
\begin{align}
    V_{n,r}(g) &= \frac{1}{\binom{n}{r}}\sum_{t\in C_{n,r}}\binom{n}{r}n^{-r}\sum_{b\in\mathrm{S}(t)} \omega_bg({\bf X}_b),
\end{align}
\noindent where $\mathrm{S}(t)$ is the set of $r$-tuples selected with replacement from $t\in C_{n, r}$ and $\omega_b=\binom{n-u(b)}{r-u(b)}^{-1}$ where $1\le u(b) \le r$ is the number of unique elements from $b$. The rescaling factor $\omega_b$ accounts for the re-appearance of $b$ in other tuples $t'\in C_{n, r}$. We can further break down $V_{n,r}(g)$ as follows:
\begin{align}
    V_{n, r}(g) = \frac{1}{\binom{n}{r}}\sum_{t\in C_{n,r}}\binom{n}{r}n^{-r}\sum_{u=1}^r \sum_{b\in\mathrm{S}_u(t)} \frac{1}{\binom{n-u}{r-u}}g({\bf X}_b),
\end{align}
\noindent where $\mathrm{S}_u(t) \subset \mathrm{S}(t)$ is the subset of tuples with $u$ unique elements. For $1\le u\le r$, we have:
\begin{align*}
    |\mathrm{S}_u(t)| = u!S(r, u)\binom{r}{u},
\end{align*}
\noindent where $S(r,u)$ is the Stirling number of the second kind. To justify the formula for the cardinality of $\mathrm{S}_u(t)$, we can select a $r$-tuple with $u$ unique elements from an original set of $r$ distinct elements as follows:
\begin{itemize}
    \item Select $u$ unique elements for the subset ($\binom{r}{u}$ ways).
    \item Among the selected unique elements, partition the $r$ positions of the tuple to $u$ bins, which correspond to the unique elements ($S(r, u)$ ways).
    \item Since the unique elements can be in any arbitrary order, there is an additional multiplicative factor of $u!$.
\end{itemize}

\noindent We can further verify that $\sum_{u=1}^r|S_u(t)|=r^r$, which is the total number of ways to sample from any $t\in C_{n, r}$ with replacement. Now, we define $g^*:\mathcal{X}^r\to\R$ as
\begin{align}
    \label{eq:V2U_kernel}
    \forall t\in C_{n,r}: g^*({\bf X}_t) = \binom{n}{r}n^{-r}\sum_{u=1}^r\sum_{b\in\mathrm{S}_u(t)} \frac{1}{\binom{n-u}{r-u}}g({\bf X}_b),
\end{align}

\noindent we can say that $V_{n, r}(g)$ is a U-Statistic with kernel $g^*$. Now, we would like to show that $g^*({\bf X}_t)$ is a weighted average of $g({\bf X}_b)$ for $b\in\mathrm{S}(t)$. In other words, we would like to show that $\sum_{u=1}^r\sum_{b\in\mathrm{S}_u(t)}\frac{\binom{n}{r}}{n^r\binom{n-u}{r-u}}=1$. We have:
\begin{equation}
\begin{aligned}
    \label{eq:sums_up_to_one}
    \sum_{u=1}^r\sum_{b\in\mathrm{S}_u(t)}\frac{1}{\binom{n-u}{r-u}} &= \sum_{u=1}^r |\mathrm{S}_u(t)|\cdot\frac{1}{\binom{n-u}{r-u}} = \sum_{u=1}^r u! S(r, u)\binom{r}{u}\cdot\frac{1}{\binom{n-u}{r-u}} \\
    &= \sum_{u=1}^r u!S(r, u)\frac{r!}{u!(r-u)!}\cdot\frac{(r-u)!(n-r)!}{(n-u)!} \\
    &= \sum_{u=1}^r S(r, u) \frac{r!(n-r)!}{(n-u)!} = \binom{n}{r}^{-1}\sum_{u=1}^r S(r, u)\frac{n!}{(n-u)!}.
\end{aligned}
\end{equation}
\noindent It is a well-known property of Stirling number that $\sum_{u=1}^r S(r, u)\frac{n!}{(n-u)!}=n^r$ (For example, see \citet{preprint:adell2024})\footnote{Similarly, $\sum_{u=1}^r|S_u(t)| = \sum_{u=1}^ru!S(r, u)\binom{r}{u}=\sum_{u=1}^rS(r, u)\frac{r!}{(r-u)!}=r^r$ for all $t\in C_{n, r}$, which is the number of ways to sample $r$ elements with replacement from a set with $t$ elements.}. Hence, we have $\sum_{u=1}^r\sum_{b\in\mathrm{S}_u(t)}\frac{1}{\binom{n-u}{r-u}}=\frac{n^r}{\binom{n}{r}}$, as desired.

\section{Classic Learning Theory Results}
\label{sec:SLTbasic}
\subsection{Rademacher Complexity}
\begin{definition}[Rademacher Complexity]
    Let $\mathcal{Z}$ be a vector space and $\mathcal{G}$ be a class of functions $g:\mathcal{Z}\to [a,b]$ where $a,b\in\R$ and $a < b$. Let $S=\{\z_1, \dots, \z_n\}$ be independent random variables such that $\z_i\sim P_i,\forall i\in[n]$. Then, the \emph{empirical Rademacher complexity} of $\mathcal{G}$ is defined as:
    \begin{equation}
        \ERC_S(\mathcal{G}) = \E_{\Rad_n}\left[
            \sup_{g\in\mathcal{G}}\frac{1}{n}\left|
                \sum_{i=1}^n \sigma_ig(\z_i)
            \right|
        \right],
    \end{equation}
    \noindent where $\Rad_n=(\sigma_1, \dots, \sigma_n)$ is a vector of $n$ independent Rademacher variables. Additionally, the \emph{expected Rademacher complexity} of $\mathcal{G}$ is defined as follows:
    \begin{equation}
        \RC_{\mathcal{Q}}(\mathcal{G}) = \E_S\left[\ERC_S(\mathcal{G})\right],
    \end{equation}
    \noindent where $\mathcal{Q}=\bigotimes_{i=1}^n P_i$ is the distribution of $S$. Intuitively, the Rademacher complexity is a measure of a function class' richness. If a class $\mathcal{G}$ is sufficiently diverse, there is a higher chance that given a random sequence of signs (represented by the sequence of Rademacher variables), we will be able to find a function $g\in\mathcal{G}$ that matches the signs. Hence, the Rademacher complexity will be large.
\end{definition}

\begin{lemma}
    \label{lem:rc_to_erc}
    Let $\mathcal{Z}$ be a vector space and $\mathcal{G}$ be a class of functions $g:\mathcal{Z}\to [a,b]$ where $a,b\in\R$ and $a < b$. Let $S=\{\z_1, \dots, \z_n\}$ be a dataset of independent random variables such that each $\z_i\sim P_i,\forall i\in[n]$. Then, for any $\delta\in(0,1)$, with probability of at least $1-\delta$, we have:
    \begin{equation}
        \RC_{\mathcal{Q}}(\mathcal{G}) \le \ERC_S(\mathcal{G}) + (b-a)\sqrt{\frac{\ln 1/\delta}{2n}},
    \end{equation}
    \noindent where $\mathcal{Q}=\bigotimes_{i=1}^n P_i$ is the distribution of the whole dataset $S$.
\end{lemma}

\begin{proof}
    Let $\phi:\mathcal{Z}^n \to \R_+$ be defined as $\phi(x_1, \dots, x_n) = \E_{\Rad_n}\left[\sup_{g\in\mathcal{G}}\frac{1}{n}\left|\sum_{j=1}^n \sigma_j g(x_j)\right|\right]$ where $x_i\in\mathcal{Z}$ for all $1\le i \le n$. Then, for all $1\le i \le n$, we have:
    \begin{align*}
        \sup_{x_i, x_i'\in\mathcal{Z}}\left|
            \phi(x_1, \dots, x_i, \dots, x_n) - \phi(x_1, \dots, x_i', \dots, x_n) 
        \right| \le \frac{b-a}{n}.
    \end{align*}
    \noindent Hence, by McDiarmid's inequality \cite{book:mcdiarmid1989}, for any $\epsilon>0$, we have:
    \begin{align*}
        \P\Big(\underbrace{\E\phi(\z_1,\dots, \z_n)}_{\RC_\mathcal{Q}(\mathcal{G})} - \underbrace{\phi(\z_1,\dots, \z_n)}_{\ERC_S(\mathcal{G})} \ge \epsilon\Big)
        &\le \exp\left(
            -\frac{2n\epsilon^2}{(b-a)^2}
        \right).
    \end{align*}
    \noindent Setting the right-hand-side to $\delta$, we have $\epsilon=(b-a)\sqrt{\frac{\ln 1/\delta}{2n}}$ and we obtained the desired bound.
\end{proof}

\begin{proposition}[Rademacher Complexity Bound]
    \label{prop:generic_rad_complexity_bound}
    Let $\mathcal{Z}$ be a vector space and $\mathcal{G}$ be a class of functions $g:\mathcal{Z}\to [a,b]$ where $a,b\in\R$ and $a < b$. Let $S=\{\z_1, \dots, \z_n\}$ be a dataset of independent variables such that $\z_i\sim P_i,\forall i\in[n]$. Then, for any $\delta\in(0,1)$, with probability of at least $1-\delta$:
    \begin{equation}
        \sup_{g\in\mathcal{G}} \left|\Theta(g) - \frac{1}{n}\sum_{i=1}^n g(\z_i)\right| \le 2\ERC_S(\mathcal{G}) + 3(b-a)\sqrt{\frac{\ln 4/\delta}{2n}}.
    \end{equation}
    \noindent where $\Theta(g)=\frac{1}{n}\sum_{i=1}^n \mathbb{E}_{\z\sim P_i}[g(\z)]$.
\end{proposition}

\begin{proof}
    Let $\phi:\mathcal{Z}^n\to\R$ be defined as $\phi(x_1, \dots, x_n) =\sup_{g\in\mathcal{G}}\left[\Theta(g) - \frac{1}{n}\sum_{i=1}^n g(x_i)\right]$ and $\mathcal{Q}=\bigotimes_{i=1}^nP_i$, i.e., the distribution of $S$. Using McDiarmid's inequality, for all $\Delta\in(0,1)$, the following inequality holds with probability of at least $1-\Delta/2$:
    \begin{align*}
        \phi(\z_1, \dots, \z_n) &\le \E_S\phi(\z_1, \dots, \z_n) + (b-a)\sqrt{\frac{\ln 2/\Delta}{2n}} \\
        &\le 2\RC_\mathcal{Q}(\mathcal{G}) + (b-a)\sqrt{\frac{\ln 2/\Delta}{2n}}. \qquad \textup{(Symmetrization - Lemma \ref{lem:symmetrization_inequality})}
    \end{align*}
    \noindent Furthermore, we have $\RC_\mathcal{Q}(\mathcal{G}) \le \ERC_S(\mathcal{G}) + (b-a)\sqrt{\frac{\ln 2/\Delta}{2n}}$ with probability of at least $1-\Delta/2$ (Lemma \ref{lem:rc_to_erc}). Hence, by the union bound, with probability of at least $1-\Delta$, we have:
    \begin{align*}
        \sup_{g\in\mathcal{G}}\left[\Theta(g) - \frac{1}{n}\sum_{i=1}^n g(x_i)\right] \le  2\ERC_S(\mathcal{G}) + 3(b-a)\sqrt{\frac{\ln 2/\Delta}{2n}}.
    \end{align*}
    \noindent Let $\phi(x_1, \dots, x_n) =\sup_{g\in\mathcal{G}}\left[\frac{1}{n}\sum_{i=1}^n g(x_i)-\Theta(g)\right]$ and repeat the above argument, we have the inequality in the other direction with probability of at least $1-\Delta$. Hence, by the union bound, with probability of at least $1-2\Delta$, we have the following two-sided inequality:
    \begin{align*}
        \sup_{g\in\mathcal{G}} \left|\Theta(g) - \frac{1}{n}\sum_{i=1}^n g(\z_i)\right| \le 2\ERC_S(\mathcal{G}) + 3(b-a)\sqrt{\frac{\ln 2/\Delta}{2n}}.
    \end{align*}
    \noindent Setting $\Delta = \delta/2$ completes the proof.
\end{proof}

\subsection{Massart Lemma \& Dudley's Entropy Integral}
\begin{lemma}[Massart's Finite Lemma]
    \label{lem:massart}
    Given $S = \{\z_1, \dots, \z_n\}$ be a sample of independent random variables. Let $\mathcal{G}$ be a finite function class. Then, we have:
    \begin{equation}
        \E_{\Rad_n}\left[
            \sup_{g\in\mathcal{G}} \frac{1}{n}\left|
                \sum_{i=1}^n \sigma_i g(\z_i)
            \right|
        \right] \le B\sqrt{\frac{2\ln 2|\mathcal{G}|}{n}},
    \end{equation}
    \noindent where $\Rad_n=(\sigma_j)_{j\in[n]}$ are independent Rademacher variables, $B = \sup_{g\in\mathcal{G}}\left(\frac{1}{n}\sum_{i=1}^n |g(\z_i)|^2\right)^{1/2}$.
\end{lemma}

\begin{proof}
    For each $g\in\mathcal{G}$, we denote $\theta_g = \frac{1}{n}{\sum_{i=1}^n \sigma_i g(\z_i)}$. Then, for $\lambda>0$, we have:
    \begin{align*}
        \lambda\E_{\Rad_n}\left[
            \sup_{g\in\mathcal{G}} \frac{1}{n}\left|
                \sum_{i=1}^n \sigma_i g(\z_i)
            \right|
        \right] &= \lambda\E_{\Rad_n}\left[\max_{g\in\mathcal{G}} |\theta_g|\right] \quad (\sup \to \max \text{ due to finite }\mathcal{G})\\
        &= \ln\exp \lambda \E_{\Rad_n}\left[\max_{g\in\mathcal{G}} |\theta_g|\right] \\
        &= \ln\exp \lambda \E_{\Rad_n}\left[\max_{g\in\mathcal{G}} \max(\theta_g, -\theta_g)\right] \\
        &\le \ln\E_{\Rad_n}\exp\left(
            \lambda \max_{g\in\mathcal{G}}\max(\theta_g, -\theta_g)
        \right) \quad (\text{Jensen's Inequality}) \\
        &= \ln \E_{\Rad_n} \max_{g\in\mathcal{G}} \exp\left(
            \lambda \max(\theta_g, -\theta_g)
        \right) \quad (\exp(\lambda x) \text{ is increasing}) \\
        &\le \ln\sum_{g\in\mathcal{G}}\E_{\Rad_n}\left[
            \exp(\lambda\theta_g) + \exp(-\lambda\theta_g)
        \right] \quad (\exp\max \le \exp \ \mathrm{sum}) \\
        &= \ln 2\sum_{g\in\mathcal{G}}\E_{\Rad_n}\left[\exp(\lambda\theta_g)\right] \quad (\text{By symmetry of } \sigma) \\
        &= \ln 2\sum_{g\in\mathcal{G}} \prod_{i=1}^n \E_{\Rad_n}\left[\exp\left(\frac{\lambda}{n}\sigma_ig(\z_i)\right)\right] \quad (\text{$\sigma_i$ are independent}) \\
        &= \ln 2\sum_{g\in\mathcal{G}} \prod_{i=1}^n \frac{\exp\left(-\frac{\lambda}{n}g(\z_i)\right) + \exp\left(\frac{\lambda}{n}g(\z_i)\right)}{2}.
    \end{align*}
    \noindent Using the inequality $e^x + e^{-x} \le 2e^{\frac{x^2}{2}}$, we have:
    \begin{align*}
        \lambda\E_{\Rad_n}\left[
            \sup_{g\in\mathcal{G}} \frac{1}{n}\left|
                \sum_{i=1}^n \sigma_i g(\z_i)
            \right|
        \right] &\le \ln 2\sum_{g\in\mathcal{G}}\prod_{i=1}^n \exp\left(
            \frac{\lambda^2g(\z_i)^2}{2n^2}
        \right) \\
        &= \ln 2 \sum_{g\in\mathcal{G}}\exp\left(
            \frac{\lambda^2\sum_{i=1}^n g(\z_i)^2}{2n^2}
        \right) \\
        &\le \ln 2|\mathcal{G}|\cdot \exp\left(
            \frac{\lambda^2 B^2}{2n}
        \right) \\
        &= \ln2|\mathcal{G}| + \frac{\lambda^2B^2}{2n}.
    \end{align*}
    \noindent From the above, let $\lambda = B^{-1}\sqrt{2n\ln 2|\mathcal{G}|}$ and we obtain the desired bound.
\end{proof}

\begin{theorem}[Dudley's Entropy Integral]
    \label{thm:dudley}
    Let $\mathcal{G}$ be a real-valued function class and the dataset $S=\{\z_1, \dots, \z_n\}$ consists of independent random variables. Then, we have:
    \begin{equation}
        \ERC_S(\mathcal{G}) \le \inf_{\alpha>0}\left(
            4\alpha + 12\int_\alpha^B \sqrt{\frac{\ln2\mathcal{N}(\mathcal{G}, \epsilon, \Lnorm_2(S))}{n}}d\epsilon
        \right),
    \end{equation}
    \noindent where $B = \sup_{g\in\mathcal{G}}\left(\frac{1}{n}\sum_{i=1}^n g(\z_i)^2\right)^{1/2}$ and $\ERC_S(\mathcal{G})$ is the empirical Rademacher complexity of the function class $\mathcal{G}$ (given the dataset $S$).
\end{theorem}

\begin{proof}
    The above result is derived by a standard chaining argument (See, for example, \citet[Proposition 22]{article:ledent2021norm} or \citet[Lemma A.5]{article:bartlett2017}). Then, apply Massart's finite lemma. The difference is that instead of using the standard Massart's lemma for the regular notion of Rademacher Complexity (without absolute value), we apply Lemma \ref{lem:massart}.
\end{proof}

\begin{lemma}{(\citet{article:hieu2024})}
    \label{lem:conversion_of_covering_number}
    Let $\mathcal{F}$ be a class of representation functions $f:\X\to\R^d$ and $\ell:\R^k\to\R_+$ be a contrastive loss function that is $\ell^\infty$-Lipschitz with constant $\eta>0$. Let $S=\Big\{(\x_j,\x_j^+, \x_{j1:k}^-)\Big\}_{j=1}^m$ be a set of tuples. Then, we have:
    \begin{align}
        \ln\mathcal{N}\Big(\mathcal{H}, \epsilon, \Lnorm_2(S)\Big) &\le \ln\mathcal{N}\Bigg(\mathcal{F}, \frac{\epsilon}{4\eta\Gamma}, \Lnorm_{\infty,2}(\widetilde S)\Bigg),\\
        \Gamma&=\sup_{f\in\mathcal{F}}\sup_{\x\in\widetilde S}\|f(\x)\|_2.
    \end{align}
    \noindent where $\widetilde S$ is the set of all vectors, including anchors, positive and negative samples in $S$. 
\end{lemma}

\subsection{Symmetrization Inequality}
\begin{lemma}[Symmetrization]
    \label{lem:symmetrization_inequality}
    Let $S=\{\z_1, \dots, \z_n\}$ be a sample of independent random variables such that $\z_i\sim P_i,\forall i\in[n]$. Let $\mathcal{G}$ denote a function class. Then, for any real-valued convex, non-decreasing function $\varphi$, we have:
    \begin{equation}
        \E_S\varphi\left[
            \sup_{g\in\mathcal{G}}\left|
                \frac{1}{n}\sum_{i=1}^n g(\z_i) - \Theta(g)
            \right|
        \right] \le \E_{S, \Rad_n}\varphi\left[2\mathcal{R}_\mathcal{G}^{S, \Rad_n}\right],
    \end{equation}
    \noindent where $\Rad_n=(\sigma_j)_{j\in[n]}$ are independent Rademacher variables and $\Theta(g)=\frac{1}{n}\sum_{i=1}^n \mathbb{E}_{\z\sim P_i}[g(\z)]$.
\end{lemma}

\begin{proof} 
    Denote $\E_S$ as the expectation taken over $S$. We introduce another sample $S' = \{\z_1', \dots, \z_n'\}$, called the ``ghost" sample, such that $\z_i'$ is identically distributed as (and independent of) $\z_i$ for all $i\in[n]$. Then, we can write $\Theta(g)$ as $\Theta(g) = \mathbb{E}_{S'}\Big[\frac{1}{n}\sum_{i=1}^n g(\z_i')\Big].$
    \noindent Then, we have:
    \begin{align*}
        \E_S\varphi\left[
            \sup_{g\in\mathcal{G}}\left|
                \frac{1}{n}\sum_{i=1}^n g(\z_i) - \Theta(g)
            \right|
        \right]
        &= \E_S\varphi\Bigg[
            \sup_{g\in\mathcal{G}}\Bigg|\frac{1}{n}\sum_{i=1}^n g(\z_i) - \mathbb{E}_{S'}\Bigg[\frac{1}{n}\sum_{i=1}^n g(\z_i')\Bigg]\Bigg|
        \Bigg] \\
        &= \E_S\varphi\Bigg[
            \sup_{g\in\mathcal{G}}\Bigg|\mathbb{E}_{S'}\Bigg[\frac{1}{n}\sum_{i=1}^n g(\z_i) - \frac{1}{n}\sum_{i=1}^n g(\z_i')\Bigg]\Bigg|
        \Bigg] \\
        &\le \E_{S,S'}\varphi\Bigg[
            \sup_{g\in\mathcal{G}}\Bigg|\frac{1}{n}\sum_{i=1}^n \Big(g(\z_i)-g(\z_i')\Big)\Bigg|
        \Bigg] \qquad \textup{(Jensen's Ineq.)}
    \end{align*}
    \noindent Now, since for all $i\in[n]$, $\z_i$ and $\z_i'$ are identically distributed. Then, let $\sigma_i$ be a Rademacher random variable, $g(\z_i)-g(\z_i')$, $g(\z_i')-g(\z_i)$ and $\sigma_i\Big(g(\z_i)-g(\z_i')\Big)$ are all identically distributed random variables. Hence, let $\Rad_n=(\sigma_j)_{j\in[n]}$ be independent Rademacher variables, we have:
    \begin{align*}
        \E_S\varphi\left[
            \sup_{g\in\mathcal{G}}\left|
                \frac{1}{n}\sum_{i=1}^n g(\z_i) - \Theta(g)
            \right|
        \right]
        &\le \E_{S,S'}\varphi\Bigg[
            \sup_{g\in\mathcal{G}}\Bigg|\frac{1}{n}\sum_{i=1}^n \Big(g(\z_i)-g(\z_i')\Big)\Bigg|
        \Bigg] \\
        &= \E_{S,S',\Rad_n}\varphi\Bigg[
            \sup_{g\in\mathcal{G}}\Bigg|\frac{1}{n}\sum_{i=1}^n \sigma_i\Big(g(\z_i)-g(\z_i')\Big)\Bigg|
        \Bigg] \\
        &= \E_{S,S',\Rad_n}\varphi\Bigg[
            \sup_{g\in\mathcal{G}}\Bigg|\frac{1}{n}\sum_{i=1}^n \sigma_ig(\z_i) + \frac{1}{n}\sum_{i=1}^n (-\sigma_i)g(\z_i')\Bigg|
        \Bigg].
    \end{align*}
    \noindent Since the Rademacher variables $\sigma_i$ are symmetric, $\sigma_i$ and $-\sigma_i$ are identically distributed for all $i\in[n]$. Hence, we have:
    \begin{align*}
        &\E_S\varphi\left[
            \sup_{g\in\mathcal{G}}\left|
                \frac{1}{n}\sum_{i=1}^n g(\z_i) - \Theta(g)
            \right|
        \right] \\
        &\le \E_{S,S',\Rad_n}\varphi\Bigg[
            \sup_{g\in\mathcal{G}}\Bigg|\frac{1}{n}\sum_{i=1}^n \sigma_ig(\z_i) + \frac{1}{n}\sum_{i=1}^n \sigma_ig(\z_i')\Bigg|
        \Bigg] \qquad\textup{(}\sigma_i \textup{ are symmetric)} \\
        &\le \E_{S,S',\Rad_n}\varphi\Bigg[
            \sup_{g\in\mathcal{G}}\Bigg|\frac{1}{n}\sum_{i=1}^n \sigma_ig(\z_i)\Bigg| + \sup_{g\in\mathcal{G}}\Bigg|\frac{1}{n}\sum_{i=1}^n \sigma_ig(\z_i')\Bigg|
        \Bigg] \\
        &\le \frac{1}{2}\E_{S,\Rad_n}\varphi\Bigg[
            \sup_{g\in\mathcal{G}}\Bigg|\frac{2}{n}\sum_{i=1}^n \sigma_ig(\z_i)\Bigg|
        \Bigg] + \frac{1}{2}\E_{S',\Rad_n}\varphi\Bigg[
            \sup_{g\in\mathcal{G}}\Bigg|\frac{2}{n}\sum_{i=1}^n \sigma_ig(\z_i')\Bigg|
        \Bigg]\quad\textup{(Jensen's Ineq.)} \\
        &= \E_{S,\Rad_n}\varphi\Big[2\mathcal{R}_\mathcal{G}^{S, \Rad_n}\Big],
    \end{align*}
    \noindent as desired.
\end{proof}

\subsection{Sub-Gaussianity of Rademacher Complexity}
\begin{definition}[Sub-Gaussian Random Variable]
    Let $X$ be a random variable with mean $\E[X]=\mu$. $X$ is then called sub-Gaussian if there exists $\xi>0$ such that:
    \begin{equation}
        M_X(t) \le \exp\left(t\mu + \frac{t^2\xi^2}{2}\right), \quad \forall t > 0,
    \end{equation}
    \noindent where $M_X$ denotes the moment generating function of $X$. We call $X$ a sub-Gaussian random variable with variance proxy $\xi^2$, denoted as $X\in\mathcal{SG}(\xi^2)$.
\end{definition}

\begin{lemma}[\cite{book:concentration_inequalities}, Theorem 2.1]
    \label{lem:property_of_subgaussian_variable}
    Let $X$ be a random variable with mean $\E[X] = \mu$. If there exists $\xi>0$ such that the following holds:
    \begin{equation}
        \label{eq:bounded_tail_prob}
        \P(|X-\mu|\ge t) \le 2\exp\left(-\frac{t^2}{2\xi^2}\right),
    \end{equation}
    \noindent then the random variable $X$ is sub-Gaussian. Specifically, $X\in\mathcal{SG}(16\xi^2)$.
\end{lemma}

\begin{proof}
    Let $Z = X-\mu$ be the centered random variable derived by translating $X$ by its mean. Firstly, we prove that Eqn. \eqref{eq:bounded_tail_prob} implies that $\E|Z|^{2q} \le q!(4\xi^2)^q$ for all integers $q\ge1$. Using the identity $\E|Z|^q = \int_0^\infty qt^{q-1}\P(|Z|\ge t)dt$, we have:
    \begin{align*}
        \E|Z|^{2q} &= 2q\int_0^\infty t^{2q-1}\P(|Z|\ge t)dt \\
            &\le 4q\int_0^\infty t^{2q-1}\exp\left(-\frac{t^2}{2\xi^2}\right)dt.
    \end{align*}
    \noindent Letting $u=\frac{t^2}{2\xi^2}$, hence $t^2 = 2u\xi^2$ and $dt = \frac{\xi^2du}{t}$, the above integral becomes:
    \begin{align*}
        \E|Z|^{2q} &\le 4q\xi^2\int_0^\infty t^{2q-2}e^{-u}du \\
            &= 4q\xi^2\int_{0}^\infty (2u\xi^2)^{q-1}e^{-u}du \\
            &= 2q\cdot(2\xi^2)^q \underbrace{\int_0^\infty u^{q-1}e^{-u}du}_{\Gamma(q)} \\
            &= 2q!(2\xi^2)^q \le q!(4\xi^2)^q.
    \end{align*}
    \noindent Let $\tilde Z$ be the i.i.d. copy of $Z$. Hence, $Z-\tilde Z$ is symmetric about $0$, which means that $\E[(Z-\tilde Z)^p] = 0$ for odd-order $p$-moments. Therefore, For all $\lambda>0$, we have:
    \begin{align*}
        M_Z(\lambda)M_{-\tilde Z}(\lambda) &= M_{Z-\tilde Z}(\lambda) \quad (\text{Due to independence})\\
            &= \E\exp\left(\lambda(Z - \tilde Z)\right) \\
            &= 1 + \sum_{q=1}^\infty \frac{\lambda^{2q}\E\left[(Z - \tilde Z)^{2q}\right]}{(2q)!}.
    \end{align*}
    \noindent By the convexity of $f(z) = z^{2q}$, for all $t\in(0,1)$, we have:
    \begin{align*}
        \left[tZ + (1-t)(-\tilde Z)\right]^{2q} \le tZ^{2q} + (1-t)\tilde Z^{2q}.
    \end{align*}
    \noindent Setting $t=\frac{1}{2}$, we have:
    \begin{align*}
        \left[\frac{Z - \tilde Z}{2}\right]^{2q} \le \frac{Z^{2q} + \tilde Z^{2q}}{2} \implies (Z - \tilde Z)^{2q} \le 2^{2q-1}(Z^{2q} + \tilde Z^{2q}).
    \end{align*}
    \noindent As a result, we have $\E[(Z-\tilde Z)^{2q}] \le 2^{2q-1}(\E[Z^{2q}] + \E[\tilde Z^{2q}]) = 2^{2q}\E[Z^{2q}]$. Plugging this back to the formula of $M_{Z-\tilde Z}(\lambda)$, we have:
    \begin{align*}
        \E[e^{\lambda Z}]\E[e^{-\lambda \tilde Z}] &\le 1 + \sum_{q=1}^\infty \frac{\lambda^{2q}2^{2q}\E[Z^{2q}]}{(2q)!} \\
        &\le 1 + \sum_{q=1}^\infty \frac{\lambda^{2q}2^{2q}(4\xi^2)^{q}q!}{(2q)!}.
    \end{align*}
    \noindent Since $\E[e^{-\lambda \tilde Z}]\ge 1$ for all $\lambda>0$ and 
    \begin{align*}
        \frac{(2q)!}{q!} &= \prod_{j=1}^q (q+j) \ge \prod_{j=1}^q (2j) = 2^qq!,
    \end{align*}
    \noindent We have:
    \begin{align*}
        \E[e^{\lambda Z}] &\le 1 + \sum_{q=1}^\infty \frac{(2\lambda^2\cdot 4\xi^2)^q}{q!} = 1 + \sum_{q=1}^\infty \frac{(8\lambda^2\xi^2)^q}{q!} = e^{8\lambda^2\xi^2}.
    \end{align*}
    \noindent As a result, we have:
    \begin{align*}
        M_X(\lambda) &= e^{\lambda\mu}\mathbb{E}[e^{\lambda Z}] \le e^{\lambda\mu + 8\lambda^2\xi^2}.
    \end{align*}
    \noindent Hence, by definition, we have $X\in\mathcal{SG}(16\xi^2)$ as desired.
\end{proof}

\begin{lemma}
    \label{lem:subgaussianity_of_rad_complexity}
    Let $\F$ be a class of bounded functions $f:\X\to[0,\mathcal{B}]$ and let $S=\{\x_1, \dots, \x_n\}$ be sampled i.i.d. from a given distribution $\mathcal{P}$. Let $\Rad_n=\set{\sigma_j}_{j=1}^n$ be a sequence of independent Rademacher variables and define the following random variable:
    \begin{equation}
        \mathcal{R}_{\F}^{S, \Rad_n} = \sup_{f\in\F}\left|\frac{1}{n}\sum_{j=1}^n \sigma_j f(\x_j)\right|,
    \end{equation}
    \noindent which is a function of both $S$ and $\Rad_n$. Then, we have:
    \begin{align*}
        \E_{S, \Rad_n}\exp\Big(t\mathcal{R}_{\F}^{S, \Rad_n}\Big) \le \exp\left(
            t\mathfrak{R}_n(\F) + \frac{t^2\mathcal{B}^2}{n}
        \right), \quad \forall t>0,
    \end{align*}
    \noindent where $\mathfrak{R}_n(\F)$ is the expected Rademacher complexity, defined as $\mathfrak{R}_n(\F) = \E_{S, \Rad_n}\left[\mathcal{R}_{\F}^{S, \Rad_n}\right]$.
\end{lemma}

\begin{proof}
    Let $S_i$ be the copy of $S$ with the $i$-th element replaced with $\x_i'$, an i.i.d. copy of $\x_i$. Similarly, let $\Rad_n^{(\ell)}$ be the copy of $\Rad_n$ where the $\ell$-th element is replaced with $\sigma_\ell'$, an i.i.d. copy of $\sigma_\ell$. Then, the random variable $\mathcal{R}_\mathcal{F}^{S,\Rad_n}$ satisfies the following bounded-difference properties:
    \begin{alignat}{2}
        \left|\mathcal{R}_\F^{S, \Rad_n} - \mathcal{R}_\F^{S_i, \Rad_n}\right| &\le \sup_{f\in\F}\left|
            \frac{1}{n}\sigma_i (f(\x_i) - f(\x_i'))
        \right| &&\le \frac{2\mathcal{B}}{n},\\
        \left|\mathcal{R}_\F^{S, \Rad_n} - \mathcal{R}_\F^{S, \Rad_n^{(\ell)}}\right| &\le \sup_{f\in\F}\left|
            \frac{1}{n}f(\x_\ell)(\sigma_\ell - \sigma_\ell')
        \right| &&\le \frac{2\mathcal{B}}{n}.
    \end{alignat}
    \noindent Hence, $\mathcal{R}_\F^{S, \Rad_n}$ is a function of the $2n$ independent random variables $\x_1,\dots,\x_n,\sigma_1,\dots,\sigma_n$ which has the bounded differences property with constants $c_1=\dots=c_{2n}=\frac{2\mathcal{B}}{n}$ (only the independence of these variables is used in what follows). By the bound on the moment generating function from which McDiarmid's inequality \cite{book:mcdiarmid1989} is derived (Hoeffding's lemma applied to the increments of the Doob martingale of $\mathcal{R}_\F^{S, \Rad_n}$), we have:
    \begin{align*}
        \E_{S, \Rad_n}\exp\left(t \mathcal{R}_\F^{S, \Rad_n}\right) \le \exp\left(
        t\mathfrak{R}_n(\F) + \frac{t^2}{8}\sum_{i=1}^{2n}c_i^2
        \right) = \exp\left(t\mathfrak{R}_n(\F) + \frac{t^2\mathcal{B}^2}{n}\right), \quad \forall t>0,
    \end{align*}
    \noindent as desired.
\end{proof}

\end{document}